\documentclass{article}

\usepackage[top=1in, bottom=1in, left=1in, right=1in]{geometry}

\usepackage{comment,url,algorithm,algorithmic,graphicx,subcaption,relsize}
\usepackage{amssymb,amsfonts,amsmath,amsthm,amscd,dsfont,mathrsfs,mathtools,nicefrac}
\usepackage{float,psfrag,epsfig,color,xcolor,url,hyperref}
\usepackage{epstopdf,bbm,enumitem}
\usepackage[toc,page]{appendix}
\usepackage[mathscr]{euscript}

\def\ddefloop#1{\ifx\ddefloop#1\else\ddef{#1}\expandafter\ddefloop\fi}

\def\ddef#1{\expandafter\def\csname bb#1\endcsname{\ensuremath{\mathbb{#1}}}}
\ddefloop ABCDEFGHIJKLMNOPQRSTUVWXYZ\ddefloop

\def\ddef#1{\expandafter\def\csname c#1\endcsname{\ensuremath{\mathcal{#1}}}}
\ddefloop ABCDEFGHIJKLMNOPQRSTUVWXYZ\ddefloop

\def\ddef#1{\expandafter\def\csname v#1\endcsname{\ensuremath{\boldsymbol{#1}}}}
\ddefloop ABCDEFGHIJKLMNOPQRSTUVWXYZabcdefghijklmnopqrstuvwxyz\ddefloop

\def\ddef#1{\expandafter\def\csname v#1\endcsname{\ensuremath{\boldsymbol{\csname #1\endcsname}}}}
\ddefloop {alpha}{beta}{gamma}{delta}{epsilon}{varepsilon}{zeta}{eta}{theta}{vartheta}{iota}{kappa}{lambda}{mu}{nu}{xi}{pi}{varpi}{rho}{varrho}{sigma}{varsigma}{tau}{upsilon}{phi}{varphi}{chi}{psi}{omega}{Gamma}{Delta}{Theta}{Lambda}{Xi}{Pi}{Sigma}{varSigma}{Upsilon}{Phi}{Psi}{Omega}{ell}\ddefloop

\def\balign#1\ealign{\begin{align}#1\end{align}}
\def\baligns#1\ealigns{\begin{align*}#1\end{align*}}
\def\balignat#1\ealign{\begin{alignat}#1\end{alignat}}
\def\balignats#1\ealigns{\begin{alignat*}#1\end{alignat*}}
\def\bitemize#1\eitemize{\begin{itemize}#1\end{itemize}}
\def\benumerate#1\eenumerate{\begin{enumerate}#1\end{enumerate}}

\newenvironment{talign*}
{\csname align*\endcsname}
{\endalign}
\newenvironment{talign}
{\csname align\endcsname}
{\endalign}

\def\balignst#1\ealignst{\begin{talign*}#1\end{talign*}}
\def\balignt#1\ealignt{\begin{talign}#1\end{talign}}

\let\originalleft\left
\let\originalright\right
\renewcommand{\left}{\mathopen{}\mathclose\bgroup\originalleft}
\renewcommand{\right}{\aftergroup\egroup\originalright}

\def\tinycitep*#1{{\tiny\citep*{#1}}}
\def\tinycitealt*#1{{\tiny\citealt*{#1}}}
\def\tinycite*#1{{\tiny\cite*{#1}}}
\def\smallcitep*#1{{\scriptsize\citep*{#1}}}
\def\smallcitealt*#1{{\scriptsize\citealt*{#1}}}
\def\smallcite*#1{{\scriptsize\cite*{#1}}}

\def\mbi#1{\boldsymbol{#1}} 
\def\mbf#1{\mathbf{#1}}
\def\mbb#1{\mathbb{#1}}

\def\mrm#1{\mathrm{#1}}

\newcommand{\norm}[1]{\left\lVert#1\right\rVert}

\newcommand{\op}{\mathrm{op}}

\theoremstyle{plain}  
\newtheorem*{remark}{\textbf{Remark}}

\def\R{\mathbb{R}}

\def\N{\mathbb{N}}

\def\<{\left\langle} 
\def\>{\right\rangle}

\def\iff{\Leftrightarrow}

\newcommand{\boldone}{\mbf{1}} 
\newcommand{\ident}{\mbf{I}} 
\def\norm#1{\left\|{#1}\right\|} 

\def\E{\mbb{E}} 

\def\P{\mbb{P}} 

\newcommand{\Tr}{\operatorname{Tr}}
\def\Var{\mrm{Var}} 

\def\Cov{\mrm{Cov}} 

\DeclareSymbolFont{rsfs}{U}{rsfs}{m}{n}
\DeclareSymbolFontAlphabet{\mathscrsfs}{rsfs}

\providecommand{\diag}{\mathop\mathrm{diag}}

\def\rank#1{\mathrm{rank}({#1})}
\def\supp#1{\mathrm{supp}({#1})}

\ifdefined\nonewproofenvironments\else
\ifdefined\ispres\else
\newtheorem{theorem}{Theorem}
\newtheorem{lemma}[theorem]{Lemma}
\newtheorem{corollary}[theorem]{Corollary}

\newtheorem{prop}[theorem]{Proposition}

\newtheorem{assum}[theorem]{Assumption}

\theoremstyle{definition}
\newtheorem{definition}[theorem]{Definition}

\renewenvironment{proof}{\noindent\textbf{Proof.}\hspace*{.3em}}{\qed\\}
\newenvironment{proof-sketch}{\noindent\textbf{Proof Sketch}
	\hspace*{0.3em}}{\qed\\}
\newenvironment{proof-idea}{\noindent\textbf{Proof Idea}
	\hspace*{0.3em}}{\qed\\}
\newenvironment{proof-of-lemma}[1][{}]{\noindent\textbf{Proof of Lemma {#1}.}
	\hspace*{0.3em}}{\qed\\}
\newenvironment{proof-of-theorem}[1][{}]{\noindent\textbf{Proof of Theorem {#1}.}
	\hspace*{0.3em}}{\qed\\}
\newenvironment{proof-of-prop}[1][{}]{\noindent\textbf{Proof of Proposition {#1}.}
	\hspace*{0.3em}}{\qed\\}
\fi

\newtheorem{proposition}[theorem]{Proposition}

\fi
\makeatletter
\@addtoreset{equation}{section}
\makeatother

\hypersetup{
colorlinks,
linkcolor={red!50!black},
citecolor={blue!50!black},
urlcolor={blue!80!black}
}

\newcommand{\relu}{\mathrm{ReLU}}

\def\1{\bm{1}}

\def\eps{{\epsilon}}

\def\vzero{{\bm{0}}}
\def\vone{{\bm{1}}}
\def\vmu{{\bm{\mu}}}
\def\vtheta{{\bm{\theta}}}
\def\va{{\bm{a}}}
\def\vb{{\bm{b}}}
\def\vc{{\bm{c}}}

\def\ve{{\bm{e}}}
\def\vf{{\bm{f}}}
\def\vg{{\bm{g}}}
\def\vh{{\bm{h}}}

\def\vm{{\bm{m}}}

\def\vq{{\bm{q}}}

\def\vs{{\bm{s}}}
\def\vt{{\bm{t}}}
\def\vu{{\bm{u}}}
\def\vv{{\bm{v}}}
\def\vw{{\bm{w}}}
\def\vx{{\bm{x}}}
\def\vy{{\bm{y}}}
\def\vz{{\bm{z}}}

\def\Kf{\vK^{(f)}}

\DeclareMathAlphabet{\mathsfit}{\encodingdefault}{\sfdefault}{m}{sl}
\SetMathAlphabet{\mathsfit}{bold}{\encodingdefault}{\sfdefault}{bx}{n}

\usepackage[table]{xcolor} 
\usepackage{booktabs}
\usepackage{authblk}
\usepackage{etoc}
\usepackage{wrapfig}
\usepackage{bm}
\allowdisplaybreaks

\definecolor{mLightRed}{HTML}{FEE2E2}
\definecolor{mRed}{HTML}{EF4444}
\definecolor{mLightYellow}{HTML}{FEF3C7}
\definecolor{mDarkGreen}{HTML}{065F46}

\newcommand{\C}{\mathbb{C}}

\newcommand{\dist}[1]{\textnormal{dist}(#1)}

\newcommand{\spec}{\operatorname{spec}}

\newcommand{\MP}{\textnormal{MP}}
 
\newcommand{\SD}[1]{O_{\prec}\left(#1 \right)}

\newcommand{\iid}{\stackrel{\mathrm{i.i.d.}}{\sim}}

\newcommand{\snr}{\mathrm{SNR}}

\newcommand{\citep}[1]{\cite{#1}}
\newcommand{\citet}[1]{\cite{#1}} 

\makeatletter
\renewcommand{\paragraph}{%
\@startsection{paragraph}{4}%
{\z@}{1.5ex \@plus 1ex \@minus .2ex}{-1em}%
{\normalfont\normalsize\bfseries}%
}
\makeatother

\makeatletter
\renewcommand{\paragraph}{%
\@startsection{paragraph}{4}%
{\z@}{1.5ex \@plus 1ex \@minus .2ex}{-1em}%
{\normalfont\normalsize\bfseries}%
}
\makeatother

\begin{document}
\title{
Eigen-Spike Emergence and Quadratic Equivalents for Conjugate Kernels on Nonlinearly Separable Data
}
\author{
\begin{minipage}{0.49\textwidth}
    \centering
  Collin Cranston$^*$ \\
  \vspace{-3mm}
  Department of Mathematics \\
  University of California, San Diego, USA \\
  \texttt{ccransto@ucsd.edu}
  \end{minipage}
  \hfill
  \begin{minipage}{0.49\textwidth}
  \centering
  Zhichao Wang$^*$ \\
  \vspace{-3mm}
  ICSI and Department of Statistics \\
  University of California, Berkeley, USA \\
  \texttt{zhichao.wang@berkeley.edu}
  \end{minipage}
  
  \vspace{0.5cm}
  \begin{minipage}{0.49\textwidth}
    \centering
  Todd Kemp \\
  \vspace{-3mm}
  Department of Mathematics \\
  University of California, San Diego, USA \\
  \texttt{tkemp@ucsd.edu}
   \end{minipage}
  \begin{minipage}{0.49\textwidth}
    \centering
  Michael W. Mahoney \\
  \vspace{-3mm}
  ICSI, LBNL and Department of Statistics \\
  University of California, Berkeley, USA \\
  \texttt{mmahoney@stat.berkeley.edu}
  \end{minipage}
  \hfill
  }

\date{}
\maketitle
 
\begingroup
\renewcommand{\thefootnote}{}
\footnotetext{$^*$Equal contribution.}
\endgroup

\begin{abstract}
Recent work in random matrix theory (RMT) has developed the notion of deterministic equivalents: typically linear surrogate models that approximate the spectral behavior of large nonlinear random matrices, such as nonlinear feature maps in neural networks (NNs). 
Such equivalents make theoretical predictions tractable by reducing a complex model to a simpler one with properties that fall under the umbrella of classical RMT tools. 
However, this leaves open the question of whether this idealized linear equivalence remains meaningful for classification of high-dimensional nonlinearly separable data.
Motivated by this, we consider the conjugate kernel (CK), which is the nonlinear feature map of a one-layer feedforward NN, under a canonical nonlinearly separable dataset for the XOR problem; and we use the study of informative outlier eigenvalues in the CK and whether their corresponding eigenvectors asymptotically align with XOR labels as a proxy for nonlinear learnability. 
We develop a robust \emph{quadratic equivalent} of the CK matrix that enables a precise analysis of emergent informative spikes, as one modifies various knobs common in ML practice: sample complexity, signal-to-noise ratio (SNR), nonlinear activation choice, and pretrained features. We identify regimes in which these knobs move the CK beyond the linear equivalent and produce BBP-type transitions to label-aligned outlier eigenspaces. Our analysis helps bring deterministic-equivalence tools from RMT to bear on problems of practical relevance in ML.
\end{abstract}

\section{Introduction}
\label{sec:intro} 

Spectral information from weight matrices, kernel matrices, and Hessians provides a quantitative window into representation learning and efficient optimization in neural networks (NNs) \citep{martin2018implicitself,martin2019heavytail,MM20_SDM,MM20a_trends_NatComm,YTHx22_TR,wang2022spectral,jacot2018neural, pennington2017nonlinear, adlam2020neural,liao2021hessianeigenspectrarealisticnonlinear,sagun2016singularityhessian,ghorbani2019hessianeigenvaluedensity}. For example, eigenvalues of NN weight matrices can capture the strength of learned directions, while the corresponding eigenvectors indicate which data features are amplified or extracted during training \citep{ba2022high,damian2022neural,moniri2023theory,dandi2024random}. 
A major direction in deep learning theory is to use spectral tools to understand when NNs can propagate, preserve, transform, or amplify task-relevant structure that may be hidden in high-dimensional datasets.

A concrete and widely used spectral mechanism is the emergence of \emph{spikes} (outlier eigen/singular values) and the \emph{alignment} of their eigen/singular vectors with signals in the dataset. 
Random matrix theory (RMT) provides a powerful tool to analyze such eigen-spike emergence. 
In high-dimensional spiked models where a low-rank signal is embedded in Gaussian noise, BBP (Baik--Ben Arous--P{\'e}ch{\'e}) phase transitions characterize when an informative outlier separates from the bulk spectrum and when its associated eigenvector becomes correlated with the underlying signal \citep{baik2005phase,nadler2008finitepca}. Recent work has developed nonlinear counterparts for random-feature and kernel-type matrices, providing sharp predictions for outlier locations and eigenvector alignment \citep{feldman2023spectral,ba2022high,ba2023learning,wang2024nonlinearspikedcovariancematrices}. Such nonlinear RMT can be applied to study deep learning theory, e.g., the infinite-width \emph{Neural Tangent Kernel} (NTK) and the \emph{Conjugate Kernel} (CK) \citep{jacot2018neural,pennington2017nonlinear,adlam2020neural} for multi-layer NNs. 
Their high-dimensional spectra have been studied extensively \citep{liao2025randommatrixtheorydeep,wang2021deformed,liao2020random}. In this work, we focus on the CK, and we study \emph{emergent spikes} induced by \emph{nonlinear} structure in the data. 
Concretely, we consider the following spectral question for CK: 
\begin{quote} 
    \textit{When can NNs transfer a nonlinear pattern in the data into a linearly accessible direction (an aligned eigen-spike) that enables linear classification?}
\end{quote} 

In this paper, the CK is defined by a one-hidden-layer random NN: 
$\vY=\frac{1}{\sqrt{N}}\,\sigma(\vW\vX)\in\R^{N\times n},$ with nonlinear activation  $\sigma$, random weights $\vW\in\R^{N\times d}$, and data $\vX\in\R^{d\times n}$. The
CK matrix is the Gram matrix of hidden activations,
$\vK=\vY^\top \vY \in \R^{n\times n}$ for $n$ data points. The spectral distribution of $\vK$ is particularly useful for generalization error of random feature regression, as demonstrated in \citet{mei2019generalization,hu2020universality} for isotropic Gaussian data $\vX$. However, real datasets often contain low-rank or nonlinear structure that may be hard to analyze using PCA \citep{pca_jolliffe2002, pca_scholkopf1998, pca_tenenbaum2000, pca_roweis2000, pca_belkin2003, pca_coifman2006, pca_johnstone2009}. To highlight genuinely nonlinear structure, we instead take $\vX$ to be the high-dimensional XOR Gaussian mixture dataset (Section~\ref{sec:model_data}). This dataset is a balanced four-component Gaussian mixture whose two classes are unions of opposite components \citep{refinetti2021classifying}. The key property is that XOR is \emph{not linearly separable} in first-order statistics: the class means vanish, and any method that only ``sees'' linear correlations cannot recover labels. However, XOR is \emph{quadratically separable}, so a second-order feature can potentially make the task linearly solvable. Heuristically, the eigenvalues of $\vK$ describe the spectrum of learned similarities, while the alignment of leading eigenvectors with task structure (e.g., labels or cluster indicators) may lead to spectral methods solving the task.
Therefore, this paper aims to answer the following spectral question for $\vK$: 
\begin{quote} 
    \textit{When do outlier eigenvalues emerge in the CK spectrum, and when do their spiked eigenvectors align with the XOR labels and enable linear classification?} 
\end{quote} 

In practice, NN performance is affected by hyperparameters, dataset structure, and architecture; and, as a practical matter, machine learning (ML) practitioners typically fiddle (quite aggressively) with various ``knobs'' of the ML training process in order to improve model performance. 
Motivated by this, we consider properties of the CK matrix for the XOR problem, as various ``knobs'' for $\vK$ are varied; and, depending on their values, we observe very different spectral behavior. 
Concretely, we vary the following knobs (which are among the standard parameters that are tuned in ML practice): sampling size scaling $n$ with input features $d$ and NN width $N$; signal-to-noise ratio (SNR) in the data $\vX$; the pretrained weight matrix $\vW$; and the choice of nonlinear activation $\sigma$.
We show that, depending on the values of these ``knobs,'' one can obtain qualitatively different spectral properties, various phase transitions, and the emergence of \textit{quadratic informative spikes} in the CK matrix.
Importantly, the emergence of these quadratically informative spikes makes the nonlinear XOR linearly classifiable.

\begin{figure}[ht]
    \centering
    \begin{minipage}[c]{0.6\linewidth}\centering
    \includegraphics[width=\linewidth]{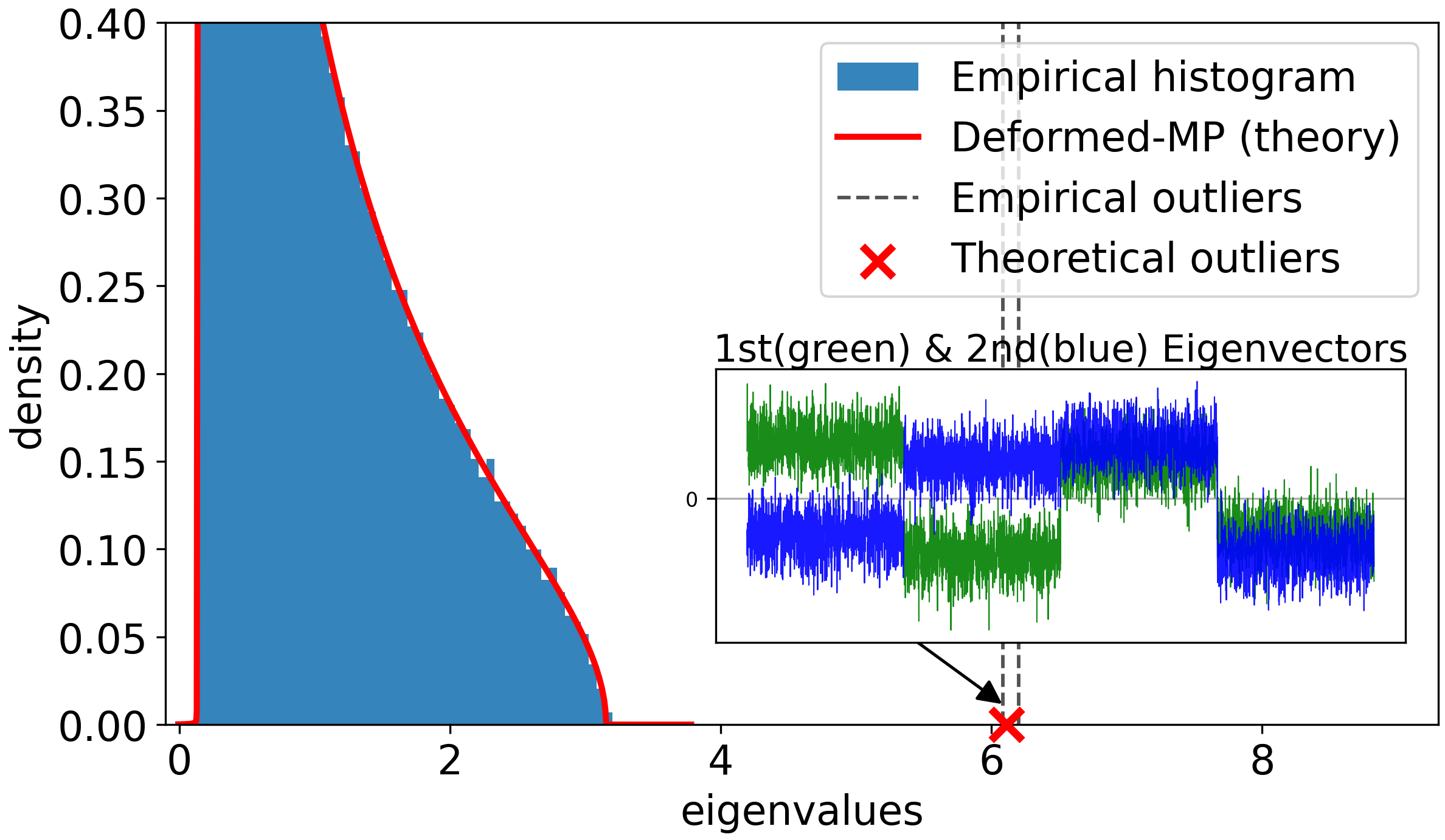}
    \end{minipage}\hfill
    \begin{minipage}[c]{0.39\linewidth}\centering
    \includegraphics[width=\linewidth]{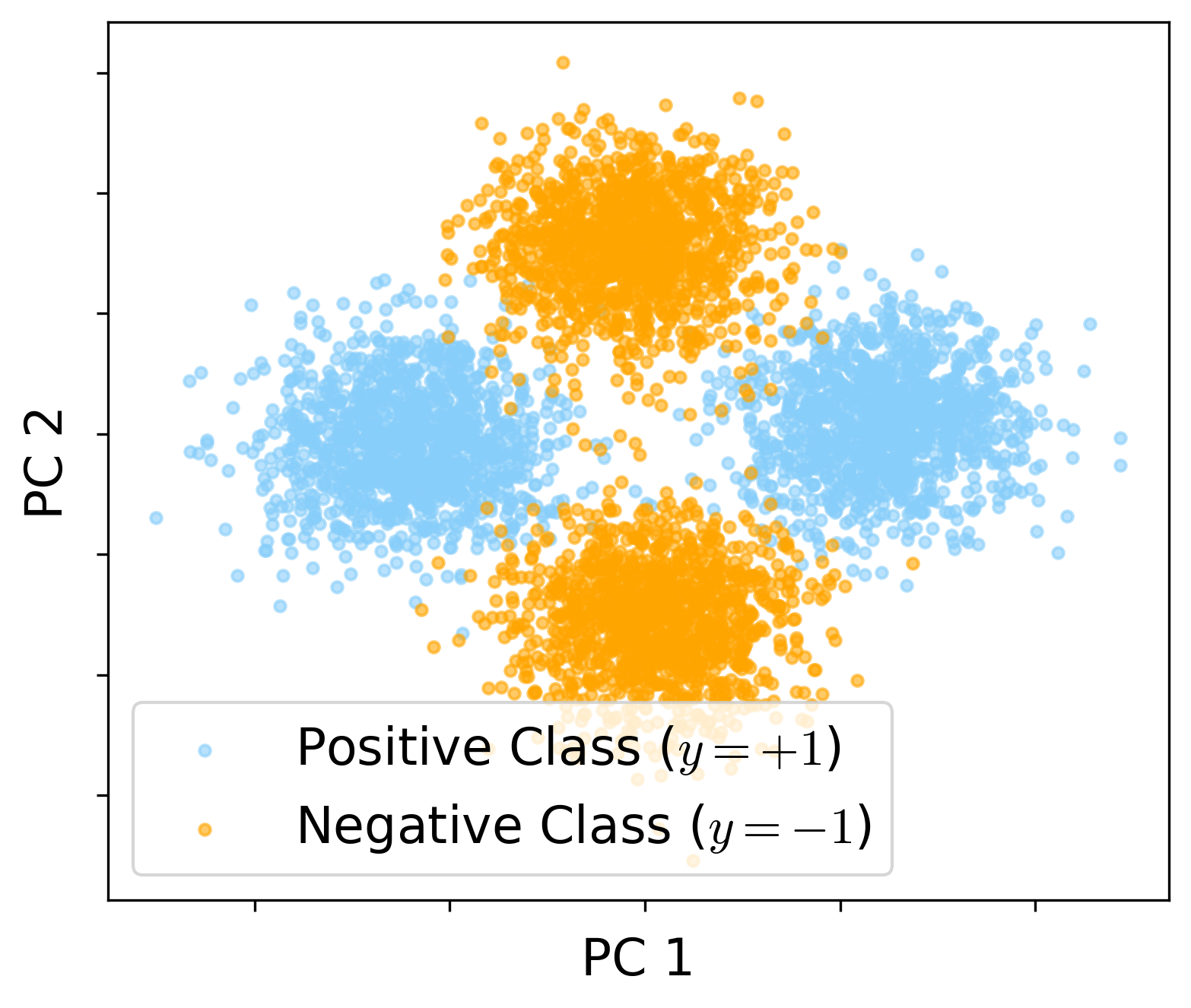}
    \end{minipage}
    \caption{\textbf{Finite-SNR proportional regime: linear classification of XOR fails.}
    Spectrum of the random CK matrix $\vK$ built from XOR data with $n=5000$, $N=d=15000$, SNR $r^2=36$, and $\sigma\propto\mathrm{ReLU}$.
    \textbf{Left:} empirical eigenvalue histogram of $\vK$ (blue) against the deformed
    MP law predicted by Theorem~\ref{thm:snr_finite_eig} (red curve); the
    red $\color{red}{\times}$ marks the predicted location of the two (coincident) outlier
    eigenvalues and the dashed lines mark the empirical values. The subfigure plots the two outlier eigenvectors against sample index, showing four-block pattern of the
    Gaussian-mixture components but not the binary labels.
    \textbf{Right:} kernel-PCA embedding using the top two principal components (PCs) of $\vK$
    (one point per sample, colored by its true label); the four clouds form the canonical
    XOR geometry and are \emph{not} linearly separable.
    The outlier eigenspace is asymptotically orthogonal to XOR label vector
    $\vy$ in this baseline case.}
    \label{fig:theorem3_figs}
\end{figure}
 
\subsection{Our Approach}
Much of the existing RMT literature on the CK matrix analysis lies in a
proportional regime ($n\asymp d\asymp N$), and it often involves the assumption of isotropic Gaussian data $\vZ\in\mathbb{R}^{d\times n}$. 
In this setting, a standard tool is a
\emph{linear equivalent} (LE) for the nonlinear feature matrix $\sigma(\vW\vZ)$.
This approach replaces the nonlinear features
with an affine Gaussian surrogate and an independent Gaussian noise term $\vN$:
\begin{equation}
\label{eq:LDE_intro}
    \sigma(\vW\vZ)\;\approx\; b_\sigma\,\vW\vZ \;+\; a_\sigma\,\vN,
    \qquad
    \vN_{ij}\iid\cN(0,1),
\end{equation}
for constants $a_\sigma,b_\sigma$ determined by $\sigma$. 
This approximation is also called \emph{Gaussian equivalence}
in prior work \citep{mei2019generalization,hu2020universality,goldt2021gaussian,bosch2023precise}. 
When applying the LE, one can show that the spectra of nonlinear CK matrices are asymptotically equivalent to a linear model, thus making classical RMT directly applicable. 
However, the LE also suggests that  the induced kernel behaves essentially like a
linear kernel plus isotropic noise; this can hide important nonlinear task structure. Figure~\ref{fig:theorem3_figs} shows that CK matrix with LE resembles a linear transform and its top PCs are orthogonal to the XOR labels. See Section~\ref{sec:numerics} for more details of the simulation.

For the XOR problem, the nonlinearity $\sigma$ is essential: it can transform a nonlinear classification task into one that is linearly accessible. 
To understand this, we vary the ML ``knobs'' listed above to identify regimes in which the LE approximation works, versus when it breaks down. 
To go beyond the classical LE in \eqref{eq:LDE_intro}, we develop a refined (quadratic) equivalent model that captures emergent spiked eigenvalues and eigenvectors carrying nonlinear information about CK and the dataset, thereby enabling linear classification on the nonlinear XOR problem. 
We refer to such spikes as \emph{quadratic informative spikes}. 
Addressing these questions will clarify the regimes in which NNs, and in particular various choices for training knobs of the ML training process, can recover and even amplify beneficial nonlinear features concealed in the dataset. 
 
\paragraph{Quadratic equivalents (QE).}
We model the XOR data as a low-rank
perturbation $\vX=\vZ+\vM$, where $\vZ$ is Gaussian noise and $\vM$ is a rank-two signal that encodes the XOR structure. Our goal is to understand how this signal
shows up in the CK spectrum. We go beyond the LE in \eqref{eq:LDE_intro}, which captures only the bulk spectrum, and derive a QE that precisely tracks all spikes via
a second-order Taylor approximation of $\sigma$. 
At a high level, this yields the decomposition
\begin{align}
\label{eq:QDE_null_intro}
    \sigma(\vW\vX)
    &\;\approx\;
    \underbrace{\sigma(\vW\vZ)}_{\text{bulk+uninformative spikes}}
    \;+\;
     \underbrace{\vT_1}_{\text{linear spikes}} 
    \;+\;
    \underbrace{ \vT_2}_{\text{quadratic spikes}} 
\end{align}
where $\vT_1$ and $\vT_2$ are finite-rank contributions induced by the signal
$\vM$ and $\sigma$ (see \eqref{eq:QE_formal} and Proposition~\ref{prop:approx}). 
In this QE, there are three terms.
The first term, $\sigma(\vW\vZ)$, determines the bulk spectrum; and it may also possibly induce \emph{uninformative outliers}, solely due to the architecture \citep{benigni2022largest}.
The second term, $\vT_1$, corresponds to a linear transformation of the signal $\vM$; while it can induce outliers, their eigenvectors turn out \emph{not} to be aligned with the XOR labels, meaning that they are not useful for nonlinear XOR classification. 
The third term, $\vT_2$, captures quadratic properties: depending on the correct tuning of the ML training ``knobs,'' this quadratic term $\vT_2$ can dominate and create \emph{quadratic informative spikes} (see \eqref{eq:QE_formal}).
This effectively provides a ``quadratic feature channel'' that makes the nonlinear XOR problem analyzable with RMT and linearly separable. In Section~\ref{sec:proof_idea}, we present the formal statement of this QE and the proof strategy of our main results.

\subsection{Our Contributions}\label{subsec:contrib}
We develop a RMT framework for the CK that is based on a QE and that captures when nonlinear
learnability emerges for XOR, in regimes where the LE behavior is insufficient. 
We are particularly interested in this emergence as a function of the following four ``knobs'' that are widely-used by ML practitioners (see Table~\ref{tab:summary}).
\begin{enumerate}[label=\textbf{\arabic*.}, leftmargin=*, itemsep=1.5pt]
  \item \textbf{Finite-SNR proportional limit ($n\asymp d$): negative results.}
  In the proportional regime with $\snr=\Theta(1)$, CK can exhibit outlier
  eigenvalues, but their eigenvectors do \emph{not} align with XOR labels; hence spectral methods with linear readouts on CK features fail to classify XOR
  (Theorem~\ref{thm:snr_finite_eig}; Figure~\ref{fig:theorem3_figs}). 
  We also prove failure for spectral clustering of Euclidean distance kernel.

  \item \textbf{Large-SNR proportional limit: emergence of quadratic informative spike.}
  We prove that with increasing SNR scaling and with proper activation choice, a quadratic informative outlier separates and aligns with the XOR labels, enabling linear spectral clustering/classification
  (Theorem~\ref{thm:large}).

  \item \textbf{Pretrained/spiked features: test-time BBP transition.}
  A low-rank perturbation of a random $\vW$ can yield the emergence of an outlier whose eigenvector aligns with the induced task
  direction (Theorem~\ref{thm:trained}) even for finite-SNR. Then, empirically, we show that linear spectral clustering occurs for CK matrix with weights from a pretrained NN on CIFAR-2 in Section~\ref{sec:numerics}.

  \item \textbf{Quadratic sample-size regime ($n\asymp d^{2}$): a quadratic kernel.}
  We prove that in this case the linear-width projected CK model ($n\asymp N$) is determined by a quadratic polynomial kernel and has an informative spike aligned with XOR labels (Theorem~\ref{thm:quadratic}).  

\end{enumerate}

\begin{table}[h]
\centering
\renewcommand{\arraystretch}{1.4}
\begin{tabular}{@{}lcccc@{}}
\toprule
\textbf{Regime} & \textbf{SNR $r^2$} & \textbf{Sample} & \textbf{Weights} & \textbf{Linear Label Align?} \\
\midrule
\rowcolor{mLightRed}
Case 1 in Section~\ref{sec:finitesnr} (Baseline) & $\Theta(1)$ & $n \asymp d$ & Random & \textcolor{mRed}{\textbf{No}} \\
\rowcolor{mLightYellow}
Case 2 in Section~\ref{subsec:large-snr} (Large SNR) & $\Theta(n^{1/2})$ & $n \asymp d$ & Random & \textcolor{mDarkGreen}{\textbf{Yes}} \\
\rowcolor{mLightYellow}
Case 3 in Section~\ref{subsec:trained} (Pretrained weight) & $\Theta(1)$ & $n \asymp d$ & Spiked & \textcolor{mDarkGreen}{\textbf{Yes}} \\
\rowcolor{mLightYellow}
Case 4 in Section~\ref{subsec:quadratic-regime} (Quad. sample size) & $\Theta(1)$ & $n \asymp d^2$ & Random & \textcolor{mDarkGreen}{\textbf{Yes}} \\
\bottomrule
\end{tabular}
\caption{\small Summary of different knobs. From Cases 2--4, $c_\sigma \neq 0$ is \textbf{necessary} in all these cases. If $c_\sigma = 0$ (e.g., tanh after centering/normalization), linear classification via CK is impossible in all cases. Here `$\textcolor{mDarkGreen}{\textbf{Yes}}$' represents non-trivial label alignment of CK spikes above BBP phase transition.}
\label{tab:summary}
\end{table}

\vspace{-5mm}
\section{Additional Literature Review}
\label{sec:literature}
 
\paragraph{Simplicity bias and low-rank structure.}
Recent work suggests that deep networks and transformers exhibit an implicit
\emph{distributional simplicity bias}, learning simpler (lower-order)
statistical structure of the data earlier in training
\citep{belrose2024neural,rende2024distributional}. Complementarily,
\citet{huh2023the} report a \emph{low-rank simplicity bias} in deep networks,
arguing that depth/over-parameterization can implicitly favor low effective-rank
feature representations. From a broader perspective,
\citet{wilson2025position} argues that many generalization phenomena often viewed
as ``mysterious'' in deep learning can be reconciled with classical theory via
soft inductive biases.

\paragraph{Spectrum of kernel random matrices.}
The study of kernel random matrices and their spectral distributions has become a major
topic in RMT.
Global convergence of the empirical spectral distribution (ESD) is known for
broad classes of nonlinear kernels
\citep{el2010spectrum,cheng2013spectrum,do2013spectrum}. In the proportional regime, \citet{cheng2013spectrum} and \citet{do2013spectrum}
characterize limiting spectra of inner-product kernels. 
The CK and NTK can be viewed as random kernels, whose spectra were
further studied \citep{fan2020spectra,adlam2020neural,peche2019note,pennington2017nonlinear,dabo2024traffic,guionnet2026global} and connected to generalization in random feature regression
\citep{mei2019generalization,adlam2020neural} and memorization of spurious feature \citep{bombari2024spurious} .

\paragraph{Nonlinear spiked covariance matrices.}
Spiked models exhibit phase transitions where informative eigenvalues separate
from the bulk \citep{baik2005phase,peche2006largest}. Recent work extends this
to \emph{nonlinear} random matrices, i.e., matrices whose entries are nonlinear
functions of high-dimensional random projections (such as $\sigma(\vW\vX)$),
leading to nonlinear analogues of BBP transitions and alignment formulas
\citep{wang2021deformed,feldman2023spectral,ba2023learning,wang2024nonlinearspikedcovariancematrices,guionnet2023spectral,benigni2022largest}. 
Outlier eigenstructure of CK and NTK has been studied \citet{li2025eigen}, using the deterministic equivalence from \citet{gu2022lossless}. 
Although they can deal with random deep networks and general Gaussian mixture datasets, their CKs and NTKs are expected kernels (namely first take width $N\to\infty$) and only equivalent to linear kernels.

\paragraph{Polynomial sample-size regimes.}
In the proportional regime ($n\asymp d$), LE/Gaussian-equivalence phenomena
often imply linear-kernel behavior for broad families of kernels. Moving to
polynomial regimes ($n\asymp d^k$), several works show that many kernel matrices
behave like polynomial kernels of degree $k$ (with rigorous deterministic
equivalents and sharp asymptotics)
\citep{xiao2022precise,lu2022equivalence,cheng2024dimensionfreeridgeregression,misiakiewicz2024nonasymptotictheorykernelridge}.
In the quadratic regime ($k=2$), nonlinear kernels can reduce in operator norm
to quadratic polynomial kernels \citep{ghorbani2019limitations,pandit2024universalitykernelrandommatrices}, enabling classification of quadratically separable data such as XOR
\citep{refinetti2021classifying}. \citet{wen2025does} recently proved a conditional Gaussian equivalence for random feature model at quadratic regime to show the asymptotic generalization error in this case.
 
\paragraph{Gaussian mixtures and XOR problem.}
Gaussian mixture classification, specifically the XOR mixture has become an important testbed for statistics and ML to evaluate various algorithms.
\citet{refinetti2021classifying} documented a sharp contrast between trained
two-layer networks and kernel methods on such mixtures, motivating a spectral
analysis of when (and why) kernels can succeed. 
Related RMT analyses for mixture
classification with kernels appear in \citep{liao2019innerkernelhd,couillet2019hdmrobust}. 
\citet{glasgow2023sgd} showed how training two layer NNs with SGD can learn this XOR problem even with a sample complexity $\tilde O(d)$. 
Recently, \cite{taheri2025theory} analyzed continual learning by studying one-hidden-layer quadratic NNs trained with gradient descent on an XOR dataset. \citet{demir2025onegrad} analyzed two-layer networks after one gradient step on $\vW$ with Gaussian-mixture data using higher-order polynomial equivalence. 
This motivates our spiked feature case in Section \ref{subsec:trained}, while our result focuses on label-aligned outliers for linear classification on XOR, unlike trained feature regression tasks in \citet{demir2025onegrad}. 

\section{Notations and Preliminaries}
\label{sec:model_data}
 
\paragraph{Settings of one-layer NN models.} We study the outputs of one-layer NNs when fed a dataset $\vX\in\R^{d\times n}$. 
Here, $d$ is the input dimension and $n$ is the sample size. 
We consider different scalings between $n$ and $d$ to determine the sample complexity of the model.
Let the weight matrix at random initialization be $\vW\in\R^{N\times d}$ with $[\vW]_{ij}\iid\cN(0,1)$ for $i\in [N],j\in[d]$, where $N$ is the width of NN. 
Denote the output of our one-layer NN and the CK matrix as
\begin{equation}\label{eq:output_Y}
    \vY:=\frac{1}{\sqrt{N}}\sigma (\vW {\vX}), \qquad \vK:=\vY^\top\vY ,
\end{equation}
where $\sigma:\R\to\R$ is a nonlinear activation. Our target is the eigenstructure of this CK matrix $\vK$ when $\vX$ has a nonlinear separable pattern. 
We now introduce the assumptions of the activation function $\sigma$.
\begin{assum}
\label{assump:sigma}
Assume $\sigma$ is three-times differentiable with
$\sup_{x\in\R}|\sigma'(x)|,|\sigma''(x)|,|\sigma'''(x)|\le \lambda_\sigma$, for some $\lambda_\sigma\in(0,\infty)$, and is centered and normalized with respect to
$\xi\sim\cN(0,1)$: $\E[\sigma(\xi)]=0$ and $\E[\sigma(\xi)^2]=1.$ Define the first and second order Hermite coefficients of $\sigma$ as
\begin{align}
    b_\sigma &:= \E[\sigma'(\xi)] = \E[\xi\,\sigma(\xi)] \in \R,
    \label{eq:b_sigma_def}\\
    c_\sigma &:= \E[\sigma''(\xi)] = \E[(\xi^2-1)\sigma(\xi)] \in \R.
    \label{eq:c_sigma_def}
\end{align}
\end{assum}
Many smooth activations satisfy Assumption~\ref{assump:sigma} after centering and normalization (e.g., $\tanh$, $\mathrm{erf}$, sigmoid, Softplus, GELU). While ReLU is not $C^{3}$, it can be approximated by smooth functions; we keep Assumption~\ref{assump:sigma} to avoid technical distractions.

\paragraph{XOR dataset.} 
Our dataset $\vX$ is a binary class Gaussian Mixture dataset with four balanced clusters. We denote the collection of $n$ data points by $\vX=[\vx_1,\ldots,\vx_n]\in\R^{d\times n}$, where
$
\vx_i\iid  \mathcal N(\frac{1}{\sqrt{d}}\vmu_{a_i},\frac{1}{d}\ident_d)
$ for $i\in[n]$, where $\vmu_{a_i}\in\R^d$ is the mean of one of the cluster defined in \eqref{eq:cluster_means} for $a_i\in[4]$. Without loss of generality, we reorder the data points so that we can write $\vX$ as 
\begin{equation}
\label{eq:XOR}
    \vX=\vM+\vZ ,
\end{equation}
where $\vZ \in\R^{d\times n}$ has i.i.d. entries with distribution $\mathcal{N}(0, \frac{1}{d} )$ and 
\begin{equation}\label{eq:def_M}
    \vM = r\sqrt{\frac{n}{2d}}\big(\vu_1\cdot \bm{v}_1 ^\top +  \vu_2\cdot \bm{v}_2^\top\big), 
\end{equation}
for some $r\ge 0$, with unit-norm vectors
\begin{equation}\label{eq:u_1u_2}
    \vu_1:= \frac{1}{\sqrt{d}}\begin{bmatrix}\mathbf{1}_{d/2} \\ -\mathbf{1}_{d/2}\end{bmatrix}, \quad \vu_2:=\frac{1}{\sqrt{d}}\begin{bmatrix}\mathbf{1}_{d/2} \\ \mathbf{1}_{d/2}\end{bmatrix}, \quad \vv_1 = \sqrt{\frac{2}{n}}\begin{bmatrix}
\mathbf{1}_{n/4} \\ -\mathbf{1}_{n/4} \\ \mathbf{0}_{n/2}
\end{bmatrix}, \quad
\vv_2 = \sqrt{\frac{2}{n}}\begin{bmatrix}
\mathbf{0}_{n/2} \\ \mathbf{1}_{n/4} \\ -\mathbf{1}_{n/4}
\end{bmatrix}.
\end{equation}
Here, the signal-to-noise ratio $(\snr)$ is defined as $\mathrm{SNR} := r^2$. In this case, the mean of each cluster is 
\begin{equation}\label{eq:cluster_means}
    \vmu_1=r\cdot\vu_1,
    \quad\vmu_2= r\cdot\vu_2,\quad
    \vmu_3=-\vmu_1,
    \quad\vmu_4=-\vmu_2.
\end{equation}
Following the above definition of $\vX$, we consider a binary classification dataset \(\{(\vx_i,y_i)\}_{i=1}^n\) with $y_i\in\{\pm1\}$.
We assume the dataset is \emph{not} linearly separable in the input space, i.e., there do not exist \(\vw\in\R^d\) and \(b\in\R\) such that $ y_i(\vw^\top \vx_i+b)>0,$ for all $i\in[n].$
For simplicity, in this case, the labels of data $\vX$ are defined by  
\begin{equation}\label{eq:y_def}
    \bm{y}^\top=\{y_i\}_{i=1}^n=[\mathbf{1}_{n/2}^\top,
-\mathbf{1}_{n/2}^\top].
\end{equation}
Our goal is to study the emergence of outlier eigenvalues in the spectrum of $\vK=\vY^\top\vY$ defined by \eqref{eq:output_Y}, and verify whether the corresponding eigenvectors (spikes) are aligned to this label vector~$\vy$, in order to (partially) recover the classes of $\vX$ defined by \eqref{eq:XOR}.

\paragraph{Marchenko–Pastur law.} 
Given a noise sample covariance $\vS=\vZ^\top\vZ \in \mathbb{R}^{n\times n}$ from \eqref{eq:XOR}, assume $d,n\to\infty$ with $\frac{n}{d}\to \psi \in(0,\infty)$. Then the ESD of $\vS$ converges weakly almost surely to the Marchenko--Pastur (MP) law \citep{marchenko1967}, denoted by $\rho^{\mathrm{MP}}_{\psi}$.
Let $a \;=\; (1-\sqrt{\psi})^2,$ and $b \;=\; (1+\sqrt{\psi})^2.$ When $\psi\le 1$, the MP law has an absolutely continuous part on $[a,b]$ with density
\[
\rho^{\mathrm{MP}}_{\psi}(x)
\;:=\;
\frac{1}{2\pi \psi x}\sqrt{(b-x)(x-a)}\;\mathbf{1}_{[a,b]}(x),
\]
and, when $\psi>1$, the density function is $\rho^{\mathrm{MP}}_{\psi}
=
\left(1-\frac{1}{\psi}\right)\delta_0
+
\rho^{\mathrm{MP}}_{\psi}(x)\,dx.$

\paragraph{Stieltjes transform.}
The Stieltjes transform provides an analytical tool for studying the limiting ESD in RMT. For $z\in\mathbb{C}\setminus[0,\infty)$, define the Stieltjes transform of $\rho^{\mathrm{MP}}_{\psi}$ by $m_{\psi}(z)
\;:=\;
\int_{\mathbb{R}}\frac{1}{x-z}\,\rho^{\mathrm{MP}}_{\psi}(dx).$
This Stieltjes transform fully captures the spectral distribution $\rho^{\mathrm{MP}}_{\psi}$, as we can recover the density from an inversion formula \citep{bai2010spectral}:  for $x>0$
$$\rho^{\mathrm{MP}}_{\phi}(x)
\;=\;
\frac{1}{\pi}\lim_{\eta\downarrow 0}\operatorname{Im} m_{\phi}(x+i\eta).$$

\paragraph{Limiting spectral distribution of CK.}
From \citet{pennington2017nonlinear,louart2018random,benigni2019eigenvalue,fan2020spectra}, we know the limiting ESD of $\vK$ is a deformed MP law, under certain conditions of $\vX$ and the proportional regime, denoted by
\begin{equation}\label{eq:deformedMP}
    \mu:=\rho^{\MP}_{\phi}\boxtimes \nu, \quad \nu:=(1-b_\sigma^2)+b_\sigma^2 \rho^{\MP}_{\psi},
\end{equation}
where $\rho^{\MP}_{\phi}$ is a standard MP law with aspect ratio $\phi\in(0,\infty)$ and $\nu$ is a shift of MP law $\rho^{\MP}_{\psi}$ with aspect ratio $\psi\in(0,\infty)$ which is compactly supported in $[0,\infty)$. Here $\boxtimes$ denotes free multiplicative convolution. We can fully characterize this deformed MP law $\mu$ by its Stieltjes transform 
\begin{equation}\label{eq:stieltjes}
    m_{\mu}(z)=\int\frac{1}{x-z} d\mu(x).
\end{equation}
We refer to \citep{bai2010spectral,fan2020spectra,wang2024nonlinearspikedcovariancematrices} for more details of free multiplicative convolution and these Stieltjes transforms. 
We denote the following transforms: 
\begin{align}
    z(s) = & -\frac{1}{s}+\phi \int\frac{x}{1+x s}\nu(dx),\qquad &\varphi(s) =& -\frac{sz'(s)}{z(s)}, \label{eq:z(s)}\\
    T(s) =&\frac{z(s)s^2-(\phi-1)s}{\phi}, \qquad &T^{-1}(t)=&-\frac{t\bigl(1-\psi b_\sigma^2\,t\bigr)}
{1+\bigl(1-\psi b_\sigma^2\bigr)t-\psi b_\sigma^2(1-b_\sigma^2)t^2},\label{eq:T_def}
\end{align}
which is used to characterize outliers of $\vK$. More properties of these transforms are given by Appendix \ref{subsec:z_toolkit}.
\section{Main Results}\label{sec:results}
We now present our main results on the outlier eigenvalues and eigenvectors of $\mathbf{K}$ across the regimes summarized in Table~\ref{tab:summary}.
In the first three regimes, we consider proportional limit, where $d,n,N\to\infty$ proportionally: Section~\ref{sec:finitesnr} studies finite $\mathrm{SNR}=\Theta(1)$ regime; 
Section~\ref{subsec:large-snr} analyzes large $\mathrm{SNR}=\Theta(n^{1/4})$ case; and
Section~\ref{subsec:trained} replaces the random weight $\mathbf{W}$ with a pretrained weight. 
Then, Section~\ref{subsec:quadratic-regime} considers the quadratic sample-size regime $n\asymp d^2$, with finite SNR and random weights.
In each case, we predict the emergence of spikes of $\vK$ and whether their eigenvectors enable \emph{linear} classification of $\vX$.

\subsection{Proportional Limit with Finite SNR Case}
\label{sec:finitesnr} 
We first consider a baseline regime in which all dimensions grow proportionally and the $\snr$ of $\vX$ is finite. 
We show that the outlier eigenvectors cannot be used to linearly classify the dataset, even partially. 
 
\begin{assum}[Proportional limit] \label{assump:asymptotics}
    Let $n,d,N\to \infty$ such that	$n/N \to \phi\in (0,+\infty)$ and $n/d \to \psi\in (0,+\infty)$ for some fixed constants $\phi$ and $\psi$.
\end{assum}
We denote some constants used to define the BBP thresholds in the following theorem by
\begin{equation}
\label{eq:spikes}
\begin{aligned}
\tau:=\frac{c_\sigma^2}{2}\psi,
\qquad
\beta_{\rm lin}:=\frac{r^2\psi}{2}b_\sigma^2,
\qquad
\tau_{\rm crit}:=b_\sigma^2\sqrt\psi(1+\sqrt\psi).
\end{aligned} 
\end{equation}
Recall the definition of $T(\cdot)$ and $T^{-1}(\cdot)$ in \eqref{eq:T_def}. When $\tau>0$ and $\beta_{\rm lin}>0$, respectively, we define 
\[
s_{\rm un}:=T^{-1}(1/\tau), \quad s_{\rm lin}:=T^{-1}(1/\beta_{\rm lin}).
\]

\begin{theorem}
\label{thm:snr_finite_eig}
Under Assumptions~\ref{assump:sigma} and \ref{assump:asymptotics}, we further assume that $r\ge 0$ is a constant. Consider $\vK$ with XOR data $\vX$ defined in Section \ref{sec:model_data}. Recall the transforms $z(\cdot),\varphi(\cdot)$ and $T(\cdot)$ defined by \eqref{eq:z(s)} and \eqref{eq:T_def}. Denote $\vu:=\frac{1}{\sqrt n}\mathbf 1_n,$ and $\vm:=\E[\sigma(\vw^\top\vZ)\mid \vZ]\in\R^n$ conditioned on $\vZ$, where $\vw\sim\cN(0,\vI_d)$. Assume in addition that $s_{\rm un}\neq s_{\rm lin}$ if both are nonzero.
Then the CK matrix $\vK$ satisfies the following.
\begin{enumerate}[label=\textbf{(\roman*)},leftmargin=*]
\item (\textbf{Bulk.}) The ESD of $\vK$ converges weakly in probability to $\mu$ defined in \eqref{eq:deformedMP}.
\item \textbf{(Uninformative outliers.)}
If $c_\sigma\neq0$ and $z'(s_{\rm un})>0$, then $\vK$ has two eigenvalues outside the support of $\mu$ satisfying
$
\widehat\lambda_{{\rm un},1},\ \widehat\lambda_{{\rm un},2}=z(s_{\rm un})+o_\P(1).
$
Let $\widehat\vP_{\rm un}$ be the spectral projector onto this two-dimensional outlier eigenvalues. Then
\begin{align}
\|\widehat\vP_{\rm un}\vu\|^2&=\varphi(s_{\rm un})+o_\P(1),\label{eq:thm_un_proj_u}\\
\left\|\widehat\vP_{\rm un}\frac{\vm}{\|\vm\|}\right\|^2&=-\frac{\varphi(s_{\rm un})}{\tau^2 s_{\rm un}^2T'(s_{\rm un})}+o_\P(1),\label{eq:thm_un_proj_m}\\
\vu^\top\widehat\vP_{\rm un}\frac{\vm}{\|\vm\|}&=o_\P(1).\label{eq:thm_un_proj_cross}
\end{align}
Besides, when $b_\sigma=0$, \eqref{eq:thm_un_proj_u} and \eqref{eq:thm_un_proj_m} are asymptotically the same.
If $c_\sigma=0$ or $z'(s_{\rm un})\le0$, then $\vK$ has no such uninformative outliers.

\item \textbf{(Linear spikes from XOR.)}
If $z'(s_{\rm lin})>0$ and $\beta_{\rm lin}>0$, then $\vK$ has additional two outlier eigenvalues outside the support of $\mu$ satisfying
$
\widehat\lambda_{{\rm lin},1},\ \widehat\lambda_{{\rm lin},2}=z(s_{\rm lin})+o_\P(1),
$
and the spectral projector $\widehat\vV_{\rm lin}$ of the corresponding two-dimensional eigenspace satisfies
\begin{equation}\label{eq:alignment_finite_snr}
\big\|(\vv_1\vv_1^\top+\vv_2\vv_2^\top)\widehat\vV_{\rm lin}\big\|_F^2
=2\cdot\left(-\frac{\varphi(s_{\rm lin})}{\beta_{\rm lin}^{\,2}s_{\rm lin}^2T'(s_{\rm lin})}\right)+o_\P(1).
\end{equation}
If $\beta_{\rm lin}=0$ or $z'(s_{\rm lin})\le0$, then $\vK$ has no such linear outliers.

\item \textbf{(No label alignment.)}
Let $\widehat\vV_{\rm out}$ denote the orthogonal projector onto the direct sum of all separated outlier clusters described in \textbf{(ii)}--\textbf{(iii)} if they exist. Then, for the XOR labels $\vy$,
\begin{equation}\label{eq:alignment_y_case1}
\frac1n\|\widehat\vV_{\rm out}\vy\|^2=o_\P(1).
\end{equation}
\end{enumerate}
\end{theorem}
Theorem~\ref{thm:snr_finite_eig} exhibits a BBP-type phase transition for the CK matrix with XOR data: e.g., $z'(s_{\rm lin})>0$ is the threshold showing when spikes separate from the deformed MP bulk $\mu$; in the separated case, the associated eigenvectors have nontrivial asymptotic overlap with certain directions.
There are up to four candidate spikes outside $\supp{\mu}$: two outliers (Theorem \ref{thm:snr_finite_eig}(iii)) induced by the linear Hermite component of $\sigma$ and spike strength $s_{\rm lin}$ depending on $r$, and two potential ``architectural'' outliers (Theorem \ref{thm:snr_finite_eig}(ii)) corresponding to $c_\sigma\neq 0$ \citep{benigni2022largest} . 
The latter is ``uninformative'' in the sense that it may persist even in the null model $r=0$ (pure noise data). 
See Section~\ref{sec:numerics} for further discussion. 
To focus on informative spikes, previous works always impose $c_\sigma=0$ \citep{wang2024nonlinearspikedcovariancematrices}.
The proof of Theorem~\ref{thm:snr_finite_eig} is given in Appendix~\ref{app:proof_finite_snr}.

\begin{remark}
In Theorem \ref{thm:snr_finite_eig}(ii), there is a necessary condition of the BBP threshold $z'(s_{\rm un})>0$ for the emergence of uninformative spikes: $z'(s_{\rm un})>0\implies\tau>\tau_{\rm crit}.$
Hence, from \eqref{eq:spikes}, uninformative spikes can be removed if \(\tau\le \tau_{\rm crit}\) holds, namely $\sqrt{\psi}\bigl( \frac{c_\sigma^2}{2} - b_\sigma^2 \bigr) \le b_\sigma^2.$ When the second Hermite coefficient $c_\sigma/\sqrt{2}$ is smaller than the first $b_\sigma$ in absolute value (for instance, a centered and normalized GELU function in Appendix~\ref{app:gelu}), this condition is satisfied for all $\psi>0$, and thus no uninformative spike emerges.
\end{remark}

\begin{remark}
From \eqref{eq:alignment_y_case1} in Theorem~\ref{thm:snr_finite_eig}(iv), all separated outlier eigenvectors are asymptotically orthogonal to the label vector $\vy$. Thus, these outliers cannot be used for linear classification. Although the spikes in Theorem~\ref{thm:snr_finite_eig}(iii) carry information about the within-class structure (alignment with $\vv_1,\vv_2$ in
\eqref{eq:alignment_finite_snr}), they cannot recover the XOR labels $\vy$ without an additional nonlinear transform on spikes.
\end{remark}

\paragraph{Kernel spectral clustering also fails.}
Previously \citet{couillet2016kernel} studied this Euclidean distance kernel clustering for general Gaussian mixture data.  
We can apply this result to the kernel spectral clustering algorithm \citet{ng2001spectralclustering} and get similar negative results as Theorem~\ref{thm:snr_finite_eig}.
Full details are provided in Appendix~\ref{app:kernel_clustering}.

\begin{theorem}[Kernel spectral clustering for XOR]
\label{thm:kernel_cluster}
Let Assumption~\ref{assump:asymptotics} hold, and let  $r=\Theta(1)$. Consider the
kernel matrix $\big(\vK^{(f)}\big)_{ij}
=
f(\|\vx_i-\vx_j\|^2),
$ for $i,j\in[n],$ and some $f(x)\in C^3$ near $x=2$, with $f(2)>0$ and $f'(2) \ne 0$, and the
normalized Laplacian
$\vL = n\,\vD^{-1/2}\vK^{(f)}\vD^{-1/2}$ where $\vD=\diag(\vK^{(f)}\mathbf{1}_n).  $
Then,
\begin{itemize}
\item \textbf{(Isolated eigenvalues.)}
If $r^{2} > 2\sqrt{\psi^{-1}}$, then $\vL$ has two nontrivial isolated
eigenvalues 
whose limits are given
by Theorem~\ref{thm:L_eigenvalues} in Appendix~\ref{app:kernel_clustering}. If $r^{2} \le 2\sqrt{\psi^{-1}}$, then $\vL$ has no such outliers.

\item \textbf{(No label alignment.)}
Let $\vPi$ be the orthogonal projector onto the span of the eigenspace associated with the above outliers. Then, $\frac{1}{n}\,\vy^\top \vPi \vy \;\to\; 0$ in probability.
 
\end{itemize}
\end{theorem}


\subsection{Large-SNR Proportional Limit Case}
\label{subsec:large-snr}
 
Theorem~\ref{thm:snr_finite_eig} shows that spectral methods based on outlier eigenvectors fail to linearly classify XOR in the finite-SNR proportional regime. Linear informative spikes may still emerge, but a purely linear approximation of $\sigma$ cannot capture the quadratic feature needed for XOR. To make XOR linearly learnable, we now increase the SNR so that the quadratic Hermite component of $\sigma$ becomes non-vanishing. 
Then two outliers created by the quadratic component become label-informative, enabling linear classification. Numerical simulations of all these cases are in Section~\ref{sec:numerics}.

In this section, we consider large SNR regime by letting $r$ in \eqref{eq:def_M} grow with $n$ as follows:
\begin{equation}\label{eq:large_snr_constants}
    r_0:=\lim_{n\to\infty} \frac{r}{n^{1/4}}\in(0,\infty),\qquad 
    \kappa_0:=\frac{r_0^4\psi^2}{4},\qquad 
        \beta_{\rm quad}:=c_\sigma^2\kappa_0 .
\end{equation} 
For \(\beta_{\rm quad}>0\), define the label-informative factor
\begin{equation}\label{eq:s_y}
        \qquad s_y:=T^{-1}(1/\beta_{\rm quad}).
\end{equation}
Recall $\tau$ from \eqref{eq:spikes}. When \(c_\sigma\neq0\), equivalently \(\tau>0\) and \(\beta_{\rm quad}>0\), we define non-label factors
\begin{equation}\label{eq:s+-}
        t_\pm
        :=\frac{1\pm\sqrt{\beta_{\rm quad}/(\tau+\beta_{\rm quad})}}{\tau},
        \qquad
        s_\pm:=T^{-1}(t_\pm).
\end{equation}

\begin{theorem}\label{thm:large}
Under Assumptions~\ref{assump:sigma} and~\ref{assump:asymptotics}, let $\vP_{\rm lin}=\vv_1\vv_1^\top+\vv_2\vv_2^\top$, $\vP_\perp=\vI-\vP_{\rm lin}$, and $\vK_\perp:=\vP_\perp \vK \vP_\perp$. Then, under the scaling in \eqref{eq:large_snr_constants}, the following statements hold.

\begin{enumerate}[label=\textbf{(\roman*)},leftmargin=*]
\item \textbf{(Bulk.)} The ESD of \(\vK\) converges
weakly in probability to \(\mu\) defined in \eqref{eq:deformedMP}.

\item \textbf{(Diverging linear outliers.)} If \(b_\sigma\neq0\), then the top two eigenvalues of \(\vK\) satisfy
\begin{equation}\label{eq:large_linear_outliers}
\widehat\lambda_1(\vK),\widehat\lambda_2(\vK)=\frac{b_\sigma^2r_0^2\psi}{2}\,n^{1/2}(1+o_\P(1)),
\end{equation}
and the associated eigenspace projector $\widehat\vP_{\rm lin}$ satisfies $\|\widehat\vP_{\rm lin}-\vP_{\rm lin}\|_{\rm F}\xrightarrow{\P}0,$ and $ \frac{1}{n}\vy^{\top}\widehat\vP_{\rm lin} \vy\xrightarrow{\P}0.$

\item \textbf{(Order-one outliers of \(\vK_\perp\).)}  Assume \(c_\sigma\neq0\), so
\(\tau,\beta_{\rm quad}>0\).  Let $ \mathcal S_{\rm cand}:=\{s_+,s_-,s_y\}$
with \(s_\pm,s_y\) defined by \eqref{eq:s_y} and
\eqref{eq:s+-}.  An element \(s\in\mathcal S_{\rm cand}\) produces a
separated order-one outlier of \(\vK_\perp\) precisely when
\begin{equation}\label{eq:large_final_admissible}
        z'(s)>0,
        \qquad z(s)\in\mathbb R\setminus(\supp\mu\cup\{0\}).
\end{equation}
More precisely, fix a limiting location \(\lambda_\star\in\mathbb R\setminus
(\supp\mu\cup\{0\})\), and define the algebraic multiplicity
\begin{equation}\label{eq:large_final_multiplicity}
        m_\star
        :=\#\{s\in\mathcal S_{\rm cand}: z(s)=\lambda_\star,\ z'(s)>0\}.
\end{equation}
For any sufficiently small $
\delta>0$ and interval
\(I_\star=(\lambda_\star-\delta,\lambda_\star+\delta)\) containing no other
candidate location and disjoint from \(\supp{\mu}\cup\{0\}\), \(\vK_\perp\) has
exactly \(m_\star\) eigenvalues in \(I_\star\), and all of them converge to
\(\lambda_\star\) in probability. \(\vK_\perp\) has at most three separated order-one outliers, counting multiplicity.   

\item \textbf{(Label alignment.)}  Suppose that \(s_y\) is
admissible in the sense of \eqref{eq:large_final_admissible}.  Let
\(\lambda_y:=z(s_y)\), and let \(\widehat\vP_y^\perp\) be the spectral projector of
\(\vK_\perp\) onto a small interval around \(\lambda_y\). Then
\begin{equation}\label{eq:large_final_label_projector}
        \frac{1}{n} \vy^{\top}\widehat\vP_y^\perp \vy
        \xrightarrow{\P}
        \gamma_y(s_y)
        :=-\frac{\varphi(s_y)}{\beta_{\rm quad}^2s_y^2T'(s_y)}  >0 .
\end{equation} 
If \(\lambda_\star=z(s_\pm)\) is admissible from \eqref{eq:large_final_admissible} and distinct from the label candidate $\lambda_y$, then the corresponding spectral projector
\(\widehat\vP_{\rm nl,\star}^\perp\) satisfies $ \frac{1}{n} \vy^{\top}\widehat\vP_{\rm nl,\star}^\perp \vy
        \xrightarrow{\P}0.$ 
\end{enumerate}
\end{theorem}
 
In contrast to Theorem~\ref{thm:snr_finite_eig}, when $r=\Theta(n^{1/4})$ the linear outliers characterized by Theorem~\ref{thm:snr_finite_eig}(iii) become diverging and their eigenvectors remain asymptotically orthogonal to the label vector $\vy$. 
Meanwhile, the quadratic component of $\sigma$ may produce three $O(1)$ outliers (Theorem~\ref{thm:large}(iii)). 
Here, \eqref{eq:large_final_label_projector} gains a positive alignment with labels $\vy$ because the quadratic features $\vv_1^{\odot 2}$ and $\vv_2^{\odot 2}$ are preserved. 
Hence, a linear readout on the $\lambda_y$-spike eigenspace in \eqref{eq:large_final_label_projector} can classify XOR. The proof of Theorem~\ref{thm:large} is in Appendix~\ref{app:large_SNR}.

When $b_\sigma=0$, the diverging outliers in~\eqref{eq:large_linear_outliers} vanish. If \(c_\sigma=0\), then \(\tau=\beta_{\rm quad}=0\), only the bulk and the possible two
diverging linear outliers in \eqref{eq:large_linear_outliers} remain. The absent quadratic component leads linear classification via these spikes impossible, thus highlighting the role of the activation knob in preserving nonlinear structure. 
For further discussion about the role of non-linearity choice, see numerical simulations Section~\ref{sec:numerics}.

\subsection{Pretrained Weight Matrix Case}
\label{subsec:trained}
 
Empirically, weight matrices in well-trained NNs often exhibit spectral spikes (even heavy tails) \citep{martin2018implicitself,martin2019heavytail,MM20_SDM,MM20a_trends_NatComm,YTHx22_TR}. 
Such spectral behavior suggests that training induces nontrivial feature structure in the weights from dataset \citep{ba2022high,ba2023learning,demir2025onegrad}.
However, the above theorems consider only random features, and therefore they do not capture task-dependent structure learned in the pretrained $\vW$. 
As a toy model for feature learning, we consider a simple but nontrivial pretrained model in which the weight is a rank-one perturbation of the random initialization:
\begin{equation}\label{eq:pretrained_weight}
\vW_1=\vW+\theta\,\va\vb^\top,
\end{equation}
where $\va$ and $\vb$ are independent of the \emph{test} dataset $\vX$. 
A related rank-one structure arises, for example, after one-step gradient descent in certain student--teacher settings \citep{ba2022high,demir2025onegrad}, and it is also connected to LoRA \citep{hu2021lora}.
Below, we analyze a \emph{test-time} BBP transition, i.e., we evaluate the CK built from $\vW_1$ on independent XOR test data $\vX$, and we characterize the emergence and alignment of the induced spike. 

Let \(\mu,z,T,m_\mu\) be the limiting law and transforms defined in Section \ref{sec:model_data}.  For any constants $r\ge 0$ and $\theta_0>0$, define the following constants:
\begin{equation}\label{eq:sw_eta_constants}
         \beta_q:=\frac{3c_\sigma^2}{4}\theta_0^4\phi(3+3r^2+\frac{r^4}{2}) .
        \qquad
        \omega_q:=\frac{8+8r^2+r^4}{12+12r^2+2r^4},
        \qquad
        \chi_y:=\frac{r^4}{12+12r^2+2r^4}.
\end{equation} 
Recall $\tau$ defined in \eqref{eq:spikes}. We define, for \(t\in\mathbb C\),
\begin{equation}\label{eq:sw_Fq_t}
        F_q(t):=(1-\tau t)^2
        -\beta_q t\,(1-\tau\omega_q t) (1-\frac23\tau t ).
\end{equation} 

\begin{theorem} \label{thm:trained}
Suppose $r\ge 0$ is a constant and let $\kappa:=1+\frac{r^2}{2}$. Let pretrained $\vW_1$ be defined in \eqref{eq:pretrained_weight} with $\vb=\mathbf 1_d,$ $\va\sim\cN(\mathbf 0,\vI_N/N),$ and $\theta=\theta_0N^{1/4},$
where \(\va\) is independent of \((\vW,\vX)\), and $\theta_0>0$ is a fixed constant. Under Assumptions~\ref{assump:sigma} and~\ref{assump:asymptotics}, consider $\vY:=\frac1{\sqrt N}\sigma(\vW_1\vX)$, pretrained CK $\vK:=\vY^\top\vY,$ and projected CK $\vK_s:=\vPi_s\vK\vPi_s$
where $\vPi_s:=\vI_n-\widehat\vs\widehat\vs^\top,$ $ \widehat\vs:=\frac{\vs}{\|\vs\|},$ $\vq:=\vs^{\odot2},$ and $\vs:=\vX^\top\mathbf \vb$.
Then,
\begin{enumerate}[label=\textbf{(\roman*)},leftmargin=*]
\item \textbf{(Bulk.)}
The ESD of \(\vK\) converges weakly in probability
to \(\mu\) defined in \eqref{eq:deformedMP}.

\item \textbf{(Diverging spike.)}
If \(b_\sigma\neq0\), then the largest eigenvalue of \(\vK\) satisfies
$
        \widehat\lambda_1(\vK)
        =b_\sigma^2\theta_0^2\phi\kappa\sqrt N\,(1+o_\P(1)).
$
If \(\widehat\vv_1\) is the associated unit eigenvector, then
\begin{equation}\label{eq:sw_giant_vec}
        \left|\left\langle \widehat\vv_1,\widehat\vs\right\rangle\right|^2
        \xrightarrow{\P}1,
        \qquad
        \frac1n\left|\widehat\vv_1^\top\vy\right|^2\xrightarrow{\P}0 .
\end{equation}
If \(b_\sigma=0\), this diverging spike is absent.

\item\textbf{(Order-one outliers of $\vK_s$.)}
All asymptotically separated order-one outliers of \(\vK_s\) are determined by the roots $s$ of the
three equations:
\begin{equation}\label{eq:sw_order_one_factors}
        1-\beta_{\rm lin}T(s)=0,
        \qquad
        1-\frac{\beta_{\rm lin}}{\kappa}T(s)=0,
        \qquad
        F_q(T(s))=0,
\end{equation} where $\beta_{\rm lin}$ is defined in \eqref{eq:spikes}.
Fix \(\lambda_\star\in\mathbb R\setminus(\supp\mu\cup\{0\})\),
and let \(m_\star\) be the total algebraic multiplicity of all real solutions
\(s\) to \eqref{eq:sw_order_one_factors} such that $z(s)=\lambda_\star$ and $z'(s)>0 .$
For any sufficiently small $\delta>0$ and interval
\(I_\star=(\lambda_\star-\delta,\lambda_\star+\delta)\), disjoint from
\(\supp\mu\cup\{0\}\) and containing no other candidate location, \(\vK_s\) has
exactly \(m_\star\) eigenvalues in \(I_\star\), and all of them converge to
\(\lambda_\star\) in probability.  If a real root $s$ satisfies \(z'(s)\le0\), it
does not produce an outlier.
\item\textbf{(Label alignment.)}
Outliers coming only from the first two equations in
\eqref{eq:sw_order_one_factors} have zero normalized label overlap.  Suppose
 that \(\widehat\vP_{\star}^{(s)}\) is the spectral projector of \(\vK_s\) onto
\(I_\star\) where the roots of \(F_q(T(s))=0\) inside \(I_\star\) are simple. Recall $m_\mu$ defined in~\eqref{eq:stieltjes}. Then
\begin{equation}\label{eq:sw_label_projector}
        \frac1n\vy^\top\widehat\vP_{\star}^{(s)}\vy
        \xrightarrow{\P}
        \sum_{s\in\mathcal R_q(\lambda_\star)} 
         -\frac{m_\mu(z(s))z'(s)}{T'(s)}
        \frac{\beta_q\chi_y\,T(s)\,(1-\tau T(s))
        \left(1-\frac23\tau T(s)\right)}{F_q'(T(s))}\ge0
\end{equation}
where \(\mathcal R_q(\lambda_\star)\) is the set of admissible roots of
\(F_q(T(s))=0\) with \(z(s)=\lambda_\star,z'(s)>0   \).
\end{enumerate}
\end{theorem}
Here \(\vK_s\) has at most five separated order-one outliers from \eqref{eq:sw_order_one_factors}, counting
multiplicity: two linear informative spikes related to $\beta_{\rm lin}$ when $b_\sigma\neq0$ and at most three roots from \(F_q(T(s))=0\) related to both uninformative and label-aligned spikes when $c_\sigma\neq0$.
This result shows that the large data-dependent spike in the pretrained weight $\vW_1$ amplifies the signals in XOR and alters the approximation: the usual LE is replaced by QE to capture the induced quadratic feature. 
Theorem \ref{thm:trained}(iv) shows this quadratic informative eigenvector aligns with labels. 
If $c_\sigma=0$, so $\beta_q=0$, then we cannot obtain label alignment.
Empirical simulations of test-time BBP are presented in Section~\ref{sec:numerics}. The proof of Theorem~\ref{thm:trained} is in Appendix~\ref{app:pretrained}.

\subsection{Quadratic Sample-size Case}\label{subsec:quadratic-regime}

Sample complexity can determine whether nonlinear structure is visible in high dimensions. In the proportional regime $n=\Theta(d)$, the LE suppresses higher-order structure, making the CK behave like a ``linear'' model. We now turn to the quadratic sample-size regime $n=\Theta(d^2)$, where the CK behaves like a quadratic polynomial kernel \citep{xiao2022precise,wen2025does}
and becomes label-informative for XOR even at finite SNR.

 
\begin{theorem}
\label{thm:quadratic}
Under Assumption~\ref{assump:sigma} with $c_\sigma\neq0$, consider the XOR model \eqref{eq:XOR}--\eqref{eq:y_def} with finite $r\in(0,\infty)$. Assume that $\frac{n}{p}\to\gamma\in(0,\infty)$ and $\frac{n}{N}\to\phi\in(0,\infty),$ where
$p:=\frac{d(d+1)}{2}$. Denote $\ell := \frac{r^4}{4}$ and $\lambda_{\mathrm{out}}(\gamma,\ell):=1+\gamma+\gamma\ell+\frac1\ell$. Then there exists a projection $\vP_\sigma$ defined in \eqref{eq:P-sigma-plus-bbp} such that the following hold for $\mbi{K}_\sigma:=\mbi{P}_\sigma\mbi{K}\mbi{P}_\sigma$.

\begin{enumerate}[label=\textbf{(\roman*)},leftmargin=*]

\item \textbf{(Bulk.)}  The ESD of $\mbi{K}_\sigma$ converges
weakly to $\mu_{\rm q}:= \rho_\phi^{\mathrm{MP}}\boxtimes\nu_{\rm q}$ where $ \nu_{\rm q}:=(1-b_\sigma^2-\frac{c_\sigma^2}{2})+\frac{c_\sigma^2}{2} \rho_\gamma^{\mathrm{MP}}$.

\item \textbf{(Spike of the population CK.)} If
$\ell>\gamma^{-1/2}$, then the population CK matrix $\E[\vK_\sigma|\vX]$ has a label-aligned outlier $\widehat{\Lambda_y}=\Lambda_y+o_{\P}(1)$ where $\Lambda_y:=(1-b_\sigma^2-\frac{c_\sigma^2}{2})+\frac{c_\sigma^2}{2} \lambda_{\mathrm{out}}(\gamma,\ell)\not\in\supp{\nu_{\rm q}}$. In this case, its associated eigenvector is aligned with label $\vy$. Otherwise, no such outlier exists.

\item \textbf{(Spike of the linear-width CK.)}  If
$\ell>\gamma^{-1/2}$ and $z'(-1/\Lambda_y)>0$, then for every sufficiently small deterministic interval $I_y$ around
\begin{equation}\label{eq:lambda-K-y-def}
        \lambda_y
        =z\!\left(-\frac1{\Lambda_y}\right)
        :=\Lambda_y\left(
          1+
          \phi\int\frac{t}{\Lambda_y-t}\,\nu_{\rm q}(dt)
          \right),
\end{equation}
which is disjoint from $\supp{\mu_{\rm q}}$, $\mbi{K}_\sigma$ has exactly one eigenvalue in $I_y$ which converges in probability to $\lambda_y$.
If $\widehat{\mbi{v}}_y$ is the associated unit-norm eigenvector, then
\begin{equation}\label{eq:final-label-overlap}
        \left|\left\langle \widehat{\mbi{v}}_y,\frac{\mbi{y}}{\sqrt n}\right\rangle\right|^2
        \xrightarrow{\P}
        \varphi\!\left(-\frac1{\Lambda_y}\right)\frac{\gamma\ell^2-1}{\gamma\ell(\ell+1)},
\qquad
        \varphi\!\left(-\frac1{\Lambda_y}\right)
        :=
        \frac{
        1-\phi\int\frac{t^2}{(\Lambda_y-t)^2}\,\nu_{\rm q}(dt)
        }{
        1+\phi\int\frac{t}{\Lambda_y-t}\,\nu_{\rm q}(dt)
        }.
\end{equation}
If either
$\ell\le\gamma^{-1/2}$ or
$z'(-1/\Lambda_y)\le0$, then no separated label-aligned eigenvalue of $\mbi{K}_\sigma$ is produced.
\end{enumerate}
\end{theorem}

Theorem~\ref{thm:quadratic} only characterizes the label-aligned outlier generated by the XOR.  The full CK may contain additional outliers which are removed by the projection $\vP_\sigma$: a
mean outlier, $d$ linear-term outliers when $b_\sigma\neq0$, and
possibly $o(n)$ quadratic nuisance outliers caused by the rank-$O(d)$
non-isotropic covariance component in $\E[\vK|\vX]$. These
outliers do not affect the limiting ESD and have asymptotically zero overlap with XOR label. The $\vP_\sigma$ defined in \eqref{eq:P-sigma-plus-bbp} only depends on $\sigma$ and $\vX$. The proof of Theorem~\ref{thm:quadratic} is in Appendix~\ref{sec:quadratic-regime-proof}.


\section{Proof Strategy for Main Results}
\label{sec:proof_idea}
Dataset $\vX$ in \eqref{eq:XOR} is decomposed as a rank-2 XOR signal $\vM$ and isotropic noise $\vZ$. 
In the proportional limit, for Theorems~\ref{thm:snr_finite_eig}, \ref{thm:large}, and \ref{thm:trained}, we expand $\sigma(\vW(\vZ+\vM))$ around $\vW\vZ$ and retain Hermite components of $\sigma$ up to second order. This yields a bulk term $\sigma(\vW\vZ)$ and a low-rank spike decomposition for $\vY$.
Let $\vY_0 := \frac{1}{\sqrt{N}}\sigma(\vW\vZ) $ and $\theta_{\snr} :=r\sqrt{\frac{n}{2d}}$.
Consider Gaussian random vectors $\vg_1=\vW\vu_1,\ \vg_2=\vW\vu_2$.
Theorems~\ref{thm:snr_finite_eig}, \ref{thm:large}, and
\ref{thm:trained} repeatedly use the following QE formulation, a formal approximation of \eqref{eq:QDE_null_intro}:
\begin{equation}
\label{eq:QE_formal}
\vY_{\mathrm{QE}}:=
\vY_0
+\underbrace{\frac{\theta_{\snr} b_\sigma}{\sqrt{N}}\Big(\vg_1\vv_1^\top+\vg_2\vv_2^\top\Big)}_{\vT_1\ \text{(linear spikes)}}
+\underbrace{\frac{\theta_{\snr}^2 c_\sigma}{2\sqrt{N}}\Big(\vg_1^{\odot 2}\vv_1^{\odot 2\top}+\vg_2^{\odot 2}\vv_2^{\odot 2\top}\Big)}_{\vT_2\ \text{(quadratic spikes)}}.
\end{equation}
\begin{prop}[QE for CK in Proportional Limit]
\label{prop:approx}
Assume $r=O(n^{1/4})$ and Assumptions~\ref{assump:sigma} and
\ref{assump:asymptotics} hold. Let $\vY_{\mathrm{QE}}$ be~\eqref{eq:QE_formal}. Then, in probability, 
$\|\vY-\vY_{\mathrm{QE}}\|\to 0.$
Consequently, if $\|\vY\|,\|\vY_{\mathrm{QE}}\|\lesssim 1$, then $\vK=\vY^\top\vY$ and $\vK_{\mathrm{QE}}:=\vY_{\mathrm{QE}}^\top\vY_{\mathrm{QE}}$
have asymptotically the same outlier eigenvalues and eigenvector alignments.
\end{prop}
With this proposition, we can analyze informative spikes from $\vT_1$ and $\vT_2$, although we have to carefully consider the uninformative spikes from the null model $\vY_0$, which gives more technical difficulty.
The magnitude of $\vT_1$ scales like $\theta_{\snr} b_\sigma$, while $\vT_2$ scales like
$\theta_{\snr}^2 c_\sigma/\sqrt{n}$. For finite SNR (or when $c_\sigma=0$), the quadratic spikes $\vT_2$ are negligible relative to the bulk. In this case, only the linear spikes in $\vT_1$ contribute; but these are useless for linear clustering.
In the large-SNR regime (with $c_\sigma\neq 0$), $\vT_2$ becomes comparable to the bulk. Then since $\vy\in\mathrm{span}\{\vv_1^{\odot 2},\vv_2^{\odot 2}\}$, the corresponding outlier eigenvectors become label-informative. 
Proposition~\ref{prop:approx} fails when we consider the quadratic sample-size regime $n=\Theta(d^2)$ in Section~\ref{subsec:quadratic-regime}. In Appendix~\ref{sec:quadratic-QE}, we develop a different QE for the population level of CK matrix in this case, which captures the concentration of the random kernel around a deterministic quadratic polynomial kernel.

\section{Numerical Simulations: Varying Knobs}\label{sec:numerics}
We now present numerical simulations that validate and complement the theoretical predictions in Section~\ref{sec:results}.

\paragraph{Baseline case: proportional limit with finite SNR.}
Figure~\ref{fig:theorem3_figs} shows the random CK spectra for the XOR model in the finite-SNR regime and proportional limit (Assumption~\ref{assump:asymptotics}), together with the theoretical spike predictions from Theorem~\ref{thm:snr_finite_eig}. For the activation function used in this simulation, uninformative spikes  described in Theorem~\ref{thm:snr_finite_eig}(ii) do not appear. Consistent with Theorem~\ref{thm:snr_finite_eig}(iii)--(iv), the two leading eigenvalues are outliers, and their eigenvectors exhibit nontrivial alignment with \(\vv_1\) and \(\vv_2\), defined in~\eqref{eq:u_1u_2}, while remaining asymptotically orthogonal to the XOR label vector \(\vy\). Consequently, although one could potentially apply downstream nonlinear transformations to these leading eigenvectors to classify the data, direct linear classification of the XOR problem from the CK eigenvectors fails in this regime.

This setting serves as our baseline. The theoretical results in Section~\ref{sec:results} identify four primary control parameters, or ``knobs,'' that govern the emergence of CK outliers and the alignment of their eigenvectors with the XOR labels \(\vy\):
(1) the SNR \(r^2\);
(2) the choice of nonlinearity \(\sigma\), through the constants \((b_\sigma,c_\sigma)\) defined in \eqref{eq:b_sigma_def} and \eqref{eq:c_sigma_def};
(3) the weight structure, comparing random and pretrained weights; and
(4) the asymptotic scaling regime relating \((n,d,N)\).
Below, we vary these knobs one at a time to isolate their effects on the CK spectrum relative to this baseline.

\subsection{Knob I: SNR of the XOR Dataset}\label{subsec:numerics_snr}
\begin{figure}[!htbp]
    \centering
    \begin{minipage}[c]{0.61\linewidth}\centering
    \includegraphics[width=\linewidth]{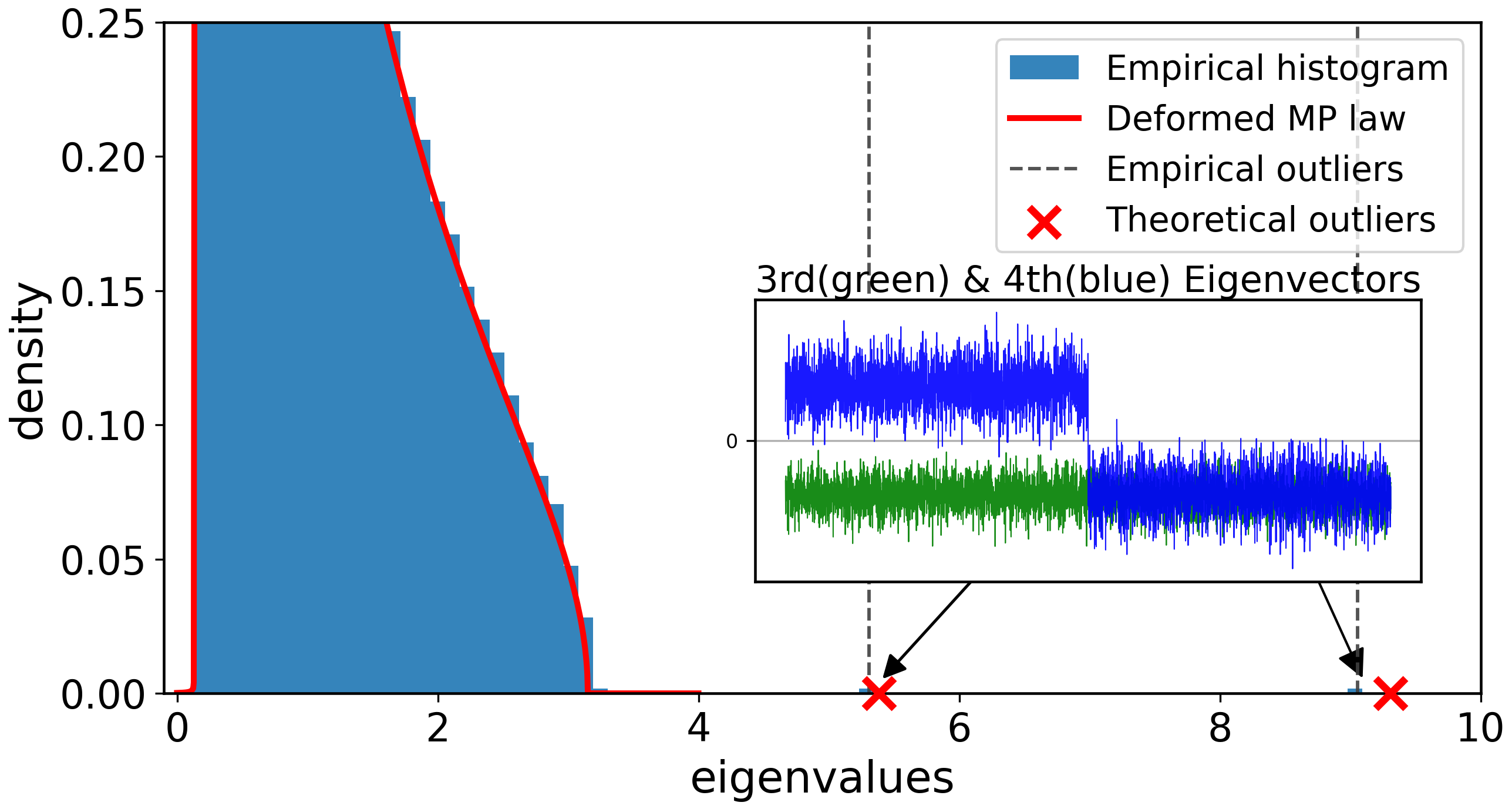}
    \end{minipage}\hfill
    \begin{minipage}[c]{0.38\linewidth}\centering
    \includegraphics[width=\linewidth]{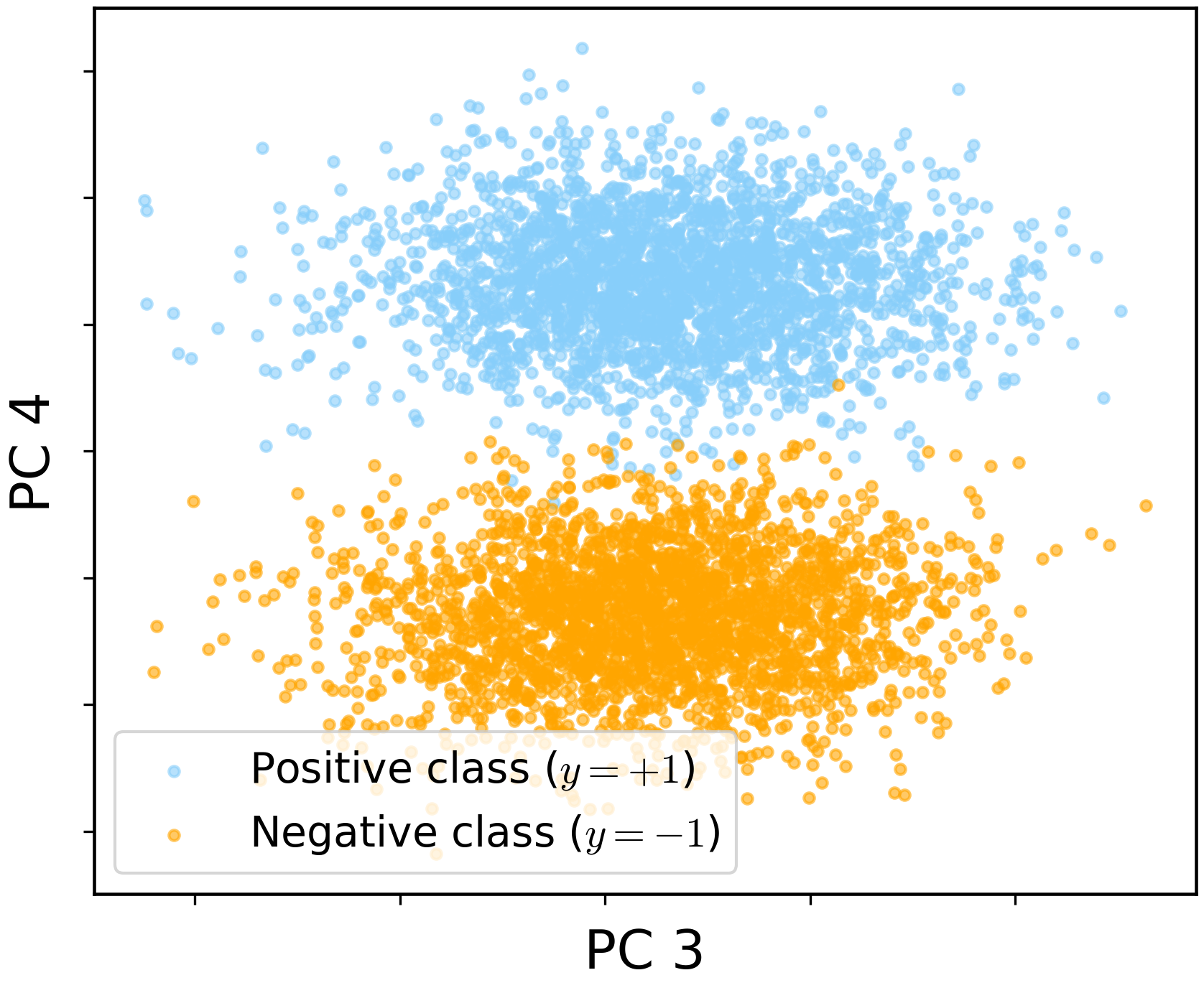}
    \end{minipage}
    \caption{\textbf{Large-SNR proportional regime: informative CK spikes recover the XOR labels.}
    We consider the same setting as in Figure~\ref{fig:theorem3_figs} only increasing SNR to \(r=35^2\).
    \textbf{Left:} spectrum of \(\vK\) (blue) together with the predicted curve from Theorem~\ref{thm:large} (red curve). To focus on the order-one eigenvalues, we do not display the first two extremely large eigenvalues
    \(\widehat\lambda_1=153.9\) and \(\widehat\lambda_2=153.1\) predicted by \eqref{eq:large_linear_outliers}.
    The red \(\color{red}{\mathbf{\times}}\) are the theoretical locations of the order-one outliers. The subfigure presents the 3rd \& 4th eigenvectors against sample index.
    The 4th outlier eigenvector of $\vK$ is informative for linear classification of XOR.
    \textbf{Right:} kernel-PCA visualization using the 3rd \& 4th PCs of \(\vK\). 
    In contrast to Figure~\ref{fig:theorem3_figs}, these two spectral coordinates clearly separate the positive and negative classes by a linear boundary.
    }
    \label{fig:theorem5_figs}
\end{figure}
Theorem~\ref{thm:large} shows that increasing the SNR to \(r^2=\Theta(n^{1/2})\) can lead to the emergence of a label-informative order-one CK outlier. In this regime, the quadratic component of \(\sigma\) is no longer negligible when \(c_\sigma\ne 0\), making linear spectral classification of XOR possible through \eqref{eq:large_final_label_projector} in Theorem~\ref{thm:large}(iv). When moving from Figure~\ref{fig:theorem3_figs} to Figure~\ref{fig:theorem5_figs}, we only increase the SNR from \(36\) to \(1225\). This change produces two additional outliers, one of whose eigenvectors asymptotically aligns with the XOR labels \(\vy\).

Figures~\ref{fig:theorem3_figs} and~\ref{fig:theorem5_figs} use the centered and normalized ReLU activation, for which both \(b_\sigma\) and \(c_\sigma\) are nonzero. In Figure~\ref{fig:CK_phase}, we show the same SNR-driven phase transition using a centered and normalized activation \(\sigma(x) \propto \sqrt{x^2+1}-1\), for which \(b_\sigma=0\) and \(c_\sigma\ne 0\). Since \(b_\sigma=0\), the diverging linear outliers predicted by Theorem~\ref{thm:snr_finite_eig}(iii) and Theorem~\ref{thm:large}(ii) are absent. As the SNR increases from \(r^2=36\) to \(625\), the second eigenvector develops a two-class structure, and XOR becomes linearly classifiable using CK eigenvectors.

\begin{figure}[!htbp]
    \centering
    \begin{minipage}[c]{0.61\linewidth}
        \centering
        \includegraphics[width=\linewidth]{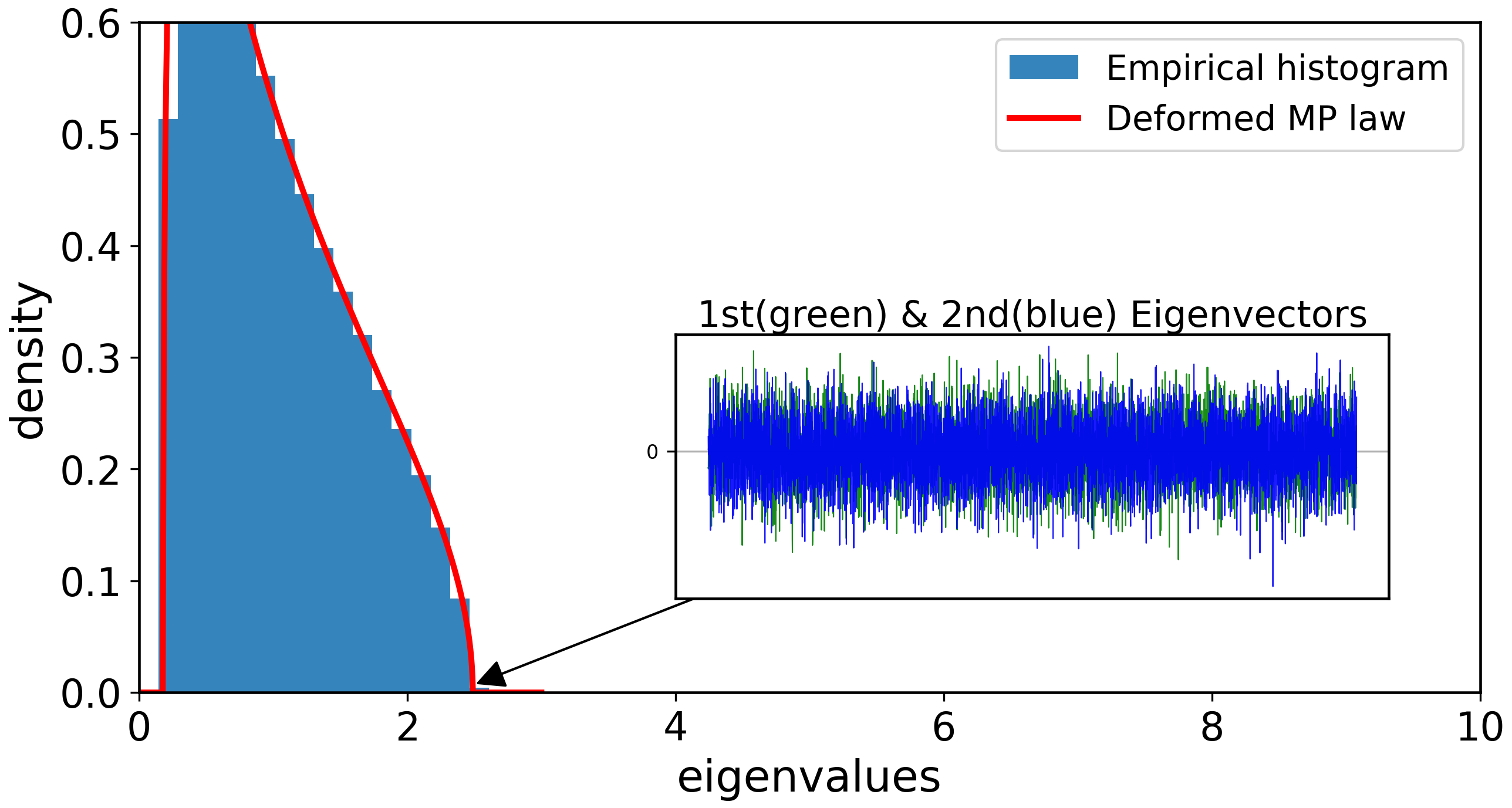}
    \end{minipage}%
    \hfill
    \begin{minipage}[c]{0.38\linewidth}
        \centering
        \includegraphics[width=\linewidth]{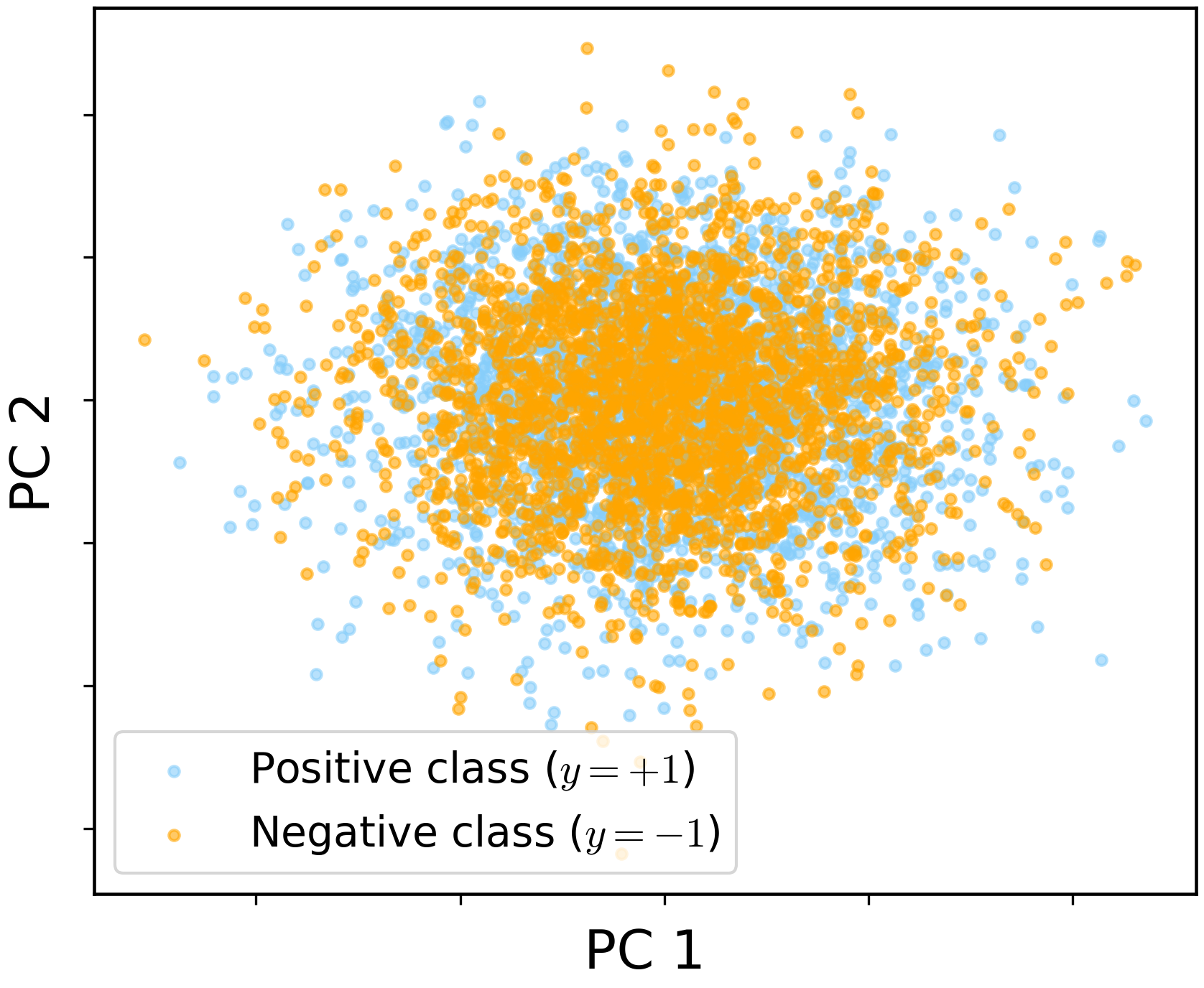}
    \end{minipage}
    \begin{minipage}[c]{0.61\linewidth}
        \centering
        \includegraphics[width=\linewidth]{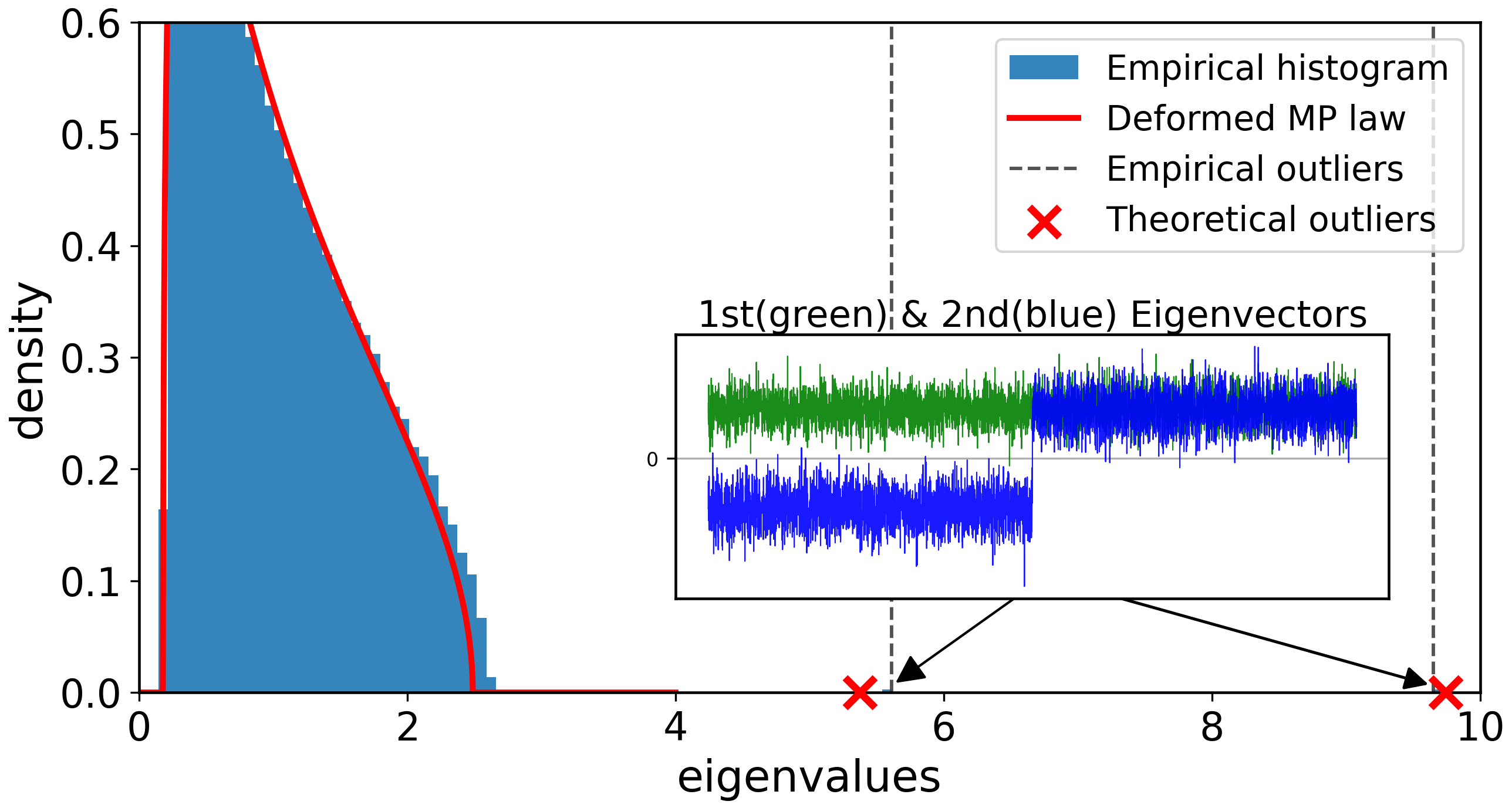}
    \end{minipage}%
    \hfill
    \begin{minipage}[c]{0.38\linewidth}
        \centering
        \includegraphics[width=\linewidth]{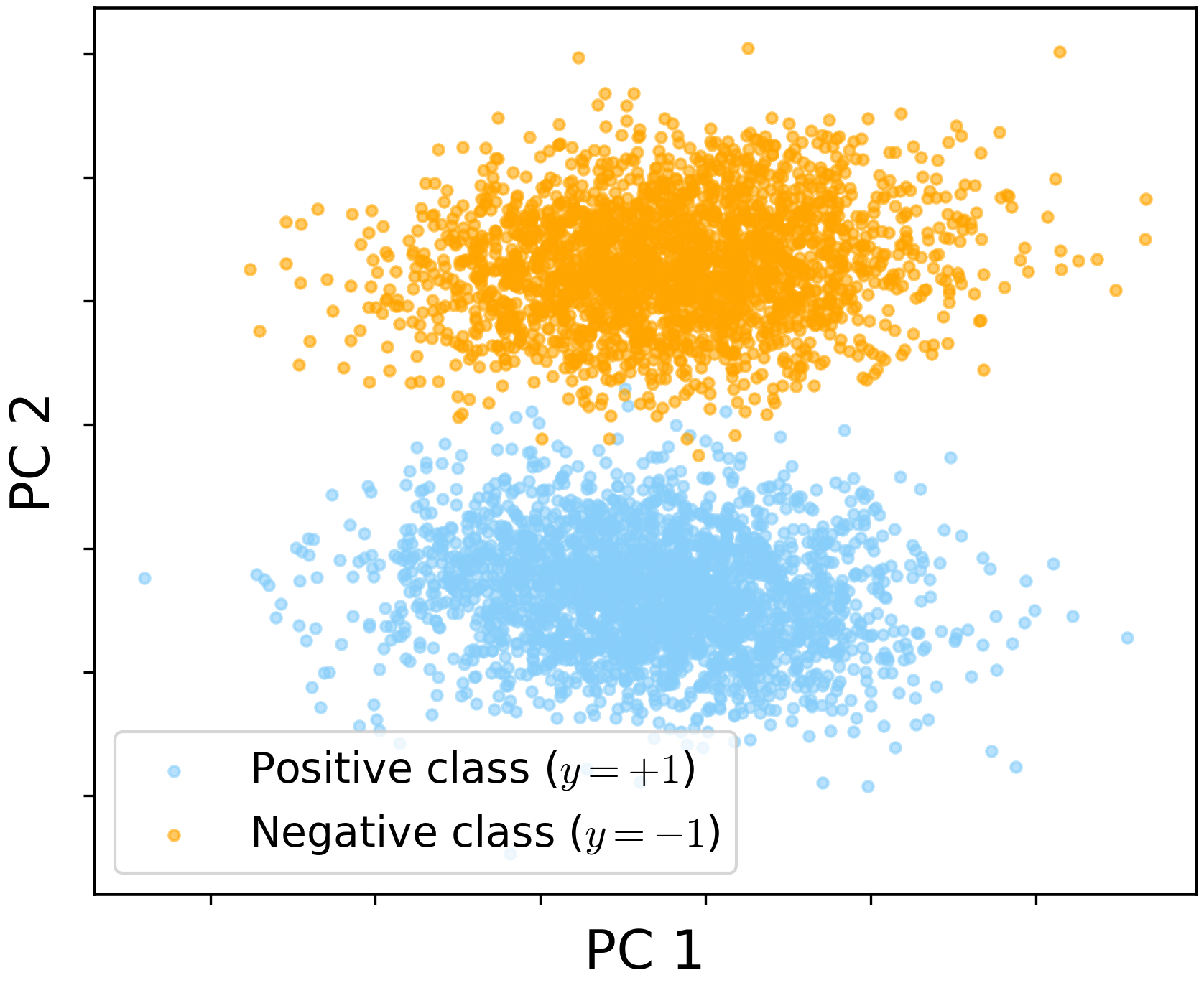}
    \end{minipage}
    \caption{\textbf{CK spectral phase transition driven by SNR.} \textbf{Top:} CK spectrum in the finite-SNR regime with \(r=6\); no eigenvalue separates from the bulk. \textbf{Bottom:} CK spectrum in the large-SNR regime with \(r=25\). Two quadratic eigenvalues separate from the bulk, and the second leading eigenvector aligns with the XOR labels. Here the activation \(\sigma(x) \propto \sqrt{x^2+1}-1\), \(n=5000\), and \(N=d=15000\).}
    \label{fig:CK_phase}
\end{figure}

\subsection{Knob II: Activation Functions}\label{subsec:numerics_activation}

Figures~\ref{fig:theorem3_figs},~\ref{fig:theorem5_figs}, and~\ref{fig:CK_phase} already show that changing the activation function can lead to qualitatively different CK spectral behavior. We now present two additional examples: the emergence of uninformative spikes and the absence of informative spikes.
\begin{figure}[!htbp]
    \centering
    \begin{minipage}[t]{0.48\linewidth}
        \centering
        \includegraphics[width=\linewidth]{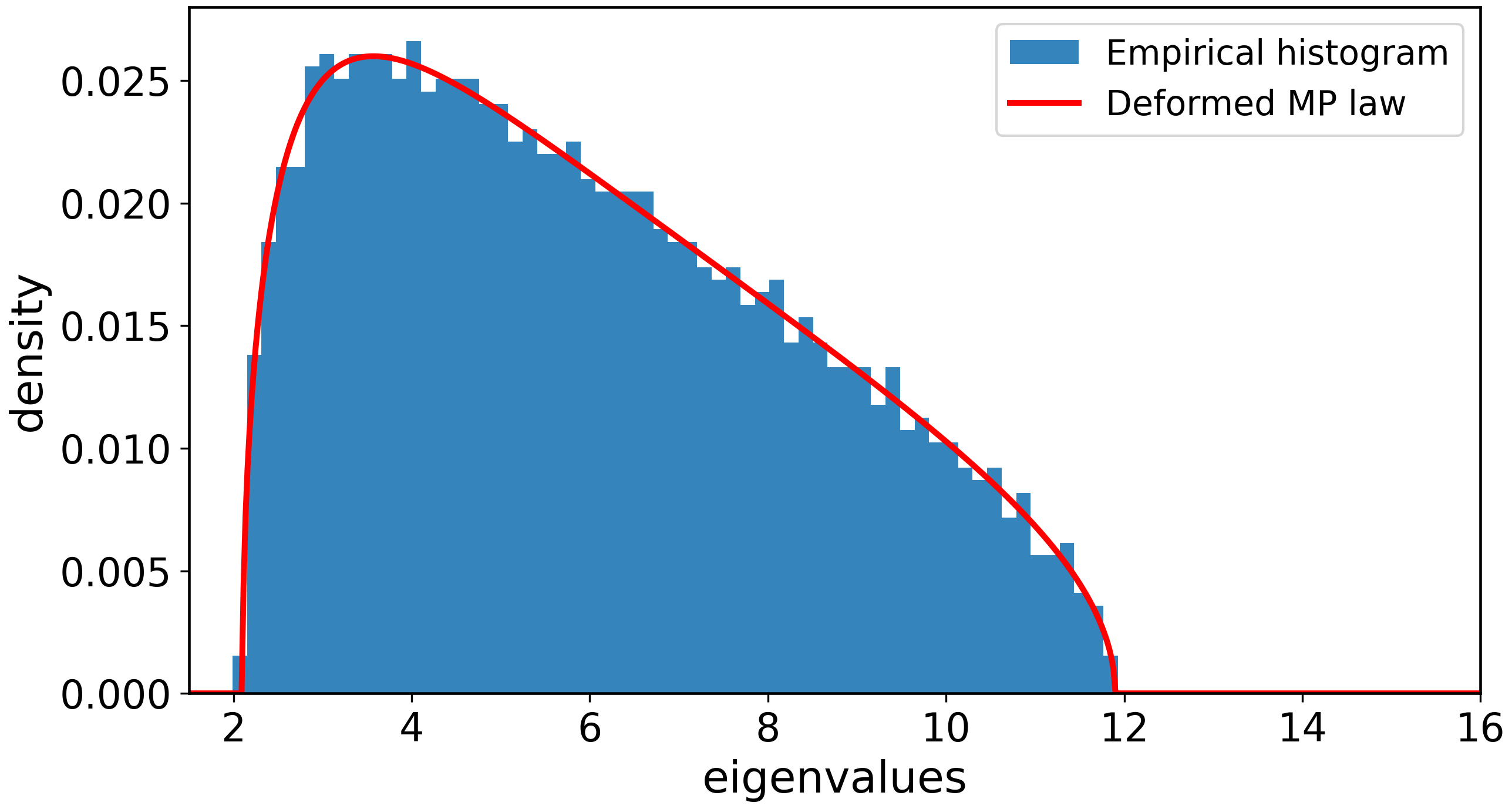}
    \end{minipage}%
    \begin{minipage}[t]{0.48\linewidth}
        \centering
        \includegraphics[width=\linewidth]{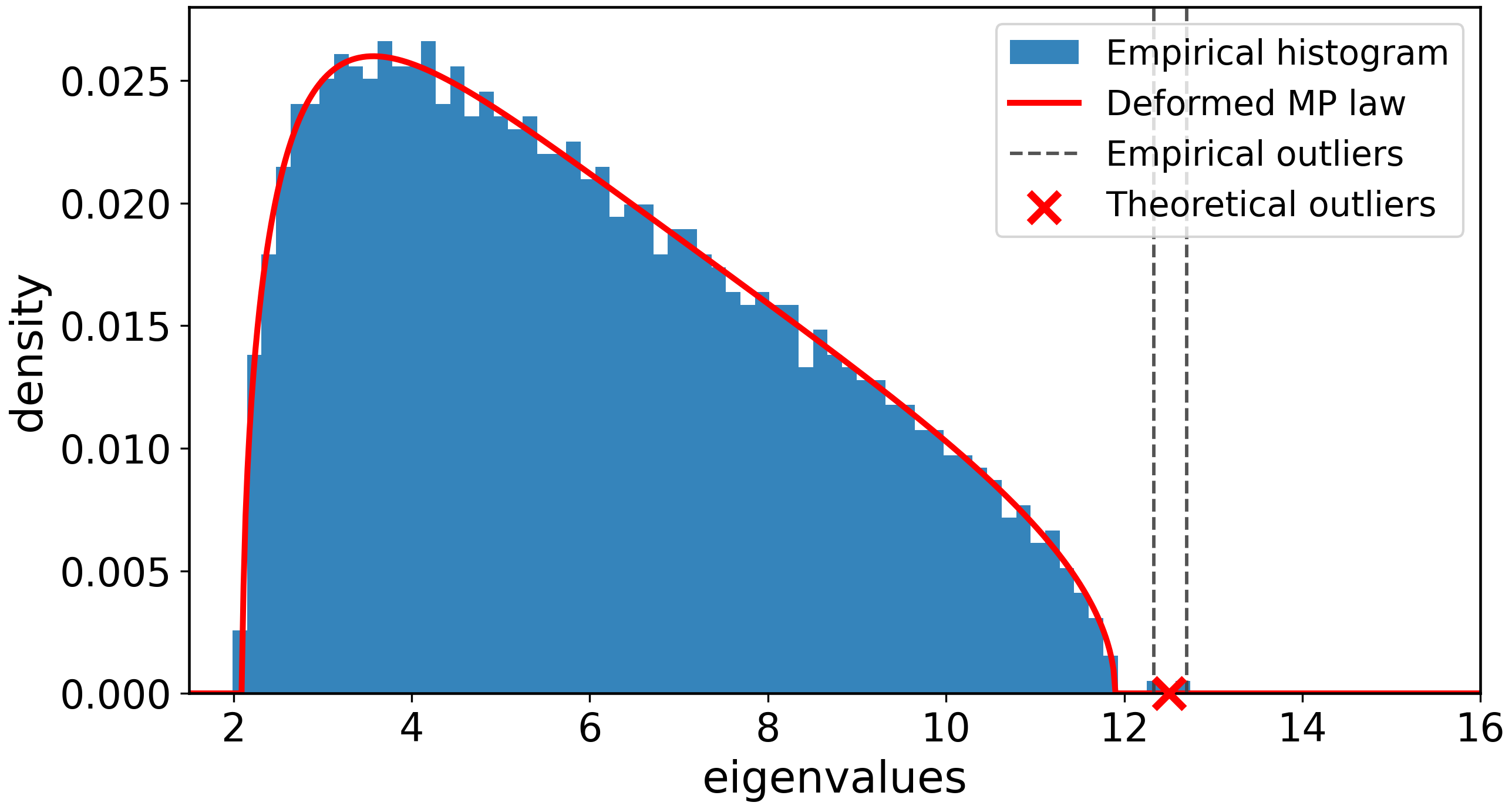}
    \end{minipage}%
    \caption{\textbf{Emergence of uninformative spikes when \(\snr=0\).} CK spectra under the activation function \(f_\alpha\) defined in \eqref{eq:f_alpha}, with \(n=6000\), \(N=d=1000\), and \(r=0\). \textbf{Left:} \(\alpha=2\). \textbf{Right:} \(\alpha=1.5\). Two uninformative spikes emerge and are close to the theoretical predictions from Theorem~\ref{thm:snr_finite_eig}(ii), shown by red \(\color{red}{\mathbf{\times}}\) markers.}
    \label{fig:activationf_experiment}
\end{figure}
\paragraph{Uninformative spikes can emerge even in the null model.}
Theorem~\ref{thm:snr_finite_eig}(ii) describes uninformative spikes of the CK matrix even when \(r=0\), namely for an isotropic Gaussian dataset. These uninformative spikes are caused by the nonzero quadratic coefficient \(c_\sigma\ne 0\). Such spikes were studied by \citet{benigni2022largest} for more general datasets, while our result gives a precise characterization of their locations for Gaussian data.
To illustrate their emergence in the CK matrix, we consider the single-parameter family of activation functions from \citet{benigni2022largest},
\begin{equation}\label{eq:f_alpha}
    f_\alpha(x) = \frac{\cos(\alpha x) - e^{-\alpha^2/2}}{\sqrt{e^{-\alpha^2}\bigl(\cosh(\alpha^2)-1\bigr)}}.
\end{equation}
These functions satisfy Assumption~\ref{assump:sigma}. Moreover, \(b_\sigma=0\) for all choices of \(\alpha\), while \(c_\sigma\) varies with \(\alpha\). Figure~\ref{fig:activationf_experiment} plots the CK spectra under \(f_\alpha\) in the null model \(\snr=0\) for \(\alpha=2\) and \(\alpha=1.5\). As \(\alpha\) decreases, two outlier spikes emerge in the CK spectrum. These two uninformative spikes arise from the population covariance of the CK and from the CK mean; Theorem~\ref{thm:snr_finite_eig}(ii) predicts that their locations are asymptotically close.

\paragraph{XOR remains linearly nonseparable at large SNR when \(c_\sigma=0\).}
If the activation function \(\sigma\) satisfies \(c_\sigma=0\), then the quadratic component in the CK is absent. As discussed in Theorem~\ref{thm:large}, this implies that the CK cannot be used to linearly classify XOR for any signal-to-noise ratio. Figure~\ref{fig:tanh_largesnr} follows the same setting as the bottom panel of Figure~\ref{fig:CK_phase}, changing only the activation to \(\sigma\propto \tanh\) with \(c_\sigma=0\). Here \(\snr=625\). By Theorem~\ref{thm:large}, only two diverging linear outliers separate from the bulk. Moreover, their associated spectral coordinates are not linearly separable with respect to the XOR labels.

\begin{figure}[!htbp]
    \centering
    \begin{minipage}[c]{0.61\linewidth}
        \centering
        \includegraphics[width=\linewidth]{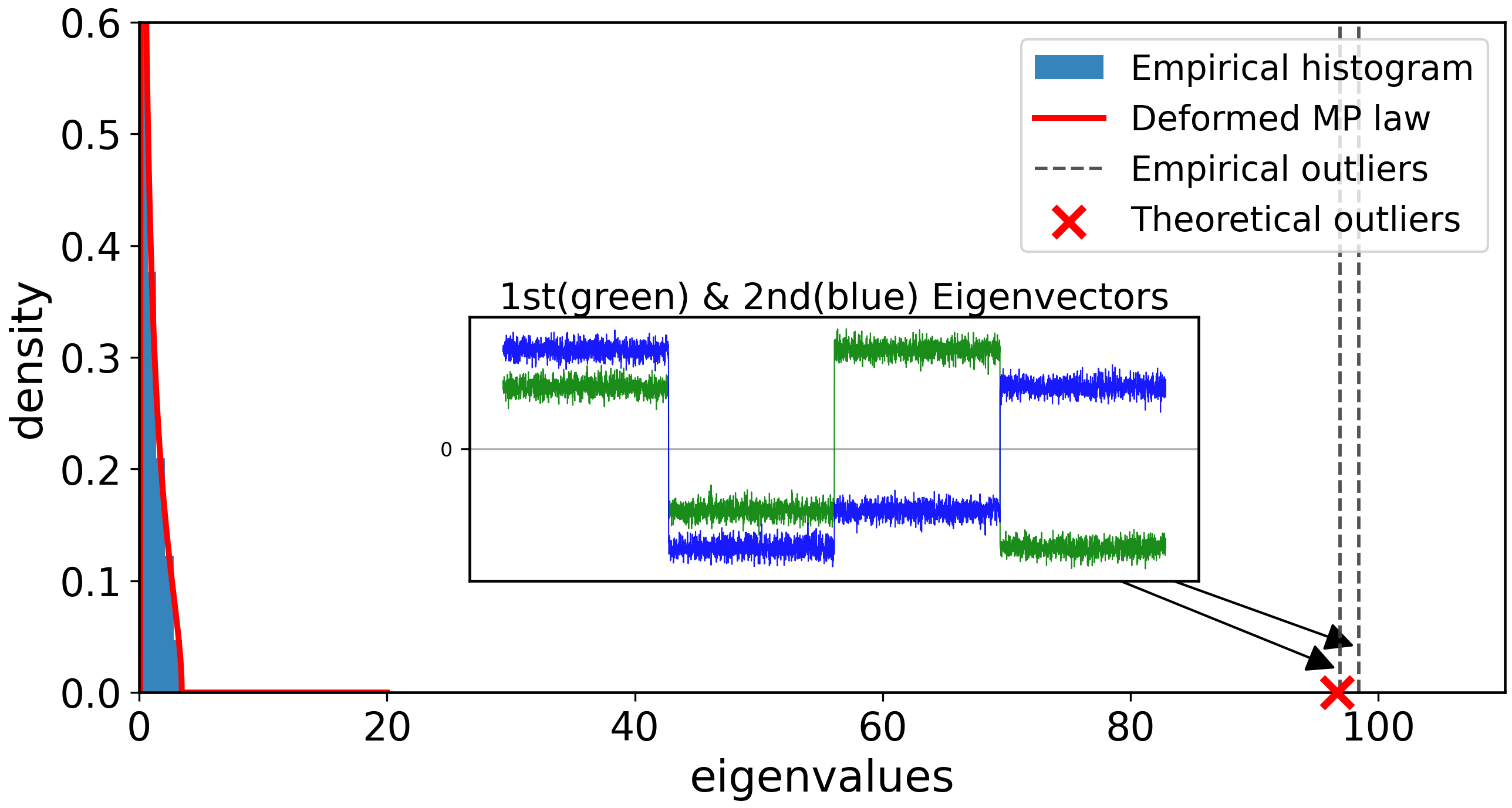}
    \end{minipage}\hfill
    \begin{minipage}[c]{0.38\linewidth}
        \centering
        \includegraphics[width=\linewidth]{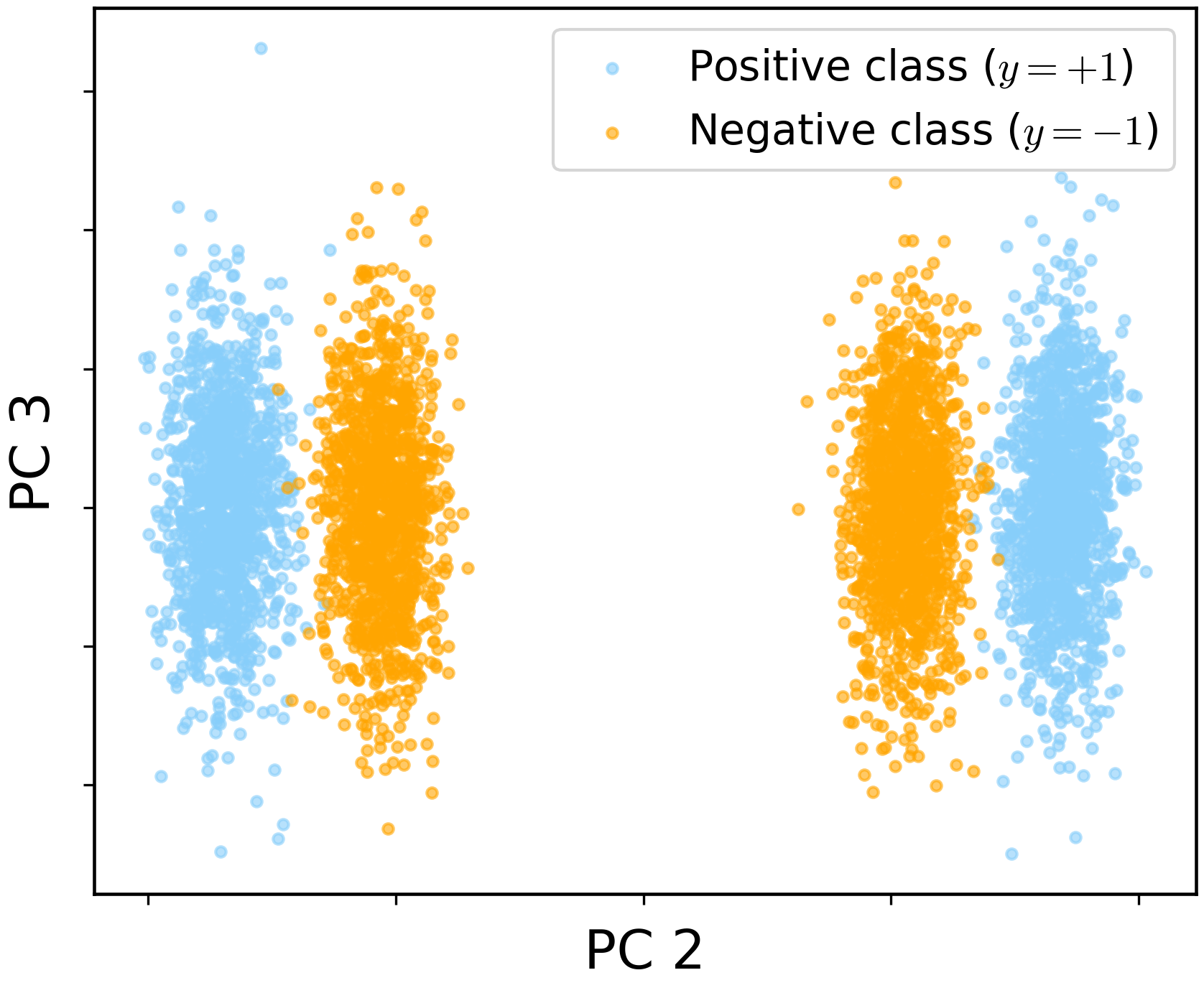}
    \end{minipage}%
    \caption{\textbf{Large-SNR regime with \(c_\sigma=0\): linear classification of XOR still fails.} Here \(n=5000\), \(N=d=15000\), \(r=25\), and \(\sigma(x)\propto\tanh(x)\), so \(c_\sigma=0\). \textbf{Left:} spectrum of \(\vK\) (blue), together with the theoretical locations of the two diverging outliers predicted by Theorem~\ref{thm:large}(ii), shown by red \(\color{red}{\times}\) markers. The outlier eigenvectors, displayed in the inset, exhibit a four-block pattern associated with the four Gaussian-mixture components but remain asymptotically orthogonal to the XOR label vector \(\vy\). \textbf{Right:} kernel-PCA visualization of \(\vK\). Thus the leading nontrivial spectral coordinates are not linearly separable with respect to the XOR labels.}
    \label{fig:tanh_largesnr}
\end{figure}

\subsection{Knob III: Weight Structure}\label{subsec:numerics_weights}

Following Theorem~\ref{thm:trained}, we generate a rank-one perturbation of the randomly initialized weights,
\(\vW_1 = \vW + \theta \va \vb^\top\),
as a model of pretrained weights. In Figure~\ref{fig:spikedW_fig}, we plot the spectrum of the projected CK matrix \(\vK_s\) and its PCs for the centered and normalized ReLU activation, an XOR dataset with SNR \(r^2=36\), and spike strength \(\theta=5\). The rank-one structure in the pretrained weights produces informative spikes. As predicted by Theorem~\ref{thm:trained}, two order-one outlier eigenvalues emerge, and the second leading eigenvector exhibits a two-class structure that enables linear classification of XOR.

\begin{figure}[!htbp]
    \centering
    \begin{minipage}[c]{0.61\linewidth}
        \centering
        \includegraphics[width=\linewidth]{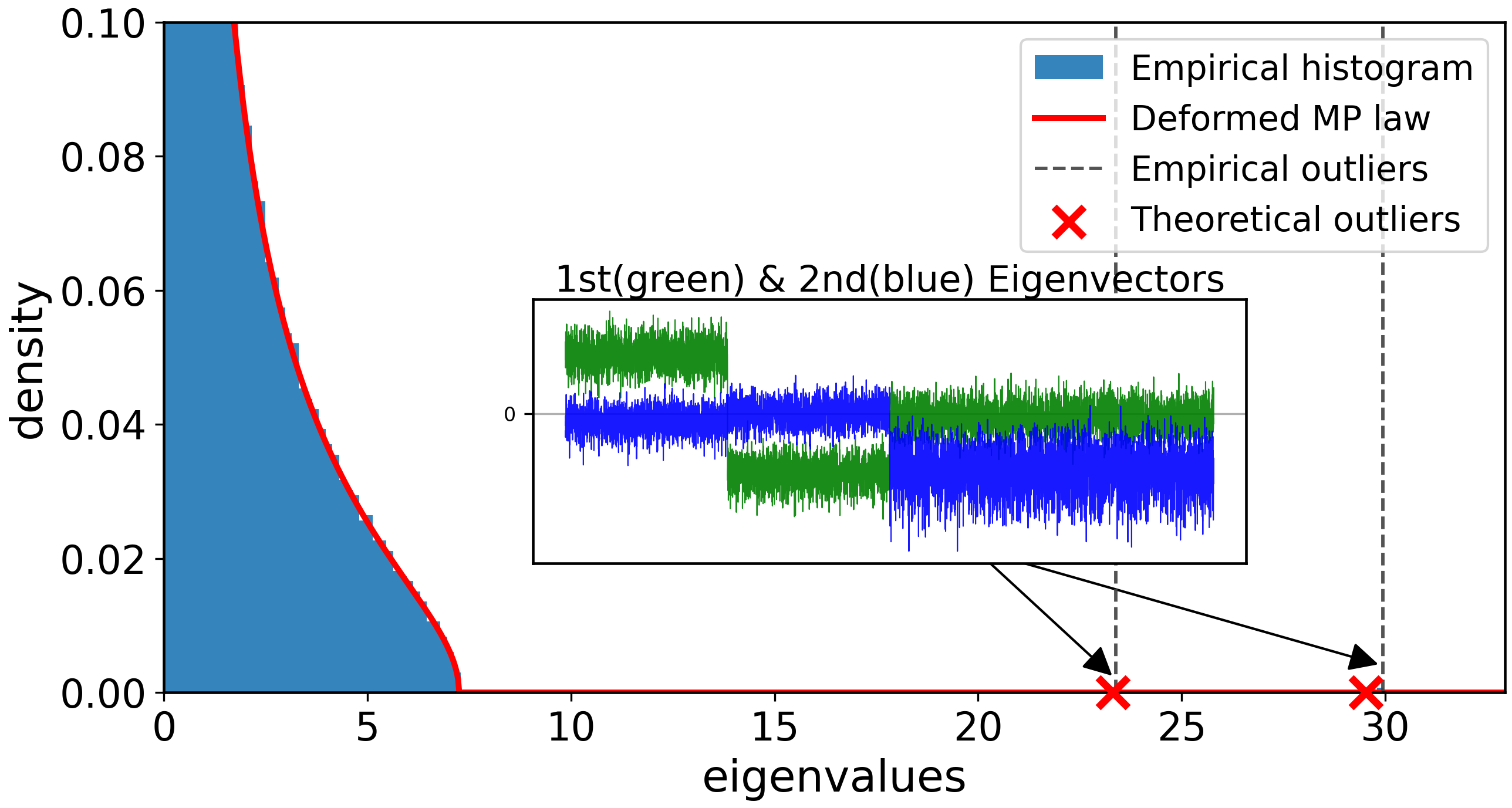}
    \end{minipage}\hfill
    \begin{minipage}[c]{0.38\linewidth}
        \centering
        \includegraphics[width=\linewidth]{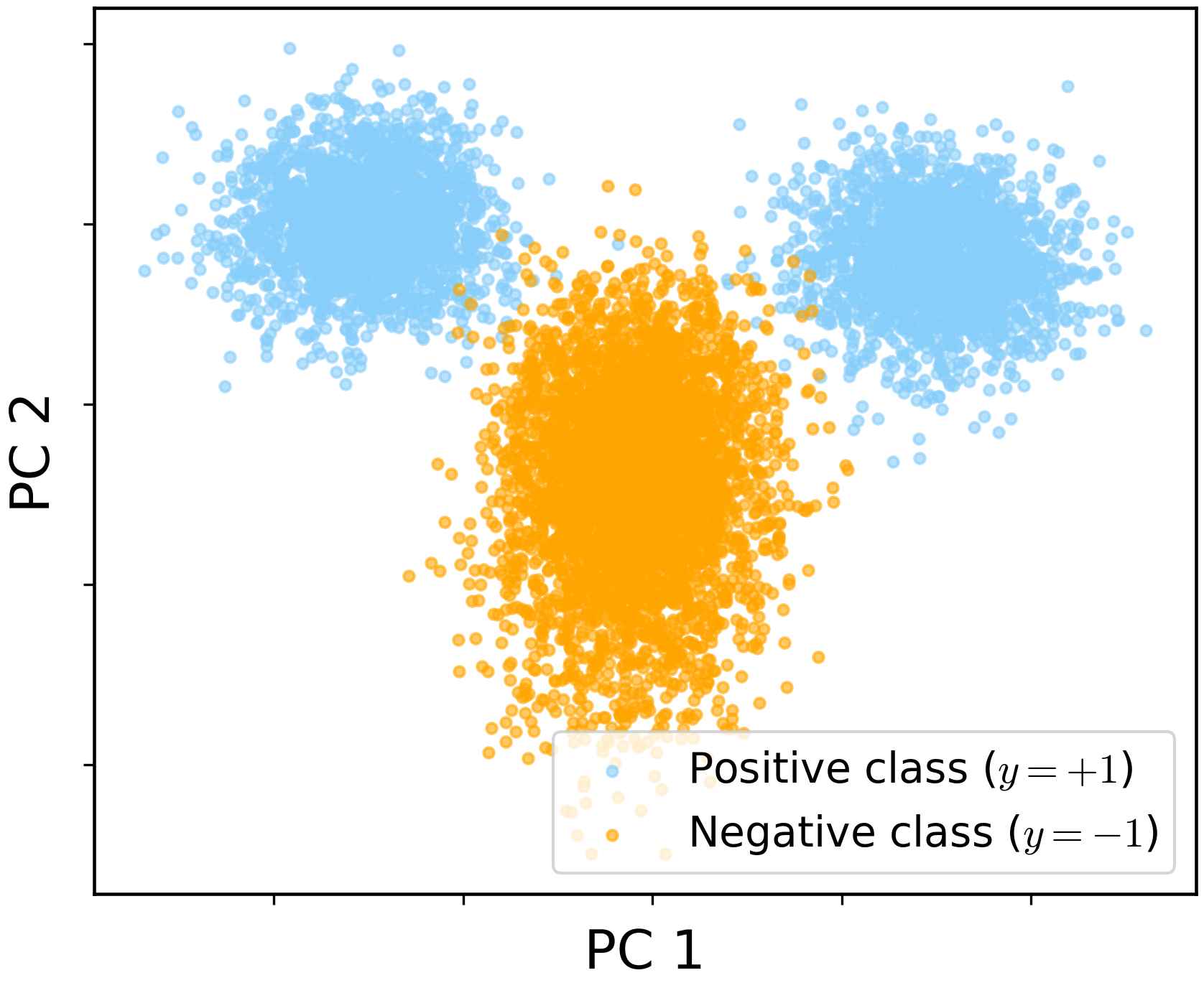}
    \end{minipage}%
    \caption{\textbf{Spiked-weight regime: a weight-induced spike aligns with XOR labels.} Here \(N=n=8000\), \(d=4000\), the SNR is \(r^2=36\), \(\sigma(x)\propto\mathrm{ReLU}(x)\), and the weight spike strength in \eqref{eq:pretrained_weight} is \(\theta=5.0\). \textbf{Left:} spectrum of \(\vK_s\) (blue), as defined in Theorem~\ref{thm:trained}, together with the isolated eigenvalue locations predicted by Theorem~\ref{thm:trained}, shown by red \(\color{red}{\times}\) markers. The top two empirical eigenvectors are shown in the insets; the second one enables linear classification of XOR. \textbf{Right:} kernel-PCA visualization of the samples using the second and first PCs of \(\vK_s\).}
    \label{fig:spikedW_fig}
\end{figure}

\paragraph{Pretrained weights from feature learning on a real-world dataset.}
The pretrained-weight model in \eqref{eq:pretrained_weight}, used in Theorem~\ref{thm:trained}, is an idealized model for our analysis. We now show that real pretrained weights obtained during training can exhibit similar label-aligned CK spikes. A recent line of work \citep{martin2018implicitself,martin2019heavytail,wang2022spectral} has tracked the evolution of the empirical spectral distribution of weight matrices throughout the training process of deep NNs. We focus on an early-training phase in which the weight spectrum is well described by an ``MP bulk$+$spikes'' structure. The emergence of spikes in the weight spectrum is strongly related to feature learning in NNs \citep{frei2022random} and suggests that these spikes carry task-relevant information.

We consider the CK matrix formed using pretrained weights extracted from the first-layer weight matrix of a 20-layer fully connected NN trained on CIFAR-2. Each hidden layer has width \(N=512\), and the input dimension is \(d=3072\). We focus on the first 50 steps of SGD training, using learning rate \(0.05\), momentum \(0.5\), and batch size \(128\). At step 50, the test accuracy is \(80.4\%\), and the training accuracy is \(91.41\%\). Figure~\ref{fig:pretrained_figs} shows the spectra of the first-layer trained weights at steps \(1\), \(2\), and \(50\) (bottom row), together with the spectra of CK matrices evaluated on XOR data using these pretrained weights (top row). The emergence of spikes in the weight spectra during training suggests that the model is learning task-dependent features.

For the CK matrices in this experiment, \(n=8000\), \(N=512\), \(d=3072\), and the SNR is \(r^2=36\). In contrast to the prediction of Theorem~\ref{thm:snr_finite_eig} for randomly initialized weights, the pretrained CK at step 50 exhibits four spikes. We omit the top outlier in Figure~\ref{fig:pretrained_figs} because it is a large-order outlier. The fourth leading eigenvector aligns with the XOR labels. After sufficient training, this eigenvector transitions from pure noise to near-perfect alignment with the XOR labels, thereby enabling linear classification.

These experiments suggest that the LE approximation of the CK can break down even during the early stages of training. They also indicate that, under suitable training regimes and parameter choices, there can be a critical phase transition in which the emergence of weight spikes makes quadratic channels in the CK nonnegligible. One possible explanation is that training the NN induces spiked eigenvalues with eigenvectors that align with XOR labels, even though the model is not trained directly on the XOR task. Theorem~\ref{thm:trained} captures this mechanism through an idealized model in which a large spike is artificially introduced into the weights, producing a spiked CK eigenvalue with an XOR-label-aligned eigenvector.

\begin{figure}[!htbp]
    \centering
    \begin{minipage}[t]{0.33\linewidth}
        \centering
        \includegraphics[width=\linewidth]{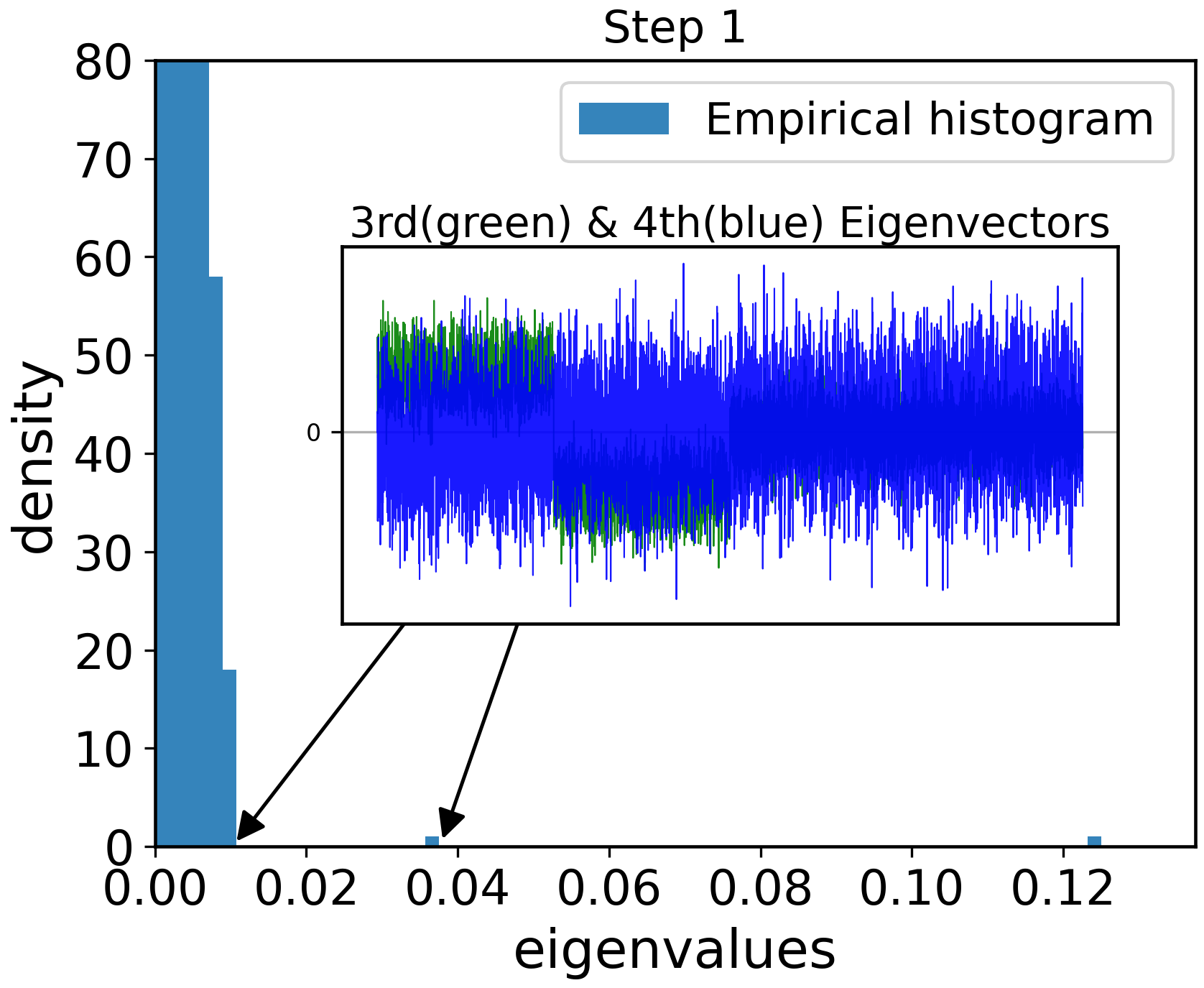}
    \end{minipage}%
    \begin{minipage}[t]{0.33\linewidth}
        \centering
        \includegraphics[width=\linewidth]{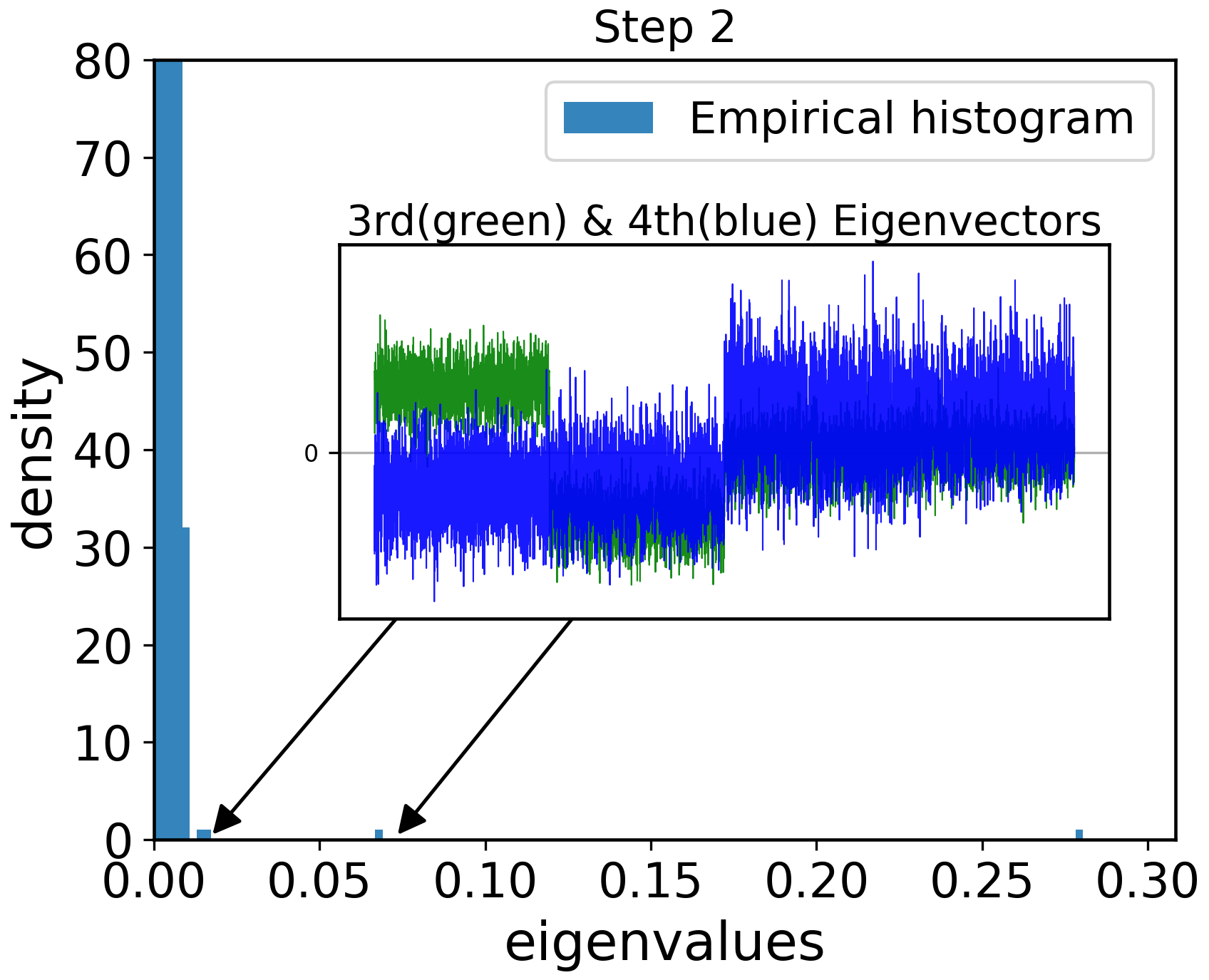}
    \end{minipage}%
    \begin{minipage}[t]{0.33\linewidth}
        \centering
        \includegraphics[width=\linewidth]{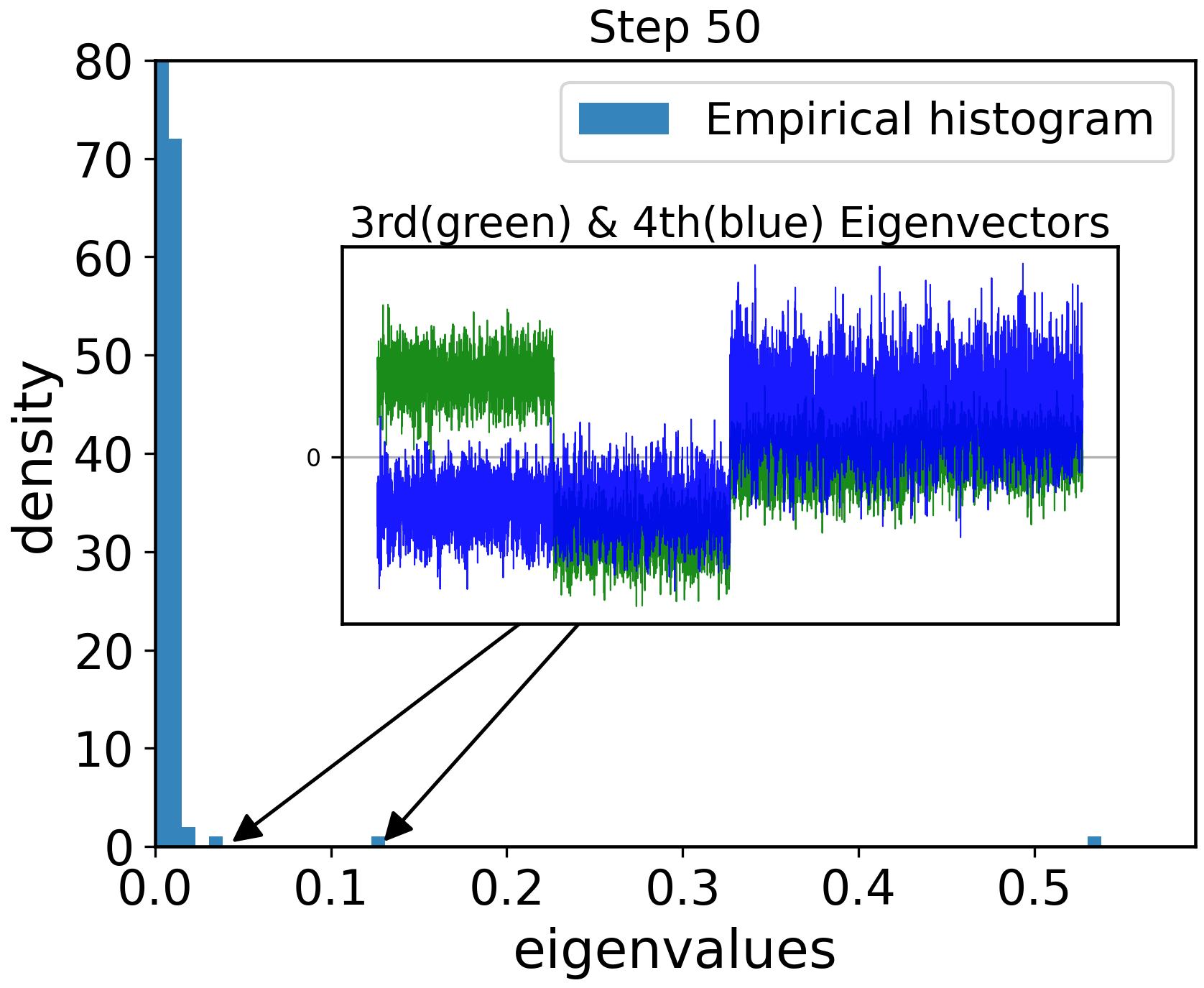}
    \end{minipage}
    \begin{minipage}[t]{0.33\linewidth}
        \centering
        \includegraphics[width=\linewidth]{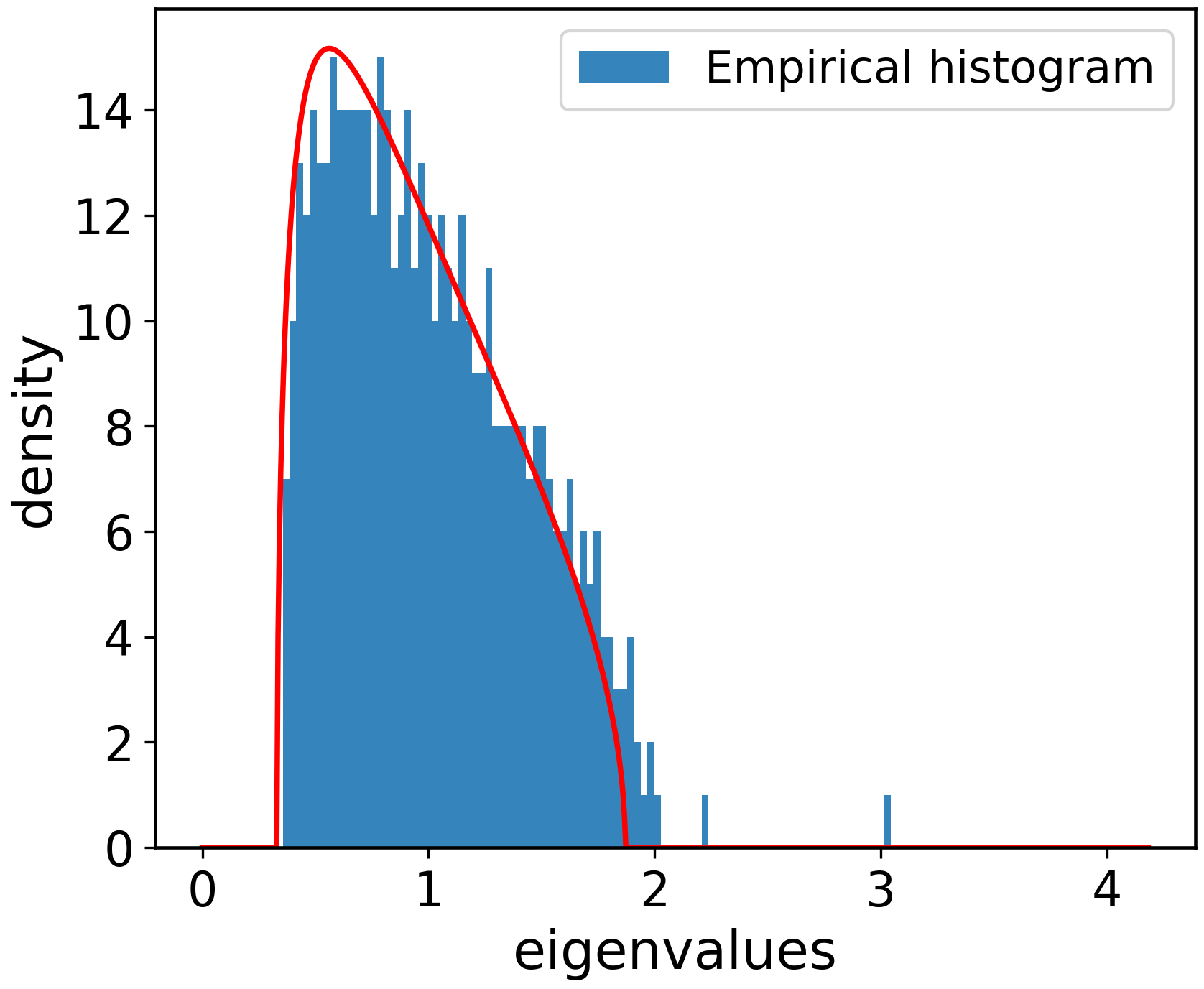}
    \end{minipage}%
    \begin{minipage}[t]{0.33\linewidth}
        \centering
        \includegraphics[width=\linewidth]{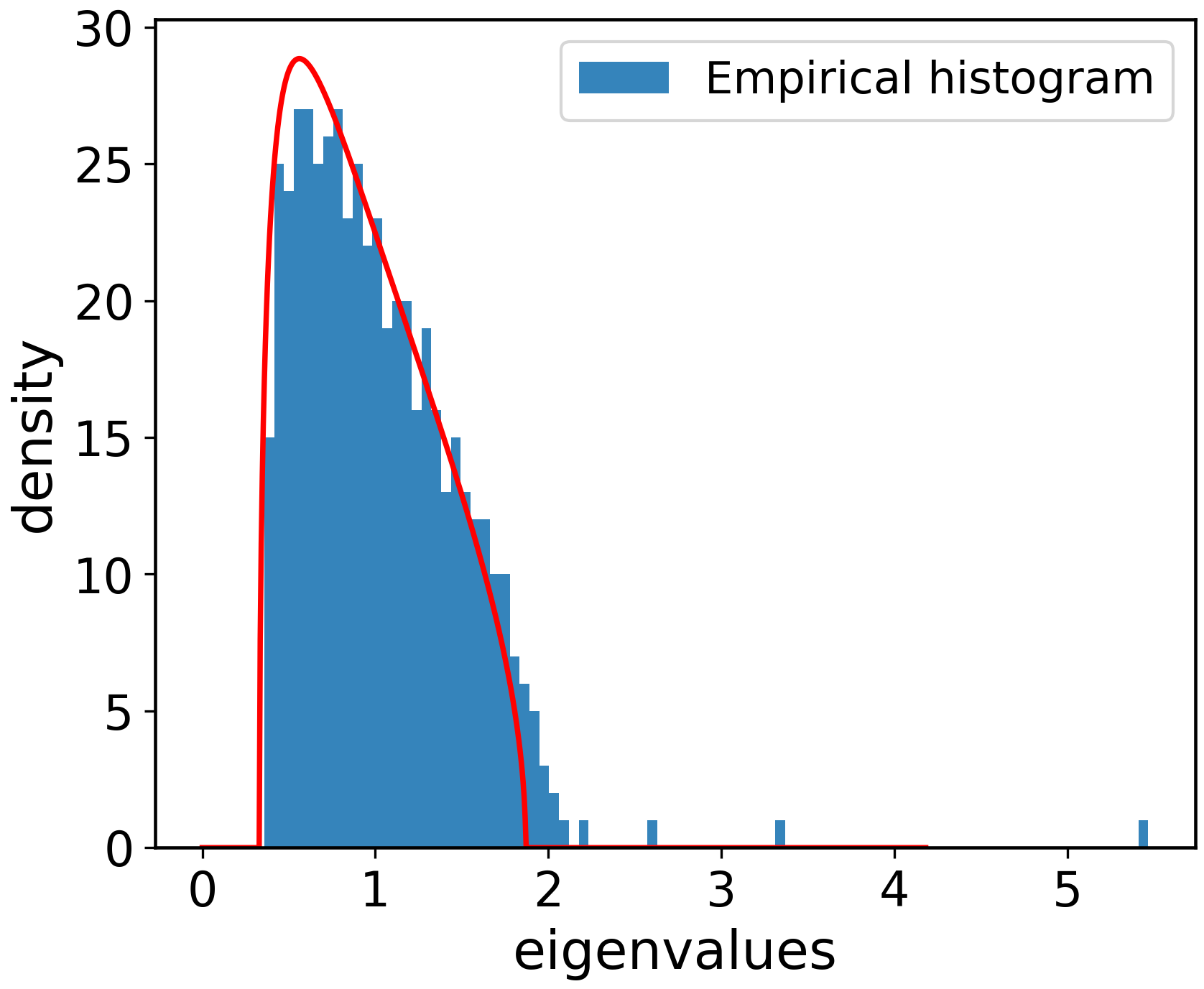}
    \end{minipage}%
    \begin{minipage}[t]{0.33\linewidth}
        \centering
        \includegraphics[width=\linewidth]{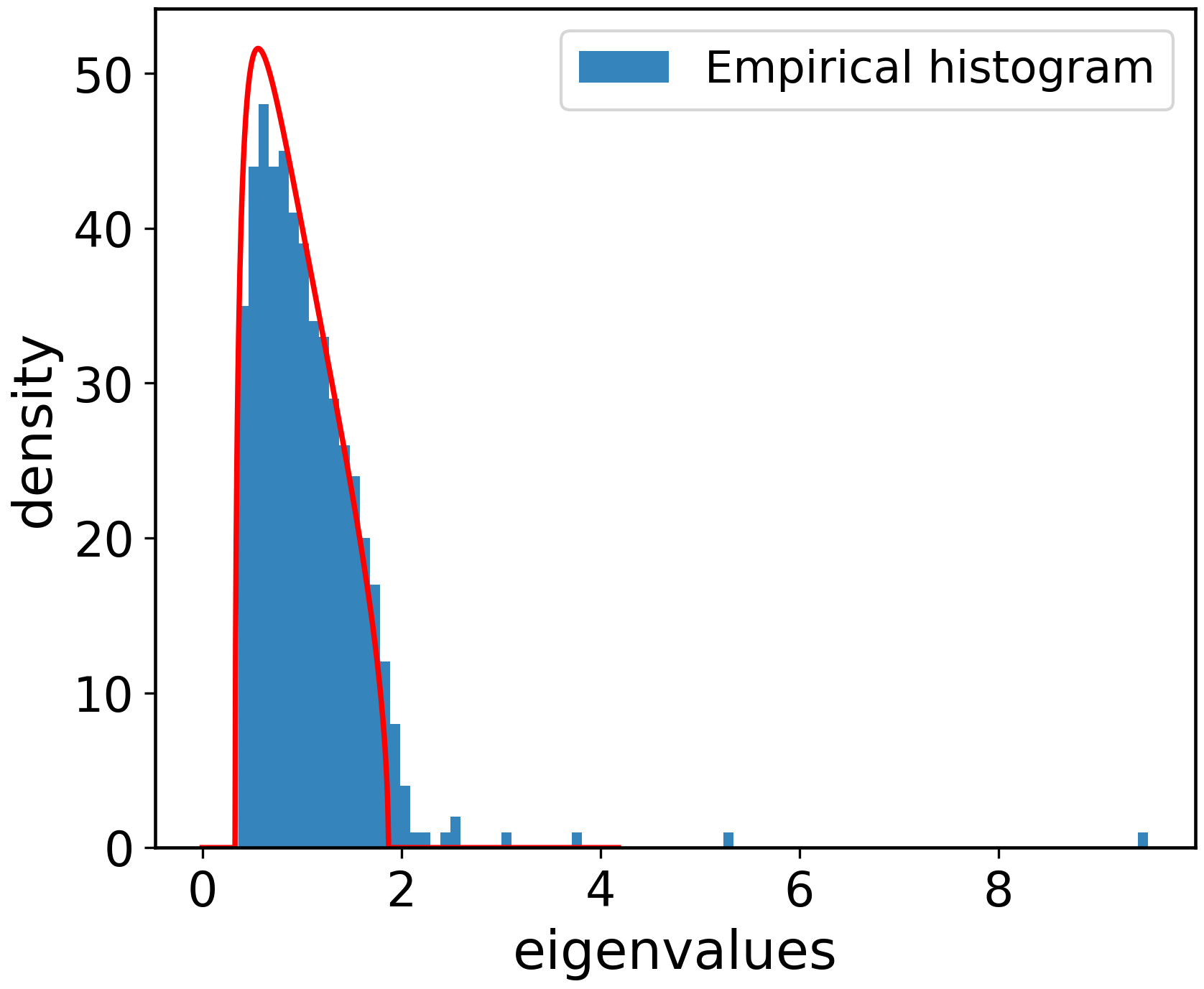}
    \end{minipage}%
    \caption{\textbf{Pretrained weights help linear classification of XOR.} \textbf{Top:} CK spectra with pretrained weights at training steps \(1\), \(2\), and \(50\). For these CK matrices, \(n=8000\), \(N=512\), \(d=3072\), the SNR is \(r^2=36\), and \(\sigma(x)\propto \relu(x)\). \textbf{Bottom:} spectra of the first-layer weight matrix of a NN trained on CIFAR-2 at steps \(1\), \(2\), and \(50\). The red curves show the MP laws.}
    \label{fig:pretrained_figs}
\end{figure}

\subsection{Knob IV: Sample Size}\label{subsec:numerics_scaling}

Following Section~\ref{subsec:quadratic-regime}, we study CK spectra under the quadratic sample-size scaling \(n=\Theta(d^2)\). Figure~\ref{fig:quad_fig} shows the spectrum and PCs of the projected CK matrix for the centered and normalized ReLU activation and an XOR dataset with finite SNR \(r^2=6.25\). Here \(n=15000\), \(N=20000\), and \(d=150\). In this scaling, many outliers emerge in the unprojected CK spectrum. We therefore focus on the projected CK matrix \(\vK_\sigma\) defined in Theorem~\ref{thm:quadratic}, which isolates the label-aligned spike. We observe a single eigenvalue outside the bulk of the deformed MP law \(\mu_{\rm q}\), and its eigenvector aligns with the XOR labels \(\vy\), even at an SNR smaller than that used in Figure~\ref{fig:theorem3_figs}.

\begin{figure}[!htbp]
    \centering
    \begin{minipage}[c]{0.6\linewidth}
        \centering
        \includegraphics[width=\linewidth]{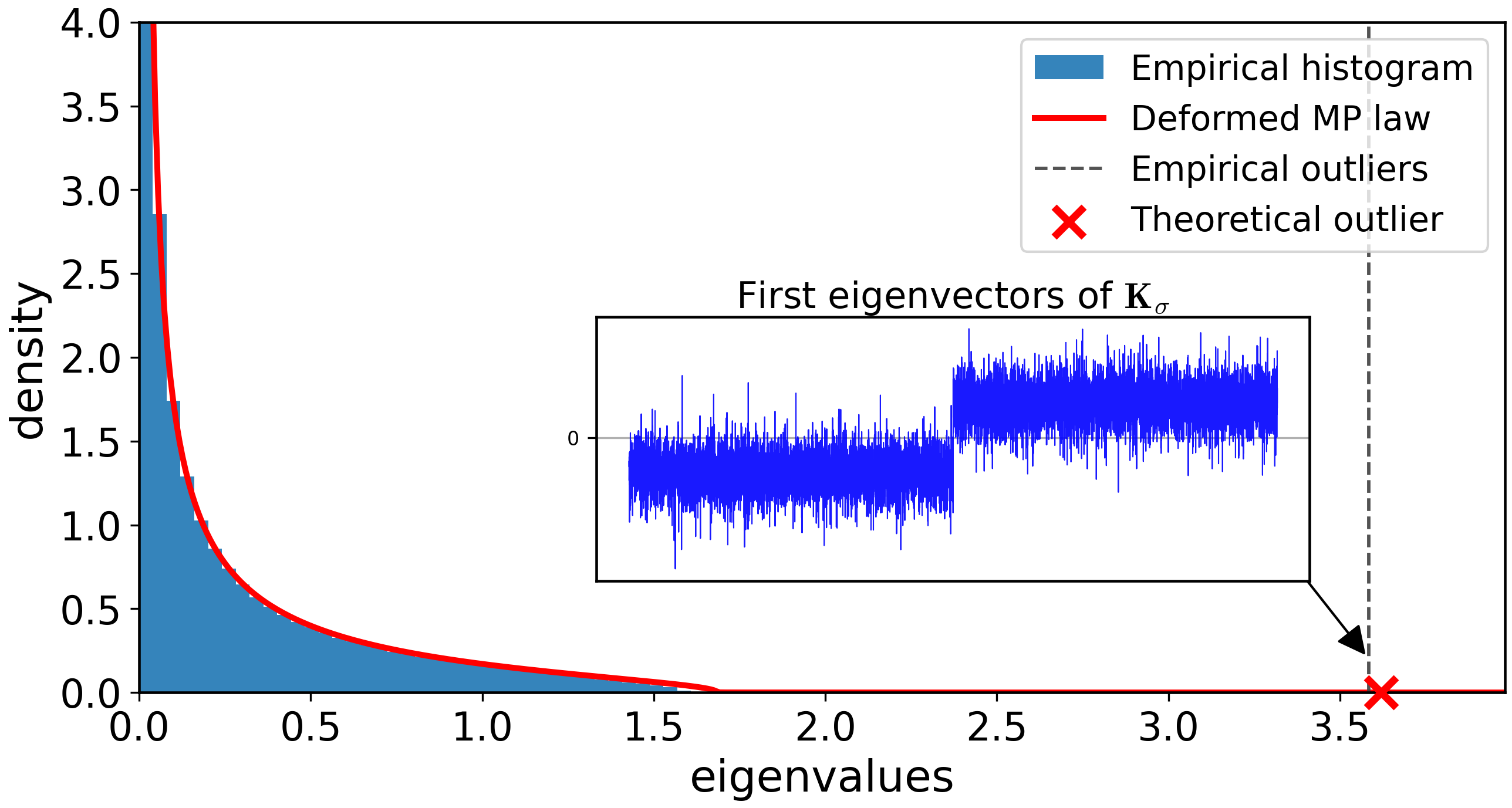}
    \end{minipage}\hfill
    \begin{minipage}[c]{0.39\linewidth}
        \centering
        \includegraphics[width=\linewidth]{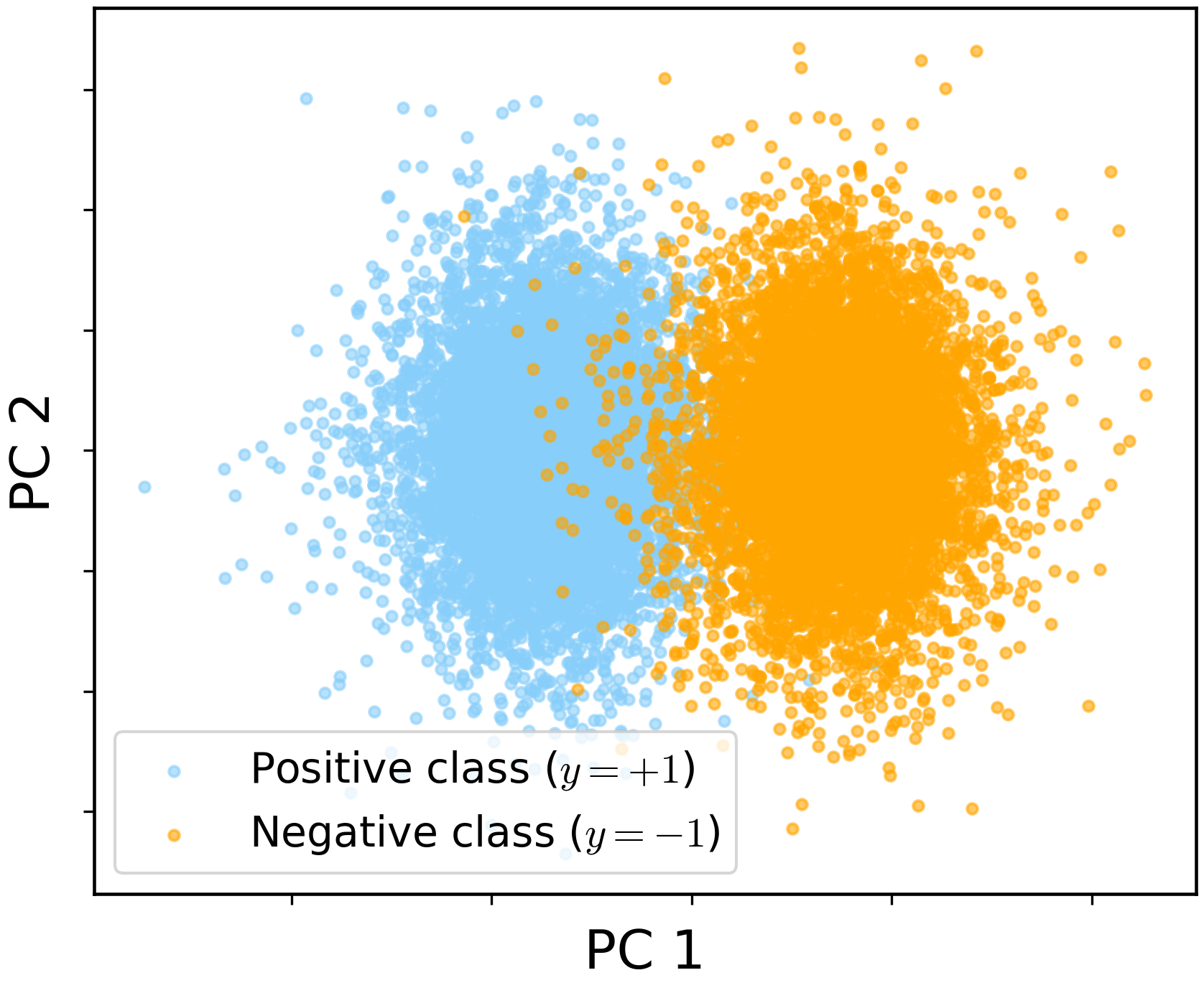}
    \end{minipage}
    \caption{\textbf{Quadratic sample-size regime: an informative CK spike aligns with XOR labels.} Here \(n=15000\), \(N=20000\), \(d=150\), the SNR is \(r^2=6.25\), and \(\sigma(x) \propto \relu(x)\). \textbf{Left:} spectrum of the projected CK matrix \(\vK_\sigma\), defined in Theorem~\ref{thm:quadratic}, together with the asymptotic spike prediction from Theorem~\ref{thm:quadratic}, shown by a red \(\color{red}{\times}\) marker. There is a single isolated spike, whose eigenvector is displayed in the inset. Its blockwise sign pattern aligns with the binary XOR partition, explaining the linear class separation seen on the right. \textbf{Right:} kernel-PCA visualization of the samples using the top two PCs of \(\vK_\sigma\).}
    \label{fig:quad_fig}
\end{figure}

\section{Conclusion}
 
We introduce a quadratic equivalent framework and illustrate its use by characterizing when nonlinear structure becomes spectrally visible in the CK matrix. 
For the nonlinear XOR problem, we show that in the proportional limit with finite SNR, the CK can exhibit outliers, but their eigenvectors are asymptotically unaligned with labels. In contrast, increasing SNR, moving to the quadratic sample-size regime, or adding pretrained structure to the weights can make the quadratic feature non-negligible and produce informative outliers that enable linear spectral classification.
Our results highlight the complex and subtle relationship between common ML training ``knobs'' and the emergence of spikes in the spectrum of ML models. 
Extending QE to more general scalings, deeper networks, and other forms of structured data is an interesting direction.

\subsubsection*{Acknowledgments}

CC and TK were supported in part by NSF grant DMS-2400246. 
ZW and MWM would like to acknowledge the NSF and the DARPA DIAL and DARPA AIQ programs for partial support of this work.


\bibliographystyle{alpha}
\bibliography{ref}


\newpage
\tableofcontents
\appendix 
\section{Preliminaries}
Throughout this paper, we utilize the stochastic domination notation introduced by \cite{erdHos2013averaging}.
\begin{definition}[Stochastic domination]\label{def:sd}
    Consider two families of random variables 
    \[X=\{X_n(u):n\in\N,u\in U_n\},\,\, Y=\{Y_n(u):n\in\N,u\in U_n\},\]
    with parameter $u\in U_n$ where $U_n$ is a possibly $n$-dependent deterministic set. Let $\cE=\{E_n\}_{n\in\N}$ be a sequence of measure subsets, where each $E_n$ is related to $X_n(u)$  and $Y_n(u)$.
    
    \begin{enumerate}
        \item We denote $|X|\prec Y$ or $X=\SD{Y}$ as the stochastic domination of $X$ by $Y$ uniformly in $u\in U_n$: if for all small $\epsilon>0$ and large $D>0$, there exists some $N_0(\epsilon,D)$, such that for all $n\ge N_0(\epsilon,D)$, 
        \[\sup_{u\in U_n}\P \left(|X_n(u)|\ge n^\epsilon Y_n(u)\right)\le n^{-D}.\]
        \item For a family of random matrices $\vA$ and a family of non-negative random variables $\zeta$, $\vA = \SD{\zeta}$ represents $|\langle \vv,\vA \vu\rangle | \prec \zeta \norm{\vv}\norm{\vu}$ uniformly for all deterministic vectors $\vv$ and $\vu$.
    \end{enumerate}
\end{definition}
If $\{\mbi{A}_n\}$ are random matrices, we interpret $\|\mbi{A}_n\|_{\op} = o_{\mathbb{P}}(1)
$ if and only if $
\|\mbi{A}_n\|_{\op}\xrightarrow{\mathbb{P}}0.$
For simplicity, in the proofs, we use $\|\mbi{A}_n\|$ to represent operator norm $\|\mbi{A}_n\|_{\op}$.

\subsection{Properties of XOR Data}
We introduce the following notion of $(\tau_n,B)$-\textbf{orthonormality} from \citep{fan2020spectra}.
\begin{definition}
    A matrix $\vX \in \R^{d\times n}$ is $(\tau_n,B)$-\textbf{orthonormal} if the columns $\{\vx_\alpha\}_{\alpha=1}^{n}$ of $\vX$ satisfy
\begin{equation}\label{eq:x_alpha_orthonormal}
     \bigl\lvert \|\vx_\alpha \| -1 \bigr\rvert \le \tau_n, \quad \bigl \lvert \|\vx_\beta \| -1 \bigr \rvert \le \tau_n, \quad \bigl\lvert \vx_\alpha^\top \vx_\beta \bigr\rvert \le \tau_n
\end{equation}
    for all pairs $\alpha \ne \beta \in [n]$, $\sum_{\alpha=1}^{n} \bigl( \|\vx_\alpha\| -1 \bigr)^2 \le B^2$ and $\|\vX \| \le B.$
\end{definition}
\begin{lemma}
\label{lem:xor_orthonormal}
Let $\mbi{X}\in\mathbb{R}^{d\times n}$ be the XOR data matrix defined in \eqref{eq:XOR}. Assume either
\begin{enumerate}[label=\textbf{(A\arabic*)},leftmargin=*]
\item $n\asymp d$ and $r=O(n^{1/4})$, \quad or
\item $n=O(d^2)$ and $r=O(1)$.
\end{enumerate}
Then there exist constants $c,C>0$  
such that with probability at least $1-n^{-c}$, $\mbi{X}$ satisfies \eqref{eq:x_alpha_orthonormal} with 
\begin{equation}\label{eq:tauB_choice}
\tau_n \;=\; C\Big(\sqrt{\frac{\log n}{d}}+\frac{r^2}{d}\Big).
\end{equation}
Moreover, on the same event,
\begin{equation}\label{eq:sumsq_choice_sqnorm}
\sum_{\alpha=1}^n\big(\|\mbi{x}_\alpha\|-1\big)^2
\;\le\;\sum_{\alpha=1}^n\big(\|\mbi{x}_\alpha\|^2-1\big)^2
\;\le\;
C\Big(\frac{n}{d}+\frac{nr^2}{d^2}+\frac{nr^4}{d^2}\Big),
\end{equation}
and also $\|\vX\|\le C\Big(1+(1+r)\sqrt{\frac{n}{d}}\Big).$
In particular, in regime (A1) one may take $\tau_n=C\sqrt{\frac{\log n}{n}}$ and $B=C(1+r)$, and in regime (A2)
one may take $\tau_n=C\sqrt{\frac{\log n}{d}}$ and $B=C \sqrt{\frac{n}{d}}$ such that $\vX$ is $(\tau_n,B)$-\textbf{orthonormal} with probability at least $1-n^{-c}$.
\end{lemma}

\begin{proof}
Recall that $\theta_\snr:=r\sqrt{\frac{n}{2d}}$, and $\mbi{u}_1,\mbi{u}_2\in\mathbb{R}^d$ and $\mbi{v}_1,\mbi{v}_2\in\mathbb{R}^n$ are orthonormal vectors defined by \eqref{eq:u_1u_2}. This proof is basically following the proof of Proposition 3.3 of \citet{fan2020spectra}. Write the columns as $\mbi{x}_\alpha=\mbi{z}_\alpha+\mbi{m}_\alpha$, where $\mbi{z}_\alpha$ is the $\alpha$-th column of $\mbi{Z}$
and $\mbi{m}_\alpha$ is the $\alpha$-th column of $\mbi{M}$.
Since $\mbi{M}=\theta_\snr\,\mbi{u}_1\mbi{v}_1^\top+\theta_\snr\,\mbi{u}_2\mbi{v}_2^\top$ and
$\mathrm{supp}(\mbi{v}_1)\cap \mathrm{supp}(\mbi{v}_2)=\emptyset$, for each $\alpha$ exactly one of
$\mbi{v}_1[\alpha],\mbi{v}_2[\alpha]$ is nonzero, and $\mbi{m}_\alpha = \theta_\snr \mbi{u}_1 \mbi{v}_1[\alpha]$ or $\theta_\snr \mbi{u}_2 \mbi{v}_2[\alpha].$
Using $|\mbi{v}_1[\alpha]|=|\mbi{v}_2[\alpha]|=\sqrt{2/n}$ on the support, we obtain for all $\alpha$
\begin{equation}\label{eq:mcol_norm_detail}
\|\mbi{m}_\alpha\|
=\theta_\snr\sqrt{\frac{2}{n}}
=r\sqrt{\frac{n}{2d}}\sqrt{\frac{2}{n}}
=\frac{r}{\sqrt d},
\qquad
\|\mbi{m}_\alpha\|^2=\frac{r^2}{d},
\qquad
\|\mbi{m}_\alpha\|^4=\frac{r^4}{d^2}.
\end{equation}
Now we verify the four requirements in the $(\tau_n,B)$-orthonormal definition:

\paragraph{Uniform control of column norms.}
For each $\alpha\in[n]$,
\begin{equation}\label{eq:xnorm2_expand}
\|\mbi{x}_\alpha\|^2
=\|\mbi{z}_\alpha\|^2 + 2\,\mbi{z}_\alpha^\top \mbi{m}_\alpha + \|\mbi{m}_\alpha\|^2.
\end{equation}
Since $\mbi{z}_\alpha\sim\mathcal N(\mbi{0},\mbi{I}_d/d)$, we have $d\|\mbi{z}_\alpha\|^2\sim\chi_d^2$.
By Lemma 1 of \citet{laurent2000adaptive},
\begin{equation}\label{eq:norm_Zalpha}
    \mathbb{P}\Big(\big|\|\mbi{z}_\alpha\|^2-1\big| \ge 2\sqrt{\tfrac{t}{d}} + 2\tfrac{t}{d}\Big)\le 2e^{-t}
\end{equation}
for all $t>0$. Taking $t=C_0\log n$ and union bounding over $\alpha\in[n]$ yields: with probability at least
$1-2n\cdot n^{-C_0}$, for $C_0$ large enough,
\begin{equation}\label{eq:znorm_sup_detail}
\max_{\alpha\in[n]}\big|\|\mbi{z}_\alpha\|^2-1\big|
\;\le\;
C\sqrt{\frac{\log n}{d}}.
\end{equation}
In particular, on the same event, $\max_\alpha\|\mbi{z}_\alpha\|^2\le 2$ for all large $n$.
Conditioning on $\mbi{m}_\alpha$, $\mbi{z}_\alpha^\top\mbi{m}_\alpha \,\big|\,\mbi{m}_\alpha
\sim\mathcal{N}\Big(0,r^2/d^2\Big),$ using \eqref{eq:mcol_norm_detail}.
Hence for $t=C_0\log n$,
\[
\mathbb{P}\Big(|\mbi{z}_\alpha^\top\mbi{m}_\alpha|\ge \frac{r}{d}\sqrt{2t}\Big)\le 2e^{-t}.
\]
A union bound over $\alpha\in [n]$ yields: with probability at least $1-n^{-c}$,
\begin{equation}\label{eq:cross_sup_detail}
\max_{\alpha\in[n]}|\mbi{z}_\alpha^\top\mbi{m}_\alpha|
\;\le\;
C\,\frac{r\sqrt{\log n}}{d}.
\end{equation}
Plugging \eqref{eq:znorm_sup_detail}, \eqref{eq:cross_sup_detail}, and $\|\mbi{m}_\alpha\|^2=r^2/d$ into
\eqref{eq:xnorm2_expand} gives, with probability at least $1-n^{-c}$,
\begin{equation}\label{eq:xnorm2_sup_detail}
\max_{\alpha\in[n]}\big|\|\mbi{x}_\alpha\|^2-1\big|
\;\le\;
C\Big(\sqrt{\frac{\log n}{d}}+\frac{r\sqrt{\log n}}{d}+\frac{r^2}{d}\Big).
\end{equation}
We now simplify the RHS under regimes (A1) and (A2). In both regimes, for large $n$, $\frac{r\sqrt{\log n}}{d}\le \sqrt{\frac{\log n}{d}}+\frac{r^2}{d}.$
Therefore \eqref{eq:xnorm2_sup_detail} implies
\begin{equation}\label{eq:xnorm2_sup_simplified}
\max_{\alpha\in[n]}\big|\|\mbi{x}_\alpha\|^2-1\big|
\;\le\;
C\Big(\sqrt{\frac{\log n}{d}}+\frac{r^2}{d}\Big).
\end{equation}
Finally, notice that
\begin{equation}\label{eq:bound_norm_x-1}
    \big|\|\mbi{x}_\alpha\|-1\big|
=\frac{\big|\|\mbi{x}_\alpha\|^2-1\big|}{\|\mbi{x}_\alpha\|+1}
\le \big|\|\mbi{x}_\alpha\|^2-1\big|,
\end{equation}
so \eqref{eq:xnorm2_sup_simplified} yields
\begin{equation}\label{eq:xnorm_sup_detail}
\max_{\alpha\in[n]}\big|\|\mbi{x}_\alpha\|-1\big|
\;\le\;
C\Big(\sqrt{\frac{\log n}{d}}+\frac{r^2}{d}\Big).
\end{equation}
This verifies the column-norm part of $(\tau_n,B)$-orthonormality with $\tau_n$ as in \eqref{eq:tauB_choice}.

\paragraph{Uniform control of pairwise inner products.}
Fix $\alpha\neq \beta\in [n]$. Expand
\begin{equation}\label{eq:inner_expand}
\mbi{x}_\alpha^\top \mbi{x}_\beta
=
\mbi{z}_\alpha^\top \mbi{z}_\beta
+\mbi{z}_\alpha^\top \mbi{m}_\beta
+\mbi{m}_\alpha^\top \mbi{z}_\beta
+\mbi{m}_\alpha^\top \mbi{m}_\beta.
\end{equation}
The scalar $\mbi{z}_\alpha^\top\mbi{z}_\beta$ is not exactly Gaussian, but it is conditionally Gaussian:
conditioning on $\mbi{z}_\alpha$, $\mbi{z}_\alpha^\top\mbi{z}_\beta \,\big|\, \mbi{z}_\alpha\sim \mathcal N(0,\|\mbi{z}_\alpha\|^2/d).$
On the event $\max_\alpha\|\mbi{z}_\alpha\|^2\le 2$  we have  
\[
\mathbb{P}\Big(|\mbi{z}_\alpha^\top\mbi{z}_\beta|\ge t\ \Big|\ \mbi{z}_\alpha\Big)
\le 2\exp\Big(-\frac{dt^2}{4\|\mbi{z}_\alpha\|^2}\Big)
\le 2\exp\Big(-c\,d t^2\Big).
\]
Taking $t=C\sqrt{\frac{\log n}{d}}$ and union bounding over all $\alpha\neq\beta\in [n]$ yields
with probability at least $1-n^{-c}$,
\begin{equation}\label{eq:zz_sup_detail}
\max_{\alpha\neq\beta}|\mbi{z}_\alpha^\top\mbi{z}_\beta|
\;\le\;
C\sqrt{\frac{\log n}{d}}.
\end{equation}
For $\alpha\neq\beta$, conditional on $\mbi{m}_\beta$, $\mbi{z}_\alpha^\top\mbi{m}_\beta \sim  \mathcal N\!(0,r^2/d^2),$
and similarly $\mbi{m}_\alpha^\top\mbi{z}_\beta\sim \mathcal N(0,r^2/d^2)$.
A union bound over all $\alpha\neq\beta$ gives, with probability at least $1-n^{-c}$,
\begin{equation}\label{eq:zm_sup_detail}
\max_{\alpha\neq\beta}|\mbi{z}_\alpha^\top\mbi{m}_\beta|
\;\le\;
C\,\frac{r\sqrt{\log n}}{d},
\qquad
\max_{\alpha\neq\beta}|\mbi{m}_\alpha^\top\mbi{z}_\beta|
\;\le\;
C\,\frac{r\sqrt{\log n}}{d}.
\end{equation}
Each $\mbi{m}_\alpha$ equals $\pm \frac{r}{\sqrt d}\mbi{u}_1$ or $\pm \frac{r}{\sqrt d}\mbi{u}_2$. By Cauchy–Schwarz inequality, $|\mbi{m}_\alpha^\top\mbi{m}_\beta|
\le \|\mbi{m}_\alpha\|\|\mbi{m}_\beta\|
=r^2/d.$ 
Hence
\begin{equation}\label{eq:mm_sup_detail}
\max_{\alpha\neq\beta}|\mbi{m}_\alpha^\top\mbi{m}_\beta|
\;\le\;\frac{r^2}{d}.
\end{equation}
Combining \eqref{eq:inner_expand}--\eqref{eq:mm_sup_detail},
we obtain with probability at least $1-n^{-c}$,
\begin{equation}\label{eq:inner_sup_detail}
\max_{\alpha\neq\beta}|\mbi{x}_\alpha^\top \mbi{x}_\beta|
\;\le\;
C\Big(\sqrt{\frac{\log n}{d}}+\frac{r^2}{d}\Big).
\end{equation}
This verifies the pairwise-inner-product part of $(\tau_n,B)$-orthonormality with $\tau_n$ as in \eqref{eq:tauB_choice}.

\paragraph{Sum of squared squared-norm deviations.}
From \eqref{eq:xnorm2_expand},
$
\|\mbi{x}_\alpha\|^2-1
=
(\|\mbi{z}_\alpha\|^2-1)
+2\,\mbi{z}_\alpha^\top\mbi{m}_\alpha
+\|\mbi{m}_\alpha\|^2.
$
Using $(a+b+c)^2\le 3(a^2+b^2+c^2)$, we get
\[
(\|\mbi{x}_\alpha\|^2-1)^2
\le
3(\|\mbi{z}_\alpha\|^2-1)^2
+12(\mbi{z}_\alpha^\top\mbi{m}_\alpha)^2
+3\|\mbi{m}_\alpha\|^4.
\]
Summing over $\alpha$ and using \eqref{eq:mcol_norm_detail} gives
\begin{equation}\label{eq:sumsq_reduce}
\sum_{\alpha=1}^n(\|\mbi{x}_\alpha\|^2-1)^2
\le
3\sum_{\alpha=1}^n(\|\mbi{z}_\alpha\|^2-1)^2
+12\sum_{\alpha=1}^n(\mbi{z}_\alpha^\top\mbi{m}_\alpha)^2
+3n\frac{r^4}{d^2}.
\end{equation}
Set $g_\alpha:=\sqrt d\,\vz_\alpha\sim\mathcal N(0,I_d)$ so that
$\|\vz_\alpha\|^2=\|g_\alpha\|^2/d$ and $\|g_\alpha\|^2\sim \chi_d^2$.
Define $U_\alpha := \|g_\alpha\|^2-d,$ and $X_\alpha := (\|\vz_\alpha\|^2-1)^2=\frac{U_\alpha^2}{d^2}.$
Then $\E[X_\alpha]=\Var(\|\vz_\alpha\|^2)=2/d$, hence $\E\sum_{\alpha=1}^n X_\alpha = 2n/d$.
Write $U_\alpha=\sum_{i=1}^d(\xi_{i\alpha}^2-1)$ with $\xi_{i\alpha}\iid\cN(0,1)$.
For any integer $p\ge2$, Rosenthal's inequality (Theorem 3 of \citet{rosenthal1970subspaces}) yields
\begin{equation}\label{eq:U_moment_bound}
\E|U_\alpha|^{2p}\ \le\ C_p\Big(d\,\E|\xi_{11}^2-1|^{2p} + (d\,\Var(\xi_{11}^2-1))^{p}\Big)
\ \le\ C_p' d^{p},
\end{equation}
where $C_p,C_p'$ depend only on $p$ (not on $d,n$). Consequently,
\begin{equation}\label{eq:X_moment_bound}
\E|X_\alpha|^{p}=\frac{\E|U_\alpha|^{2p}}{d^{2p}}\ \le\ \frac{C_p'}{d^{p}},
\qquad
\E|X_\alpha-\E X_\alpha|^{p}\ \le\ \frac{C_p''}{d^{p}}.
\end{equation}
Let $\widetilde X_\alpha := X_\alpha-\E X_\alpha$ so that $\E\widetilde X_\alpha=0$.
Applying Rosenthal's inequality again (now over $\alpha=1,\dots,n$) gives, for any integer $p\ge2$,
\[
\E\Big|\sum_{\alpha=1}^n \widetilde X_\alpha\Big|^{p}
\le C_p\Big(n\,\E|\widetilde X_1|^{p} + (n\,\E\widetilde X_1^2)^{p/2}\Big)
\le C_p'\Big(\frac{n}{d^2}\Big)^{p/2},
\]
using \eqref{eq:X_moment_bound} and $\E\widetilde X_1^2 \lesssim d^{-2}$.
Therefore, by Markov's inequality, for any fixed $t>0$,
\[
\P\Big(\sum_{\alpha=1}^n X_\alpha \ge \frac{2n}{d}+t\frac{n}{d}\Big)
=
\P\Big(\sum_{\alpha=1}^n \widetilde X_\alpha \ge t\frac{n}{d}\Big)
\le
\frac{\E|\sum_{\alpha=1}^n \widetilde X_\alpha|^{p}}{(t n/d)^{p}}
\le
C_p''\,t^{-p}\,n^{-p/2}.
\]
Taking $t=1$ and $p$ large enough (so that $p/2\ge c$) yields, for some constants $c,C>0$,
\begin{equation}\label{eq:sumsq_z_detail}
\sum_{\alpha=1}^n(\|\vz_\alpha\|^2-1)^2 \le C\frac{n}{d}
\qquad\text{with probability at least }1-n^{-c}.
\end{equation}
This bound is uniform in $(n,d)$ and in particular covers both regimes $n\asymp d$ and $n=O(d^2)$.

Conditioning on $\mbi{m}_\alpha$, the random variable $\mbi{z}_\alpha^\top\mbi{m}_\alpha$ is Gaussian
$\mathcal N(0,r^2/d^2)$, hence $(\mbi{z}_\alpha^\top\mbi{m}_\alpha)^2$ is sub-exponential with mean $r^2/d^2$.
By the Bernstein inequality,
\begin{equation}\label{eq:sumsq_zm_detail}
\sum_{\alpha=1}^n(\mbi{z}_\alpha^\top\mbi{m}_\alpha)^2 \le C\frac{nr^2}{d^2}
\qquad\text{with probability at least }1-n^{-c}.
\end{equation}

Plugging \eqref{eq:sumsq_z_detail}, \eqref{eq:sumsq_zm_detail} into \eqref{eq:sumsq_reduce} gives, with probability at least $1-n^{-c}$,
\[
\sum_{\alpha=1}^n(\|\mbi{x}_\alpha\|^2-1)^2
\le
C\Big(\frac{n}{d}+\frac{nr^2}{d^2}+\frac{nr^4}{d^2}\Big).
\] 
This proves \eqref{eq:sumsq_choice_sqnorm} because of \eqref{eq:bound_norm_x-1}.

\paragraph{Operator norm bound.}
We bound $\|\mbi{X}\| \le \|\mbi{Z}\| +\|\mbi{M}\|.$
A standard Gaussian operator norm bound states that for all $t\ge 0$,
\[
\mathbb{P}\Big(\|\mbi{Z}\| \ge 1+\sqrt{\frac{n}{d}}+t\Big)\le 2\exp(-c d t^2).
\]
Taking $t=\sqrt{\frac{\log n}{d}}$ yields $\|\mbi{Z}\| \le C(1+\sqrt{n/d})$ with probability at least $1-n^{-c}$. 
Since $\mbi{u}_1,\mbi{u}_2$ are orthonormal and $\mbi{v}_1,\mbi{v}_2$ are orthonormal,   $\|\mbi{M}\| =\theta_\snr=r\sqrt{\frac{n}{2d}}\le C r\sqrt{\frac{n}{d}}$.
 Therefore, with probability at least $1-n^{-c}$,
\[
\|\mbi{X}\| 
\le
C\Big(1+\sqrt{\frac{n}{d}}\Big)+C r\sqrt{\frac{n}{d}}
\le
C\Big(1+(1+r)\sqrt{\frac{n}{d}}\Big).
\] 
This completes the proof.
\end{proof}

\subsection{Properties of CK Matrix}

The following intermediary results are necessary to construct the QE in Appendix \ref{app:quad_de}.

\begin{prop} \label{prop:chid}
Let \(\chi_d\) denote the chi distribution with \(d\) degrees of freedom.
Then, as \(d\to\infty\),
\begin{equation}\label{eq:chid_result}
    \frac{\mathbb E[\chi_d]}{\sqrt d}
= 1 - \frac{1}{4d} + \frac{1}{32 d^2} + \frac{5}{128 d^3} + O(d^{-4}).
\end{equation}
\end{prop}

\begin{proof}
It is classical that $\mathbb E[\chi_d]=\sqrt{2}\,\frac{\Gamma\!\left(\frac{d+1}{2}\right)}{\Gamma\!\left(\frac d2\right)}.$
(See, e.g., \cite[Ch.~6]{abramowitz1964handbook}, \cite[Sec.~5.11]{DLMF}.)
Set \(m:=d/2\). Then
\[
\frac{\mathbb E[\chi_d]}{\sqrt d}
=\sqrt{\frac{2}{2m}}\,\frac{\Gamma(m+\tfrac12)}{\Gamma(m)}
=\frac{\Gamma(m+\tfrac12)}{\Gamma(m)\,\sqrt m}
=:Q_m .
\]
We use the Stirling series for \(\log\Gamma\) (with Bernoulli numbers),
valid as \(z\to\infty\) in \(|\arg z|<\pi\) \cite[(5.11.2)]{DLMF}:
\begin{equation}\label{eq:stirling}
\log\Gamma(z)
=(z-\tfrac12)\log z - z + \tfrac12\log(2\pi)
+\frac{1}{12 z}-\frac{1}{360 z^3}+O(z^{-5}).
\end{equation}
Write $\log Q_m
=(\log\Gamma(m+\tfrac12)-\log\Gamma(m))-\tfrac12\log m .$
Using \eqref{eq:stirling} with \(z=m+\tfrac12\) and \(z=m\), we obtain
\begin{align*}
\log Q_m
&= m\log(m+\tfrac12) - (m+\tfrac12) - (m-\tfrac12)\log m + m \\
&\quad + \frac{1}{12(m+\tfrac12)} - \frac{1}{12 m}
 - \frac{1}{360(m+\tfrac12)^3} + \frac{1}{360 m^3}
 - \tfrac12\log m + O(m^{-5}).
\end{align*}
The logarithmic terms simplify to $m\log(1+\frac{1}{2m}) - \frac{1}{2}.$
Expanding with the binomial and Taylor series
\(\log(1+x)=x-\tfrac{x^2}{2}+\tfrac{x^3}{3}+O(x^4)\) and
\((1+x)^{-1}=1-x+x^2-x^3+O(x^4)\), we get
\begin{align}
m\log\!\left(1+\frac{1}{2m}\right)
&= \frac12 - \frac{1}{8m} + \frac{1}{24 m^2} - \frac{1}{64 m^3} + O(m^{-4}),\\
\frac{1}{12(m+\tfrac12)} - \frac{1}{12m}
&= -\frac{1}{24 m^2} + \frac{1}{48 m^3} + O(m^{-4}).
\end{align}
Further, the difference
\(-\frac{1}{360(m+\tfrac12)^3} + \frac{1}{360 m^3}\)
is \(O(m^{-4})\).
Collecting terms, we have $\log Q_m
= -\frac{1}{8m} + \frac{1}{192 m^3} + O(m^{-4}).$
Exponentiating and using \(\exp(u)=1+u+\tfrac{u^2}{2}+\tfrac{u^3}{6}+O(u^4)\) with
\(u=-\tfrac{1}{8m}+\tfrac{1}{192 m^3}\), we obtain \eqref{eq:chid_result} by substituting \(m=d/2\).
\end{proof}
Notice that the same expansion follows from the general ratio formula \cite{tricomi1951asymptotic,DLMF}, for \(z\to\infty\)
\[\frac{\Gamma(z+a)}{\Gamma(z+b)}
\sim z^{a-b}\!\left(1+\frac{(a-b)(a+b-1)}{2z}+\cdots\right).\]

When the dataset is pure standard noise, let us define the null model as 
\begin{equation}\label{eq:null_model}
    \vY_0:=\frac{1}{\sqrt{N}}\sigma(\vW\vZ).
\end{equation} 
\begin{lemma}\label{lem:expectation}
    Under the Assumption~\ref{assump:sigma}, we know that $\|\E[\vY_0]\|\lesssim \frac{\sqrt{n}}{d
    }$
    as $n,d \to\infty$.
\end{lemma}
\begin{proof}
    Notice that $\E[\vY_0]=\frac{C_d}{\sqrt{N}}\boldone_N\boldone_n^\top$ where $C_d:=\E[\sigma(\frac{1}{\sqrt{d}}\vw^\top\vz)]$ where $\vz,\vw\iid \cN(0,\ident_d)$. Hence, $\|\E[\vY_0]\| = \sqrt{n}|C_d|.$
    Denote $\xi \sim \cN(0,1)$ which is independent with $\vw$. Applying $\E[\sigma(\xi)]=0$,
$\E[\sigma'(\xi)\xi]=\E[\sigma''(\xi)]=c_\sigma$, and a Taylor approximation of $\sigma$ at $\xi$, we can get
	\begin{align*}
	C_d=~&\E[\sigma(\frac{1}{\sqrt{d}}\norm{\vw}\xi)]-\E[\sigma(\xi)]\\
	= ~&
\E[\sigma'(\xi)\xi(\frac{1}{\sqrt{d}}\norm{\vw}-1)]+\E[\sigma''(\eta)\xi^2(\frac{1}{\sqrt{d}}\norm{\vw}-1)^2]\\
= ~&
c_\sigma(\E[\frac{1}{\sqrt{d}}\norm{\vw}]-1)+\E[\sigma''(\eta)\xi^2(\frac{1}{\sqrt{d}}\norm{\vw}-1)^2],
	\end{align*}
for some $\eta$ between $\xi$ and $\frac{1}{\sqrt{d}}\norm{\vw}\xi$. Notice that $\frac{1}{\sqrt{d}}\|\vw\|\overset{\cD}{\sim} \frac{1}{\sqrt{d}}\chi_{d}$ where $\chi_d$ is chi distribution with degree $d$. Thus, 
\begin{equation}\E[\frac{1}{\sqrt{d}}\norm{\vw}] = \frac{\sqrt{2}\Gamma(\frac{1}{2}(d+1))}{\sqrt{d}\Gamma(\frac{1}{2}d)} = \sqrt{1-\frac{1}{2d}}+o(d^{-1})=1-\frac{1}{4d}+o(d^{-1}).\label{eq:chimean}\end{equation}
Then, we can get
\[|C_d|\lesssim \frac{1}{d}+\E[(\frac{1}{\sqrt{d}}\norm{\vw}-1)^2]\lesssim \frac{1}{d},\]
where we apply the Assumption~\ref{assump:sigma} and the fact $\E[(\frac{1}{\sqrt{d}}\norm{\vw}-1)^2] = 2 - 2 \E[\frac{1}{\sqrt{d}}\norm{\vw}] = O(\frac{1}{d})$.
\end{proof}

\begin{lemma}\label{lem:concentration_coeff}
Let $\vz,\vw\iid \cN(0,\ident_d)$. Under the Assumption~\ref{assump:sigma}, we know that 
\begin{align}
    |\E[\sigma'(\frac{1}{\sqrt{d}}\vw^\top\vz)]-b_\sigma|\lesssim~& 1/d, \qquad
|\E[\sigma''(\frac{1}{\sqrt{d}}\vw^\top\vz)]-c_\sigma|\lesssim~ 1/\sqrt{d}.
\end{align}
\end{lemma}
\begin{proof}
Denote $\xi \sim \cN(0,1)$ which is independent with $\vw$. Recall that $\E[\sigma'(\xi)]=b_\sigma$ and 
$\E[\sigma'(\xi)\xi]=\E[\sigma''(\xi)]=c_\sigma$. Similarly to the proof of Lemma~\ref{lem:expectation}, we can apply a Taylor approximation of $\sigma'$ at $\xi$:
\begin{align}
    |\E[\sigma'(\frac{1}{\sqrt{d}}\vw^\top\vz)]-b_\sigma|=~&|\E[\sigma'(\frac{1}{\sqrt{d}}\|\vw\|\xi)]-\E[\sigma'(\xi)]|\\
    \le ~&
|\E[\sigma''(\xi)\xi]|\cdot |\frac{1}{\sqrt{d}}\E[\norm{\vw}]-1|+|\E[\sigma'''(\eta)\xi^2]|\E(\frac{1}{\sqrt{d}}\norm{\vw}-1)^2\\
\lesssim ~&
|\frac{1}{\sqrt{d}}\E[\chi_d]-1|+\E(\frac{1}{\sqrt{d}}\chi_d-1)^2\lesssim \frac{1}{d},
\end{align} for some $\eta$ between $\xi$ and $\frac{1}{\sqrt{d}}\|\vw\|\xi$. Similar proof can be applied for $\sigma''$ to complete the proof.
\end{proof}

\begin{lemma}\label{lem:infinity_norm}
    Consider $\vg\sim\cN(0,\ident_N)$. Then we have $\|\vg\|_\infty\prec \sqrt{2\log N}$.
\end{lemma}
\begin{proof}
    Apply Proposition 2.7.6 of \cite{vershynin2018high} and use the fact that $\E[ \|\vg\|_{\infty} ]=\E[\max_{i\in [N]} |g_i|]$.
\end{proof}

\begin{lemma}\label{lem:hardamart_identity}
    Consider a matrix $\vM\in\R^{N\times n}$ and two vectors $\va\in\R^N$, $\vb\in\R^n$. We have the following identity
    \[\vM\odot(\va\vb^\top) = \diag(\va)\vM\diag(\vb).\]
\end{lemma}
\noindent
This is a classical linear algebra result, so we will omit the proof here.

\begin{lemma}\label{lem:concentration_operator}
Assume $\sigma$ is globally $L$-Lipschitz for a constant $L>0$.
Then there exist absolute constants $C,c>0$ such that for every $t>0$, with probability at least
$1-6e^{-c t^2}$,
\[ 
\norm{\frac1{\sqrt{N}}\Big(\sigma(\mbi{W}\mbi{Z})-\E\,[\sigma(\mbi{W}\mbi{Z})]\Big)}
\;\le\;
C\,L\Big(1+\sqrt{\frac nd}+\frac{t}{\sqrt d}\Big)\Big(1+\sqrt{\frac nN}+\frac{t}{\sqrt N}\Big).
\]
\end{lemma}
\begin{proof}
Let $\mbi{z}_j\in\R^d$ denote the $j$th column of $\mbi{Z}$.
By exchangeability, $\E\,\vY_0 = \frac{m}{\sqrt{N}}\,\1_N\1_n^\top$ where $\vY_0$ is defined in \eqref{eq:null_model} and $m:=\E_{\,\mbi{w}\sim\cN(\mbi{0},\mbi{I}_d),\ \mbi{z}\sim\cN(\mbi{0},\mbi{I}_d/d)}\big[\sigma({\mbi{w}}^\top{\mbi{z}})\big].$ Conditioned on $\mbi{Z}$, define $\mbi{\mu}(\mbi{Z})\in\R^n$ whose entries are
\[
\mu_j(\mbi{Z}) := \E_{\mbi{w}\sim\cN(\mbi{0},\mbi{I}_d)}\big[\sigma(\mbi{w}^\top \mbi{z}_j)\big],\qquad j=1,\dots,n.
\]
Then each row of $\vY_0$ has conditional mean $\mbi{\mu}(\mbi{Z})^\top$, hence $\E_{\mbi{W}}[\vY_0 \mid \mbi{Z}] = \frac{1}{\sqrt{N}}\1_N\,\mbi{\mu}(\mbi{Z})^\top.$
Write
\begin{equation}\label{eq:basic_triangle_fixed}
\norm{\vY_0-\E\vY_0}
=
\underbrace{\norm{\big(\vY_0-\E_{\mbi{W}}[\vY_0\mid \mbi{Z}]\big)}}_{T_1}
+
\underbrace{\norm{\big(\E_{\mbi{W}}[\vY_0\mid \mbi{Z}] - \E\vY_0\big)}}_{T_2}.
\end{equation}
We then control $T_1$ and $T_2$ separately.

\smallskip
\noindent\textbf{Bound $T_2$.}
We have
$
T_2
=
\norm{\mbi{\mu}(\mbi{Z})-m\1_n}.
$
Define $\mu(\mbi{z}) := \E_{\mbi{w}}[\sigma(\mbi{w}^\top \mbi{z})]$. For any $\mbi{z},\mbi{z}'\in\R^d$,
\[
|\mu(\mbi{z})-\mu(\mbi{z}')|
\le
\E_{\mbi{w}}\big|\sigma(\mbi{w}^\top \mbi{z})-\sigma(\mbi{w}^\top \mbi{z}')\big|
\le
L\,\E_{\mbi{w}}|\mbi{w}^\top(\mbi{z}-\mbi{z}')|
=
L\sqrt{\frac{2}{\pi}}\|\mbi{z}-\mbi{z}'\|_2.
\]
Thus $\mu(\cdot)$ is $(L\sqrt{2/\pi})$-Lipschitz. Write $\mbi{z}_j=\mbi{g}_j/\sqrt{d}$ with $\mbi{g}_j\sim\cN(\mbi{0},\mbi{I}_d)$.
Then $f(\mbi{g}) := \mu(\mbi{g}/\sqrt{d})$ is Lipschitz with constant $O(L/\sqrt{d})$.
By Gaussian concentration for Lipschitz functions (e.g., Chapter 5 of \citep{vershynin2018high}),
each coordinate $\mu(\mbi{z}_j)-\E\mu(\mbi{z}_j)=\mu(\mbi{z}_j)-m$ is subgaussian with subgaussian norm
\[
\|\mu(\mbi{z}_j)-m\|_{\psi_2} \le C\frac{L}{\sqrt{d}}.
\]
Since $(\mbi{z}_j)$ are independent across $j$, the coordinates of $\mbi{\mu}(\mbi{Z})-m\1_n$ are independent subgaussian.
A standard concentration bound for the Euclidean norm of an i.i.d. subgaussian vector
(e.g., Section~3.1 of \citep{vershynin2018high})
yields: for every $t\ge0$, with probability at least $1-2e^{-ct^2}$,
\begin{equation}\label{eq:T2_bound_fixed}
T_2=\norm{\mbi{\mu}(\mbi{Z})-m\1_n}_2
\le C\frac{L}{\sqrt{d}}(\sqrt{n}+t)
=
CL\Big(\sqrt{\frac nd}+\frac{t}{\sqrt d}\Big).
\end{equation}

\smallskip
\noindent\textbf{Bound $T_1$.}
Let $\mbi{A}:=
\sigma(\mbi{W}\mbi{Z})-\1_N\mbi{\mu}(\mbi{Z})^\top.$
Then $T_1=\|\mbi{A}/\sqrt{N}\|$.
Conditional on $\mbi{Z}$, the rows $\mbi{A}_i\in\R^n$ are i.i.d.\ (over rows $\mbi{w}_i$ of $\mbi{W}$) and mean-zero.
Fix any $\mbi{x}\in\mathbb{S}^{n-1}$ and define the scalar function of $\mbi{w}\in\R^d$
$
F_{\mbi{x}}(\mbi{w}) := 
\mbi{x}^\top \sigma(\mbi{w}^\top \mbi{Z}).
$
Then $\mbi{A}_i^\top\mbi{x}=F_{\mbi{x}}(\mbi{w}_i)-\E[F_{\mbi{x}}(\mbi{w}_i)\mid \mbi{Z}]$.
Moreover, $F_{\mbi{x}}$ is Lipschitz with constant at most $L\|\mbi{Z}\|$:
\begin{align*}
|F_{\mbi{x}}(\mbi{w})-F_{\mbi{x}}(\mbi{w}')|
&\le \sum_{j=1}^n |x_j|\,|\sigma(\mbi{w}^\top\mbi{z}_j)-\sigma(\mbi{w}'^\top\mbi{z}_j)|
\le L\sum_{j=1}^n |x_j|\,|(\mbi{w}-\mbi{w}')^\top\mbi{z}_j| \\
&\le L\|\mbi{x}\|_2\cdot \| \mbi{Z}^\top(\mbi{w}-\mbi{w}')\|_2
\le L\|\mbi{Z}\|\,\|\mbi{w}-\mbi{w}'\|_2,
\end{align*}
for any $\mbi{w},\mbi{w}'$, since $\|\mbi{x}\|_2=1$.
By Gaussian concentration for Lipschitz functions,
conditional on $\mbi{Z}$ we have
\begin{equation}\label{eq:subg_inner_prod}
\P\Big(|\mbi{A}_i^\top \mbi{x}|\ge u \,\Big|\,\mbi{Z}\Big)
\le 2\exp\Big(-c\,\frac{u^2}{L^2\|\mbi{Z}\|^2}\Big)
\end{equation}
for all $u\ge0$. Now fix $\mbi{x}\in\mathbb{S}^{n-1}$ and $\mbi{y}\in\mathbb{S}^{N-1}$. Then $\langle \mbi{A}\mbi{x},\mbi{y}\rangle
=
\sum_{i=1}^N y_i\,(\mbi{A}_i^\top\mbi{x})$
is a sum of independent centered subgaussian random variables.
Using \eqref{eq:subg_inner_prod} and $\|\mbi{y}\|_2=1$, we obtain the conditional tail bound
\begin{equation}\label{eq:bilinear_subg}
\P\Big(|\langle \mbi{A}\mbi{x},\mbi{y}\rangle|\ge u \,\Big|\,\mbi{Z}\Big)
\le 2\exp\Big(-c\,\frac{u^2}{L^2\|\mbi{Z}\|^2}\Big).
\end{equation}
Next, take $1/4$-nets $\mathcal{N}_n\subset\mathbb{S}^{n-1}$ and $\mathcal{N}_N\subset\mathbb{S}^{N-1}$ with
$|\mathcal{N}_n|\le 9^n$ and $|\mathcal{N}_N|\le 9^N$. We have
\begin{equation}\label{eq:net_reduction}
\|\mbi{A}\|
=
\sup_{\mbi{x}\in\mathbb{S}^{n-1},\ \mbi{y}\in\mathbb{S}^{N-1}}
\langle \mbi{A}\mbi{x},\mbi{y}\rangle
\ \le\
2\max_{\mbi{x}\in\mathcal{N}_n,\ \mbi{y}\in\mathcal{N}_N}\langle \mbi{A}\mbi{x},\mbi{y}\rangle.
\end{equation}
Combining \eqref{eq:bilinear_subg}--\eqref{eq:net_reduction} with a union bound yields, for any $u\ge0$,
\[
\P\Big(\|\mbi{A}\|\ge 2u \,\Big|\,\mbi{Z}\Big)
\le
2\,|\mathcal{N}_n|\,|\mathcal{N}_N|\,
\exp\Big(-c\,\frac{u^2}{L^2\|\mbi{Z}\|^2}\Big)
\le
2\exp\Big((n+N)\log 9 -c\,\frac{u^2}{L^2\|\mbi{Z}\|^2}\Big).
\]
Choosing $u = C L\|\mbi{Z}\|\big(\sqrt{n}+\sqrt{N}+t\big)$ and taking $C$ large enough (absorbing $(n+N)\log 9$)
gives: for every $t\ge0$, with conditional probability at least $1-2e^{-ct^2}$,
\[
\|\mbi{A}\|
\le
CL\|\mbi{Z}\|\big(\sqrt{n}+\sqrt{N}+t\big),
\quad\text{hence}\quad
T_1=\Big\|\frac{\mbi{A}}{\sqrt{N}}\Big\|
\le
CL\|\mbi{Z}\|\Big(1+\sqrt{\frac{n}{N}}+\frac{t}{\sqrt{N}}\Big).
\]
Since the bound holds conditionally on $\mbi{Z}$ with failure probability $\le 2e^{-ct^2}$, it also holds unconditionally.

\smallskip
\noindent\textbf{Bound $\|\mbi{Z}\|$.}
By the norm bound
for Gaussian matrices,
with probability at least $1-2e^{-ct^2}$,
\[
\|\mbi{Z}\|
\le
1+\sqrt{\frac{n}{d}}+\frac{t}{\sqrt{d}}.
\]
On the intersection of the high-probability events
\eqref{eq:T2_bound_fixed}, the $T_1$ bound above, and the $\|\mbi{Z}\|$ bound, we get
\[
T_1 \le CL\Big(1+\sqrt{\frac nd}+\frac{t}{\sqrt d}\Big)\Big(1+\sqrt{\frac nN}+\frac{t}{\sqrt N}\Big),
\qquad
T_2 \le CL\Big(\sqrt{\frac nd}+\frac{t}{\sqrt d}\Big)
\le
CL\Big(1+\sqrt{\frac nd}+\frac{t}{\sqrt d}\Big)\Big(1+\sqrt{\frac nN}+\frac{t}{\sqrt N}\Big),
\]
and then \eqref{eq:basic_triangle_fixed} yields the desired bound.
A union bound over the three events gives total failure probability at most $6e^{-ct^2}$,
which completes the proof.
\end{proof}


\subsection{Properties of z-transform and T-transform}\label{subsec:z_toolkit}

Let \(\nu\) be a compactly supported probability measure on \([0,\infty)\) and \(\phi\in(0,\infty)\).
For \(s\in\R\setminus\{0\}\) such that \(1+\lambda s\neq 0\) for all \(\lambda\in\supp{\nu}\),
\begin{equation}\label{eq:z_def}
z(s)
=
-\frac{1}{s}+\phi\int \frac{\lambda}{1+\lambda s}\,\nu(d\lambda),
\qquad
z'(s)
=
\frac{1}{s^2}-\phi\int\frac{\lambda^2}{(1+\lambda s)^2}\,\nu(d\lambda),
\end{equation} are well-defined and finite. 
Then the alignment map $\varphi(s)$
and $T(s)$ in \eqref{eq:z(s)} and~\eqref{eq:T_def} respectively, are also well-defined and finite for such \(s\). \(z(s)\) is called the z-transform of the deformed MP law \(\rho_\phi^{\rm MP}\boxtimes\nu\).


\begin{lemma} 
\label{lem:z_parametrizes_complement}
Let $\nu$ be a compactly supported probability measure on $[0,\infty)$ and let $\phi\in(0,\infty)$.
Let $\mu=\rho_\phi^{\rm MP}\boxtimes \nu$ be the deformed MP law, and define its companion
probability measure
$
\tilde\mu := \phi\,\mu + (1-\phi)\,\delta_0.
$
Let $m_\mu(z)=\int (x-z)^{-1}\,\mu(dx)$ and $m_{\tilde\mu}(z)=\int (x-z)^{-1}\,\tilde\mu(dx)$ be the Stieltjes
transforms, defined on $\C\setminus\supp{\mu}$ and $\C\setminus\supp{\tilde\mu}$ respectively. Then for all
$z\in\C\setminus(\supp{\mu}\cup\{0\})$, $m_{\tilde\mu}(z) = \tilde m_\mu(z)$ where
\begin{equation}\label{eq:companion_identity_fixed}
\tilde m_\mu(z):=\frac{\phi-1}{z}+\phi\,m_\mu(z).
\end{equation}
Define the set of singularities $\cT:=\{0\}\cup\{-1/\lambda:\lambda\in\supp{\nu}\}.$
For $s\in\C\setminus \cT$, $z(s)$ and $z'(s)$ in \eqref{eq:z_def} are well defined. 
Then, we have
\begin{enumerate}[label=\textbf{(\roman*)},leftmargin=*]
\item  For any real $\lambda\in\R\setminus\supp{\tilde\mu}$, the quantity
$s:=m_{\tilde\mu}(\lambda)$ belongs to $\R\setminus \cT$
and satisfies
\begin{equation}\label{eq:z_inverse_relation}
\lambda \;=\; z(s),
\qquad\text{and}\qquad
z'(s)\;>\;0.
\end{equation}

Equivalently, since $\supp{\mu}$ and $\supp{\tilde\mu}$ may differ only at $\{0\}$, the same conclusion holds
for any real $\lambda\in\R\setminus(\supp{\mu}\cup\{0\})$ with $s=\tilde m_\mu(\lambda)$.

\item Conversely, if $s\in\R\setminus \cT$ satisfies $z'(s)>0$,
then $\lambda:=z(s)$ belongs to $\R\setminus\supp{\tilde\mu}$ and $m_{\tilde\mu}(\lambda)=s.$
Moreover, let $I$ be a maximal open interval contained in $\R\setminus \cT$ on which $z'(s)>0$.
Then $z$ is strictly increasing on $I$ and maps $I$ bijectively onto a connected component of
$\R\setminus\supp{\tilde\mu}$ (an exterior component or an internal gap).
Finally, every nonzero boundary point $\lambda_\star\neq 0$ of $\supp{\tilde\mu}$ can be written as
$\lambda_\star=z(s_\star)$ for some $s_\star\in\R\setminus \cT$ satisfying $z'(s_\star)=0$.
\end{enumerate}
\end{lemma}

\begin{proof}
This is exactly the classical result of Silverstein--Choi (Theorems.\ 4.1--4.2 of \citet{silverstein1995analysis}).
\end{proof}

\begin{lemma} \label{lem:T_relation}
With the same notation as Lemma~\ref{lem:z_parametrizes_complement},  recall the $T$-transform defined by \eqref{eq:T_def}.
For any real $\lambda\in\R\setminus(\supp{\mu}\cup\{0\})$, let $s=\tilde m_\mu(\lambda)=m_{\tilde\mu}(\lambda)$.
Then $\lambda=z(s)$ and
\begin{equation}\label{eq:T_equals_lsms}
\lambda\,s\,m_\mu(\lambda)\;=\;T(s).
\end{equation}
\end{lemma}

\begin{proof}
Since $\lambda\neq 0$ and $s=\tilde m_\mu(\lambda)=\frac{\phi-1}{\lambda}+\phi m_\mu(\lambda)$,
we can get
$m_\mu(\lambda)=\frac{1}{\phi}\Bigl(s-\frac{\phi-1}{\lambda}\Bigr).$
Therefore,
\[
\lambda s m_\mu(\lambda)
=
\lambda s\cdot \frac{1}{\phi}\Bigl(s-\frac{\phi-1}{\lambda}\Bigr)
=
\frac{\lambda s^2-(\phi-1)s}{\phi}.
\]
By Lemma~\ref{lem:z_parametrizes_complement}(i), $\lambda=z(s)$, so the right-hand side equals $T(s)$ by
definition. This proves \eqref{eq:T_equals_lsms}.
\end{proof}

\begin{lemma}[Strict monotonicity of \(T\) on each analytic branch]\label{lem:T_monotone}
On any real interval \(I\subset\R\) that does not cross a pole \(s=-1/\lambda\) with \(\lambda\in\supp{\nu}\),
the function \(T\) is \(C^1\) and strictly decreasing:
\begin{equation}\label{eq:Tprime_negative}
T'(s)<0,\qquad s\in I.
\end{equation} Hence, $T$ is invertible on this interval.
\end{lemma}

\begin{proof}
For \(s\in I\), we define
\[
I_1(s):=\int\frac{\lambda}{1+\lambda s}\,\nu(d\lambda),
\qquad
I_2(s):=\int\frac{\lambda^2}{(1+\lambda s)^2}\,\nu(d\lambda).
\]
Then \(I_1\in C^1(I)\) and \(I_1'(s)=-I_2(s)\).
Using \(T(s)=-s+s^2 I_1(s)\) (equivalent to \eqref{eq:T_def} and \eqref{eq:z_def}) we compute
\begin{equation}\label{eq:Tprime_raw}
T'(s)=-1+2sI_1(s)-s^2I_2(s).
\end{equation}
Let \(F(\lambda):=\frac{\lambda}{1+\lambda s}\). Then \(I_1(s)=\E_\nu[F]\) and \(I_2(s)=\E_\nu[F^2]\).
Hence $T'(s)=-1+2s\E[F]-s^2\E[F^2]
=-(1-s\E[F])^2 - s^2\Var(F).$
If \(\nu\) is not a single atom, then \(\Var(F)>0\) for all \(s\in I\), so \(T'(s)<0\). In the degenerate atomic case, \(T'< 0\) still holds since $(1-s\E[F])^2>0$.
This proves \eqref{eq:Tprime_negative}.
\end{proof}

\begin{remark}[Relation to the classical \(D\)-transform in \citet{benaych2012singular}]
If one defines $D(\lambda):=\lambda m_\mu(\lambda)\tilde m_\mu(\lambda)$,
then by Lemma~\ref{lem:T_relation},
\(D(\lambda)=T(s)\) whenever \(\lambda=z(s)\) and \(s=\tilde m_\mu(\lambda)\).
Thus all additive outlier equations \(\beta D(\lambda)=1\) are exactly \(\beta T(s)=1\) in \(z\)-coordinates.
This $D$-transform and \(\beta D(\lambda)=1\) are used by \citet{benaych2012singular} to characterize the outlier singular values in additive spiked random rectangualr matrices.
\end{remark}

\begin{lemma}[Closed-form inverse of $T$ for a shifted MP bulk]\label{lem:T_inverse_shiftMP}
Fix $\psi\in(0,\infty)$ and let $b_\sigma^2\in[0,1]$. Let $m_{{\rm MP},\psi}(\cdot)$ denote the Stieltjes transform of $\rho^{\rm MP}_\psi$. Define the shifted MP law $\nu:=(1-b_\sigma^2)+b_\sigma^2\,\rho^{\rm MP}_\psi,$
and for any $s\in\C\setminus\{0\}$ such that $1+\lambda s\neq 0$ for all $\lambda\in\supp{\nu}$, $T$ transform defined by \eqref{eq:T_def} can be written as
\begin{equation}\label{eq:T_def_shiftMP}
T(s)\;=\;-s+s^2\int \frac{\lambda}{1+\lambda s}\,\nu(d\lambda).
\end{equation}

\begin{enumerate}
\item[\textbf{(i)}] If $b_\sigma=0$, then $T(s)=-s/(1+s)$. Otherwise, let $\xi(s)\;:=\;-\frac{1+(1-b_\sigma^2)s}{b_\sigma^2\,s}$ and we have
\begin{equation}\label{eq:T_as_MP}
T(s)\;=\;-\frac{1}{b_\sigma^2}\,m_{{\rm MP},\psi}\!\big(\xi(s)\big).
\end{equation}
\item[\textbf{(ii)}]
Recall the definition of $T^{-1}(t)$ in \eqref{eq:T_def}
for $t\in\C$ such that the denominator in \eqref{eq:T_def} is nonzero. 
Then $T\!\big(T^{-1}(t)\big)=t$. In particular, on any real monotone branch of $T$
(e.g., $s<0$), \eqref{eq:T_def} is the functional inverse of $T$.
\end{enumerate}
\end{lemma}

\begin{proof}
Write $\alpha:=1-b_\sigma^2$ and $\beta:=b_\sigma^2$ for short, so that $\nu=\alpha+\beta\rho^{\rm MP}_\psi$.
Let $X\sim\rho^{\rm MP}_\psi$ and set $\lambda=\alpha+\beta X$. Then
$
\int \frac{\lambda}{1+\lambda s}\,\nu(d\lambda)
\;=\;
\mathbb E\Big[\frac{\alpha+\beta X}{1+s(\alpha+\beta X)}\Big].
$
Now note the following exact algebraic identity:
\begin{equation}\label{eq:algebra_key}
\frac{\alpha+\beta X}{1+s(\alpha+\beta X)}
=
\frac{1}{s}\Big(1-\frac{1}{1+s(\alpha+\beta X)}\Big).
\end{equation}
Taking expectations in \eqref{eq:algebra_key} gives
\begin{equation}\label{eq:I1_s}
\int \frac{\lambda}{1+\lambda s}\,\nu(d\lambda)
=
\frac{1}{s}\Big(1-\mathbb E\Big[\frac{1}{1+s(\alpha+\beta X)}\Big]\Big).
\end{equation}
Insert \eqref{eq:I1_s} into the definition \eqref{eq:T_def_shiftMP} to get 
\begin{equation}\label{eq:T_as_expectation}
T(s)
=
-s\;\mathbb E\Big[\frac{1}{1+s\alpha+s\beta X}\Big].
\end{equation}
Now rewrite the expectation into Stieltjes transform. 
Therefore
\[
\mathbb E\Big[\frac{1}{1+s\alpha+s\beta X}\Big]
=
\frac{1}{s\beta}\,
\mathbb E\Big[\frac{1}{X-\xi(s)}\Big],
\qquad
\xi(s)=-\frac{1+s\alpha}{s\beta}.
\]
By definition of the MP Stieltjes transform, $\mathbb E\big[\frac{1}{X-\xi}\big]=m_{{\rm MP},\psi}(\xi)$, so $\mathbb E\Big[\frac{1}{1+s\alpha+s\beta X}\Big]
=
\frac{1}{s\beta}\,m_{{\rm MP},\psi}\!\big(\xi(s)\big).$
Plug this into \eqref{eq:T_as_expectation} to obtain \eqref{eq:T_as_MP}.

\medskip
Lemma \ref{lem:T_monotone} shows $T$ is invertible on certain intervals. Let $t=T(s)$. By \eqref{eq:T_as_MP}, this is equivalent to $m_{{\rm MP},\psi}\!\big(\xi(s)\big)=-\beta t.$ Set $m:=-\beta t.$
For MP, $m=m_{{\rm MP},\psi}(\xi)$ satisfies the exact quadratic identity
\begin{equation}\label{eq:MP_quadratic_again}
\psi\,\xi\,m^2+(\xi+\psi-1)m+1=0.
\end{equation}
Solve \eqref{eq:MP_quadratic_again} for $\xi$ in terms of $m$ to obtain
$
\xi
=
-\frac{1+(\psi-1)m}{m(1+\psi m)}.
$
Thus, with $m=-\beta t$, we have
\begin{equation}\label{eq:xi_of_t_raw}
\xi
=
-\frac{1+(\psi-1)(-\beta t)}{(-\beta t)\bigl(1+\psi(-\beta t)\bigr)}
=
\frac{1-(\psi-1)\beta t}{\beta t(1-\psi\beta t)}.
\end{equation}
By definition of $\xi(s)$, $\xi(s)=-\frac{1+\alpha s}{\beta s}.$
Hence, we have $s=-\frac{1}{\alpha+\beta\xi}.$
Substituting $\xi$ from \eqref{eq:xi_of_t_raw} yields the claimed closed form
\[
s
=
-\frac{t(1-\psi\beta t)}{1+(1-\psi\beta)t-\psi\beta(1-\beta)t^2}.
\]
Recalling $\beta=b_\sigma^2$ proves \eqref{eq:T_def}.
Let $s=T^{-1}(t)$ given by \eqref{eq:T_def}. The above computation  is reversible:
from $t$ we obtain $\xi$ via \eqref{eq:xi_of_t_raw}, then $s$ via $s=-1/(\alpha+\beta\xi)$, and finally
$m_{{\rm MP},\psi}(\xi)=-\beta t$ by construction. Plugging into \eqref{eq:T_as_MP} gives
$T(s)=-(1/\beta)m_{{\rm MP},\psi}(\xi)=t$. This shows $T(T^{-1}(t))=t$.
\end{proof}

\subsection{Centered and Normalized GELU Function}\label{app:gelu}
Let $\xi\sim\cN(0,1)$ and denote the standard normal PDF and CDF by
\[
\varphi(x):=\frac{1}{\sqrt{2\pi}}e^{-x^2/2},
\qquad
\Phi(x):=\int_{-\infty}^x \varphi(t)\,dt.
\]
The  GELU activation function \citep{hendrycks2016gaussian} is defined by $g(x)=x\Phi(x)$.
Since $\E[\sigma(\xi)]=0$ and $\E[\sigma(\xi)^2]=1$, we construct a centered and normalized version of GELU $\sigma(x):=\frac{g(x)-\mu_g}{v_g}$ by $\mu_g:=\E[g(\xi)]$ and $v_g^2:=\Var(g(\xi))$.

\paragraph{Assumption \ref{assump:sigma} is satisfied for $\sigma(\xi)$.}
A direct calculation gives
\[
g'(x)=\Phi(x)+x\varphi(x),\qquad
g''(x)=(2-x^2)\varphi(x),\qquad
g^{(3)}(x)=(x^3-4x)\varphi(x).
\]
Since $0\le \Phi(x)\le 1$ and $|x|^k\varphi(x)$ is bounded for any fixed $k$,
we have $\sup_x\{|g'(x)|,|g''(x)|,|g^{(3)}(x)|\}<\infty$, hence the same holds for $\sigma$
(with an extra factor $1/v_g$). Therefore $\sigma$ satisfies Assumption~\ref{assump:sigma}.

\paragraph{Compute the centering constant $\mu_g=\E[g(\xi)]$.}
Using $\varphi'(x)=-x\varphi(x)$, we write $x\varphi(x)\,dx=-d\varphi(x)$ and integrate by parts to get $\mu_g=\int_{\R}\varphi(x)^2\,dx.$
Since $\varphi(x)^2=\frac{1}{2\pi}e^{-x^2}$, we obtain
\begin{equation}\label{eq:mu_gelu}
\mu_g=\int_{\R}\frac{1}{2\pi}e^{-x^2}\,dx=\frac{1}{2\pi}\sqrt{\pi}=\frac{1}{2\sqrt{\pi}}.
\end{equation}

\paragraph{Compute the variance $v_g^2=\Var(g(\xi))$.}
Let $
m_{2,g}:=\E[g(\xi)^2]=\int_{\R}x^2\Phi(x)^2\varphi(x)\,dx.$
Define the antiderivative $V(x):=\int_{-\infty}^x t^2\varphi(t)\,dt.$
A standard identity gives $V(x)=-x\varphi(x)+\Phi(x)$.
Integrating by parts yields
\begin{align*}
m_{2,g}
&=\Big[\Phi(x)^2V(x)\Big]_{-\infty}^{\infty}-\int_{\R}V(x)\,d(\Phi(x)^2)\\
&=1-2\int_{\R}V(x)\Phi(x)\varphi(x)\,dx
=1-2\int_{\R}(-x\varphi(x)+\Phi(x))\Phi(x)\varphi(x)\,dx\\
&=1+2\int_{\R}x\Phi(x)\varphi(x)^2\,dx-2\int_{\R}\Phi(x)^2\varphi(x)\,dx.
\end{align*}
The last term equals $\E[\Phi(\xi)^2]$. Since $U:=\Phi(\xi)\sim \mathrm{Unif}(0,1)$, we have
\begin{equation}\label{eq:Phi_sq}
\int_{\R}\Phi(x)^2\varphi(x)\,dx=\E[U^2]=\frac{1}{3}.
\end{equation}
With integration by parts, we can also get
\begin{equation}\label{eq:mixed_term}
\int_{\R}x\Phi(x)\varphi(x)^2\,dx=\frac{1}{2\pi}\int_{\R}x\Phi(x)e^{-x^2}\,dx=\frac{1}{4\pi\sqrt{3}}.
\end{equation}
Plugging \eqref{eq:Phi_sq}--\eqref{eq:mixed_term} into the expression for $m_{2,g}$ gives $m_{2,g}=\frac{1}{3}+\frac{1}{2\pi\sqrt{3}}.$
Finally, by \eqref{eq:mu_gelu}, $v_g^2=\Var(g(\xi))=m_{2,g}-\mu_g^2
=\frac{1}{3}+\frac{1}{2\pi\sqrt{3}}-\frac{1}{4\pi}.$

\paragraph{Compute $b_\sigma$ and $c_\sigma$.}
Since
\[
\sigma'(x)=\frac{g'(x)}{v_g}=\frac{\Phi(x)+x\varphi(x)}{v_g},
\qquad
\sigma''(x)=\frac{g''(x)}{v_g}=\frac{(2-x^2)\varphi(x)}{v_g}.
\]
Hence, using the standard Gaussian integrals,
\begin{align}
b_\sigma
&=\E[\sigma'(\xi)]
=\frac{1}{v_g}\E[\Phi(\xi)+\xi\varphi(\xi)]
=\frac{1}{v_g}\Big(\E[\Phi(\xi)]+\E[\xi\varphi(\xi)]\Big)
=\frac{1}{v_g}\cdot\frac{1}{2}
=\frac{1}{2v_g},
\label{eq:b_sigma_gelu}
\\[2mm]
c_\sigma
&=\E[\sigma''(\xi)]
=\frac{1}{v_g}\E[(2-\xi^2)\varphi(\xi)]
=\frac{1}{v_g}\int_{\R}(2-x^2)\varphi(x)^2\,dx=\frac{3}{4\sqrt{\pi}\,v_g}.
\label{eq:c_sigma_gelu_start}
\end{align}
Numerically, $b_\sigma \approx 0.85$ and $c_\sigma \approx 0.72.$
Thus GELU does not have uninformative spikes in Theorem \ref{thm:snr_finite_eig}(ii) since  \(\tau\le \tau_{\rm crit}\) and $\sqrt{\psi}\bigl( \frac{c_\sigma^2}{2} - b_\sigma^2 \bigr) \le b_\sigma^2.$

\section{Quadratic Equivalence for Proportional Limit}
\label{app:quad_de}
Let us recall the null model $\vY_0$ in \eqref{eq:null_model}
and QE model $\vY_{\mathrm{QE}}$ of the random feature matrix on XOR dataset defined in \eqref{eq:QE_formal}. Below we control the difference between $\vY_{\mathrm{QE}}$ and $\vY$.

\bigskip
\begin{proof-of-prop}[\ref{prop:approx}]
In this proof, we aim to prove a more precise statement:
\begin{equation}\label{eq:QDE_rate}
   \|\vY-\vY_{\mathrm{QE}}\|\ \prec\ n^{-1/4}. 
\end{equation}
Recall that $\vY=\frac{1}{\sqrt{N}}\sigma(\vW\vZ+\vW\vM)$ where $\vZ$ and $\vM$ are defined by \eqref{eq:def_M}. Then, we can take Taylor approximations for each entry of $\vY$ around $\vY_0$ to obtain
    \begin{align}\label{eq:decomp_Y}
        \vY=\vY_0+\frac{1}{\sqrt{N}}\sigma'(\vW\vZ)\odot(\vW\vM)+\frac{1}{2\sqrt{N}}\sigma''(\vW\vZ)\odot(\vW\vM)^{\odot 2} +\frac{1}{3!\sqrt{N}}\sigma'''(\vTheta)\odot(\vW\vM)^{\odot 3}
    \end{align}for some random matrix $\vTheta\in\R^{N\times n}$ such that $\Theta_{i,j}$ is between $[\vW\vZ]_{i,j}$ and $[\vW\vX]_{i,j}$ for all $i\in[N]$ and $j\in[n]$. Notice that 
    \begin{equation}
        \vW\vM = \theta_\snr(\vg_1\vv_1^\top+\vg_2\vv_2^\top),
    \end{equation} and $\vg_1,\vg_2\iid\cN(0,\ident_N)$. By the definitions of $\vv_1$ and $\vv_2\in\R^n$, we have
    \begin{equation}
        (\vW\vM)^{\odot 2} = \theta^2_\snr (\vg_1^{\odot 2}\vv_1^{\odot 2\top}+\vg_2^{\odot 2}\vv_2^{\odot 2\top} + (\vg_1\odot \vg_2)(\vv_1\odot\vv_2)^\top) = \theta^2_\snr (\vg_1^{\odot 2}\vv_1^{\odot 2\top}+\vg_2^{\odot 2}\vv_2^{\odot 2\top}).
    \end{equation}
Thus, from \eqref{eq:decomp_Y}, we have that 
    \begin{align}
        \vY=~&\vY_0 +\frac{\E[\sigma'(\frac{1}{\sqrt{d}}\vw^\top\vz)]\theta_\snr}{\sqrt{N}} (\vg_1\vv_1^\top+\vg_2\vv_2^\top) +\frac{\E[\sigma''(\frac{1}{\sqrt{d}}\vw^\top\vz)]\theta_\snr^2}{2\sqrt{N}} (\vg_1^{\odot 2}\vv_1^{\odot 2\top}+\vg_2^{\odot 2}\vv_2^{\odot 2\top})\\ 
        ~&+ \frac{1}{3!\sqrt{N}}\sigma'''(\vTheta)\odot(\vW\vM)^{\odot 3}
        +\frac{1}{\sqrt{N}}(\sigma'(\vW\vZ)-\E[\sigma'(\vW\vZ)])\odot(\vW\vM)\\
        ~&+\frac{1}{2\sqrt{N}}(\sigma''(\vW\vZ)-\E[\sigma''(\vW\vZ)])\odot(\vW\vM)^{\odot 2},\label{eq:decomp_Y1}
    \end{align}where $\vz,\vw\iid \cN(0,\ident_d)$.
    Next, we address the terms on the right-hand side of \eqref{eq:decomp_Y1} individually.

    \begin{enumerate}[label=(\alph*)]
        \item Since $\|\sigma'''\|_\infty = O(1)$, we have that 
        \begin{align}
            \|\frac{1}{3!\sqrt{N}}\sigma'''(\vTheta)\odot(\vW\vM)^{\odot 3}\|\lesssim ~&\theta^3\|\frac{1}{\sqrt{N}}\sigma'''(\vTheta)\|\sum_{i,j,k=1}^2 \|\vg_i\|_{\infty}\|\vg_j\|_{\infty}\|\vg_k\|_{\infty}\|\vv_i\|_{\infty}\|\vv_j\|_{\infty}\|\vv_k\|_{\infty}\\
            \lesssim ~&n^{-3/4}\|\frac{1}{\sqrt{N}}\sigma'''(\vTheta)\|_F\|\vg_1\|_{\infty}^3
            \prec \frac{1}{n^{1/4}},
        \end{align}
        where we apply Lemmas~\ref{lem:infinity_norm} and \ref{lem:hardamart_identity}.
        \item Next, we have that
        \begin{align}
            &\|\frac{1}{\sqrt{N}}(\sigma'(\vW\vZ)-\E[\sigma'(\vW\vZ)])\odot(\vW\vM)\|\\
            \le &~\|\frac{1}{\sqrt{N}}(\sigma'(\vW\vZ)-\E[\sigma'(\vW\vZ)])\|\cdot \theta_\snr\cdot (\|\vg_1\|_{\infty}\|\vv_1\|_{\infty}+\|\vg_2\|_{\infty}\|\vv_2\|_{\infty})\\
            \lesssim &~\frac{\sqrt{\log n}}{n^{1/4}} \|\frac{1}{\sqrt{N}}(\sigma'(\vW\vZ)-\E[\sigma'(\vW\vZ)])\|\prec \frac{1}{n^{1/4}}
        \end{align}
        where we utilize Lemma~\ref{lem:concentration_operator} in the last line since $\sigma'$ is Lipschitz.
        \item Analogously, we have that 
        \[\frac{1}{2\sqrt{N}}\|(\sigma''(\vW\vZ)-\E[\sigma''(\vW\vZ)])\odot(\vW\vM)^{\odot 2}\|\prec \frac{1}{\sqrt{n}}.\]
        \item  Lastly, we apply Lemma~\ref{lem:concentration_coeff} to obtain that 
            \begin{align}
        \vY=~&\vY_0 +\frac{\theta_\snr b_\sigma}{\sqrt{N}} (\vg_1\vv_1^\top+\vg_2\vv_2^\top) +\frac{\theta_\snr^2 c_\sigma}{2\sqrt{N}} (\vg_1^{\odot 2}\vv_1^{\odot 2\top}+\vg_2^{\odot 2}\vv_2^{\odot 2\top})+O_{\prec}(n^{-1/4}).
    \end{align}
    \end{enumerate}
Combining all these together,  we proved \eqref{eq:QDE_rate}.  Then, if $\|\vY\|,\|\vY_{\rm QE}\|\lesssim 1$, we can conclude that 
\[\|\vK-\vK_{\rm QE}\|\prec n^{-1/4}\]
and they have asymptotically the same outlier eigenvalues and eigenvector alignments.
\end{proof-of-prop}

\section{Proof for Finite-SNR and Proportional Limit Regime}\label{app:proof_finite_snr}

In this section, we present the proof for Theorem \ref{thm:snr_finite_eig} when we consider proportional limit and finite SNR regime. 
In this case, the rank-two XOR signal satisfies $\vM = \theta_{\snr}\,(\vu_1\vv_1^\top+\vu_2\vv_2^\top),$ where $\theta_{\snr}:=r\sqrt{\frac{n}{2d}}\to r\sqrt{\frac{\psi}{2}}.$
We know that both \(\vu_1,\vu_2\in\R^d\) and \(\vv_1,\vv_2\in\R^n\) are  deterministic orthonormal vectors
satisfying
\begin{equation}\label{eq:v_orth_u}
\vv_1^\top \vv_2=0,\qquad \|\vv_1\|=\|\vv_2\|=1,
\qquad \vu^\top\vv_1=\vu^\top\vv_2=0,
\qquad \vu:=\frac{\mbi{1}_n}{\sqrt n}.
\end{equation}
Recall that \(\vZ\in\R^{d\times n}\) has i.i.d.\ \(\cN(0,1/d)\) entries, and \(\vW\in\R^{N\times d}\) has i.i.d.\ \(\cN(0,1)\) entries,
independent of \(\vZ\). Recall the null feature matrix $\vY_0$ in \eqref{eq:null_model}
and its columnwise conditional mean $\vm := \E[\sigma(\vw^\top\vZ)\mid \vZ]\in\R^n,$
whose entries are 
\begin{equation}\label{eq:mean_vector_entry}
m_j = F(\|\vz_j\|^2),\quad
F(s):=\E_{\xi\sim\cN(0,1)}[\sigma(\sqrt{s}\,\xi)],\quad j\in[n].
\end{equation}
Here \(\vz_j\) denotes the \(j\)-th column of \(\vZ\), and \(\vw\sim\cN(0,\vI_d)\) independent of \(\vZ\). Notice that conditioning on $\vZ$, $\vY_0$ has independent and identically distributed rows. However, when $c_\sigma\neq 0$, the mean of rows in $\vY_0$, $\vm$, is not vanishing which will potentially give us some uninformative spike, see \citep{benigni2022largest,wang2024nonlinearspikedcovariancematrices}. Hence, we first need to subtract the mean vector $\vm$ to make all rows centered and to apply theorems in \citet{wang2024nonlinearspikedcovariancematrices}. Then we consider the effect of the mean vector $\vm$.
Define the centered null feature matrix by
\begin{equation}\label{eq:G_def}
\vG := \vY_0 - \frac{1}{\sqrt N}\mbi{1}_N\vm^\top
\end{equation}
and the null model of CK matrix by $\vK_0 := \vG^\top\vG.$


\subsection{Analysis of the Population Covariance Matrix}\label{subsec:Sigma_population_spike}

Conditional on \(\vZ\), the rows of \(\vG\) are i.i.d.\ centered vectors in \(\R^n\).
Let \(\bar\vg^\top:=\sigma(\vw^\top\vZ)-\vm^\top\) for \(\vw\sim\cN(0,\vI_d)\) independent of \(\vZ\), and define the conditional
population covariance
\begin{equation}\label{eq:Sigma_def}
\vSigma := \E[\bar\vg\,\bar\vg^\top\mid \vZ]\in\R^{n\times n}.
\end{equation}

\begin{lemma}[Approximation of \(\vSigma\)]\label{lem:Sigma_trunc}
Let \(\vS:=\vZ^\top\vZ\). Under Assumptions~\ref{assump:sigma} and \ref{assump:asymptotics}, we have
\begin{equation}\label{eq:Sigma_trunc}
\vSigma
=
(1-b_\sigma^2)\vI_n
+b_\sigma^2\,\vS
+\frac{c_\sigma^2}{2}\big(\vS^{\odot 2}-\vI_n\big)
+\vR_\Sigma,
\end{equation} such that $\|\vR_\Sigma\|=\widetilde O_{\P}(n^{-1/2}).$
\end{lemma}

\begin{proof}
With Lemma 5.2 of \citet{wang2021deformed} and Lemma \ref{lem:xor_orthonormal}, we have
\[
\vSigma
=
(1-b_\sigma^2)\vI_n
+b_\sigma^2\,\vS
+\frac{c_\sigma^2}{2}\big(\vS^{\odot 2}-\vI_n\big)+d_\sigma^2 (\vS^{\odot 3}-\vI_n)
+O_{\P}(n^{-1/2}),
\]
where $d_\sigma$ is the third-order Hermite coefficient of $\sigma$.
Note that under event of Lemma \ref{lem:xor_orthonormal},
\[
\|(\vZ^{\top}\vZ)^{\odot 3} - \ident_n \| \le \| \text{diag}\bigl( (\vZ^{\top}\vZ)^{\odot 3} - \ident_n \bigr) \| + \|\text{offdiag}\bigl(\vZ^{\top}\vZ \bigr)^{\odot 3} \|_F
\]
\begin{equation}
\le \max_{\alpha \in [n]} \bigl \lvert |\vz_\alpha |^{6}-1 \bigr \rvert + n \cdot \max_{\alpha \ne \beta \in [n]} |\vz_\alpha^\top \vz_\beta |^{3} \le C \bigl( \tau_n + n \tau_n^3 \bigr) \le 2C \sqrt{(\log^3 n)/n}
\label{eq:27_one}
\end{equation}
for some constant $C \ge 0$. This completes the proof. 
\end{proof}

\begin{lemma}\label{lem:hadamard_rankone}
Let \(\vS=\vZ^\top\vZ\) and \(\vu=\mbi{1}_n/\sqrt n\).
Then $\vS^{\odot 2}-\vI_n=\frac{n}{d}\,\vu\vu^\top+\vE_2$ with $\|\vE_2\|=o_\P(1).$
\end{lemma}

\begin{proof}
The proof is directly derived from the proof of Theorem 2.1 of \citet{el2010spectrum}, where $\vS^{\odot 2}$ is the quadratic term in the Taylor expansion of the inner-product kernel of $\vZ$.
\end{proof}

\begin{corollary}\label{cor:Sigma_decomp}
Define $\vSigma_0:=(1-b_\sigma^2)\vI_n+b_\sigma^2\vS,$ and $\vu:=\frac{\mbi{1}_n}{\sqrt n}.$ Let $\tau:=\frac{c_\sigma^2}{2}\frac{n}{d}\to \frac{c_\sigma^2}{2}\psi$.
Then
\begin{equation}\label{eq:Sigma_decomp}
\vSigma=\vSigma_0+\tau\,\vu\vu^\top+\vR,
\qquad
\|\vR\|=o_\P(1).
\end{equation}
Moreover the limiting ESD of \(\vSigma_0\) (and of \(\vSigma\)) equals
\begin{equation}\label{eq:nu_def}
\nu=(1-b_\sigma^2)+b_\sigma^2\rho^{\rm MP}_\psi.
\end{equation}
Thus, the right edge of $\nu$ is $
\lambda_+^\nu=(1-b_\sigma^2)+b_\sigma^2(1+\sqrt\psi)^2$.
\end{corollary}

\begin{proof}
Combine Lemmas~\ref{lem:Sigma_trunc} and~\ref{lem:hadamard_rankone} and absorb all \(o_\P(1)\) terms into \(\vR\).
The rank-one term \(\tau\vu\vu^\top\) does not change the limiting ESD. The limiting ESD of
\(\vSigma_0=(1-b_\sigma^2)\vI+b_\sigma^2\vZ^\top\vZ\) is the shifted MP law \(\nu\) in \eqref{eq:nu_def}.
\end{proof}

\subsection{Analysis of Spike in Population Covariance}\label{subsec:Sigma_pop_spike}

\begin{lemma}[Population outlier from \(\vSigma_0+\tau\vu\vu^\top\)]\label{lem:Sigma_pop_equation}
Let \(\vA:=\vSigma_0\) and consider \(\vA+\tau\vu\vu^\top\).
For $\tau\ge 0$, any deterministic limit point \(\Lambda_\tau>\lambda_+^\nu\) of an isolated eigenvalue of \(\vA+\tau\vu\vu^\top\)
must satisfy
\begin{equation}\label{eq:pop_secular_eq}
1+\tau\,m_\nu(\Lambda_\tau)=0,
\end{equation}
where \(m_\nu\) is the Stieltjes transform of \(\nu\).
Moreover, such a solution exists iff
\begin{equation}\label{eq:tau_crit_def}
\tau>\tau_{\rm crit}:=b_\sigma^2\sqrt\psi(1+\sqrt\psi).
\end{equation}
\end{lemma}

\begin{proof} 
For any \(\lambda\notin\spec(\vA)\), the rank-one perturbation identity gives
\begin{equation}\label{eq:rankone_secular_eq}
\lambda\in\spec(\vA+\tau\vu\vu^\top)\setminus\spec(\vA)
\quad\Longleftrightarrow\quad
1+\tau\,\vu^\top(\vA-\lambda\vI)^{-1}\vu=0.
\end{equation} 
Write \(\vA=(1-b_\sigma^2)\vI+b_\sigma^2\vS\) with \(\vS=\vZ^\top\vZ\).
For \(\lambda>\lambda_+^\nu+\kappa\) (fixed \(\kappa>0\)), define
$
\zeta:=\frac{\lambda-(1-b_\sigma^2)}{b_\sigma^2}>(1+\sqrt\psi)^2+\kappa',
$
so that $(\vA-\lambda\vI)^{-1}=\frac{1}{b_\sigma^2}\,(\vS-\zeta\vI)^{-1}.$
The isotropic MP local law \citep{bloemendal2014isotropic} yields, uniformly for such \(\zeta\),
$
\vu^\top(\vS-\zeta\vI)^{-1}\vu
=
m_{{\rm MP},\psi}(\zeta)+o_\P(1),
$
hence
\begin{equation}\label{eq:u_resolvent_limit}
\vu^\top(\vA-\lambda\vI)^{-1}\vu
=
\frac{1}{b_\sigma^2}m_{{\rm MP},\psi}(\zeta)+o_\P(1)
=
m_\nu(\lambda)+o_\P(1).
\end{equation}
Define \(F(\lambda):=1+\tau m_\nu(\lambda)\) for \(\lambda>\lambda_+^\nu\).
Since \(m_\nu(\lambda)<0\) and is strictly increasing on \((\lambda_+^\nu,\infty)\),
\(F\) is strictly increasing and \(\lim_{\lambda\to\infty}F(\lambda)=1\).
Thus a solution exists iff \(F(\lambda_+^\nu)<0\), i.e.
\(\tau>-1/m_\nu(\lambda_+^\nu)\). At the MP edge,
\(
m_{{\rm MP},\psi}((1+\sqrt\psi)^2)=-\frac{1}{\sqrt\psi(1+\sqrt\psi)}.
\)
Using \(m_\nu(\lambda)=\frac1{b_\sigma^2}m_{{\rm MP},\psi}(\zeta)\), at \(\lambda=\lambda_+^\nu\) we get $m_\nu(\lambda_+^\nu)=-\frac{1}{b_\sigma^2\sqrt\psi(1+\sqrt\psi)}.$
Hence the critical \(\tau\) is
\(
\tau_{\rm crit}=-1/m_\nu(\lambda_+^\nu)=b_\sigma^2\sqrt\psi(1+\sqrt\psi),
\)
proving \eqref{eq:tau_crit_def}.
\end{proof}

\begin{lemma}[Closed form for \(\Lambda_\tau\) when \(\nu\) is shifted MP]\label{lem:Lambda_tau_closed_form}
Let \(\nu=(1-b_\sigma^2)+b_\sigma^2\rho_\psi^{\rm MP}\) and assume \(\tau>\tau_{\rm crit}\).
Then the unique solution \(\Lambda_\tau>\lambda_+^\nu\) of \(1+\tau m_\nu(\Lambda_\tau)=0\) equals
\begin{equation}\label{eq:Lambda_tau}
\Lambda_\tau
=
(1-b_\sigma^2)+\tau+\frac{b_\sigma^2\,\tau}{\tau-b_\sigma^2\psi}.
\end{equation}
\end{lemma}

\begin{proof}
Write \(\Lambda=(1-b_\sigma^2)+b_\sigma^2\zeta\). Then \(m_\nu(\Lambda)=\frac1{b_\sigma^2}m_{{\rm MP},\psi}(\zeta)\),
and  
\begin{equation}\label{eq:MP_target_m}
m_{{\rm MP},\psi}(\zeta)=-\frac{b_\sigma^2}{\tau}.
\end{equation}
The Stieltjes transform \(m=m_{{\rm MP},\psi}(\zeta)\) satisfies the quadratic equation
\begin{equation}\label{eq:MP_quad}
\psi\,\zeta\,m^2+(\zeta+\psi-1)m+1=0.
\end{equation}
Solving \eqref{eq:MP_quad} for \(\zeta\) in terms of \(m\) yields
\begin{equation}\label{eq:zeta_in_m}
\zeta=-\frac{1+(\psi-1)m}{m(1+\psi m)}=-\frac{1}{m}+\frac{1}{1+\psi m}.
\end{equation}
Substitute \(m=-b_\sigma^2/\tau\) from \eqref{eq:MP_target_m} into \eqref{eq:zeta_in_m}. Therefore, we obtain \eqref{eq:Lambda_tau}. Uniqueness follows from strict monotonicity of \(m_\nu(\lambda)\) on \((\lambda_+^\nu,\infty)\).
\end{proof}

\begin{corollary}[Subcritical case: no population outlier]\label{cor:no_pop_outlier}
If \(\tau\le\tau_{\rm crit}\), then \(\vSigma\) has no isolated eigenvalue outside \(\supp{\nu}\) and
\(\lambda_{\max}(\vSigma)=\lambda_+^\nu+o_\P(1)\).
\end{corollary}

\begin{proof}
If \(\tau\le\tau_{\rm crit}\), then Lemma~\ref{lem:Sigma_pop_equation} shows the limiting secular equation has no root
in \((\lambda_+^\nu,\infty)\). Using \eqref{eq:rankone_secular_eq}--\eqref{eq:u_resolvent_limit} and \citet{bai1998no},
no separated eigenvalue exists and \(\lambda_{\max}(\vSigma)=\lambda_+^\nu+o_\P(1)\).
\end{proof}

\begin{lemma}\label{lem:uninf_branch_identity}
Recall the shifted MP law $\nu=(1-b_\sigma^2)+b_\sigma^2\rho_\psi^{\rm MP}$ and define interval $I_{\rm un}:=\Bigl(-\frac1{\lambda_+^\nu},0\Bigr)$ where $
\lambda_+^\nu=(1-b_\sigma^2)+b_\sigma^2(1+\sqrt\psi)^2.$
Recall $\tau$ and $\tau_{\rm crit}$ defined in \eqref{eq:spikes}, and $\Lambda_\tau$ in \eqref{eq:Lambda_tau}.
Then the following hold.
\begin{enumerate}[label=\textbf{(\roman*)},leftmargin=*]
\item For every $\Lambda>\lambda_+^\nu$,
\begin{equation}\label{eq:T_mnu_identity_patch}
T\!\left(-\frac1\Lambda\right)=-m_\nu(\Lambda).
\end{equation}
\item The map $T$ is strictly decreasing on $I_{\rm un}$ and maps $I_{\rm un}$ bijectively onto $(0,1/\tau_{\rm crit})$.
\item For $\tau>\tau_{\rm crit}$ there exists a unique $s_{\rm un}\in I_{\rm un}$ such that $T(s_{\rm un})=\frac1\tau.$
Moreover, $s_{\rm un}=-\frac1{\Lambda_\tau}.$
Equivalently, on the uninformative branch one has the exact identity
\begin{equation}\label{eq:s_mean_s_cov_patch}
s_{\rm mean}:=T^{-1}(1/\tau)=s_{\rm un}=s_{\rm cov}:=-1/\Lambda_\tau.
\end{equation}
\item If $s_{\rm un}\in I_{\rm un}$ and $z'(s_{\rm un})>0$, then necessarily $\tau>\tau_{\rm crit}$.
\end{enumerate}
\end{lemma}

\begin{proof}
For $s\in I_{\rm un}$ write $\Lambda:=-1/s$. Then $\Lambda>\lambda_+^\nu$ and
\[
T(s)=-s+s^2\int \frac{\lambda}{1+\lambda s}\,\nu(d\lambda)
=-s\int \frac{1}{1+\lambda s}\,\nu(d\lambda)
\] is well-defined.
Substituting $s=-1/\Lambda$ gives
\[
T\!\left(-\frac1\Lambda\right)
=\frac1\Lambda\int \frac{1}{1-\lambda/\Lambda}\,\nu(d\lambda)
=\int \frac{1}{\Lambda-\lambda}\,\nu(d\lambda)
=-m_\nu(\Lambda),
\]
which proves \eqref{eq:T_mnu_identity_patch}. Since $m_\nu(\Lambda)$ is strictly increasing and negative on $(\lambda_+^\nu,\infty)$, the map $-m_\nu(\Lambda)$ is strictly decreasing in $\Lambda$; equivalently, $T$ is strictly decreasing in $s\in I_{\rm un}$. Moreover,
\[
\lim_{\Lambda\to\infty} -m_\nu(\Lambda)=0,
\qquad
\lim_{\Lambda\downarrow\lambda_+^\nu} -m_\nu(\Lambda)=-m_\nu(\lambda_+^\nu)=\frac1{\tau_{\rm crit}},
\]
where the last identity is exactly the edge computation in Lemma~\ref{lem:Sigma_pop_equation}. This proves item \textbf{(ii)}.

Item \textbf{(iii)} now follows immediately: $T(s)=1/\tau$ has a unique solution $s\in I_{\rm un}$ iff $1/\tau\in(0,1/\tau_{\rm crit})$, i.e. iff $\tau>\tau_{\rm crit}$. For such $\tau$, let $s_{\rm un}$ be the unique solution. Setting $\Lambda_\tau:=-1/s_{\rm un}$ and using \eqref{eq:T_mnu_identity_patch}, $\frac1\tau=T(s_{\rm un})=-m_\nu(\Lambda_\tau),$
so $1+\tau m_\nu(\Lambda_\tau)=0$. By Lemma~\ref{lem:Lambda_tau_closed_form}, the unique solution of this secular equation above $\lambda_+^\nu$ is precisely $\Lambda_\tau$ from \eqref{eq:Lambda_tau}, hence $s_{\rm un}=-1/\Lambda_\tau$. Finally, item \textbf{(iv)} is immediate from item \textbf{(iii)}: existence of $s_{\rm un}\in I_{\rm un}$ already forces $\tau>\tau_{\rm crit}$.
\end{proof}

\subsection{BBP Transition for the Null Model}\label{subsec:K0_tau}

Let \(z,\varphi\) be defined from \((\nu,\phi)\) by \eqref{eq:z(s)} where $\nu$ is defiend by \eqref{eq:nu_def}. In this section, we analyze the uninformative spike of $\vK_0$ derived from the potential spike in the population covariance matrix $\vSigma$.

\begin{lemma}\label{lem:K0_bulk_and_tau_outlier}
Let \(\vK_0=\vG^\top\vG\). The following hold.
\begin{enumerate}[label=\textbf{(\roman*)},leftmargin=*]
\item \textbf{(Bulk.)} The ESD of \(\vK_0\) converges weakly in probability to \(\mu=\rho_{\phi}^{\rm MP}\boxtimes\nu\).

\item \textbf{(BBP for covariance outlier.)} 
If \(\tau>\tau_{\rm crit}\), let \(\Lambda_\tau\) be the population outlier of \(\vSigma\) in \eqref{eq:Lambda_tau}.
Then:
\[
\text{If additionally } z'\!\Big(-\frac{1}{\Lambda_\tau}\Big)>0,\quad
\vK_0 \text{ has a separated outlier at }\ \widehat\lambda_\tau=z\!\Big(-\frac{1}{\Lambda_\tau}\Big)+o_\P(1),
\]
where \(\widehat\lambda_\tau\) lies in a connected component of \((0,\infty)\setminus\supp{\mu}\).
If \(z'(-1/\Lambda_\tau)\le 0\) or \(\tau\le\tau_{\rm crit}\), the corresponding eigenvalue sticks to \(\supp{\mu}\) and \(\vK_0\) has no outlier.

\item \textbf{(Eigenvector alignment.)} In the above separated case, if \(\widehat\vv_\tau\) is a unit eigenvector of \(\vK_0\) associated to \(\widehat\lambda_\tau\), then $ |\widehat\vv_\tau^\top \vu|^2 \ \xrightarrow{\P}\ \varphi\!\Big(-\frac{1}{\Lambda_\tau}\Big).$
In general, any fixed unit vectors $\vv\in\R^n$ (potentially $\vZ$ dependent), 
\begin{equation}\label{eq:alignment_right}
    |\widehat\vv_\tau^\top \vv|^2 - \varphi\!\Big(-\frac{1}{\Lambda_\tau}\Big)|\vu^\top \vv|^2\ \xrightarrow{\P}\ 0.
\end{equation}
\end{enumerate}
\end{lemma}

\begin{proof}
This lemma is an application of Theorem 12 of \citet{wang2024nonlinearspikedcovariancematrices}. It suffices to verify Assumptions 5 and 6 of \citet{wang2024nonlinearspikedcovariancematrices}. Assumption 6 is verified by Lemmas \ref{lem:Sigma_pop_equation} and \ref{lem:Lambda_tau_closed_form}, and Corollary \ref{cor:no_pop_outlier}: \(\vSigma\) has bulk law \(\nu\) and at most one population spike \(\Lambda_\tau\)
(in the \(\vu\) direction), present iff \(\tau>\tau_{\rm crit}\). Notice that from Corollary 3.5 of \citet{wang2021deformed}, we have 
\[\bar\vg^\top\vA\bar\vg -\Tr\vA\vSigma\prec \|\vA\|_F,\]
for any deterministic matrix $\vA\in\C^{n\times n}$.
Then, Assumption 5 can be verified based on Lemmas 32, 33, and 34 in \citet{wang2024nonlinearspikedcovariancematrices}. We ignore the details here for simplicity.
\end{proof}


\subsection{BBP Transition for Rank-three Additive Deformation}\label{subsec:QDE_rank3}

First, we prove that, for finite SNR \(r=\Theta(1)\), the explicit quadratic term in Proposition \ref{prop:approx} is negligible.

\begin{lemma}[Quadratic spike is negligible]\label{lem:T2_negligible}
Assume \(r=\Theta(1)\) and the assumptions of Proposition \ref{prop:approx}.
Recall
\[
\vT_2=\frac{\theta_{\snr}^2c_\sigma}{\sqrt N}\big(\vg_1^{\odot 2}\vv_1^{\odot 2\top}+\vg_2^{\odot 2}\vv_2^{\odot 2\top}\big).
\]
Then \(\|\vT_2\|=O_\prec(n^{-1/2})\).
Furthermore, \(\|\vY_{\rm QE}\|\prec 1\) and
\[
\big\|(\vY_{\rm QE}-\vT_2)^\top(\vY_{\rm QE}-\vT_2)-\vY_{\rm QE}^\top\vY_{\rm QE}\big\|\prec n^{-1/2}.
\]
\end{lemma}

\begin{proof}
We bound
\[
\|\vT_2\|
\le
\frac{|\theta_{\snr}^2c_\sigma|}{\sqrt N}\sum_{\ell=1}^2 \|\vg_\ell^{\odot 2}\|\,\|\vv_\ell^{\odot 2}\|.
\]
Since \(\vg_\ell\sim\cN(0,\vI_N)\), \(\|\vg_\ell^{\odot 2}\|^2=\sum_{i=1}^N g_{\ell,i}^4=O_\prec(N)\), hence \(\|\vg_\ell^{\odot 2}\|=O_\prec(\sqrt N)\).
Also \(\|\vv_\ell^{\odot 2}\|^2=\sum_{j=1}^n v_{\ell,j}^4=2/n\) for XOR sign vectors, so \(\|\vv_\ell^{\odot 2}\|=\sqrt{2/n}\).
Therefore \(\|\vT_2\|=O_\prec(n^{-1/2})\). 

By Lemmas \ref{lem:expectation} and \ref{lem:concentration_operator}, we have $\sigma(\vW\vZ)/\sqrt{N}\prec 1$. And since $\vg_1,\vg_2\prec \sqrt{N}$ and $\|\vv_1\|=\|\vv_2\|=1$, we have that $\vY_{\rm QE}\prec 1$ by the definition of $\vY_{\rm QE}$ in \eqref{eq:QE_formal}. The Gram difference bound follows by expansion:
\[
(\vY_{\rm QE}-\vT_2)^\top(\vY_{\rm QE}-\vT_2)-\vY_{\rm QE}^\top\vY_{\rm QE}
=
-\vY_{\rm QE}^\top\vT_2-\vT_2^\top\vY_{\rm QE}+\vT_2^\top\vT_2,
\]
hence, applying \(\|\vY_{\rm QE}\|\prec 1\) and \(\|\vT_2\|=O_\prec(n^{-1/2})\) we can get the final result.
\end{proof}

In the next proposition, we show that, in addition to the potential spike in $\vG$ characterized in Lemma~\ref{lem:K0_bulk_and_tau_outlier}, the matrix $\vY$ may exhibit further potential spikes arising from a rank-three additive perturbation of $\vG$.
\begin{proposition}[Rank-three additive deformation of \(\vG\)]\label{prop:Y_rank3_decomp}
Assume \(r=\Theta(1)\).
Define
\begin{equation}\label{eq:Ysharp_def}
\vY^\sharp
:=
\vG
+\underbrace{\frac{1}{\sqrt N}\mbi{1}_N\vm^\top}_{\text{mean spike}}
+\underbrace{\frac{\theta_{\snr}b_\sigma}{\sqrt N}\big(\vg_1\vv_1^\top+\vg_2\vv_2^\top\big)}_{\text{two XOR linear spikes}}.
\end{equation}
Denote a new kernel matrix by $\vK^\sharp:=\vY^{\sharp\top}\vY^\sharp.$ Then $\|\vY-\vY^\sharp\|\prec n^{-1/4},$ and hence $\|\vK-\vK^\sharp\|\prec n^{-1/4}.$
In particular, all the isolated eigenvalues and their eigenvectors of \(\vK\) and \(\vK^\sharp\) are asymptotically the same.
\end{proposition}

\begin{proof}
By the proof of Proposition~\ref{prop:approx} and Lemma~\ref{lem:T2_negligible}, in finite SNR, we can further drop \(\vT_2\) to get \(\|\vY-\vY^\sharp\|\prec n^{-1/4}\).
From the proof of Lemma~\ref{lem:T2_negligible}, we also know that \(\|\vY\|,\|\vY_{\rm QE}\|,\|\vY^\sharp\|\prec 1\). Hence we have
\[
\|\vK-\vK^\sharp\|
=
\|\vY^\top\vY-\vY^{\sharp\top}\vY^\sharp\|
\le
\|\vY-\vY^\sharp\|(\|\vY\|+\|\vY^\sharp\|)
\prec n^{-1/4},
\]
which implies asymptotic matching of isolated eigenvalues/eigenspaces by Weyl's inequality and Davis--Kahan Theorem \citep{davis1970rotation,wedin1972perturbation}.
\end{proof}

 
\subsubsection{Orthogonal Decomposition of Weight}\label{subsec:W_decomp_indep}
Next, to analyze two XOR linear spikes in \eqref{eq:Ysharp_def}, we need to make bulk matrix $\vG$ independent with the vectors $\vg_1,\vg_2$ by the following decomposition.

\begin{lemma}[Orthogonal decomposition and independence]\label{lem:W_decomp}
Let \(\vP_\perp:=\vI_d-\vu_1\vu_1^\top-\vu_2\vu_2^\top\) and define
\[
\vW_\perp:=\vW\vP_\perp,
\qquad
\vg_1 =\vW\vu_1,
\qquad
\vg_2 =\vW\vu_2.
\] 
\begin{enumerate}[label=\textbf{(\roman*)},leftmargin=*]
\item \(\vW=\vW_\perp+\vg_1\vu_1^\top+\vg_2\vu_2^\top\).
\item \(\vg_1,\vg_2\) are independent vectors with distribution \(\cN(0,\vI_N)\), and are independent of \(\vW_\perp\).
\end{enumerate}
\end{lemma}

\begin{proof}
For each row \(\vw_i^\top\) of \(\vW\),  \(\vw_i=(\vw_i^\top\vu_1)\vu_1+(\vw_i^\top\vu_2)\vu_2+\vP_\perp\vw_i\in\R^d\).
Stacking over \(i\) yields \(\vW=\vW_\perp+\vg_1\vu_1^\top+\vg_2\vu_2^\top\).
Since \(\vw_i\sim\cN(0,\vI_d)\), the scalar components \(\vw_i^\top\vu_1\), \(\vw_i^\top\vu_2\), and the vector \(\vP_\perp\vw_i\)
are independent Gaussians because they are orthogonal linear functionals of a Gaussian vector.
Stacking over \(i\) gives independence of \(\vg_1,\vg_2\) and \(\vW_\perp\).
\end{proof}

\begin{lemma}[The null bulk is asymptotically independent of \(\vg_1,\vg_2\)]\label{lem:G_indep_g}
For $\widetilde{\vw}\sim\cN(0,\vP_\perp)$, define
\[
\widetilde{\vY}_0:=\frac{1}{\sqrt N}\sigma(\vW_\perp\vZ),
\qquad
\widetilde{\vm}:=\E[\sigma(\widetilde{\vw}^\top\vZ)\mid \vZ], 
\qquad
\widetilde{\vG}:=\widetilde{\vY}_0-\frac{1}{\sqrt N}\mbi{1}_N\widetilde{\vm}^\top.
\]
Then \(\widetilde{\vG}\) is independent of \((\vg_1,\vg_2)\) and $\|\vm-\widetilde{\vm}\|, \|\vG-\widetilde{\vG}\|\prec 1/\sqrt{N}.$
\end{lemma}

\begin{proof}
By Lemma~\ref{lem:W_decomp}, we have the decomposition
$
\vW\vZ = \vW_\perp\vZ + \vg_1(\vu_1^\top\vZ)+\vg_2(\vu_2^\top\vZ).
$
Let \(\vt_1^\top:=\vu_1^\top\vZ\) and \(\vt_2^\top:=\vu_2^\top\vZ\). Each \(\vt_k\in\R^n\) has i.i.d.\ \(\cN(0,1/d)\) entries and is independent of \(\vW\).
Set \(\vA:=\vW_\perp\vZ\) and \(\vDelta:=\vg_1\vt_1^\top+\vg_2\vt_2^\top\), so \(\sigma(\vW\vZ)=\sigma(\vA+\vDelta)\) and \(\sigma(\vW_\perp\vZ)=\sigma(\vA)\).

Following the same proof of Proposition \ref{prop:approx}, we can apply
entrywisely the following Taylor approximation, with \(a=A_{ij}\), \(\delta=\Delta_{ij}\), bounded \(\sigma^{(3)}\) gives
$
\sigma(a+\delta)=\sigma(a)+\sigma'(a)\delta+\frac12\sigma''(\xi)\delta^2 
$
for some $\xi$ between $a$ and $a+\delta$.
Therefore
\[
\vY_0-\widetilde{\vY}_0
=\frac{1}{\sqrt N}\big(\sigma(\vA+\vDelta)-\sigma(\vA)\big)
=
\widetilde\vT_1+\widetilde\vT_2,
\]
where $\widetilde\vT_1:=\frac{1}{\sqrt N}\big(\sigma'(\vA)\odot\vDelta\big)$ and $\widetilde\vT_2:=\frac{1}{2\sqrt N}\big(\sigma''(\vXi)\odot\vDelta^{\odot 2}\big).$
First, we have
\[
\|\widetilde\vT_1\|
\le
\frac{1}{\sqrt N}\sum_{i=1}^2\|\diag(\vg_i)\|\,\|\sigma'(\vA)\|\,\|\diag(\vt_i)\|\prec 1/\sqrt{n},
\]where we apply Lemmas \ref{lem:infinity_norm} and \ref{lem:hardamart_identity}, and  
\(\|\sigma'(\vA)\|=O_\prec(\sqrt N)\) because \(\sigma'\) is bounded and \(\vA\) has i.i.d.\ bounded sub-Gaussian norm of rows.
Similarly we have \(\|\widetilde\vT_2\|\prec 1/\sqrt{n}\) since \(\|\sigma';(\vXi)\|=O_\prec(1)\).
A parallel argument applies to the conditional mean vectors, giving \(\|\vm-\widetilde{\vm}\|\prec 1/\sqrt{n}\) since \(\|\vm-\widetilde{\vm}\|\le \E[\|\vY_0-\widetilde{\vY}_0\||\vZ]\).
Therefore we complete the proof.
\end{proof}


\subsubsection{Mean Spike Strength and Orthogonality}\label{subsec:spike_strengths}

 \begin{lemma}[Mean spike strength and orthogonality]\label{lem:mean_strength_orth}
Under the Assumptions \ref{assump:sigma} and \ref{assump:asymptotics}, we have
\begin{equation}\label{eq:mean_norm_limit}
\|\vm\|^2\xrightarrow{\P}\beta_{\rm mean}:=\frac{c_\sigma^2}{2}\psi,
\end{equation}
and
\begin{equation}\label{eq:m_orthogonality}
\vu^\top\vm=o_\P(1),
\qquad
\vv_1^\top\vm=o_\P(1),
\qquad
\vv_2^\top\vm=o_\P(1).
\end{equation}
\end{lemma}

\begin{proof}
Fix $j\in[n]$ and condition on $\vz_j$. Since $\vw\sim\cN(0,\vI_d)$ is independent of $\vz_j$,
\(
\vw^\top \vz_j \sim \cN(0,\|\vz_j\|^2).
\)
Define for $s>0$ $F(s):=\E_{\xi\sim\cN(0,1)}\big[\sigma(\sqrt{s}\,\xi)\big],$
so that $m_j=F(\|\vz_j\|^2)$. By $\E[\sigma(\xi)]=0$ we have $F(1)=0$.

\smallskip 
Under Assumption~\ref{assump:sigma}, $F$ is $C^2$ on a neighborhood of $1$
and we can differentiate under the expectation
\[
F'(s)=\frac{1}{2\sqrt{s}}\E\!\left[\sigma'(\sqrt{s}\xi)\,\xi\right].
\]
Hence, using Stein's identity $\E[\xi f(\xi)]=\E[f'(\xi)]$ with $f=\sigma'$, $F'(1)=\frac12\E[\sigma'(\xi)\xi]=\frac12\E[\sigma''(\xi)]=\frac{c_\sigma}{2}.$
Therefore, for $s$ in a fixed neighborhood of $1$, Taylor expansion of $F$ at $s=1$ gives us
\begin{equation}\label{eq:F_taylor}
F(s)=\frac{c_\sigma}{2}(s-1) + R(s),
\qquad
|R(s)| \le C (s-1)^2,
\end{equation}
for some constant $C$ depending only on $\sup_{t\in[1-\delta,1+\delta]}|F''(t)|$.

Then let $\Delta_j:=\|\vz_j\|^2-1$. By \eqref{eq:mean_vector_entry} and \eqref{eq:F_taylor}, $m_j=\frac{c_\sigma}{2}\Delta_j + R_j,$ where $|R_j|\le C\Delta_j^2$ on the event $\|\vz_j\|^2\in[1-\delta,1+\delta]$, which holds w.h.p.\ since $\|\vz_j\|^2=(1/d)\chi_d^2$.
Then
\[
m_j^2=\frac{c_\sigma^2}{4}\Delta_j^2 + \underbrace{\frac{c_\sigma}{1}\Delta_j R_j + R_j^2}_{=:e_j},
\qquad |e_j|\le C\big(|\Delta_j|^3+\Delta_j^4\big).
\]
Summing over $j$ gives
\begin{equation}\label{eq:m_norm_expand}
\|\vm\|^2
=
\frac{c_\sigma^2}{4}\sum_{j=1}^n \Delta_j^2
+\sum_{j=1}^n e_j.
\end{equation}
Since $\Delta_j=(\chi_d^2-d)/d$ has moments $\E|\Delta_j|^3=O(d^{-3/2})$ and $\E[\Delta_j^4]=O(d^{-2})$,
and $n/d\to\psi$, we have
\[
\E\Big[\sum_{j=1}^n |e_j|\Big]
\le C n\Big(\E|\Delta_j|^3+\E[\Delta_j^4]\Big)
=O\Big(\frac{n}{d^{3/2}}+\frac{n}{d^2}\Big)
=o(1),
\]
hence $\sum_{j=1}^n e_j=o_\P(1)$ by Markov inequality. Thus \eqref{eq:m_norm_expand} reduces to
\begin{equation}\label{eq:m_norm_reduce}
\|\vm\|^2
=
\frac{c_\sigma^2}{4}\sum_{j=1}^n \Delta_j^2
+o_\P(1).
\end{equation}
Lastly, we give the concentration of $\sum \Delta_j^2$.
Because $\vz_j\sim\cN(0,\vI_d/d)$, we have $\Delta_j=(\chi_{d,j}^2-d)/d$ with i.i.d.\ $\chi_{d,j}^2$.
Then
\[
\E[\Delta_j^2]=\Var(\|\vz_j\|^2)=\frac{2}{d},
\qquad
\Var(\Delta_j^2)=O(d^{-2}).
\]
Therefore
$
\Var\Big(\sum_{j=1}^n \Delta_j^2\Big)=O\Big(\frac1d\Big)\to0,
$
so the Chebyshev inequality implies $\sum_{j=1}^n \Delta_j^2=\frac{2n}{d}+o_\P(1).$
Plugging into \eqref{eq:m_norm_reduce} yields \eqref{eq:mean_norm_limit}.
 
Let $\vv\in\R^n$ be any deterministic unit vector (in particular $\vu,\vv_1,\vv_2$). Since $m_j$ are independent and
$\Var(m_j)=O(1/d)$ (from $m_j=(c_\sigma/2)\Delta_j+O(\Delta_j^2)$ and $\Var(\Delta_j)=2/d$), we have $\Var(\vv^\top\vm)= O(1/d)\to0.$ Also $\E[m_j]=O(1/d)$ because $\E[\Delta_j]=0$ and the remainder is $O(\Delta_j^2)$. Hence
\[
|\E[\vv^\top\vm]|
=
\Big|\sum_{j=1}^n v_j\E[m_j]\Big|
\le \|\vv\|_1\,O(1/d)\le \sqrt{n}\,O(1/d)=O(d^{-1/2})\to0.
\]
Thus $\vv^\top\vm\to0$ in $L^2$ and hence in probability. Taking $\vv=\vu,\vv_1,\vv_2$ yields \eqref{eq:m_orthogonality}.
\end{proof}
 
\begin{lemma}\label{lem:g_norms}
Let \(\vg_k=\vW\vu_k\) for fixed orthonormal \(\vu_1,\vu_2\).
Then
\[
\frac{\|\vg_k\|^2}{N}-1\prec \frac{1}{\sqrt{N}},\qquad \frac{\vg_1^\top\vg_2}{N} \prec\frac{1}{\sqrt{N}}.
\]
Consequently, the two XOR additive spikes in \eqref{eq:Ysharp_def} have asymptotic strength $\theta_{\snr}^2 b_\sigma^2 \to r^2 b_\sigma^2\psi/2=\beta_{\rm lin}.$
\end{lemma}
\begin{proof}
Each \(\vg_k\sim\cN(0,\vI_N)\) an d \((\vg_1,\vg_2)\) is jointly Gaussian with \(\Cov(\vg_1,\vg_2)=\mathbf{0}\). Therefore, we can directly apply Proposition 2.6.6 and Theorem 3.1.1 of \citet{vershynin2018high} to get conclusion. 
\end{proof}


\subsubsection{Resolvent Quadratic Forms for the Spike Analysis}\label{subsec:resolvent_forms}

Fix any compact set \(\mathcal D\subset\R\) such that \(\mathcal D\cap\supp{\mu}=\emptyset\).
For \(\lambda\in\mathcal D\), define \(s(\lambda)\in\R\) by the unique relation
\begin{equation}\label{eq:s_lambda_def}
\lambda=z(s(\lambda)),\qquad z'(s(\lambda))>0,
\end{equation}
which is well-defined by Lemma~\ref{lem:z_parametrizes_complement}.
Define the left and right resolvents
\begin{equation}
\vQ_L(\lambda):=(\vG\vG^\top-\lambda\vI_N)^{-1},
\qquad
\vQ_R(\lambda):=(\vK_0-\lambda\vI_n)^{-1}.
\end{equation}

\begin{lemma}[Left resolvent limits for quadratic forms]\label{lem:left_resolvent}
Uniformly for \(\lambda\in\mathcal D\subset\R\setminus\{0\}\),
\begin{align}
\frac{1}{N}\mbi{1}_N^\top\vQ_L(\lambda)\mbi{1}_N &= s(\lambda)+o_\P(1),\label{eq:left_ones}\\
\frac{1}{N}\vg_k^\top\vQ_L(\lambda)\vg_\ell &= \delta_{k\ell}\,s(\lambda)+o_\P(1),\label{eq:left_gkgl}\\
\frac{1}{N}\mbi{1}_N^\top\vQ_L(\lambda)\vg_k &= o_\P(1).\label{eq:left_cross}
\end{align}
\end{lemma}

\begin{proof}
First, for \eqref{eq:left_ones}, 
we can directly utilize Theorem 2.10 of \citet{fan2026anisotropic} to derive the anisotropic local law of $\vG\vG^\top$ and obtain
\begin{equation}\frac{1}{N}\mbi{1}_N^\top\vQ_L(\lambda)\mbi{1}_N - s(\lambda)\prec \frac{1}{\sqrt{N}}\end{equation}
uniformly for \(\lambda\in\mathcal D\subset\R\setminus\{0\}\) where \(\mathcal D\subset\R\) is any fixed compact set such that \(\mathcal D\cap(\supp{\mu}\cup\{0\})=\emptyset\).

Second, we can replace \(\vG\) by an independent copy of the bulk.
By Lemma~\ref{lem:G_indep_g}, \(\|\vG-\widetilde{\vG}\|\prec 1/\sqrt{n}\), hence
\(
\|\vG\vG^\top-\widetilde{\vG}\widetilde{\vG}^\top\|\prec 1/\sqrt{n}.
\)
Resolvent stability gives, uniformly on \(\mathcal D\),
\[
\|\vQ_L(\lambda)-\widetilde{\vQ}_L(\lambda)\|\prec 1/\sqrt{n},
\qquad
\widetilde{\vQ}_L(\lambda):=(\widetilde{\vG}\widetilde{\vG}^\top-\lambda\vI)^{-1}.
\]
Thus it suffices to prove \eqref{eq:left_gkgl}--\eqref{eq:left_cross} with \(\vQ_L\) replaced by \(\widetilde{\vQ}_L\). By Lemma~\ref{lem:W_decomp}, \(\vg_1,\vg_2\) are independent standard Gaussian random vectors and independent of \(\widetilde{\vG}\) and \(\widetilde{\vQ}_L(\lambda)\).
Conditional on \(\widetilde{\vQ}_L(\lambda)\), for \(k\neq \ell\),
\(\vg_k^\top \widetilde{\vQ}_L(\lambda)\vg_\ell\) is centered with variance \(O(N)\), so \(N^{-1}\vg_k^\top \widetilde{\vQ}_L(\lambda)\vg_\ell=o_\P(1)\).
For \(k=\ell\),  
$
\E\Big[\frac{1}{N}\vg_k^\top \widetilde{\vQ}_L(\lambda)\vg_k\ \Big|\ \widetilde{\vQ}_L(\lambda)\Big]
=
\frac{1}{N}\Tr \widetilde{\vQ}_L(\lambda) 
$
and the Hanson--Wright inequality yields concentration around this trace at \(O_\P(N^{-1/2})\), uniformly on \(\mathcal D\)
since \(\|\widetilde{\vQ}_L(\lambda)\|\le 1/\dist{\lambda,\spec(\widetilde{\vG}\widetilde{\vG}^\top)}\le C(\mathcal D)\).
Similarly, \(N^{-1}\mbi{1}^\top\widetilde{\vQ}_L(\lambda)\vg_k=o_\P(1)\).
\end{proof}


\begin{lemma}[Wishart resolvent quadratic form with a column-dependent vector]\label{lem:dependent_vector_resolvent}
Let \(\mbi Z \in\R^{d\times n}\) have i.i.d.\ entries
\(Z_{\ell j}\sim\cN(0,1/d)\), and set $\mbi S:=\mbi Z^\top \mbi Z\in\R^{n\times n}$ and $\mbi G(\zeta):=(\mbi S-\zeta \mbi I_n)^{-1},\qquad \zeta\in\C\setminus\R .$
Assume \(n/d\to\psi\in(0,\infty)\), and define MP edges $a:=(1-\sqrt{\psi})^2$ and $b:=(1+\sqrt{\psi})^2.$
Fix \(\kappa>0\) and \(\omega\in(0,1)\), and consider the outside-spectrum domain
\[
\mathcal S_{\rm out}:=\Bigl\{\zeta=E+i\eta:\ \dist{E,[a,b]}\ge \kappa,\ \ n^{-1+\omega}\le \eta\le 1\Bigr\}.
\]
Let \(f:\R_+\to\R\) be Lipschitz with \(f(1)=0\) and polynomial growth, and define the random vector $\mbi m\in\R^n$ such that its $j$-th entry is $m_j:=f(\|\mbi z_j\|^2)$ for $j\in[n]$, where $\vz_j$ is the $j$-th column of $\mbi Z$.
Let \(m_{{\rm MP},\psi}(\zeta)\) be the Stieltjes transform of \(\rho^{\rm MP}_\psi\).
Then uniformly for \(\zeta\in\mathcal S_{\rm out}\),
\begin{equation}\label{eq:depvec_main_refined}
\mbi m^\top \mbi G(\zeta)\,\mbi m
\;=\;
m_{{\rm MP},\psi}(\zeta)\,\|\mbi m\|^2
\;+\;
O_\prec(d^{-1/2}).
\end{equation}
In particular, for any fixed real \(E\) with \(\dist{E,[a,b]}\ge \kappa\), $\mbi m^\top(\mbi S-E\mbi I)^{-1}\mbi m
=
m_{{\rm MP},\psi}(E)\,\|\mbi m\|^2+o_\P(1).$
\end{lemma}

\begin{proof}
 Since \(\|\mbi z_j\|^2=1+O_\prec(d^{-1/2})\) and \(f\) is Lipschitz with \(f(1)=0\),
\begin{equation}\label{eq:m_entry_size}
m_j=O_\prec(d^{-1/2})\qquad\text{uniformly in }j.
\end{equation}
Hence \(\|\mbi m\|^2=\sum_{j=1}^n m_j^2=O_\prec(n/d)=O_\prec(1)\), and moreover
$\|\mbi m\|=O_\prec(1).$

We first recall the isotropic local law of \citet{bloemendal2014isotropic} for $\vS$.
By the entrywise local law for sample covariance matrices and its minor-stability
\citep{bloemendal2014isotropic,knowles2017anisotropic}, we have
\begin{equation}\label{eq:entrywise_local_law_refined}
\max_{i,j\in[n]}\bigl|\vG_{ij}(\zeta)-\delta_{ij}\,m_{{\rm MP},\psi}(\zeta)\bigr|
\ \prec\ \Psi(\zeta):=n^{-1/2},
\end{equation}
and the same holds for all resolvent minors \(\mbi G^{(j)}(\zeta):=(\mbi S^{(j)}-\zeta\mbi I)^{-1}\), uniformly in \(j\in [n]\)), where  $\vS^{(j)}$ denote the $(n-1)\times(n-1)$ minor of $\vS$ obtained by removing $j$-th row and column.

We will also use the standard resolvent comparison identity between \(\mbi G\) and \(\mbi G^{(j)}\):
for any \(i,k\neq j\),
\begin{equation}\label{eq:minor_resolvent_identity}
\vG_{ik}(\zeta)
=
\vG^{(j)}_{ik}(\zeta)
+
\frac{\vG_{ij}(\zeta)\,\vG_{jk}(\zeta)}{\vG_{jj}(\zeta)}.
\end{equation}
Since \(\vG_{ij}=O_\prec(n^{-1/2})\) for \(i\neq j\) and \(\vG_{jj}=O_\prec(1)\) on \(\mathcal S_{\rm out}\),
\begin{equation}\label{eq:minor_replacement_error}
\max_{i,k\neq j}\big|\vG_{ik}(\zeta)-\vG^{(j)}_{ik}(\zeta)\big|
\ \prec\ \Psi(\zeta)^2=n^{-1}.
\end{equation}

Write the block form of \(\mbi S-\zeta\mbi I\) w.r.t.\ index \(j\):
\[
\mbi S-\zeta\mbi I
=
\begin{pmatrix}
q_j-\zeta & \mbi r_j^\top\\
\mbi r_j & \mbi S^{(j)}-\zeta\mbi I
\end{pmatrix},
\qquad
q_j:=\|\mbi z_j\|^2,\quad
\mbi r_j:=(\mbi z_k^\top \mbi z_j)_{k\neq j}\in\R^{n-1}.
\]
Denote vector $\mbi G_{j,-j}(\zeta)=(\vG_{jk}(\zeta))_{k\neq j}\in\R^{1\times (n-1)}$. The Schur complement identities give
\begin{align}
\label{eq:schur_Gjj_refined}
\vG_{jj}(\zeta)&=\frac{1}{q_j-\zeta-\mbi r_j^\top \mbi G^{(j)}(\zeta)\,\mbi r_j},\\
\label{eq:schur_Gjminusj_refined}
\mbi G_{j,-j}(\zeta)&=-\,\vG_{jj}(\zeta)\,\mbi r_j^\top \mbi G^{(j)}(\zeta).
\end{align}
Let \(\mbi m^{(j)}\) be \(\mbi m\) with the \(j\)-th component removed, and define
$a_j(\zeta):=\mbi r_j^\top \mbi G^{(j)}(\zeta)\,\mbi m^{(j)}.$
Then \((\mbi G\mbi m)_j=\vG_{jj}(m_j-a_j(\zeta))\), so
\begin{equation}\label{eq:Q_decomp_refined}
\mbi m^\top \mbi G(\zeta)\,\mbi m
=
\sum_{j=1}^n m_j^2\,\vG_{jj}(\zeta)
-\sum_{j=1}^n m_j\,\vG_{jj}(\zeta)\,a_j(\zeta)
=:T_1(\zeta)-T_2(\zeta).
\end{equation}
Now we control $T_1(\zeta)$ and $T_2(\zeta)$ separately.

For \(T_1(\zeta)\) term, by \eqref{eq:entrywise_local_law_refined} with \(i=j\),
\(\vG_{jj}(\zeta)=m_{{\rm MP},\psi}(\zeta)+O_\prec(n^{-1/2})\).
Using \(\|\mbi m\|^2=O_\prec(1)\),
\begin{equation}\label{eq:T1_refined}
T_1(\zeta)
=
m_{{\rm MP},\psi}(\zeta)\,\|\mbi m\|^2
+
O_\prec(n^{-1/2}).
\end{equation}
Hence, it suffices to control the dependence term \(T_2(\zeta)\) below. 
Write \(X_j(\zeta):=m_j\,\vG_{jj}(\zeta)\,a_j(\zeta)\), so \(T_2(\zeta)=\sum_{j=1}^n X_j(\zeta)\). We follow the fluctuation average argument in \citep{wang2024nonlinearspikedcovariancematrices} to show that \(T_2(\zeta)\) is vanishing.
 
Firstly, we compute the mean of this \(T_2(\zeta)\). Let \(\E_j[\cdot]\) be conditional expectation over the single column \(\mbi z_j\)
(i.e.\ conditioning on \(\mbi Z^{(j)}\)).
Then, we have that:
(i) \(m_j=f(\|\mbi z_j\|^2)\) is even in \(\mbi z_j\);
(ii) \(\vG_{jj}(\zeta)\) is even in \(\mbi z_j\) by \eqref{eq:schur_Gjj_refined};
(iii) \(a_j(\zeta)\) is odd in \(\mbi z_j\) because \(\mbi r_j=\mbi Z^{(j)\top}\mbi z_j\) is linear in \(\mbi z_j\),
while \(\mbi G^{(j)}\) and \(\mbi m^{(j)}\) are \(\mbi Z^{(j)}\)-measurable, independent with $\vz_j$.
Hence \(X_j(\zeta)\) is odd in \(\mbi z_j\), so
\begin{equation}\label{eq:EjXj0_refined}
\E_j[X_j(\zeta)]=0.
\end{equation}

Next, we control the size of \(a_j\) and \(X_j\). 
Conditional on \(\mbi Z^{(j)}\), we have \(\mbi r_j=\mbi Z^{(j)\top}\mbi z_j\), with
\(\mbi z_j\sim\cN(\mbi 0,\mbi I_d/d)\) independent of \(\mbi Z^{(j)}\). Thus \(\mbi r_j\) is centered Gaussian with
\(\Cov_j(\mbi r_j)=(1/d)\,\mbi S^{(j)}\).
Therefore \(a_j(\zeta)=\mbi m^{(j)\top}\mbi G^{(j)}\mbi r_j\) is centered Gaussian with conditional variance
\[
\Var_j(a_j(\zeta))
=
\frac{1}{d}\,\mbi m^{(j)\top}\mbi G^{(j)}(\zeta)\,\mbi S^{(j)}\,\mbi G^{(j)}(\zeta)\,\mbi m^{(j)}
\ \le\ \frac{C}{d}\,\|\mbi m\|^2,
\]
because on \(\mathcal S_{\rm out}\), \(\|\mbi G^{(j)}(\zeta)\|\le C(\kappa)\) and \(\|\mbi S^{(j)}\|=O_\prec(1)\).
Using \(\|\mbi m\|=O_\prec(1)\), we get
\begin{equation}\label{eq:aj_size_refined}
a_j(\zeta)=O_\prec(d^{-1/2}).
\end{equation}
Combining \eqref{eq:m_entry_size}, \(\vG_{jj}(\zeta)=O_\prec(1)\), and \eqref{eq:aj_size_refined} yields
\begin{equation}\label{eq:Xj_size_refined}
X_j(\zeta)=O_\prec(d^{-1}).
\end{equation}

Finally, we take high-moment fluctuation averaging for \(T_2\).
Fix an integer \(p\ge 1\). We claim
\begin{equation}\label{eq:T2_moment_bound_refined}
\E\big|T_2(\zeta)\big|^{2p}
\ \le\ C_p\Big(\frac{n}{d^2}\Big)^p
\ \asymp\ C_p\,d^{-p},
\qquad \text{uniformly for }\zeta\in\mathcal S_{\rm out}.
\end{equation}
This implies \(T_2(\zeta)=O_\prec(\sqrt n/d)=O_\prec(d^{-1/2})\).

We can expand moments and classify by the number of distinct indices. Notice that
\[
\E|T_2|^{2p}
=\sum_{j_1,\dots,j_{2p}\in[n]}
\E\big[X_{j_1}\cdots X_{j_{2p}}\big].
\]
For any fixed multi-index \(\mbi j=(j_1,\dots,j_{2p})\), let \(\mathcal I(\mbi j)\) be the set of distinct values in \(\mbi j\),
and write \(r:=|\mathcal I(\mbi j)|\).
For each \(t\in\mathcal I(\mbi j)\), let \(\nu_t\) be its multiplicity in the list \(\mbi j\),
so \(\sum_{t\in\mathcal I}\nu_t=2p\).
Let \(\mathcal S(\mbi j):=\{t\in\mathcal I(\mbi j):\nu_t=1\}\) be the set of singleton indices,
and write \(q:=|\mathcal S(\mbi j)|\). Notice that 
\[
2p=\sum_{t\in\mathcal I}\nu_t \ge \sum_{t\in\mathcal S}1+\sum_{t\notin\mathcal S}2
= q+2(r-q)=2r-q,
\]
hence \(q\ge 2r-2p\) for any fixed multi-index \(\mbi j\).

Fix a monomial \(X_{j_1}\cdots X_{j_{2p}}\) and pick a singleton index \(t\in\mathcal S(\mbi j)\).
Write the product as \(X_t\cdot F_t\), where \(F_t\) is the product of the remaining \(2p-1\) factors. Define \(F_t^{(t)}\) by replacing in \(F_t\) every resolvent entry of \(\mbi G\) (and of any minors)
by the corresponding entry of a resolvent/minor with index \(t\) removed whenever the indices allow.
Formally, \(F_t^{(t)}\) is constructed so that it is \(\mbi Z^{(t)}\)-measurable, hence independent of \(\mbi z_t\),
and such that
\begin{equation}\label{eq:F_decouple_one_singleton}
F_t = F_t^{(t)} + \Delta_t,
\qquad
|\Delta_t|\ \prec\ \Psi(\zeta)^2\,\prod_{\alpha\neq t}|X_{j_\alpha}|.
\end{equation}
The estimate \eqref{eq:F_decouple_one_singleton} follows by repeated use of the resolvent identity
\eqref{eq:minor_resolvent_identity} and the replacement bound \eqref{eq:minor_replacement_error}:
each replacement introduces at least one factor \(\vG_{it}\vG_{tk}/\vG_{tt}\), hence a factor
\(O_\prec(\Psi^2)\), and the total number of replacements is \(O_p(1)\) since \(p\) is fixed. Hence, \eqref{eq:F_decouple_one_singleton} decouples each singleton index and gains a factor \(\Psi(\zeta)^2\) per singleton. 

Now take \(\E_t[\cdot]\). Since \(F_t^{(t)}\) is \(\mbi Z^{(t)}\)-measurable, by \eqref{eq:EjXj0_refined}, we have $\E_t[X_t F_t^{(t)}]=F_t^{(t)}\,\E_t[X_t]=0.$
Hence
$
\E[X_t F_t]=\E[X_t\Delta_t].
$
Using \eqref{eq:Xj_size_refined} and \eqref{eq:F_decouple_one_singleton}, we obtain that
\[
\big|\E[X_t\Delta_t]\big|
\ \prec\
\Psi(\zeta)^2\,\E\Big[ |X_t|\prod_{\alpha\neq t}|X_{j_\alpha}|\Big]
\ \prec\
\Psi(\zeta)^2\,d^{-2p}.
\]
Then, we can repeat this decoupling \emph{for each singleton index} in \(\mathcal S(\mbi j)\),
and each repetition contributes an additional factor \(\Psi(\zeta)^2\).
Therefore, for the monomial with \(q\) singleton indices,
\begin{equation}\label{eq:monomial_bound_q_singletons}
\big|\E[X_{j_1}\cdots X_{j_{2p}}]\big|
\ \prec\
d^{-2p}\,\Psi(\zeta)^{2q}.
\end{equation}
For fixed \(r=|\mathcal I|\), the number of multi-indices with \(r\) distinct values is \(O(n^r)\).
Using \eqref{eq:monomial_bound_q_singletons} and \(q\ge 2r-2p\), we get
\[
\sum_{\mbi j:\,|\mathcal I(\mbi j)|=r}
\big|\E[X_{j_1}\cdots X_{j_{2p}}]\big|
\ \prec\
n^r\,d^{-2p}\,\Psi(\zeta)^{2(2r-2p)}.
\]
On \(\mathcal S_{\rm out}\), \(\Psi(\zeta)=n^{-1/2}\), so
\(
n^r\,d^{-2p}\,\Psi^{4r-4p}
=
n^{2p-r}\,d^{-2p}
\ \asymp\
d^{-r}.
\)
Summing over \(r=1,\dots,2p\) gives
\[
\E|T_2(\zeta)|^{2p}
\ \prec\
\sum_{r=1}^{2p} d^{-r}
\ \prec\
d^{-1}.
\]
This is already \(o(1)\), but we need the sharper \(d^{-p}\) scaling. For that, we treat separately the
case \(r\le p\), where there are no (or few) singletons and pairing dominates:
\begin{itemize}
\item  If \(r\le p\), then trivially \(q\ge 0\) and we use the crude size bound \(|X_j|\prec d^{-1}\):
there are \(O(n^r)\le O(n^p)\) such monomials, each of size \(O_\prec (d^{-2p})\), so their total contribution is
$O_\prec(\ n^p d^{-2p})=O_\prec( d^{-p}).$
 
\item  If \(r\ge p+1\), then \(q\ge 2r-2p\ge 2\), and the singleton-decoupling bound above yields a total contribution
\(O_\prec (\sum_{r\ge p+1} d^{-r})\prec d^{-(p+1)}\), which is negligible compared to \(d^{-p}\).
\end{itemize}
Combining both regimes yields \eqref{eq:T2_moment_bound_refined}, proving \(T_2(\zeta)=O_\prec(d^{-1/2})\).
By \eqref{eq:Q_decomp_refined}, \eqref{eq:T1_refined}, and \(T_2(\zeta)=O_\prec(d^{-1/2})\), we can conclude that
\[
\mbi m^\top \mbi G(\zeta)\,\mbi m
=
m_{{\rm MP},\psi}(\zeta)\,\|\mbi m\|^2
+
O_\prec(n^{-1/2})
+
O_\prec(d^{-1/2}).
\]
Since \(n\asymp d\), this is exactly \eqref{eq:depvec_main_refined}.
\end{proof}

\begin{lemma}[Right resolvent limits for quadratic forms]\label{lem:right_resolvent}
Uniformly for \(\lambda\in\mathcal D\),
\begin{align}
\vv_k^\top\vQ_R(\lambda)\vv_\ell &= \delta_{k\ell}\,m_\mu(\lambda)+o_\P(1),\label{eq:right_v}\\
\vm^\top\vQ_R(\lambda)\vm &= \|\vm\|^2\,m_\mu(\lambda)+o_\P(1),\label{eq:right_m}\\
\vv_k^\top\vQ_R(\lambda)\vm &= o_\P(1).\label{eq:right_vm_cross}
\end{align}
\end{lemma}

\begin{proof}
Conditional on \(\vZ\), \(\vK_0=\vG^\top\vG\) is a general sample covariance matrix with population covariance \(\vSigma\)
and aspect ratio \(\phi=n/N\). Similar as the proof of Lemma \ref{lem:K0_bulk_and_tau_outlier}, we can verify that Assumption 5 of \cite{wang2024nonlinearspikedcovariancematrices} is satisfied for $\vG$. Hence, we can apply Theorem 10 of \cite{wang2024nonlinearspikedcovariancematrices} to obtain a deterministic equivalence of quadratic form of the resolvent $\vQ_R(\lambda)$.
Precisely for any deterministic unit vectors \(\vb,\vb'\) (possibly \(\vZ\)-measurable),
\begin{equation}\label{eq:aniso_DE}
\vb^\top(\vK_0-\lambda\vI)^{-1}\vb'
=
\vb^\top\Big(-\lambda \tilde m_\mu(\lambda)\vSigma-\lambda\vI\Big)^{-1}\vb'+O_\prec(N^{-1/2}),
\end{equation}
uniformly for \(\lambda\in\mathcal D\).

By Corollary~\ref{cor:Sigma_decomp}, \(\vSigma=\vSigma_0+\tau\vu\vu^\top+\vR\) with \(\|\vR\|=o_\P(1)\).
Let \(\vA:=-\lambda s(\lambda)\vSigma_0-\lambda\vI\). Then
\[
-\lambda s(\lambda)\vSigma-\lambda\vI
=
\vA-\lambda s(\lambda)\tau\vu\vu^\top+o_\P(1).
\]
By Woodbury identity, we have
\[
(\vA-\lambda s\tau\vu\vu^\top)^{-1}
=
\vA^{-1}
+
\frac{\lambda s\tau}{1-\lambda s\tau\,\vu^\top\vA^{-1}\vu}\,\vA^{-1}\vu\vu^\top\vA^{-1}.
\]
If \(\vb^\top\vu=0\) and \(\vb'^\top\vu=0\), then \(\vb^\top(\vA^{-1}\vu)=O(1)\cdot(\vb^\top\vu)=0\), hence the rank-one correction vanishes:
\[
\vb^\top(\vA-\lambda s\tau\vu\vu^\top)^{-1}\vb'=\vb^\top\vA^{-1}\vb'.
\]
Applying this with \(\vb=\vv_k\), \(\vb'=\vv_\ell\) (note \(\vv_k^\top\vu=0\) by \eqref{eq:v_orth_u}) yields that the \(\tau\)-spike does not affect \eqref{eq:right_v}. Also, Lemma \ref{lem:mean_strength_orth} implies that $\vb^\top(\vA-\lambda s\tau\vu\vu^\top)^{-1}\vb'=\vb^\top\vA^{-1}\vb'+o_\P(1) $ for \(\vb=\vb'=\vm/\|\vm\|\).

For \(\vv_k^\top\vA^{-1}\vv_\ell\), since \(\vSigma_0=(1-b_\sigma^2)\vI+b_\sigma^2\vS\) is orthogonally invariant (shifted Wishart),
the anisotropic local law from \citet{knowles2017anisotropic} implies \(\vv_k^\top\vA^{-1}\vv_\ell=\delta_{k\ell}m_\mu(\lambda)+o_\P(1)\),
giving \eqref{eq:right_v}. Here, we have removed the rank-one \(\tau\vu\vu^\top\) term in \(\vSigma\) when applying \eqref{eq:aniso_DE}.

For the quadratic from with respect to \(\vm\),
we apply \eqref{eq:aniso_DE} with \(\vb=\vb'=\vm/\|\vm\|\) and then show that
\(
\vm^\top\vA^{-1}\vm=\|\vm\|^2 m_\mu(\lambda)+o_\P(1)
\)
even though \(\vm\) depends on \(\vZ\).
Write
\[
\vA^{-1}
=
\Big(-\lambda s(\lambda)\big((1-b_\sigma^2)\vI+b_\sigma^2\vS\big)-\lambda\vI\Big)^{-1}
=
\alpha(\lambda)\,(\vS-\zeta(\lambda)\vI)^{-1},
\]
for explicit scalars \(\alpha(\lambda)\) and \(\zeta(\lambda)\) depending on \(\lambda\) (and bounded away from \(\supp{\rho_\psi^{\rm MP}}\) for \(\lambda\in\mathcal D\)).
From Lemma \ref{lem:dependent_vector_resolvent}, we can derive that
\begin{equation}\label{eq:dependent_m_wishart}
\vm^\top(\vS-\zeta\vI)^{-1}\vm
=
m_{{\rm MP},\psi}(\zeta)\,\|\vm\|^2+o_\P(1),
\end{equation}
uniformly for \(\zeta=\zeta(\lambda)\) in the relevant domain. 

For the cross term \(\vv_k^\top\vQ_R(\lambda)\vm\),
we use \eqref{eq:aniso_DE} with \(\vb=\vv_k\), \(\vb'=\vm\).
The deterministic equivalent reduces to \(\vv_k^\top\vA^{-1}\vm\) since \(\vv_k\perp \vu\).
Because \(\vv_k^\top\vm=o_\P(1)\) (Lemma~\ref{lem:mean_strength_orth}) and \(\|\vA^{-1}\|=O(1)\) on \(\mathcal D\),
we get \(\vv_k^\top\vA^{-1}\vm=o_\P(1)\), yielding \eqref{eq:right_vm_cross}.
\end{proof}


\subsubsection{Analyses of Additive Spikes and Alignments}\label{subsec:additive_spikes}

In this section, we define the rank-three spike matrix
\begin{equation}\label{eq:P_def}
\vP
:=
\frac{1}{\sqrt N}\mbi{1}_N\vm^\top
+\frac{\theta_{\snr}b_\sigma}{\sqrt N}\big(\vg_1\vv_1^\top+\vg_2\vv_2^\top\big)
=
\sum_{k=0}^2 \sqrt{\beta_k}\,\va_k\vb_k^\top,
\end{equation}
where
\begin{equation}\label{eq:def_ai_bi}
\va_0:=\frac{\mbi{1}_N}{\sqrt N},
\quad
\vb_0:=\frac{\vm}{\|\vm\|},
\quad
\va_1:=\frac{\vg_1}{\|\vg_1\|},
\quad
\vb_1:=\vv_1,
\quad
\va_2:=\frac{\vg_2}{\|\vg_2\|},
\quad
\vb_2:=\vv_2,
\end{equation}
and by Lemmas~\ref{lem:mean_strength_orth} and~\ref{lem:g_norms}, asymptotically,
\begin{equation}\label{eq:beta_limits}
\beta_0=\|\vm\|^2 \to \beta_{\rm mean}=\frac{c_\sigma^2}{2}\psi=\tau,
\qquad
\beta_1=\beta_2=\theta_{\snr}^2 b_\sigma^2 \to \beta_{\rm lin}=\frac{r^2\psi}{2}b_\sigma^2.
\end{equation}
Then, we recall  \(\vY^\sharp=\vG+\vP\) and \(\vK^\sharp=\vY^{\sharp\top}\vY^\sharp\) defined by Proposition \ref{prop:Y_rank3_decomp}.


\begin{lemma}\label{lem:tau_outlier_stability_collision}
Recall$\mbi G\in\R^{N\times n}$ in \eqref{eq:G_def} and $\mbi K_0=\mbi G^\top \mbi G$.
Assume $\mbi K_0$ has an \emph{isolated} outlier eigenvalue
$\widehat\lambda_\tau^{(0)}$ with unit-norm eigenvector $\widehat{\mbi u}_\tau\in\R^n$ so that for any $\epsilon>0$, $\mbi K_0\,\widehat{\mbi u}_\tau=\widehat\lambda_\tau^{(0)}\,\widehat{\mbi u}_\tau,$ and
$\widehat\lambda_\tau^{(0)}\notin\supp{\mu}+(-\epsilon,\epsilon)$ for all large $n$.
Define the associated unit left singular vector of $\vG$ as 
\[
\widehat{\mbi w}_\tau:=\frac{\mbi G\,\widehat{\mbi u}_\tau}{\sqrt{\widehat\lambda_\tau^{(0)}}}\in\R^N,
\qquad \|\widehat{\mbi w}_\tau\|=1.
\]
Consider the rank--$3$ additive spike matrix $\mbi P = \sum_{k=0}^2 \sqrt{\beta_k}\,\mbi a_k \mbi b_k^\top
$ with $
\|\mbi a_k\|=\|\mbi b_k\|=1$ and $\beta_k=O(1),$
$\mbi Y^\sharp = \mbi G+\mbi P,$ and $\mbi K^\sharp = \mbi Y^{\sharp\top}\mbi Y^\sharp.$
Assume that
\begin{equation}\label{eq:tau_decouple_conditions}
\max_{k\in\{0,1,2\}} |\mbi b_k^\top \widehat{\mbi u}_\tau| = o_\P(1),
\qquad
\max_{k\in\{0,1,2\}} |\mbi a_k^\top \widehat{\mbi w}_\tau| = o_\P(1).
\end{equation}
Then, we have
\begin{enumerate}[label=\textbf{(\roman*)},leftmargin=*]
\item \textbf{(Eigenvalue stability.)}
There exists an eigenvalue $\widehat\lambda_\tau^\sharp\in\spec(\mbi K^\sharp)$ such that $\big|\widehat\lambda_\tau^\sharp-\widehat\lambda_\tau^{(0)}\big| = o_\P(1).$
Hence, $\widehat\lambda_\tau^\sharp$ is an outlier eigenvalue of $\vK^\sharp$, i.e., for any $\epsilon>0$,
$\widehat\lambda_\tau^\sharp\notin\supp{\mu}+(-\epsilon,\epsilon)$ for all large $n$.
\item \textbf{(Eigenvector-alignment stability.)}
Fix any constant $\delta>0$.
Let $\mbi\Pi_\delta$ be the spectral projector of $\mbi K^\sharp$ onto $I_\delta:=(\widehat\lambda_\tau^{(0)}-\delta,\widehat\lambda_\tau^{(0)}+\delta)$.
Then $\|( \mbi I-\mbi\Pi_\delta)\widehat{\mbi u}_\tau\| = o_\P(1).$
In particular, the normalized vector
\[
\widehat{\mbi u}_\tau^\sharp:=\frac{\mbi\Pi_\delta \widehat{\mbi u}_\tau}{\|\mbi\Pi_\delta \widehat{\mbi u}_\tau\|}
\in\mathrm{Range}(\mbi\Pi_\delta)
\]
satisfies $|\langle \widehat{\mbi u}_\tau^\sharp,\widehat{\mbi u}_\tau\rangle|^2=1-o_\P(1)$.
Thus $\widehat{\mbi u}_\tau^\sharp$ inherits the same eigenvector alignment limit as
$\widehat{\mbi u}_\tau$ (e.g.\ alignment with $\mbi u=\mbi 1_n/\sqrt n$ in Lemma \ref{lem:K0_bulk_and_tau_outlier}).

\item \textbf{(Projected resolvents remove the singularity at $\widehat\lambda_\tau^{(0)}$.)}
Define the orthogonal projectors
\[
\mbi P_R^\perp:=\mbi I_n-\widehat{\mbi u}_\tau\widehat{\mbi u}_\tau^\top,
\qquad
\mbi P_L^\perp:=\mbi I_N-\widehat{\mbi w}_\tau\widehat{\mbi w}_\tau^\top,
\]
and the spike-removed resolvents
\begin{align}
\mbi Q_R^\perp(\lambda)
:= \mbi P_R^\perp(\mbi K_0-\lambda\mbi I)^{-1}\mbi P_R^\perp, \qquad
\mbi Q_L^\perp(\lambda)
:= \mbi P_L^\perp(\mbi G\mbi G^\top-\lambda\mbi I)^{-1}\mbi P_L^\perp,
\end{align}
which admit a continuous extension to $\lambda=\widehat\lambda_\tau^{(0)}$.
For each fixed $k,\ell\in\{0,1,2\}$, $\mbi b_k^\top \mbi Q_R^\perp(\widehat\lambda_\tau^{(0)})\,\mbi b_\ell$ and $\mbi a_k^\top \mbi Q_L^\perp(\widehat\lambda_\tau^{(0)})\,\mbi a_\ell$ are well-defined and remain $O_\P(1)$.  Hence, applying \eqref{eq:tau_decouple_conditions}, we have
\begin{equation}
    \mbi b_k^\top \mbi Q_R(\widehat\lambda_\tau^{(0)})\,\mbi b_\ell = \mbi b_k^\top \mbi Q_R^\perp (\widehat\lambda_\tau^{(0)})\,\mbi b_\ell+o_\P(1),
\end{equation}
\begin{equation}
\mbi a_k^\top \mbi Q_L(\widehat\lambda_\tau^{(0)})\,\mbi a_\ell =\mbi a_k^\top \mbi Q_L^\perp (\widehat\lambda_\tau^{(0)})\,\mbi a_\ell+o_\P(1).
\end{equation}
\end{enumerate}

\end{lemma}

\begin{proof}
We can expand $\mbi K^\sharp$ as a low-rank perturbation of $\mbi K_0$.
By definition,
\[
\mbi K^\sharp
=
(\mbi G+\mbi P)^\top(\mbi G+\mbi P)
=
\mbi K_0+\underbrace{\mbi G^\top\mbi P+\mbi P^\top\mbi G+\mbi P^\top\mbi P}_{=: \ \mbi\Delta}.
\]
Thus $\mbi\Delta$ has rank at most $9$ and $\|\mbi\Delta\|=O_\P(1)$ (since $\|\mbi G\|=O_\P(1)$ and $\|\mbi P\|=O_\P(1)$).

\textbf{Part (i).} Let $\mbi A$ be symmetric, $\lambda\in\R$, and $\mbi v$ a unit vector. Then
\begin{equation}\label{eq:dist_spec_bound}
\mathrm{dist}\big(\lambda,\spec(\mbi A)\big)\ \le\ \|(\mbi A-\lambda\mbi I)\mbi v\|.
\end{equation}
Indeed, considering eigen-decomposition $\mbi A=\sum_i \lambda_i \mbi q_i\mbi q_i^\top$ and $\mbi v=\sum_i c_i \mbi q_i$,
we have $\|(\mbi A-\lambda\mbi I)\mbi v\|^2=\sum_i (\lambda_i-\lambda)^2|c_i|^2 \ge
\min_i(\lambda_i-\lambda)^2\sum_i|c_i|^2 = \mathrm{dist}(\lambda,\spec(\mbi A))^2$. Since $\mbi K_0\widehat{\mbi u}_\tau=\widehat\lambda_\tau^{(0)}\widehat{\mbi u}_\tau$, we have
\[
(\mbi K^\sharp-\widehat\lambda_\tau^{(0)}\mbi I)\widehat{\mbi u}_\tau
=
(\mbi K_0-\widehat\lambda_\tau^{(0)}\mbi I)\widehat{\mbi u}_\tau+\mbi\Delta\widehat{\mbi u}_\tau
=
\mbi\Delta\widehat{\mbi u}_\tau.
\]
Therefore,
$
\mathrm{dist}\big(\widehat\lambda_\tau^{(0)},\spec(\mbi K^\sharp)\big)\ \le\ \|\mbi\Delta\widehat{\mbi u}_\tau\|.
$
It remains to show $\|\mbi\Delta\widehat{\mbi u}_\tau\|=o_\P(1)$.

Write $\mbi P=\mbi A\mbi\Theta\mbi B^\top$ with
$\mbi A=[\mbi a_0,\mbi a_1,\mbi a_2]$, $\mbi B=[\mbi b_0,\mbi b_1,\mbi b_2]$ and
$\mbi\Theta=\mathrm{diag}(\sqrt{\beta_0},\sqrt{\beta_1},\sqrt{\beta_2})$.
Then
\begin{align*}
\mbi G^\top\mbi P\,\widehat{\mbi u}_\tau
&=\mbi G^\top\mbi A\mbi\Theta\,(\mbi B^\top \widehat{\mbi u}_\tau),
\\
\mbi P^\top\mbi P\,\widehat{\mbi u}_\tau
&=\mbi B\mbi\Theta(\mbi A^\top\mbi A)\mbi\Theta\,(\mbi B^\top \widehat{\mbi u}_\tau),
\\
\mbi P^\top\mbi G\,\widehat{\mbi u}_\tau
&=\mbi B\mbi\Theta\,\mbi A^\top(\mbi G\,\widehat{\mbi u}_\tau)
=\sqrt{\widehat\lambda_\tau^{(0)}}\ \mbi B\mbi\Theta\,(\mbi A^\top \widehat{\mbi w}_\tau),
\end{align*}
where we used $\mbi G\widehat{\mbi u}_\tau=\sqrt{\widehat\lambda_\tau^{(0)}}\,\widehat{\mbi w}_\tau$.
Hence,
\[
\|\mbi\Delta\widehat{\mbi u}_\tau\|
\ \le\
\|\mbi G^\top\mbi A\mbi\Theta\|\cdot\|\mbi B^\top\widehat{\mbi u}_\tau\|
+\sqrt{\widehat\lambda_\tau^{(0)}}\ \|\mbi B\mbi\Theta\|\cdot\|\mbi A^\top\widehat{\mbi w}_\tau\|
+\|\mbi B\mbi\Theta(\mbi A^\top\mbi A)\mbi\Theta\|\cdot\|\mbi B^\top\widehat{\mbi u}_\tau\|.
\]
Now $\|\mbi G^\top\mbi A\mbi\Theta\|\le \|\mbi G\|\,\|\mbi A\|\,\|\mbi\Theta\|=O_\P(1)$,
$\|\mbi B\mbi\Theta\|=O(1)$, and $\|\mbi B\mbi\Theta(\mbi A^\top\mbi A)\mbi\Theta\|=O(1)$.
By the decoupling assumptions \eqref{eq:tau_decouple_conditions}, $\|\mbi B^\top\widehat{\mbi u}_\tau\|=o_\P(1),$ and $\|\mbi A^\top\widehat{\mbi w}_\tau\|=o_\P(1).$
Therefore $\|\mbi\Delta\widehat{\mbi u}_\tau\|=o_\P(1)$, and combining with \eqref{eq:dist_spec_bound} yields
$\mathrm{dist}\big(\widehat\lambda_\tau^{(0)},\spec(\mbi K^\sharp)\big)=o_\P(1),$
which proves \textbf{(i)}.

\textbf{Part (ii).}
Let $\mbi K^\sharp=\sum_i \widehat\lambda_i^\sharp \widehat{\mbi v}_i\widehat{\mbi v}_i^\top$ be an eigen-decomposition,
and write $\widehat{\mbi u}_\tau=\sum_i c_i \widehat{\mbi v}_i$.
Then
\[
\|(\mbi K^\sharp-\widehat\lambda_\tau^{(0)}\mbi I)\widehat{\mbi u}_\tau\|^2
=
\sum_i (\widehat\lambda_i^\sharp-\widehat\lambda_\tau^{(0)})^2 |c_i|^2
=
\|\mbi\Delta\widehat{\mbi u}_\tau\|^2
=o_\P(1).
\]
Split the sum into indices with $\widehat\lambda_i^\sharp\in I_\delta$ and those with $\widehat\lambda_i^\sharp\notin I_\delta$.
For $\widehat\lambda_i^\sharp\notin I_\delta$ we have $|\widehat\lambda_i^\sharp-\widehat\lambda_\tau^{(0)}|\ge \delta$, hence
\[
\delta^2 \sum_{\widehat\lambda_i^\sharp\notin I_\delta} |c_i|^2
\ \le\
\sum_{\widehat\lambda_i^\sharp\notin I_\delta} (\widehat\lambda_i^\sharp-\widehat\lambda_\tau^{(0)})^2|c_i|^2
\ \le\
\|(\mbi K^\sharp-\widehat\lambda_\tau^{(0)}\mbi I)\widehat{\mbi u}_\tau\|^2
=o_\P(1).
\]
Hence $\sum_{\widehat\lambda_i^\sharp\notin I_\delta} |c_i|^2=o_\P(1)$, i.e.
$\|(\mbi I-\mbi\Pi_\delta)\widehat{\mbi u}_\tau\|=o_\P(1)$, proving \textbf{(ii)}.
This argument makes no assumption about how many eigenvalues lie inside $I_\delta$,
so it remains valid when there are more than one eigenvalues of $\mbi K^\sharp$ in $I_\delta$.

\textbf{Part (iii).}
Let $\{\widehat\lambda_i^{(0)},\widehat{\mbi u}_i^{(0)}\}$ be an eigen-decomposition of $\mbi K_0$ with
$\widehat\lambda_\tau^{(0)}$ corresponding to $\widehat{\mbi u}_\tau$.
Then,
\[
(\mbi K_0-\lambda\mbi I)^{-1}
=
\frac{1}{\widehat\lambda_\tau^{(0)}-\lambda}\,\widehat{\mbi u}_\tau\widehat{\mbi u}_\tau^\top
+\sum_{i\neq \tau}\frac{1}{\widehat\lambda_i^{(0)}-\lambda}\,\widehat{\mbi u}_i^{(0)}\widehat{\mbi u}_i^{(0)\top}
\]  for $\lambda\neq \widehat\lambda_\tau^{(0)}$.
Multiplying by $\mbi P_R^\perp$ on both sides removes the pole term, giving
$
\mbi Q_R^\perp(\lambda)
=
\sum_{i\neq \tau}\frac{1}{\widehat\lambda_i^{(0)}-\lambda}\,\widehat{\mbi u}_i^{(0)}\widehat{\mbi u}_i^{(0)\top},
$
which is continuous at $\lambda=\widehat\lambda_\tau^{(0)}$ (since the sum is over $i\neq \tau$).
The same argument applies to $\mbi Q_L^\perp(\lambda)$ using the eigen-decomposition of $\mbi G\mbi G^\top$.
Hence the quadratic forms in \textbf{(iii)} are well-defined and $O_\P(1)$.
This is exactly what is needed to write the additive outlier determinant equation in a way
that still makes sense at the collision point $\lambda=\widehat\lambda_\tau^{(0)}$.
\end{proof}
  
\begin{lemma}[Verification of the decoupling conditions \texorpdfstring{$\eqref{eq:tau_decouple_conditions}$}{(E.34)} for $\vK_0$]\label{lem:verify-E34-K0}
Under the assumptions of Lemma~\ref{lem:K0_bulk_and_tau_outlier}, additionally suppose the separated regime, that is, $\tau>\tau_{\mathrm{crit}}$ and $z'(-1/\Lambda_\tau)>0$. Under these conditions, the matrix $\vK_0=\vG^\top \vG$ has a unique isolated outlier eigenvalue, denoted by $\lambda_{\tau}$, associated with a unit-norm eigenvector $\widehat{\mbi u}_{\tau}$. 
Define the associated left singular vector
$\widehat{\mbi w}_{\tau}:=\vG\widehat{\mbi u}_{\tau}/\sqrt{\lambda_{\tau}}$.
Then the decoupling conditions in \eqref{eq:tau_decouple_conditions} hold where the spike directions $\va_k$ and $\vb_k$ defined in \eqref{eq:def_ai_bi}.
\end{lemma}

\begin{proof}
First we consider the right decoupling condition $\max_k|\mbi b_k^\top \widehat{\mbi u}_{\tau}|=o_\P(1)$.
By the XOR construction, $\mbi v_1^\top \mbi u=\mbi v_2^\top \mbi u=0$ where $\vu=\frac{1}{\sqrt{n}}\mbi 1_n$.
Applying \eqref{eq:alignment_right} in Lemma~\ref{lem:K0_bulk_and_tau_outlier} with $\mbi v=\mbi v_k$ gives $\mbi v_k^\top \widehat{\mbi u}_{\tau}=o_\P(1),$ for $k\in\{1,2\}.$ 
Moreover, Lemma~\ref{lem:mean_strength_orth} gives $\mbi u^\top \mbi m=o_\P(1)$ and $\|\mbi m\|\to \sqrt{\beta_{\mathrm{mean}}}>0$.
Therefore, we can obtain
\[
\mbi b_0^\top \mbi u = \frac{\mbi m^\top \mbi u}{\|\mbi m\|}=o_\P(1).
\]
Applying Lemma~\ref{lem:K0_bulk_and_tau_outlier} again with $\mbi v=\mbi b_0$ in \eqref{eq:alignment_right} yields $\mbi b_0^\top \widehat{\mbi u}_{ \tau}=o_\P(1).$
Taking the maximum over $k\in\{0,1,2\}$ completes the first part.

Second, we consider the left decoupling condition $\max_k|\mbi a_k^\top \widehat{\mbi w}_{\tau}|=o_\P(1)$.
Let $\vG^e$ be $\widetilde{\vG}$ defined in Lemma~\ref{lem:G_indep_g}, so Lemma~\ref{lem:G_indep_g} indicates that $\|\vG-\vG^e\|\prec N^{-1/2}$ and $\vG^e$ is independent of
$(\mbi g_1,\mbi g_2)$.
Define $\vK_0^e:=(\vG^e)^\top \vG^e$ and let $(\lambda_{\tau}^e,\mbi u_{\tau}^e)$ denote the isolated eigenpair of $\vK_0^e$ corresponding to the outlier of $\vK_0$. In fact, Lemma~\ref{lem:K0_bulk_and_tau_outlier} holds for $\vK_0^e$ as well.
Set $\mbi w_{\tau}^e:=\vG^e\mbi u_{\tau}^e/\sqrt{\lambda_{\tau}^e}$. Since
\[
\|\vK_0-\vK_0^e\|
=\|\vG^\top \vG-(\vG^e)^\top \vG^e\|
\le (\|\vG\|+\|\vG^e\|)\,\|\vG-\vG^e\|
=O_\P(1)\cdot O_\P(N^{-1/2})
=o_\P(1),
\]
we can claim that
\begin{equation}\label{eq:left-transfer}
\|\widehat{\mbi w}_{\tau}-\mbi w_{\tau}^e\|=o_\P(1).
\end{equation}
by Davis--Kahan Theorem. Also, we have 
$\|\widehat{\mbi u}_{\tau}-\mbi u_{\tau}^e\|=o_\P(1)$ and $|\lambda_{\tau}-\lambda_{\tau}^e|=o_\P(1)$.

Fix $k\in\{1,2\}$. Since $\mbi a_k$ is unit,
\[
|\mbi a_k^\top \widehat{\mbi w}_{\tau}|
\le |\mbi a_k^\top \mbi w_{\tau}^e|+\|\widehat{\mbi w}_{\tau}-\mbi w_{\tau}^e\|
=|\mbi a_k^\top \mbi w_{\tau}^e|+o_\P(1)
\]
by \eqref{eq:left-transfer}.
Condition on $\mbi w_{\tau}^e$ (which is measurable w.r.t.\ $\vG^e$).
By Lemmas~\ref{lem:W_decomp} and \ref{lem:G_indep_g}, $\mbi g_k\sim\mathcal N(0,\vI_N)$ is independent of $\vG^e$, hence independent of
$\mbi w_{\tau}^e$. Therefore $\mbi g_k^\top \mbi w_{\tau}^e\sim\mathcal N(0,1)$ and
$\|\mbi g_k\|=\sqrt N+o_\P(\sqrt N)$, so
\[
|\mbi a_k^\top \mbi w_{\tau}^e|
=\frac{|\mbi g_k^\top \mbi w_{\tau}^e|}{\|\mbi g_k\|}
=O_\P(1)\cdot\frac{1}{\sqrt N+o_\P(\sqrt N)}
=o_\P(1).
\]
Hence $|\mbi a_k^\top \mbi w_{\tau}|=o_\P(1)$ for $k\in\{1,2\}$.

Lastly, we consider $\mbi a_0=\mbi 1_N/\sqrt N$.
Let $\vPi_\delta$ be the spectral projector of $\vG\vG^\top$ onto interval
$(\lambda_{\tau}-\delta,\lambda_{\tau}+\delta)$ for some fixed and small $\delta>0$.
On the event that the outlier is simple, $\vPi_\delta =\widehat{\mbi w}_{\tau} \widehat{\mbi w}_{\tau}^\top$, so
\[
|\mbi a_0^\top \widehat{\mbi w}_{\tau} |^2
=\mbi a_0^\top \vPi_\delta \,\mbi a_0
=-\frac{1}{2\pi i}\oint_{\Gamma } \mbi a_0^\top \vQ_L (z)\mbi a_0\,dz,
\]
where $\vQ_L (z)=\big(\vG\vG^\top-z\vI_N\big)^{-1},$ $\Gamma :=\{z\in\mathbb C:|z-\lambda_{ \tau} |=\delta/2\}$, and $\mbi a_0^\top \vQ_L (z)\mbi a_0$ is well-defined on $\Gamma$. Following the proof of Lemma~\ref{lem:left_resolvent}, we can apply Theorem 2.10 in \cite{fan2026anisotropic} to get: uniformly for $z\in\Gamma $,
\[
\mbi a_0^\top \vQ_L(z)\mbi a_0 = s(z)+o_\P(1),
\]
where $s(z)=\tilde m_\mu(z)$ is analytic in a neighborhood of $\lambda_\tau$ containing $\Gamma $.
Therefore $\oint_{\Gamma } s(z)\,dz=0$ and we obtain $|\mbi a_0^\top \widehat{\mbi w}_{\tau} |^2=o_\P(1)$.
\end{proof}

\begin{lemma}[Master equation for additive outliers of $\mbi K^\sharp$]
\label{lem:finite_det_eq}
Fix $\lambda\in\C$ such that
\begin{equation}\label{eq:away_spec}
\lambda\notin \supp{\mu}\cup\{0\}.
\end{equation}
Define $\mbi Q_L(\lambda)=(\mbi G\mbi G^\top-\lambda \mbi I_N)^{-1}$ and
$\mbi Q_R(\lambda)=(\mbi K_0-\lambda \mbi I_n)^{-1}$.
Let us define three $3\times 3$ matrices as
\begin{equation}\label{eq:A(lambda)}
    \mbi A(\lambda):=\big(\mbi a_k^\top \mbi Q_L(\lambda)\mbi a_\ell\big)_{k,\ell=0,1,2}\qquad
\mbi B(\lambda):=\big(\mbi b_k^\top \mbi Q_R(\lambda)\mbi b_\ell\big)_{k,\ell=0,1,2}
\end{equation}
and $\mbi \Theta:=\diag(\sqrt{\beta_0},\sqrt{\beta_1},\sqrt{\beta_2})$.
Then $\lambda$ is an outlier eigenvalue of $\mbi K^\sharp$ if 
\begin{equation}\label{eq:additive_det_eq}
\det\Big(\mbi I_3-\lambda\,\mbi\Theta\,\mbi A(\lambda)\,\mbi\Theta\cdot \mbi B(\lambda)\Big)=0.
\end{equation}
Moreover, \eqref{eq:additive_det_eq} characterizes all outlier eigenvalues of $\mbi K^\sharp$, induced by $\{\va_k,\vb_k\}_{k=0}^2$,
that lie in any compact set $\mathcal D\subset\C$ satisfying \eqref{eq:away_spec}.
\end{lemma}

\begin{proof}
Notice that from the last part of Lemma \ref{lem:tau_outlier_stability_collision}, we know all the entries of $\vA(\lambda)$ and $\vB(\lambda)$ are well-defined and finite, even if $\lambda=\widehat\lambda_\tau^{(0)}\notin \supp{\mu}\cup\{0\}$ where $\widehat\lambda_\tau^{(0)}$ is characterized by Lemma \ref{lem:K0_bulk_and_tau_outlier}.
 
Define
\[
\mbi H^\sharp:=\begin{pmatrix}\mbi 0 & \mbi Y^\sharp\\ \mbi Y^{\sharp\top} & \mbi 0\end{pmatrix},\qquad
\mbi H_0:=\begin{pmatrix}\mbi 0 & \mbi G\\ \mbi G^\top & \mbi 0\end{pmatrix}.
\]
If $\rho\neq 0$, then $\rho\in\spec(\mbi H^\sharp)$ iff $\rho^2\in\spec(\mbi K^\sharp)$.
Set $\lambda:=\rho^2$. 
We now consider the low rank matrix $\mbi P=\mbi A_0\mbi\Theta \mbi B_0^\top$ where
$\mbi A_0:=[\mbi a_0,\mbi a_1,\mbi a_2]\in\R^{N\times 3}$ and
$\mbi B_0:=[\mbi b_0,\mbi b_1,\mbi b_2]\in\R^{n\times 3}$.
Then
\[
\mbi H^\sharp=\mbi H_0+\mbi\Delta,\qquad
\mbi\Delta:=\begin{pmatrix}\mbi 0 & \mbi P\\ \mbi P^\top & \mbi 0\end{pmatrix}.
\]
Factor $\mbi\Delta=\mbi U\mbi C\mbi U^\top$ with
\[
\mbi U:=\begin{pmatrix}\mbi A_0 & \mbi 0\\ \mbi 0 & \mbi B_0\end{pmatrix}\in\R^{(N+n)\times 6},\qquad
\mbi C:=\begin{pmatrix}\mbi 0 & \mbi\Theta\\ \mbi\Theta& \mbi 0\end{pmatrix}\in\R^{6\times 6}.
\]
Under \eqref{eq:away_spec}, $\rho\notin\spec(\mbi H_0)$ so $(\mbi H_0-\rho\mbi I)$ is invertible and
\[
\det(\mbi H^\sharp-\rho\mbi I)=\det(\mbi H_0-\rho\mbi I)\,
\det\Big(\mbi I_6+\mbi C\,\mbi U^\top(\mbi H_0-\rho\mbi I)^{-1}\mbi U\Big).
\]
Hence $\rho\in\spec(\mbi H^\sharp)\setminus\spec(\mbi H_0)$ iff the $6\times 6$ determinant above vanishes. 
For $\lambda=\rho^2\notin\spec(\mbi K_0)\cup\spec(\mbi G\mbi G^\top)$ we have
\[
(\mbi H_0-\rho\mbi I)^{-1}
=
\begin{pmatrix}
-\rho(\mbi G\mbi G^\top-\lambda\mbi I)^{-1} & \mbi G(\mbi K_0-\lambda\mbi I)^{-1}\\
(\mbi K_0-\lambda\mbi I)^{-1}\mbi G^\top & -\rho(\mbi K_0-\lambda\mbi I)^{-1}
\end{pmatrix}
=
\begin{pmatrix}
-\rho\,\mbi Q_L(\lambda) & \mbi G\mbi Q_R(\lambda)\\
\mbi Q_R(\lambda)\mbi G^\top & -\rho\,\mbi Q_R(\lambda)
\end{pmatrix}.
\]
Using the identity $\mbi Q_L(\lambda)\mbi G=\mbi G\mbi Q_R(\lambda)$, by the Schur complement, the $6\times 6$ determinant reduces to the equivalent $3\times 3$ condition
\[
\det\Big(\mbi I_3-\lambda\,\mbi\Theta\,\mbi A_0^\top\mbi Q_L(\lambda)\mbi A_0\,\mbi\Theta\cdot
\mbi B_0^\top\mbi Q_R(\lambda)\mbi B_0\Big)=0,
\]
which is exactly \eqref{eq:additive_det_eq}. 
The final uniform statement over $\mathcal D$ follows because all steps are analytic in $\lambda$ on the domain
\eqref{eq:away_spec}.
\end{proof}

\begin{lemma}[Scalar additive outlier equations]\label{lem:additive_scalar_eq}
Let \(\lambda\in\mathcal D\) and \(s=s(\lambda)\) satisfy \eqref{eq:s_lambda_def}.
Then, asymptotically, the determinant equation \eqref{eq:additive_det_eq} reduces to three scalar equations:
\begin{equation}\label{eq:scalar_equations}
1-\beta_k\,T(s)=o_\P(1),
\qquad k\in\{0,1,2\}.
\end{equation}
Equivalently, each additive spike of strength \(\beta\in\{\beta_0,\beta_1,\beta_2\}\) produces an outlier eigenvalue at
\[
\lambda=z(s)+o_\P(1)
\quad\text{for any real } s \text{ with } z'(s)>0 \text{ satisfying }\ \beta\,T(s)=1.
\]
\end{lemma}

\begin{proof}
Fix \(\lambda\in\mathcal D\) and let \(s=s(\lambda)\). 
By Lemma~\ref{lem:left_resolvent}, using \(\va_0=\mbi{1}/\sqrt N\), \(\va_1=\vg_1/\|\vg_1\|\), \(\va_2=\vg_2/\|\vg_2\|\),
we have
\[
\vA_{k\ell}(\lambda)
=
\va_k^\top\vQ_L(\lambda)\va_\ell
=
\delta_{k\ell}\,s + o_\P(1),
\qquad k,\ell\in\{0,1,2\}.
\]
Thus \(\vA(\lambda)=s\,\vI_3+o_\P(1)\).
 
By Lemma~\ref{lem:right_resolvent} and orthogonality \(\vv_k^\top\vu=0\), \(\vv_k^\top\vm=o_\P(1)\),
we obtain
\[
\vB_{k\ell}(\lambda)=\vb_k^\top\vQ_R(\lambda)\vb_\ell
=
\delta_{k\ell}\,m_\mu(\lambda)+o_\P(1),
\qquad k,\ell\in\{0,1,2\}.
\]
Thus \(\vB(\lambda)=m_\mu(\lambda)\,\vI_3+o_\P(1)\).
 
Plugging the above approximations of \(\vA(\lambda) \) and \(\vB(\lambda) \) into \eqref{eq:additive_det_eq} gives
\[
\det\Big(\vI_3-\lambda\,s\,m_\mu(\lambda)\,\diag(\beta_0,\beta_1,\beta_2)\Big)=o_\P(1),
\]
hence, $\lambda$ satisfies at least one of the following equations
\[
1-\beta_k\,\lambda\,s\,m_\mu(\lambda)=o_\P(1),
\] for $k=0,1,2.$
By Lemmas~\ref{lem:z_parametrizes_complement} and~\ref{lem:T_relation}, \(\lambda s m_\mu(\lambda)=T(s)\), proving \eqref{eq:scalar_equations}.
\end{proof}

\begin{lemma}[Eigenvector alignment for additive outliers]\label{lem:additive_alignment}
Fix a distinct spike strength value \(\beta_\star\in\{\beta_0,\beta_1,\beta_2\}\), and define $J_\star:=\{k\in\{0,1,2\}:\beta_k=\beta_\star\},$ with $r_\star:=|J_\star|.$ Let
\[
\vPi_{B,\star}:=\sum_{k\in J_\star}\vb_k\vb_k^\top,
\qquad
\vPi_{A,\star}:=\sum_{k\in J_\star}\va_k\va_k^\top.
\]
Assume that \(s_\star\in\R\) satisfies $\beta_\star T(s_\star)=1,$ $z'(s_\star)>0,$ and
$
\lambda_\star:=z(s_\star)\in \mathbb{R}\setminus (\supp{\mu}\cup \{0\}).
$
Then, there exists a sufficiently small $\delta>0$ with interval $I_\delta=(\lambda_\star-\delta,\lambda_\star+\delta)\subset\mathbb{R}\setminus (\supp{\mu}\cup \{0\})$ such that the spectral projector \(\widehat\vP_\star\) of \(\vK^\sharp\) onto interval $I_\delta$ and the corresponding spectral projector \(\widehat\vP_\star^{(L)}\) of \(\vY^\sharp \vY^{\sharp\top}\) onto $I_\delta$ satisfy:
\begin{equation}\label{eq:right_block_overlap}
\Tr\!\bigl(\widehat\vP_\star \vPi_{B,\star}\bigr)
=
r_\star\cdot \gamma_R(s_\star)+o_\P(1),
\end{equation}
\begin{equation}\label{eq:left_block_overlap}
\Tr\!\bigl(\widehat\vP_\star^{(L)} \vPi_{A,\star}\bigr)
=
r_\star\cdot \gamma_L(s_\star)+o_\P(1),
\end{equation}
where
\begin{equation}\label{eq:gammaR_forms}
\gamma_R(s)
:=
\frac{m_\mu(z(s))\,z'(s)}{\beta_\star T'(s)}
=
-\frac{T(s)}{\beta_\star s^2 T'(s)}\,\varphi(s)
=
-\frac{\varphi(s)}{\beta_\star^{\,2}s^2T'(s)}.
\end{equation}
and
\begin{equation}\label{eq:gammaL_forms}
\gamma_L(s)
:=
\frac{s\,z'(s)}{\beta_\star T'(s)}
=
-\frac{z(s)}{\beta_\star T'(s)}\,\varphi(s).
\end{equation}
In particular, if \(r_\star=1\) and \(J_\star=\{k\}\), then the unique outlier eigenvalue inside \(I_\delta\)
is simple, and if \(\widehat\vv_k\) denotes its unit right eigenvector, then
\begin{equation}\label{eq:simple_right_overlap_varphi}
|\langle \widehat\vv_k,\vb_k\rangle|^2
=
\gamma_R(s_\star)+o_\P(1).
\end{equation}
\end{lemma}

\begin{proof} 
By Lemma~\ref{lem:additive_scalar_eq}, every additive outlier of \(\vK^\sharp\) in \(\cD\) is asymptotically
described by a real solution \(s\) of $\beta\,T(s)=1$ such that $z'(s)>0$
for some \(\beta\in\{\beta_0,\beta_1,\beta_2\}\). Let $\lambda=z(s)$.
Since \(z'(s_\star)>0\), Lemma~\ref{lem:z_parametrizes_complement} implies that \(\lambda_\star=z(s_\star)\) lies in a connected
component of \(\R\setminus\supp{\mu}\), and \(s(\lambda)\) is real-analytic near \(\lambda_\star\).

Define $f_\star(\lambda):=1-\beta_\star T(s(\lambda)).$
Then \(f_\star(\lambda_\star)=0\), and
\begin{equation}\label{eq:fstarprime}
f_\star'(\lambda_\star)
=
-\beta_\star T'(s_\star)\,\frac{ds}{d\lambda}(\lambda_\star)
=
-\beta_\star\frac{T'(s_\star)}{z'(s_\star)}.
\end{equation}
Hence \(f_\star'(\lambda_\star)\neq 0\) whenever \(T'(s_\star)\neq 0\), so \(f_\star\) has a simple zero at \(\lambda_\star\). Now choose a sufficient small constant \(\delta>0\) such that:
 $I_\delta\subset\cD$; \(\lambda_\star\) is the only zero of \(f_\star\) inside \(I_\delta\);
 and for every \(\beta\neq \beta_\star\), the function \(1-\beta T(s(\lambda))\) does not vanish on \(I_\delta\).
By Lemma~\ref{lem:additive_scalar_eq}, the eigenvalues of \(\vK^\sharp\) inside \(I_\delta\) are precisely the
outlier cluster associated with the repeated block \(J_\star\); its total multiplicity is \(r_\star\).
 
Let
\[
\vR^\sharp(\lambda):=(\vK^\sharp-\lambda \vI_n)^{-1},
\qquad
\vQ_R(\lambda):=(\vK_0-\lambda \vI_n)^{-1},
\qquad
\vQ_L(\lambda):=(\vG\vG^\top-\lambda \vI_N)^{-1}.
\]
Recall from Lemma~\ref{lem:finite_det_eq} that
\[
\vY^\sharp=\vG+\vP,\qquad \vP=\vA_0\vTheta \vB_0^\top,
\qquad
\vA_0=[\va_0,\va_1,\va_2],\quad \vB_0=[\vb_0,\vb_1,\vb_2].
\]
Introduce the linearizations
\[
\vH^\sharp:=
\begin{pmatrix}
0 & \vY^\sharp\\
\vY^{\sharp\top} & 0
\end{pmatrix},
\qquad
\vH_0:=
\begin{pmatrix}
0 & \vG\\
\vG^\top & 0
\end{pmatrix},
\]
and
\[
\vU:=
\begin{pmatrix}
\vA_0 & 0\\
0 & \vB_0
\end{pmatrix}\in\R^{(N+n)\times 6},
\qquad
\vC:=
\begin{pmatrix}
0 & \vTheta\\
\vTheta & 0
\end{pmatrix}\in\R^{6\times 6},
\]
so that \(\vH^\sharp=\vH_0+\vU\vC\vU^\top\). Fix \(\rho\) with \(\rho^2=\lambda\), and set
\[
\vR_0(\rho):=(\vH_0-\rho \vI_{N+n})^{-1},
\qquad
\vS(\rho):=\vU^\top \vR_0(\rho)\vU.
\]
By the Woodbury identity,
\begin{align}
(\vH^\sharp-\rho \vI)^{-1}
&=
\vR_0(\rho)-\vR_0(\rho)\vU\vC\bigl(\vI_6+\vS(\rho)\vC\bigr)^{-1}\vU^\top \vR_0(\rho).
\label{eq:woodbury_H}
\end{align}
Multiplying \eqref{eq:woodbury_H} on the left by \(\vU^\top\) and on the right by \(\vU\), we obtain the exact identity
\begin{equation}\label{eq:projected_linearized_resolvent}
\vU^\top (\vH^\sharp-\rho \vI)^{-1}\vU
=
\bigl(\vI_6+\vS(\rho)\vC\bigr)^{-1}\vS(\rho),
\end{equation}
since $\vS-\vS\vC(\vI_6+\vS\vC)^{-1}\vS
=
(\vI_6+\vS\vC)^{-1}\vS.$ Using the block formula for \((\vH_0-\rho \vI)^{-1}\) from Lemma~\ref{lem:finite_det_eq},
\[
\vS(\rho)=
\begin{pmatrix}
-\rho\,\vA(\lambda) & \vX(\lambda)\\
\vX(\lambda)^\top & -\rho\,\vB(\lambda)
\end{pmatrix},
\qquad
\vX(\lambda):=\vA_0^\top \vG \vQ_R(\lambda)\vB_0
           =\vA_0^\top \vQ_L(\lambda)\vG \vB_0.
\]
Here $\vA(\lambda)$ and $\vB(\lambda)$ are defined by \eqref{eq:A(lambda)}.
Also, the lower-right block of \((\vH^\sharp-\rho \vI)^{-1}\) equals $-\rho\,(\vK^\sharp-\lambda \vI_n)^{-1}=-\rho\,\vR^\sharp(\lambda).$
Therefore, taking the lower-right \(3\times3\) block in \eqref{eq:projected_linearized_resolvent}, there exists
a \(3\times3\) matrix-valued function \(\vH_B(\lambda)\), analytic on a neighborhood of \(\cD\), such that
\begin{equation}\label{eq:projected_resolvent_exact}
\vB_0^\top \vR^\sharp(\lambda)\vB_0
=
\vB(\lambda)\,\vM(\lambda)^{-1}
+
\vH_B(\lambda),
\qquad
\vM(\lambda):=\vI_3-\lambda\,\vTheta\,\vA(\lambda)\,\vTheta\,\vB(\lambda).
\end{equation}
Hence every pole of \(\vB_0^\top \vR^\sharp(\lambda)\vB_0\) inside \(\cD\) comes from a zero of \(\det\vM(\lambda)\),
i.e., from an additive outlier of \(\vK^\sharp\).
 
By Lemmas~\ref{lem:left_resolvent} and~\ref{lem:right_resolvent}, uniformly for \(\lambda\in \cD\subset\C\),
\[
\vA(\lambda)=s(\lambda)\vI_3+o_\P(1),
\qquad
\vB(\lambda)=m_\mu(\lambda)\vI_3+o_\P(1).
\]
Therefore, uniformly on \(\cD\), we have
\begin{equation}\label{eq:M_scalarized}
\vM(\lambda)
=
\diag\!\bigl(1-\beta_0T(s(\lambda)),\,1-\beta_1T(s(\lambda)),\,1-\beta_2T(s(\lambda))\bigr)
+o_\P(1),
\end{equation}
because of \(\lambda s(\lambda)m_\mu(\lambda)=T(s(\lambda))\).

Restrict now to the block \(J_\star\) and let us define $E_\star:=\sum_{k\in J_\star}\mathbf e_k\mathbf e_k^\top\in\R^{3\times 3},$
where \(\mathbf e_0,\mathbf e_1,\mathbf e_2\) are the standard basis vectors of \(\R^3\).
 Since \(\beta_k=\beta_\star\) for all \(k\in J_\star\),
\[
E_\star \vM(\lambda) E_\star
=
\bigl(1-\beta_\star T(s(\lambda))\bigr) E_\star + o_\P(1),
\qquad \lambda\in \cD.
\]
Hence,
\begin{equation}\label{eq:block_inverse}
E_\star \vM(\lambda)^{-1} E_\star
=
\frac{1}{1-\beta_\star T(s(\lambda))}\,E_\star + o_\P(1),
\qquad \lambda\in \cD.
\end{equation} Combining \eqref{eq:projected_resolvent_exact}, \(\vB(\lambda)=m_\mu(\lambda)\vI_3+o_\P(1)\), and
\eqref{eq:block_inverse}, we obtain
\begin{equation}\label{eq:block_projected_resolvent}
E_\star \vB_0^\top \vR^\sharp(\lambda)\vB_0 E_\star
=
\frac{m_\mu(\lambda)}{1-\beta_\star T(s(\lambda))}\,E_\star + \vH_\star(\lambda) + o_\P(1),
\qquad \lambda\in \cD,
\end{equation}
where \(\vH_\star(\lambda):=E_\star \vH_B(\lambda)E_\star\) is analytic on a neighborhood of \(\cD\).
 
By the contour formula for spectral projectors,
\[
\widehat\vP_\star
=
-\frac{1}{2\pi i}\oint_{\Gamma_\star}\vR^\sharp(\lambda)\,d\lambda,
\]for some circle $\Gamma_\star\subset\cD$ with $\Gamma_\star\cap\R=\{\lambda_\star-\delta,\lambda_\star+\delta\}$.
Therefore
\begin{align}
\Tr\!\bigl(\widehat\vP_\star \vPi_{B,\star}\bigr)
&=
-\frac{1}{2\pi i}\oint_{\Gamma_\star}
\Tr\!\bigl(\vPi_{B,\star}\vR^\sharp(\lambda)\bigr)\,d\lambda=
-\frac{1}{2\pi i}\oint_{\Gamma_\star}
\Tr\!\bigl(E_\star \vB_0^\top \vR^\sharp(\lambda)\vB_0 E_\star\bigr)\,d\lambda.
\label{eq:right_projector_contour}
\end{align}
Insert \eqref{eq:block_projected_resolvent}. The analytic term \(\vH_\star(\lambda)\) has zero contour integral, and
\(\Tr(E_\star)=r_\star\), so
\begin{equation}\label{eq:right_projector_residue1}
\Tr\!\bigl(\widehat\vP_\star \vPi_{B,\star}\bigr)
=
-\frac{1}{2\pi i}\oint_{\Gamma_\star}
\frac{r_\star\,m_\mu(\lambda)}{1-\beta_\star T(s(\lambda))}\,d\lambda
+o_\P(1).
\end{equation}
Now \(1-\beta_\star T(s(\lambda))\) has a simple zero at \(\lambda_\star\), and by \eqref{eq:fstarprime}
\[
\frac{d}{d\lambda}\Bigl(1-\beta_\star T(s(\lambda))\Bigr)\Big|_{\lambda=\lambda_\star}
=
-\beta_\star\frac{T'(s_\star)}{z'(s_\star)}.
\]
Hence, $-\mathrm{Res}_{\lambda=\lambda_\star}
\frac{m_\mu(\lambda)}{1-\beta_\star T(s(\lambda))}
=
\frac{m_\mu(\lambda_\star)\,z'(s_\star)}{\beta_\star T'(s_\star)}$ and Cauchy Residue Theorem implies that 
\[
\Tr\!\bigl(\widehat\vP_\star \vPi_{B,\star}\bigr)
=
r_\star\cdot
\frac{m_\mu(\lambda_\star)\,z'(s_\star)}{\beta_\star T'(s_\star)}
+o_\P(1),
\]
which is \eqref{eq:right_block_overlap}.
By Lemma~\ref{lem:T_relation}, $T(s)=z(s)\,s\,m_\mu(z(s))$ and $m_\mu(z(s))
=
\frac{T(s)}{z(s)s}.$
Substituting into \eqref{eq:right_block_overlap},
\[
\gamma_R(s)
=
\frac{T(s)}{\beta_\star z(s)s}\cdot \frac{z'(s)}{T'(s)}=-\frac{\varphi(s_\star)}{\beta_\star^{\,2}s_\star^2T'(s_\star)}
\]
since $\varphi(s):=-\frac{s\,z'(s)}{z(s)}$
and \(\beta_\star T(s_\star)=1\).

The proof on the left singular vectors is identical, replacing \(\vK^\sharp\) by \(\vY^\sharp\vY^{\sharp\top}\), \(\vB_0\) by \(\vA_0\),
and \(\vB(\lambda)\) by \(\vA(\lambda)\). Since \(\vA(\lambda)=s(\lambda)\vI_3+o_\P(1)\), the numerator becomes
\(s(\lambda)\) instead of \(m_\mu(\lambda)\), yielding
\[
\Tr\!\bigl(\widehat\vP_\star^{(L)} \vPi_{A,\star}\bigr)
=
r_\star\cdot \frac{s_\star z'(s_\star)}{\beta_\star T'(s_\star)}+o_\P(1).
\]
If \(r_\star=1\), then \(\widehat\vP_\star=\widehat\vv_k\widehat\vv_k^\top\), and hence $|\langle \widehat\vv_k,\vb_k\rangle|^2
=
\Tr(\widehat\vP_\star \Pi_{B,\star})
=
\gamma_R(s_\star)+o_\P(1),$
which proves \eqref{eq:simple_right_overlap_varphi}.
\end{proof}

Notice that from Lemmas~\ref{lem:uninf_branch_identity} and \ref{lem:K0_bulk_and_tau_outlier}, we observe that the mean spike generated by $\beta_0$ are asymptotically identified with the separated covariance outlier of $\vK_0$ generated by $\tau$ if they exist. The following lemma shows that this mean spike is collision-safe and produces a stable additive outlier of $\vK^\sharp$ with non-trivial overlap on the mean spike direction $\vb_0$, even in the presence of the linear outlier when $\beta_{\rm lin}>0$.
\begin{lemma}\label{lem:mean_collision_cluster}
Assume $c_\sigma\neq0$, $z'(s_{\rm un})>0$, and, when $\beta_{\rm lin}>0$ and $z'(s_{\rm lin})>0$, assume $s_{\rm un}\neq s_{\rm lin}$. Let $\lambda_{\rm un}:=z(s_{\rm un}).$
Then there exists a sufficiently small deterministic constant $\delta>0$ such that, with probability tending to one,
\begin{equation}\label{eq:I_un_def_patch}
I_{\rm un,\delta}:=(\lambda_{\rm un}-\delta,\lambda_{\rm un}+\delta)
\end{equation}
is disjoint from $\supp\mu\cup\{0\}$ and from the linear outlier location $z(s_{\rm lin})$ (when present), and $\vK^\sharp$ has exactly one additive mean-spike eigenvalue in $I_{\rm un,\delta}$.

More precisely, let $\Gamma_{\rm un}$ be the positively oriented circle of radius $\delta/2$ centered at $\lambda_{\rm un}$ and write, for $z\in\Gamma_{\rm un}$,
\[
\vQ_R(z)=\frac{\widehat\vu_\tau\widehat\vu_\tau^\top}{\widehat\lambda_\tau^{(0)}-z}+\vQ_R^\perp(z),
\qquad
\vQ_L(z)=\frac{\widehat\vw_\tau\widehat\vw_\tau^\top}{\widehat\lambda_\tau^{(0)}-z}+\vQ_L^\perp(z),
\]
where $\vQ_R^\perp,\vQ_L^\perp$ are the projected resolvents from Lemma~\ref{lem:tau_outlier_stability_collision}. Then uniformly on $\Gamma_{\rm un}$,
\begin{align}
\vA^\perp(z):=\bigl(\va_k^\top\vQ_L^\perp(z)\va_\ell\bigr)_{k,\ell=0}^2&=s(z)\vI_3+o_\P(1),\label{eq:A_perp_un_patch}\\
\vB^\perp(z):=\bigl(\vb_k^\top\vQ_R^\perp(z)\vb_\ell\bigr)_{k,\ell=0}^2&=m_\mu(z)\vI_3+o_\P(1),\label{eq:B_perp_un_patch}
\end{align}
and therefore, uniformly on $\Gamma_{\rm un}$,
\begin{equation}\label{eq:det_perp_un_patch}
\det\Big(\vI_3-z\,\vTheta\,\vA^\perp(z)\,\vTheta\,\vB^\perp(z)\Big)
=
\bigl(1-\tau T(s(z))\bigr)\bigl(1-\beta_{\rm lin}T(s(z))\bigr)^2+o_\P(1).
\end{equation}
Let $\widehat\vP_{{\rm mean},{\rm un}}$ be the spectral projector of $\vK^\sharp$ onto the unique additive outlier in $I_{\rm un,\delta}$. Then,
\begin{align}
\Tr\bigl(\widehat\vP_{{\rm mean},{\rm un}}\,\vb_0\vb_0^\top\bigr)&=\gamma_R(s_{\rm un})+o_\P(1),\label{eq:mean_cluster_b0_overlap_patch}\\
\Tr\bigl(\widehat\vP_{{\rm mean},{\rm un}}\,\vu\vu^\top\bigr)&=o_\P(1).\label{eq:mean_cluster_u_overlap_patch}
\end{align}
\end{lemma}

\begin{proof}
Since $z'(s_{\rm un})>0$, Lemma~\ref{lem:uninf_branch_identity} implies $\tau>\tau_{\rm crit}$ and hence $\vK_0$ has a separated covariance outlier $\widehat\lambda_\tau^{(0)}=\lambda_{\rm un}+o_\P(1)$ by Lemma~\ref{lem:K0_bulk_and_tau_outlier}. Choose $\delta>0$ so small that the interval \eqref{eq:I_un_def_patch} is disjoint from the bulk and from the linear outlier location when the latter exists.

On the contour $\Gamma_{\rm un}$, the raw resolvents are well-defined. By the spectral decomposition of $\vK_0$ and Lemma~\ref{lem:tau_outlier_stability_collision},
\[
\vQ_R(z)=\frac{\widehat\vu_\tau\widehat\vu_\tau^\top}{\widehat\lambda_\tau^{(0)}-z}+\vQ_R^\perp(z),
\qquad
\vQ_L(z)=\frac{\widehat\vw_\tau\widehat\vw_\tau^\top}{\widehat\lambda_\tau^{(0)}-z}+\vQ_L^\perp(z),
\]
and the pole parts contribute only $o_\P(1)$ to quadratic forms against the spike directions because
\[
\max_{k\in\{0,1,2\}}|\vb_k^\top\widehat\vu_\tau|=o_\P(1),
\qquad
\max_{k\in\{0,1,2\}}|\va_k^\top\widehat\vw_\tau|=o_\P(1)
\]
by Lemma~\ref{lem:verify-E34-K0}, while $|\widehat\lambda_\tau^{(0)}-z|\asymp1$ uniformly on $\Gamma_{\rm un}$. Therefore the same proofs as in Lemmas~\ref{lem:left_resolvent} and \ref{lem:right_resolvent}, now applied to the projected resolvents, yield \eqref{eq:A_perp_un_patch} and \eqref{eq:B_perp_un_patch}. The factorization \eqref{eq:det_perp_un_patch} then follows exactly as in Lemma~\ref{lem:additive_scalar_eq}.

Because $s_{\rm un}$ is the unique solution of $\tau T(s)=1$ on the connected component containing $I_{\rm un,\delta}$, because $z'(s_{\rm un})>0$, and because the linear factor is bounded away from zero on $\Gamma_{\rm un}$ by the separation assumption $s_{\rm un}\neq s_{\rm lin}$, Rouch\'e's theorem implies that the projected determinant has exactly one zero inside $\Gamma_{\rm un}$, hence $\vK^\sharp$ has exactly one additive mean-spike eigenvalue in $I_{\rm un,\delta}$.

Finally, \eqref{eq:mean_cluster_b0_overlap_patch} is the same computation as in Lemma~\ref{lem:additive_alignment} with the projected resolvents replacing the raw ones. Also, \eqref{eq:mean_cluster_u_overlap_patch} follows from the same contour representation with the test vector $\vu$, because the only possible residue is proportional to quadratic forms of the type $\vu^\top\vQ_R^\perp(z)\vb_0=o_\P(1)$ and because $\vu^\top\vb_0=o_\P(1)$ by Lemma~\ref{lem:mean_strength_orth}.
\end{proof}

\begin{lemma} \label{lem:merged_uninf_cluster}
Assume $c_\sigma\neq0$, $z'(s_{\rm un})>0$, and, when $\beta_{\rm lin}>0$ and $z'(s_{\rm lin})>0$, assume $s_{\rm un}\neq s_{\rm lin}$. Let $\lambda_{\rm un}=z(s_{\rm un})$ and let $I_{\rm un,\delta}$ be as in Lemma~\ref{lem:mean_collision_cluster}. Then, with probability tending to one, $\vK^\sharp$ has exactly two eigenvalues in $I_{\rm un,\delta}$: one inherited from the covariance outlier of $\vK_0$ and one produced by the additive mean spike. If $\widehat\vP_{\rm un}^\sharp$ denotes the spectral projector onto this two-dimensional cluster, then
\begin{align}
\|\widehat\vP_{\rm un}^\sharp\vu\|^2&=\varphi(s_{\rm un})+o_\P(1),\label{eq:merged_un_u_patch}\\
\left\|\widehat\vP_{\rm un}^\sharp\frac{\vm}{\|\vm\|}\right\|^2&=-\frac{\varphi(s_{\rm un})}{\tau^2 s_{\rm un}^2T'(s_{\rm un})}+o_\P(1),\label{eq:merged_un_m_patch}\\
\vu^\top\widehat\vP_{\rm un}^\sharp\frac{\vm}{\|\vm\|}&=o_\P(1).\label{eq:merged_un_cross_patch}
\end{align}
\end{lemma}

\begin{proof}
By Lemma~\ref{lem:tau_outlier_stability_collision}, the separated covariance outlier of $\vK_0$ persists under the rank-three additive perturbation and yields one eigenvalue of $\vK^\sharp$ inside $I_{\rm un,\delta}$. By Lemma~\ref{lem:mean_collision_cluster}, the additive mean spike yields exactly one further eigenvalue in the same interval. Since the linear spikes are excluded from $I_{\rm un,\delta}$ by construction, there are exactly two eigenvalues in the interval. The $\vu$-alignment comes from the inherited covariance direction. Indeed, by Lemma~\ref{lem:K0_bulk_and_tau_outlier},
\[
|\widehat\vu_\tau^\top\vu|^2=\varphi(s_{\rm un})+o_\P(1).
\]
Since $\widehat\vu_\tau$ lies asymptotically inside $\mathrm{Range}(\widehat\vP_{\rm un}^\sharp)$ by Lemma~\ref{lem:tau_outlier_stability_collision}, we obtain \eqref{eq:merged_un_u_patch}; the additive mean-spike contribution to the $\vu$-mass is negligible by \eqref{eq:mean_cluster_u_overlap_patch}.
Similarly, the $\vm/\|\vm\|$-alignment comes entirely from the additive mean-spike component: the inherited covariance eigenvector contributes only $o_\P(1)$ in the $\vb_0=\vm/\|\vm\|$ direction by \eqref{eq:alignment_right} in Lemma~\ref{lem:K0_bulk_and_tau_outlier} together with Lemma~\ref{lem:mean_strength_orth}, while Lemma~\ref{lem:mean_collision_cluster} yields the main term \eqref{eq:merged_un_m_patch}. The cross term \eqref{eq:merged_un_cross_patch} follows by combining the two preceding decoupling statements.
\end{proof}

\begin{proposition}[No label alignment with any separated outliers]
\label{prop:no_label_alignment}
Under the assumptions of Theorem~\ref{thm:snr_finite_eig}, let $\widehat\vV_{\rm out}$ denote the orthogonal projector onto the direct sum of all separated outlier eigenvalues of $\vK$. Then
\begin{equation}\label{eq:no_label_alignment_main_corrected}
\frac1n\|\widehat\vV_{\rm out}\vy\|^2=\frac1n\vy^\top\widehat\vV_{\rm out}\vy\xrightarrow{\P}0.
\end{equation}
\end{proposition}

\begin{proof}
By Proposition~\ref{prop:Y_rank3_decomp}, it suffices to prove the statement for $\vK^\sharp$. In the separated regime covered by Theorem~\ref{thm:snr_finite_eig}, the outlier projector of $\vK^\sharp$ decomposes as the orthogonal sum of the uninformative cluster projector $\widehat\vP_{\rm un}^\sharp$ from Lemma~\ref{lem:merged_uninf_cluster} and, when present, the linear cluster projector $\widehat\vP_{\rm lin}^\sharp=\widehat\vP_{\star}$ from Lemma~\ref{lem:additive_alignment} by setting $\lambda_\star=z(s_{\rm lin})$ and $r_\star=2$ in the latter. Hence, $\widehat\vV_{\rm out}^\sharp=\widehat\vP_{\rm un}^\sharp+\widehat\vP_{\rm lin}^\sharp.$
It is therefore enough to show
\[
\frac1n\vy^\top\widehat\vP_{\rm un}^\sharp\vy=o_\P(1),
\qquad
\frac1n\vy^\top\widehat\vP_{\rm lin}^\sharp\vy=o_\P(1).
\]
Set $\bar\vy:=\vy/\sqrt n$. By the XOR construction,
\[
\bar\vy^\top\vu=0,
\qquad
\bar\vy^\top\vv_1=0,
\qquad
\bar\vy^\top\vv_2=0.
\]
Moreover Lemma~\ref{lem:mean_strength_orth} gives $\bar\vy^\top\vm/\|\vm\|=o_\P(1)$. Recall that $s_{\rm cov}=-1/\Lambda_\tau$, $s_{\rm mean}=T^{-1}(1/\beta_{\rm mean})$ and $s_{\rm lin}=T^{-1}(1/\beta_{\rm lin})$. Denote $\lambda_0=z(s_{\rm mean})$, $\lambda_1=\lambda_2=z(s_{\rm lin})$, and $\lambda_3=z(s_{\rm cov})$. Lemma~\ref{lem:uninf_branch_identity} implies $\lambda_0=\lambda_3$. 

For the linear outliers, consider $\lambda=\lambda_1=\lambda_2$. Then, by Lemmas~\ref{lem:tau_outlier_stability_collision} and~\ref{lem:additive_scalar_eq}, there exists a small constant $\delta>0$ such that for sufficiently large $n$, $\dist{I_\delta,\supp{\mu}\cup\{0\}}>\delta$ and all eigenvalues of $\vK^\sharp$ inside $I_\delta$ will converge to $\lambda$ in probability, where $I_\delta:=(\lambda-\delta,\lambda+\delta).$
Let \(\widehat \vV_\lambda^\sharp\) be the spectral projector of \(\vK^\sharp\) onto  outlier cluster in $I_\delta$,
and let \(\Gamma_\lambda\subset\C\) be a small contour enclosing that cluster and no other part of \(\spec(\vK^\sharp)\).
Then
\[
\widehat \vV_\lambda^\sharp
=
-\frac{1}{2\pi i}\oint_{\Gamma_\lambda} (\vK^\sharp-z \vI)^{-1}\,dz.
\]
Hence
\begin{equation}\label{eq:y_projector_contour}
\bar \vy^\top \widehat \vV_\lambda^\sharp \bar \vy
=
-\frac{1}{2\pi i}\oint_{\Gamma_\lambda} \bar \vy^\top (\vK^\sharp-z \vI)^{-1}\bar \vy\,dz.
\end{equation}

Now use the same linearization trick as in Lemma~\ref{lem:additive_alignment}. Recall that \(\vA_0=[\va_0,\va_1,\va_2]\), \(\vB_0=[\vb_0,\vb_1,\vb_2]\), and \(\vTheta=\diag(\sqrt{\beta_0},\sqrt{\beta_1},\sqrt{\beta_2})\).
Then one obtains the exact projected-resolvent representation
\begin{equation}\label{eq:resolvent_woodbury_label}
(\vK^\sharp-z \vI)^{-1}
=
\vQ_R(z)
+
\vQ_R(z)\vB_0 \,\vS(z)\,\vB_0^\top \vQ_R(z),
\end{equation}
where $\vS(z)
=
z\vTheta \vA(z)\vTheta \bigl(\vI_3-z \vTheta \vA(z)\vTheta \vB(z)\bigr)^{-1}$ is a \(3\times3\) matrix-valued meromorphic function, and all poles inside \(\cD\) come from the additive outlier equation described by Lemma~\ref{lem:additive_scalar_eq}. Hence, $\vS(z)$ is analytic and bounded on $\Gamma_\lambda$.
Insert \eqref{eq:resolvent_woodbury_label} into \eqref{eq:y_projector_contour}. The possible pole term is
\[
\bar \vy^\top \widehat \vV_\lambda^\sharp \bar \vy
=
-\frac{1}{2\pi i}\oint_{\Gamma_\lambda} \bar \vy^\top \vQ_R(z)\bar \vy\,dz-\frac{1}{2\pi i}\oint_{\Gamma_\lambda}\bar \vy^\top \vQ_R(z)\vB_0 \,\vS(z)\,\vB_0^\top \vQ_R(z)\bar \vy dz.
\]
By \eqref{eq:right_v}, $\bar\vy\vQ_R(z)\bar\vy=m_\mu(z)+o_\P(1)$ and 
\[
\vB_0^\top \vQ_R(z)\bar\vy
=
\begin{bmatrix}
\vb_0^\top \vQ_R(z)\bar \vy\\
\vb_1^\top \vQ_R(z)\bar \vy\\
\vb_2^\top \vQ_R(z)\bar \vy
\end{bmatrix}
=
o_\P(1)
\]
uniformly on the contour \(z\in\Gamma_\lambda\).  
Since $m_\mu(z)$ is analytic on \(\Gamma_\lambda\), its contour integral vanishes. Hence $\bar \vy^\top \widehat \vP_{\rm lin}^\sharp \bar \vy = o_\P(1).$

For the uninformative eigenvalues, let $\Gamma_{\rm un}$ be the contour from Lemma~\ref{lem:mean_collision_cluster}. Using the resolvent decomposition into the covariance pole plus the projected resolvent,
\[
\vQ_R(z)=\frac{\widehat\vu_\tau\widehat\vu_\tau^\top}{\widehat\lambda_\tau^{(0)}-z}+\vQ_R^\perp(z),
\]
we note first that $\bar\vy^\top\widehat\vu_\tau=o_\P(1)$ by \eqref{eq:alignment_right} in Lemma~\ref{lem:K0_bulk_and_tau_outlier}, because $\bar\vy^\top\vu=0$. Hence, the covariance pole contributes only $o_\P(1)$ to the contour integral for $\bar\vy^\top\widehat\vP_{\rm un}^\sharp\bar\vy$. Similarly as the linear outlier case, we have that uniformly on $\Gamma_{\rm un}$, $\bar\vy^\top\vQ_R^\perp(z)\bar\vy=m_\mu(z)+o_\P(1),$ and $\bar\vy^\top\vQ_R^\perp(z)\vb_k=o_\P(1),$ for $k=0,1,2$. Thus, the contour integral of the analytic main term vanishes and the remainder is $o_\P(1)$. Therefore $\bar\vy^\top\widehat\vP_{\rm un}^\sharp\bar\vy=o_\P(1)$. Combining the two clusters yields \eqref{eq:no_label_alignment_main_corrected} for $\vK^\sharp$, and hence also for $\vK$.
\end{proof}

\subsection{Proof of Theorem \ref{thm:snr_finite_eig}}\label{subsec:main_theorem}
By Proposition~\ref{prop:Y_rank3_decomp}, $\|\vK-\vK^\sharp\|=o_\P(1)$. Therefore it suffices to analyze spectrum of $\vK^\sharp$.

Item \textbf{(i)} follows from Lemma~\ref{lem:K0_bulk_and_tau_outlier}(i): the bulk ESD of $\vK_0$ is $\mu$, and the finite-rank perturbation from $\vK_0$ to $\vK^\sharp$ does not change the limiting ESD.

For item \textbf{(ii)}, Lemma~\ref{lem:uninf_branch_identity} shows that the covariance parameter and the mean parameter are identical on the uninformative branch, i.e. $s_{\rm cov}=s_{\rm mean}=s_{\rm un}$.
Thus the covariance and mean mechanisms collide at the same first-order location $\lambda_{\rm un}=z(s_{\rm un})$. Lemma~\ref{lem:merged_uninf_cluster} proves that $\vK^\sharp$ has exactly two eigenvalues in a small interval around $\lambda_{\rm un}$ and that the corresponding spectral projector satisfies \eqref{eq:thm_un_proj_u}--\eqref{eq:thm_un_proj_cross}. By Proposition~\ref{prop:Y_rank3_decomp}, the same statements hold for $\vK$.

Next, for item \textbf{(iii)}, the linear spikes are away from the covariance outlier by $s_{\rm un}\neq s_{\rm lin}$. Therefore Lemmas~\ref{lem:additive_scalar_eq} and \ref{lem:additive_alignment} apply directly with $\beta_1=\beta_2=\beta_{\rm lin}$, giving the two linear outlier eigenvalues at $z(s_{\rm lin})$ and the projector-overlap formula \eqref{eq:alignment_finite_snr}.

Finally, item \textbf{(iv)} is proved by Proposition~\ref{prop:no_label_alignment}.

\section{Kernel Spectral Clustering for XOR}\label{app:kernel_clustering}
In this section, we present the performance of the kernel spectral clustering for XOR in the finite-SNR proportional limit regime given by Theorem \ref{thm:kernel_cluster}. Results in this section on kernel spectral clustering are applications of the main results of \citet{couillet2016kernel}.

Recall the  four-component XOR Gaussian mixture  data $\vX \in \R^{d\times n}$ defined in \eqref{eq:XOR}, with cluster means deined in \eqref{eq:cluster_means}
and common covariance $\vC=\ident_d$. Define the matrix of cluster means $\vM:=[\vmu_1,\vmu_2,\vmu_3,\vmu_4] \in \mathbb{R}^{d\times 4}$ from \eqref{eq:cluster_means}, and cluster-indicator matrix $\vJ$
\begin{equation*}
    \vJ\in\{0,1\}^{n\times4},\qquad J_{i,a}=1\iff \vx_i\text{ belongs to cluster }a.
\end{equation*}
Each cluster $a\in [4]$ has size $n_a=n/4$.  For an affinity function $f:\mathbb{R} \to \mathbb{R}$, define the Euclidean distance kernel matrix $\Kf \in \R^{n \times n}$ as
\begin{equation}
    (\Kf)_{ij} :=f\Bigl( \|\vx_i - \vx_j \|^2 \Bigr), \qquad i,j \in [n]
\end{equation}
We consider the normalized Laplacian of $\Kf$, which is defined as
\[
\vL:=n\,\vD^{-1/2}\Kf\,\vD^{-1/2}, 
\]
where $\vD=\diag( \Kf \boldone_n )$ is the degree matrix. Similar to the results on the spectra of the CK in Section \ref{sec:finitesnr}, the analysis of the eigenvalues of $\vL$ relies on a LE for $\vL$. We refer to Theorem 1 of \citet{couillet2016kernel} for more details.

It is not difficult to prove that under Assumption \ref{assump:asymptotics}, the pairwise distances between distinct samples $\vx_i$ and $\vx_j$ of the XOR dataset concentrate around $\tau=2$ (see Lemma \ref{lem:xor_orthonormal} for details). The core of the analysis of $\vL$ then relies upon a careful entry-wise Taylor expansion of $\Kf$ around the limiting value $\tau$. We therefore make the following assumption on the kernel function $f:\R\to \R$.
\begin{assum}
    \label{assump:kernelfunction}
 Let $\tau=2$. The kernel function $f$ is three-times differentiable in some neighborhood of $\tau$. Furthermore,  $f(\tau)>0$ and $f'(\tau) \ne 0$.  
\end{assum}

\begin{theorem}[A detailed statement of Theorem \ref{thm:kernel_cluster}] \label{thm:L_eigenvalues}
Suppose $\snr$ $r^2=\Theta(1)$ and let Assumptions \ref{assump:asymptotics} and \ref{assump:kernelfunction} hold. Then we have that:
\begin{enumerate}[label=\textbf{(\roman*)},leftmargin=*]
    \item \textbf{(BBP for two linear informative outliers.)} If $r^{2} >2  \sqrt{\psi^{-1} }$, then there exists two isolated eigenvalues of $\vL$, which are asymptotically approximated by
\begin{equation}
-\frac{2 f'(\tau)}{f(\tau)}\rho + \frac{ f(0)-f(\tau) +\tau f'(\tau)}{f(\tau)} \label{eq:L_eig}
\end{equation}
in probability, where $\rho = \psi\bigl( 1+\frac{r^{2}}{2} \bigr)+ \frac{ 2+r^{2}}{r^{2}}$.
\item \textbf{(BBP for additional uninformative outlier.)} Let $\ell_{+} =  \frac{ 5f'(\tau)}{4f(\tau) } -\frac{ f''(\tau) }{f'(\tau) }$. If, in addition,  $(1-\ell_+)^2 > \psi^{-1}$, then there is an additional corresponding isolated eigenvalue of $\vL$, which is approximated by
\begin{equation}\label{eq:L_eig2}
   \lambda_+:= -\frac{ 2f'(\tau)}{f(\tau)} \rho_+ + \frac{ f(0) - f(\tau) + \tau f'(\tau) }{f(\tau)},
\end{equation}
in probability, where $\rho_+ = \ell_+\bigl( \psi - \frac{1}{1-\ell_+}\bigr)$.
 \item \textbf{(Eigenvector alignment.)} If $\rho$ is such an isolated eigenvalue of $\vL$ corresponding to \eqref{eq:L_eig}, let $\vPi_\rho$ denote the orthogonal projection onto the two-dimensional eigenspace corresponding to $\rho$. Then we have, in probability,
 \begin{equation}
     \frac1d \vJ^\top \vPi_\rho \vJ \rightarrow \frac{1}{1+\frac{r^2}{2}} \Bigl( \psi -\frac{4}{r^4} \Bigr) \frac{r^2}{16} \begin{pmatrix}
         1 & -1 \\
         -1 & 1
     \end{pmatrix} \otimes \ident_2
 \end{equation}
  where $\otimes$ is the Kronecker product. In particular, 
$ \frac1n \vy^\top \vPi_\rho \vy \xrightarrow{\P}0$, where $\vy$ denotes the XOR-labels.
\end{enumerate}
\end{theorem}

\begin{remark}
The trivial eigenvalue $n$, and the possible additional isolated eigenvalue of $\vL$ corresponding to $\lambda_+$ carry no relevant clustering information. The other non-trivial outlier eigenvalues of $\vL$ characterized in Theorem \ref{thm:L_eigenvalues} (i) turn out to be the only spikes that carry relevant clustering information. However, similar to the eigenvectors of the CK matrix in Theorem~\ref{thm:snr_finite_eig}, the last statement of Theorem \ref{thm:L_eigenvalues} implies that the outlier eigenvectors of the Laplacian $\vL$ are asymptotically orthogonal to the XOR-labels $\vy$. 
\end{remark}

\begin{proof-of-theorem}[\ref{thm:L_eigenvalues}]
    The common covariance $\vC= \ident_d$ of the XOR dataset greatly simplifies the analysis of the spectrum of $\vL$. We apply Corollary 1 from \citet{couillet2016kernel}, which gives the isolating eigenvalues of $\vL$ in terms of the isolated eigenvalues of $\vC+ \vM \mathcal{D}(c) \vM^{\top} $, where $\vM= [\vmu_{1}, \vmu_{2}, \vmu_{3}, \vmu_{4}]$,  $\vc=( \frac14 )_{\alpha=1}^{4}$ is the vector of class proportions, and $\mathcal{D}(\cdot)$ is the diagonal operator. 
    For the XOR dataset, this matrix becomes
    \begin{equation}
    \vC+\vM\mathcal{D}(c) \vM^{\top} = \ident_d + \frac{1}{4} \sum_{k=1}^{4} \vmu_{k} \vmu_{k}^{\top} =\ident_d + \frac{r^{2}}{d}\begin{bmatrix}
        \boldone_{\frac{d}{2}} \boldone_{\frac{d}{2}}^\top & \mathbf{0} \\
        \mathbf{0} &  \boldone_{\frac{d}{2}} \boldone_{\frac{d}{2}}^\top
    \end{bmatrix}. \label{eq:cmdc}
    \end{equation}
    The eigenvalues of \eqref{eq:cmdc} are $\ell=1+\frac{r^2}{2}$ with multiplicity two, and the rest are equal to 1. The phase transition in Corollary 1 of \citet{couillet2016kernel} can therefore be restated as  $| \ell - 1 | = \frac{r^{2}}{2} >  \sqrt{\psi^{-1}}.$
    In this case, by Corollary 1 of \citet{couillet2016kernel} the isolating eigenvalues of $\vL$ are asymptotically well approximated via \eqref{eq:L_eig}
in probability, where $\rho = \psi\bigl(1+\frac{r^{2}}{2} \bigr)+\frac{ 2+r^{2}}{r^{2} }$.
Similarly, define $\ell_+ = \frac{ 5 f'(\tau) }{4 f(\tau) } - \frac{ f''(\tau) }{f'(\tau) }$. The additional isolating eigenvalue of $\vL$ corresponding to $\rho_+ = \ell_+\bigl( \psi - \frac{1}{1-\ell_{+}} \bigr)$ emerges if $ (1-\ell_+)^2 > \psi^{-1}$,
in which case, the corresponding outlier eigenvalue of $\vL$ converges to, in probability, $\lambda_+$ defined in \eqref{eq:L_eig2}.

For eigenvector alignment of the eigenvalues relevant for clustering, we observe the two-dimensional eigenspace corresponding to $\ell=1 +\frac{r^{2} }{2}$ is spanned by the unit-norm eigenvectors $\{ \Upsilon_{1}, \Upsilon_{2} \}$, where
    \begin{equation}
    \Upsilon_{1} = \frac{1}{\sqrt{d/2}}\begin{pmatrix}
        \mathbf{1}_{d/2} \\
        \mathbf{0}_{d/2}
    \end{pmatrix}, \qquad \Upsilon_{2} = \frac{1}{\sqrt{d/2}} \begin{pmatrix}
    \mathbf{0}_{d/2}\\
    \mathbf{1}_{d/2}
    \end{pmatrix}.
    \label{eq:upsilon}\end{equation}
 Let $\Upsilon_{\rho} = ( \Upsilon_{1},  \Upsilon_{2} ) \in \R^{d\times 2}$ be the column-concatenated eigenvectors, then we apply Corollary 2 of \citet{couillet2016kernel} to obtain
\begin{equation}
    \tfrac1d \vJ^{\top} \vPi_{\rho} \vJ =\frac{1}{\ell}\Bigl(\psi-\frac{ 1}{(1-\ell)^{2}} \Bigr)\mathcal{D}(c) \vM^{\top} \Upsilon_{\rho}\Upsilon_{\rho}^{\top}\vM \mathcal{D}(c) + o_\P(1).
    \label{eq:kernel_align}\end{equation}
    Or equivalently,
    \begin{equation}
        \frac1d \vJ^{\top} \vPi_{\rho} \vJ \rightarrow \frac{r^2}{16 + 8r^{2} }\Bigl(\psi- \frac{ 4 } {r^{4}} \Bigr)\begin{pmatrix}
         1 & -1 \\
         -1 & 1
     \end{pmatrix}\otimes \ident_{2}
    \label{eq:cluster_align}  \end{equation}
    in probability. 
    Applying \eqref{eq:cluster_align} to $\vy$, we have
    \begin{equation*}
        \frac{1}{d} \vy^{\top} \vPi_{\rho} \vy \rightarrow \frac{r^2}{16+8r^{2} }\bigl(\psi - \frac{4}{r^4} \bigr) \begin{pmatrix}
    1 \\ -1 \\ 1 \\ -1
\end{pmatrix}^{\top} \begin{pmatrix}
     1 & -1  \\
         -1 & 1  \\
\end{pmatrix} \otimes \ident_2
    \begin{pmatrix}
    1 \\ -1 \\ 1 \\ -1
\end{pmatrix} =0
    \end{equation*}
    in probability. Hence, we complete the proof of this theorem.
\end{proof-of-theorem}

\section{Proof for Large-SNR Regime}
\label{app:large_SNR}

This appendix presents the proof for Theorem~\ref{thm:large}.  
Throughout this section, we denote $r=r_0n^{1/4},$ for some constant $r_0\in(0,\infty)$, and we set $\theta_{\snr}:=r\sqrt{\frac{n}{2d}}$. Denote that 
\begin{equation}\label{eq:large_final_tau_beta}
        \tau=\frac{c_\sigma^2}{2}\psi,\qquad
        \kappa_0=\lim_{n\to\infty}\frac{\theta_{\snr}^4}{n}
        =\frac{r_0^4\psi^2}{4},\qquad
        \alpha:=c_\sigma\sqrt{\kappa_0},\qquad
        \beta:=\beta_{\rm quad}:=\alpha^2=c_\sigma^2\kappa_0 .
\end{equation}
We write
\begin{equation}\label{eq:large_final_notation}
        \overline\vy:=\frac{\vy}{\sqrt n},\qquad
        \vP_{\rm lin}:=\vv_1\vv_1^\top+\vv_2\vv_2^\top,
        \qquad
        \vP_\perp:=\vI_n-\vP_{\rm lin},
        \qquad
        \vK_\perp:=\vP_\perp\vK\vP_\perp .
\end{equation}

\subsection{Projected QE Model}

\begin{lemma}\label{lem:large_final_projected_QE}
Recall $\vY_{\rm QE}$ defined by \eqref{eq:QE_formal}. Let $ \vK_{{\rm QE},\perp}:=(\vY_{\rm QE}\vP_\perp)^\top
        (\vY_{\rm QE}\vP_\perp).$
Then
\begin{equation}\label{eq:QE_projected}
          \|\vY\vP_\perp-\vY_{\rm QE}\vP_\perp\|=O_\prec(n^{-1/4}),
        \qquad
        \|\vK_\perp-\vK_{{\rm QE},\perp}\|=O_\prec(n^{-1/4}).
\end{equation}
Moreover,
\begin{equation}\label{eq:large_final_YQE_perp}
        \vY_{\rm QE}\vP_\perp
        =\vY_0\vP_\perp
        +\frac{\theta_{\snr}^2c_\sigma}{2\sqrt N}
        \left(\vg_1^{\odot2}\vv_1^{\odot2\top}
             +\vg_2^{\odot2}\vv_2^{\odot2\top}\right),
\end{equation}
because \(\vP_\perp\vv_k=0\) and, in the XOR basis,
\(\vP_\perp\vv_k^{\odot2}=\vv_k^{\odot2}\).
\end{lemma}

\begin{proof}
The QE error in \eqref{eq:QE_projected} follows from Proposition~\ref{prop:approx} after multiplication on
the right by the contraction \(\vP_\perp\).  The linear QE term vanishes after the
projection because its right directions are \(\vv_1,\vv_2\).  The projected
quadratic term has bounded operator norm since
\(\|\vg_k^{\odot2}\|=O_\prec(\sqrt N)\),
\(\|\vv_k^{\odot2}\|=O(n^{-1/2})\), and
\(\theta_{\snr}^2=O(n^{1/2})\).  Also
\(\|\vY_0\vP_\perp\|=O_\prec(1)\) by the null CK bounds used in
Appendix~\ref{app:proof_finite_snr}.  Hence
\(\|\vY_{\rm QE}\vP_\perp\|=O_\prec(1)\), and
\[
\begin{aligned}
\|\vK_\perp-\vK_{{\rm QE},\perp}\|
\le
\bigl(\|\vY\vP_\perp\|+\|\vY_{\rm QE}\vP_\perp\|\bigr)
\|\vY\vP_\perp-\vY_{\rm QE}\vP_\perp\| =O_\prec(n^{-1/4}).
\end{aligned}
\]
This proves the lemma.
\end{proof}

\begin{lemma}\label{lem:large_final_rank_form}
Let $\vG:=\vY_0-\frac1{\sqrt N}\mathbf 1_N\vm^\top$ and $\vb_0:=\frac{\vP_\perp\vm}{\|\vP_\perp\vm\|}$
when \(\|\vP_\perp\vm\|>0\).  Define the centered square vectors
\[
        \widetilde\vg_k:=\vg_k^{\odot2}-\mathbf 1_N,
        \qquad
        \vh_u:=\widetilde\vg_1+\widetilde\vg_2,
        \qquad
        \vh_y:=\widetilde\vg_1-\widetilde\vg_2,
\]
and
\[
       \va_0:=\frac{\mathbf 1_N}{\sqrt N}, \qquad \va_u:=\frac{\vh_u}{\|\vh_u\|},
        \qquad
        \va_y:=\frac{\vh_y}{\|\vh_y\|}.
\]
Then, in operator norm,
\begin{equation}\label{eq:large_final_rank_model}
\vY_{\rm QE}\vP_\perp
=
\vG\vP_\perp
+\va_0\bigl(\sqrt\tau\,\vb_0^\top+\alpha\,\vu^\top\bigr)
+\alpha\,\va_u\vu^\top
+\alpha\,\va_y\overline\vy^{\top}
+o_\P(1).
\end{equation}
Furthermore,
\[
        \|\vP_\perp\vm\|^2\xrightarrow{\P}\tau,
        \qquad
        \vb_0^\top\vu=o_\P(1),\qquad
        \vb_0^\top\overline\vy=o_\P(1),
        \qquad
        \vu^\top\overline\vy=0,
\]
and the left directions \(\va_0,\va_u,\va_y\) are asymptotically orthonormal.
\end{lemma}
\begin{proof}
Recall $\vu=\frac{1}{\sqrt n}\mathbf 1_n$. Notice that
\[
        \vv_1^{\odot2}=\frac1{\sqrt n}(\vu+\overline\vy),
        \qquad
        \vv_2^{\odot2}=\frac1{\sqrt n}(\vu-\overline\vy).
\]
Therefore
\begin{equation}\label{eq:large_final_quad_reallocation}
\vg_1^{\odot2}\vv_1^{\odot2\top}+\vg_2^{\odot2}\vv_2^{\odot2\top}
=\frac1{\sqrt n}
\Bigl[(\vh_u+2\mathbf 1_N)\vu^\top+\vh_y\overline\vy^{\top}\Bigr].
\end{equation}
The deterministic term in \eqref{eq:large_final_quad_reallocation} gives
\[
        \frac{\theta_{\snr}^2c_\sigma}{2\sqrt N}\cdot
        \frac{2}{\sqrt n}\mathbf 1_N\vu^\top
        =\frac{\theta_{\snr}^2c_\sigma}{\sqrt n}\,\va_0\vu^\top
        =\alpha\,\va_0\vu^\top+o_\P(1),
\]
because \(\theta_{\snr}^2c_\sigma/\sqrt n\to c_\sigma\sqrt{\kappa_0}=\alpha\) and \(\alpha^2=\beta\) where $\alpha$ and $\beta$ are defined in \eqref{eq:large_final_tau_beta}.
Similarly,
\[
        \frac{\theta_{\snr}^2c_\sigma}{2\sqrt N\sqrt n}\vh_u\vu^\top
        =\alpha\,\va_u\vu^\top+o_\P(1),
        \qquad
        \frac{\theta_{\snr}^2c_\sigma}{2\sqrt N\sqrt n}\vh_y\overline\vy^\top
        =\alpha\,\va_y\overline\vy^\top+o_\P(1),
\]
since
\[
        N^{-1}\|\vh_u\|^2=4+o_\P(1),\qquad
        N^{-1}\|\vh_y\|^2=4+o_\P(1),\qquad
        N^{-1}\vh_u^\top\vh_y=o_\P(1).
\]
The null mean term is
\[
        \frac1{\sqrt N}\mathbf 1_N(\vP_\perp\vm)^\top
        =\va_0\|\vP_\perp\vm\|\vb_0^\top
        =\sqrt\tau\,\va_0\vb_0^\top+o_\P(1),
\]
by Lemma~\ref{lem:mean_strength_orth}; the same argument as in that lemma, with
the deterministic vector \(\overline\vy\), gives
\(\overline\vy^\top\vm=o_\P(1)\).  Since
\(\vm\) is also asymptotically orthogonal to \(\vv_1,\vv_2,\vu\), projection by
\(\vP_\perp\) does not change its norm at first order and
\(\vb_0^\top\vu=\vb_0^\top\overline\vy=o_\P(1)\).  Finally
\(\va_0^\top\va_u=\va_0^\top\va_y=o_\P(1)\) because \(\vh_u,\vh_y\) are centered,
and \(\va_u^\top\va_y=o_\P(1)\) by the covariance estimate.
\end{proof}

\subsection{The Covariance Spike and the Projected Resolvents}

Recall $\vP_\perp$ defined in \eqref{eq:large_final_notation} and $\vG$ defined in Lemma~\eqref{lem:large_final_rank_form}. Set
\begin{equation}\label{eq:large_final_K0_def}
        \vK_0:=(\vG\vP_\perp)^\top(\vG\vP_\perp).
\end{equation}
The population covariance of the centered rows of \(\vG\vP_\perp\) has the same
rank-one spike \(\tau\vu\vu^\top\) as in Appendix~\ref{app:proof_finite_snr},
because \(\vu\perp\vv_1,\vv_2\).  Thus, conditional on \(\vZ\),
\[
        \vP_\perp\vSigma\vP_\perp
        =\vP_\perp\vSigma_0\vP_\perp+\tau\vu\vu^\top+o_\P(1).
\]
The projection \(\vP_\perp\) does not change the limiting ESD of
bulk covariance, hence we have the following lemma.

\begin{lemma}[BBP for $\vK_0$]\label{lem:large_final_K0_cov_spike}
The ESD of \(\vK_0\) in \eqref{eq:large_final_K0_def} converges weakly in probability to \(\mu\).  If
\(\tau>\tau_{\rm crit}\), let \(\Lambda_\tau\) be the population outlier from
\eqref{eq:Lambda_tau}.  If in addition \(z'(-1/\Lambda_\tau)>0\), then \(\vK_0\)
has one separated outlier
\[
        \widehat\lambda_\tau^{(0)}
        =z\!\left(-\frac1{\Lambda_\tau}\right)+o_\P(1).
\]
If \(\widehat\vu_\tau\) is its unit right eigenvector, then $|\widehat\vu_\tau^\top\vu|^2
        \xrightarrow{\P}
        \varphi\!\left(-\frac1{\Lambda_\tau}\right).$
For every deterministic unit vector \(\vb\) with \(\vb^\top\vu=o(1)\), we have
\(|\widehat\vu_\tau^\top\vb|=o_\P(1)\).  In particular the covariance outlier is
asymptotically orthogonal to \(\vb_0\) and \(\overline\vy\), but it is generally
not orthogonal to the additive quadratic right direction \(\vu\).
\end{lemma}
\begin{proof}
This is Lemma~\ref{lem:K0_bulk_and_tau_outlier} applied to the projected matrix.
The assumptions used there are unchanged by removing the two deterministic
directions \(\vv_1,\vv_2\), since \(\vu,\vb_0,\overline\vy\) are orthogonal to
\(\vv_1,\vv_2\) up to \(o_\P(1)\), and the removed subspace has fixed dimension.
The last statement follows from the general alignment formula
\eqref{eq:alignment_right} in Lemma~\ref{lem:K0_bulk_and_tau_outlier}.
\end{proof}

\begin{lemma}[Resolvent limits retaining the covariance pole]
\label{lem:large_final_resolvents_with_pole}
Let \(\cD\subset\mathbb C\setminus(\supp\mu\cup\{0\})\) be a compact spectral
domain.  For \(\lambda\in \cD\), set
\[
        \vQ_R(\lambda):=(\vK_0-\lambda\vI_n)^{-1},
        \qquad
        \vQ_L(\lambda):=((\vG\vP_\perp)(\vG\vP_\perp)^\top-
        \lambda\vI_N)^{-1},
        \qquad
        s=s(\lambda):=\tilde m_\mu(\lambda).
\]
Let $\vA:=[\va_0,\va_u,\va_y]$ and $\vB:=[\vb_0,\vu,\overline\vy]$. 
Uniformly on $\cD$ such that \(1-\tau T(s(\lambda))\neq 0\), we have
\begin{equation}\label{eq:large_final_A_resolvent}
        \vA^\top\vQ_L(\lambda)\vA=s(\lambda)\vI_3+o_\P(1),
\end{equation}
and
\begin{equation}\label{eq:large_final_B_resolvent}
        \vB^\top\vQ_R(\lambda)\vB
        =m_\mu(\lambda)
        \begin{pmatrix}
        1&0&0\\
        0&(1-\tau T(s(\lambda)))^{-1}&0\\
        0&0&1
        \end{pmatrix}
        +o_\P(1).
\end{equation}
\end{lemma}
\begin{proof}
The left-resolvent estimate in \eqref{eq:large_final_A_resolvent} is the same 
argument as Lemma~\ref{lem:left_resolvent}. The possible covariance outlier does
not contribute to the quadratic forms because its left singular vector
is asymptotically orthogonal to \(\va_0,\va_u,\va_y\), as shown in
Lemma~\ref{lem:verify-E34-K0}.  The random vectors \(\va_u,\va_y\) are handled by
conditioning on the independent copy of the centered bulk and applying the
Hanson--Wright concentration inequality \citep{gotze2021concentration}.

For the right resolvent in \eqref{eq:large_final_B_resolvent}, the deterministic equivalent in Lemma~\ref{lem:right_resolvent} for a general sample
covariance matrix gives
\[
        \vb^\top\vQ_R(\lambda)\vb'
        =\vb^\top\bigl(-\lambda s\vSigma-\lambda\vI_n\bigr)^{-1}\vb'
        +o_\P(1).
\]
Using
\(\vSigma=\vSigma_0+\tau\vu\vu^\top+o_\P(1)\) and the Sherman--Morrison formula,
with
\(\vA_0(\lambda):=-\lambda(\vI_n+s\vSigma_0)\), we get
\[
\bigl(\vA_0-\lambda s\tau\vu\vu^\top\bigr)^{-1}
=\vA_0^{-1}
+\frac{\lambda s\tau}{1-\lambda s\tau\,\vu^\top\vA_0^{-1}\vu}
\vA_0^{-1}\vu\vu^\top\vA_0^{-1}.
\]
The isotropic law for \(\vSigma_0\) yields $ \vu^\top\vA_0^{-1}\vu=m_\mu(\lambda)+o_\P(1).$ Recalling \(T(s)=\lambda s m_\mu(\lambda)\), we have
\[
        \vu^\top\vQ_R(\lambda)\vu
        =\frac{m_\mu(\lambda)}{1-\tau T(s(\lambda))}+o_\P(1).
\]
For \(\vb\in\{\vb_0,\overline\vy\}\), \(\vb^\top\vu=o_\P(1)\). Thus the
rank-one Sherman--Morrison term is negligible and
\(\vb^\top\vQ_R(\lambda)\vb=m_\mu(\lambda)+o_\P(1)\); and all cross terms are
\(o_\P(1)\) by the same argument.
\end{proof}

\subsection{Master Determinant Equation for Coupled Outliers}

Define the deterministic \(3\times3\) coefficient matrix
\begin{equation}\label{eq:large_final_C_matrix}
        \vC:=
        \begin{pmatrix}
        \sqrt\tau & \alpha & 0\\
        0          & \alpha & 0\\
        0          & 0      & \alpha
        \end{pmatrix},
\end{equation}
so that the finite-rank perturbation in \eqref{eq:large_final_rank_model} is
\[
        \vP:=\vA\vC\vB^\top
        =\va_0(\sqrt\tau\,\vb_0^\top+\alpha\,\vu^\top)
         +\alpha\,\va_u\vu^\top
         +\alpha\,\va_y\overline\vy^\top
\]where $\vA$ and $\vB$ are defined in Lemma~\ref{lem:large_final_resolvents_with_pole}. Let
\[
        \vY^\sharp:=\vG\vP_\perp+\vP,
        \qquad
        \vK^\sharp:=\vY^{\sharp\top}\vY^\sharp .
\]
By Lemmas~\ref{lem:large_final_projected_QE} and~\ref{lem:large_final_rank_form}, we know that
\begin{equation}\label{eq:large_final_Kperp_Ksharp_close}
        \|\vK_\perp-\vK^\sharp\|=o_\P(1).
\end{equation}

\begin{lemma}[Coupled outlier determinant]\label{lem:large_final_coupled_det}
Let \(\lambda\in\mathbb C\setminus(\supp\mu\cup\{0\})\) and
\(s=s(\lambda)=\tilde m_\mu(\lambda)\).  Away from the possible covariance
pole of \(\vK_0\), the relative determinant for additive outliers of
\(\vK^\sharp\) satisfies
\begin{align}
D_{\rm rel}(\lambda)
&:=\det\Bigl(\vI_3-
\lambda\,\vC^\top(\vA^\top\vQ_L(\lambda)\vA)\vC
       (\vB^\top\vQ_R(\lambda)\vB)\Bigr) \label{eq:large_final_relative_det}\\
&=\bigl(1-\beta T(s)\bigr)
  \frac{1-2(\tau+\beta)T(s)+\tau(\tau+\beta)T(s)^2}
       {1-\tau T(s)}
  +o_\P(1).\label{eq:large_final_relative_det_expanded}
\end{align}
Consequently, after multiplying by the covariance-spike factor of \(\vK_0\), the
full outlier master equation for \(\vK^\sharp\) outside \(\supp\mu\) is
\begin{equation}\label{eq:large_final_full_equation}
        \bigl(1-\beta T(s)\bigr)
        \bigl[1-2(\tau+\beta)T(s)+\tau(\tau+\beta)T(s)^2\bigr]=0 .
\end{equation}
\end{lemma}
\begin{proof}
Similarly as Lemma~\ref{lem:finite_det_eq}, the finite-rank additive determinant obtained from the linearization
\(\left(\begin{smallmatrix}0&\vY^\sharp\\ \vY^{\sharp\top}&0\end{smallmatrix}\right)\)
gives \eqref{eq:large_final_relative_det}.
Substituting \eqref{eq:large_final_A_resolvent}--\eqref{eq:large_final_B_resolvent}
and \(T(s)=\lambda s m_\mu(\lambda)\) gives
\[
        D_{\rm rel}(\lambda)
        =\det\left(\vI_3-T(s)\,\vC^\top\vC
        \begin{pmatrix}
        1&0&0\\
        0&(1-\tau T(s))^{-1}&0\\
        0&0&1
        \end{pmatrix}\right)+o_\P(1).
\]
From \eqref{eq:large_final_C_matrix}, we have
\[
        \vC^\top\vC=
        \begin{pmatrix}
        \tau&\alpha\sqrt\tau&0\\
        \alpha\sqrt\tau&2\beta&0\\
        0&0&\beta
        \end{pmatrix}.
\]
Hence we get \eqref{eq:large_final_relative_det_expanded}.

The denominator \(1-\tau T(s)\) is not an extra root of \(\vK^\sharp\); it is the
pole of the resolvent of the null model \(\vK_0\) at the covariance outlier.
In the characteristic determinant of \(\vK^\sharp\), the zero of
\(\det(\vK_0-\lambda\vI)\) at this covariance location is multiplied by the pole
of \(D_{\rm rel}\).  The product is governed by the numerator in
\eqref{eq:large_final_relative_det_expanded}.  Equivalently, an argument-principle count
on a contour around any separated cluster counts zeros of
\eqref{eq:large_final_full_equation}, with algebraic multiplicity.  
\end{proof}

The outlier master equation \eqref{eq:large_final_full_equation} captures the asymptotic locations of the eigenvalues of \(\vK^\sharp\) outside the bulk spectrum. Recall the functions \(z,\varphi,T\) from \eqref{eq:z(s)}--\eqref{eq:T_def}.
For \(\beta>0\), define the label factor and the non-label coupled factor, respectively, as
\begin{equation}\label{eq:large_final_Fy}
        F_y(s):=1-\beta T(s),
        \qquad F_{\rm nl}(s)
        :=1-2(\tau+\beta)T(s)+\tau(\tau+\beta)T(s)^2.
\end{equation}
Then \eqref{eq:large_final_full_equation} is equivalent to \(F_y(s)F_{\rm nl}(s)=0\).  The label factor \(F_y=0\) has a root at \(s_y:=T^{-1}(1/\beta)\), and the non-label factor \(F_{\rm nl}=0\) has two possible roots at \(s_\pm\):
when \(c_\sigma\neq0\), equivalently \(\tau>0\) and \(\beta>0\), 
\begin{equation}\label{eq:large_final_tpm}
        t_\pm
        :=\frac{1\pm\sqrt{\beta/(\tau+\beta)}}{\tau},
        \qquad
        s_\pm:=T^{-1}(t_\pm).
\end{equation}
The factor \(F_{\rm nl}\) is the joint contribution of three effects: the
population-covariance spike of the centered null matrix, the mean spike from $\vm$, and the large-SNR quadratic \(\vu\)-channel.  The original covariance
outlier studied in Lemma~\ref{lem:K0_bulk_and_tau_outlier} for finite SNR case withlocation \(s_\tau=T^{-1}(1/\tau)=-1/\Lambda_\tau\) is shifted into the two roots
\(s_\pm\) by $\beta_{\rm quad}$ if they exist.

\begin{lemma}[Cluster count and no-label property]\label{lem:large_final_cluster_count}
Let \(\lambda_\star\in\mathbb R\setminus(\supp\mu\cup\{0\})\) be a limiting
location generated by the roots of \eqref{eq:large_final_full_equation}, and let
\(m_\star\) be its multiplicity as in \eqref{eq:large_final_multiplicity}.  Then
\(\vK^\sharp\) has exactly \(m_\star\) eigenvalues in a sufficiently small
interval around \(\lambda_\star\).  If the interval contains no label root \(s_y\), then
the associated spectral projector \(\widehat\vP_\star^\sharp\) satisfies $\overline\vy^\top\widehat\vP_\star^\sharp\overline\vy=o_\P(1).$
If the interval contains the label root \(s_y\), then
\[
        \overline\vy^\top\widehat\vP_\star^\sharp\overline\vy
        \xrightarrow{\P}
        \frac{m_\mu(z(s_y))z'(s_y)}{\beta T'(s_y)} .
\]
\end{lemma}
\begin{proof}
Choose a small contour \(\Gamma_\star\) enclosing \(\lambda_\star\), disjoint from
\(\supp\mu\), and containing no other limiting candidate.  On this contour the
deterministic factors in \eqref{eq:large_final_full_equation} are bounded away
from zero except for the roots inside the contour.  Lemma~\ref{lem:large_final_coupled_det}
and Rouch\'e's theorem give exactly \(m_\star\) zeros of the finite-dimensional
outlier determinant, after the covariance-pole cancellation described above.
The linearization determinant is equivalent to the eigenvalue equation of
\(\vK^\sharp\) outside the bulk, so \(\vK^\sharp\) has exactly \(m_\star\)
eigenvalues in the interval.

For the projector statements, use the contour representation
\[
        \overline\vy^\top\widehat\vP_\star^\sharp\overline\vy
        =-\frac{1}{2\pi i}\oint_{\Gamma_\star}
        \overline\vy^\top(\vK^\sharp-\lambda\vI)^{-1}\overline\vy\,d\lambda .
\]
The same Woodbury-linearization calculation used for additive outlier overlaps
shows that, uniformly on \(\Gamma_\star\),
\begin{equation}\label{eq:large_final_y_resolvent}
        \overline\vy^\top(\vK^\sharp-\lambda\vI)^{-1}\overline\vy
        =\frac{m_\mu(\lambda)}{1-\beta T(s(\lambda))}+H_y(\lambda)+o_\P(1),
\end{equation}
where \(H_y\) is analytic in a neighborhood of the contour.  The non-label block
does not appear in the leading singular term because the left direction
\(\va_y\) and the right direction \(\overline\vy\) are asymptotically orthogonal
to the non-label directions \(\va_0,\va_u\) and \(\vb_0,\vu\).

If the contour contains no label root, the leading term in
\eqref{eq:large_final_y_resolvent} is analytic and its contour integral is zero,
giving the no-label statement.  If the contour contains \(s_y\), then
\(1-\beta T(s(\lambda))\) has a simple zero at \(\lambda_y=z(s_y)\), and
\[
        \frac{d}{d\lambda}\{1-\beta T(s(\lambda))\}\Big|_{\lambda=\lambda_y}
        =-\beta\frac{T'(s_y)}{z'(s_y)}.
\]
Thus
\[
        -\operatorname*{Res}_{\lambda=\lambda_y}
        \frac{m_\mu(\lambda)}{1-\beta T(s(\lambda))}
        =\frac{m_\mu(z(s_y))z'(s_y)}{\beta T'(s_y)}.
\]
This proves the label projector formula.  Using
\(T(s)=z(s)s m_\mu(z(s))\), \(\beta T(s_y)=1\), and
\(\varphi(s)=-s z'(s)/z(s)\), the same quantity equals
\(-\varphi(s_y)/(\beta^2s_y^2T'(s_y))\).
\end{proof}

\subsection{Proof of Theorem~\ref{thm:large}}
 
The bulk statement follows by comparing the full CK matrix with its QE model in
normalized trace norm.  Proposition~\ref{prop:approx} gives
\(\|\vY-\vY_{\rm QE}\|=O_\prec(n^{-1/4})\), hence
\(\|\vY-\vY_{\rm QE}\|_{F}=O_\prec(n^{1/4})\).  Since
\(n^{-1}\|\vY\|_{F}^2\) and \(n^{-1}\|\vY_{\rm QE}\|_{F}^2\) are tight, the normalized Schatten-1 norm difference of the Gram matrices satisfies
\[
        \frac1n\|\vY^\top\vY-\vY_{\rm QE}^\top\vY_{\rm QE}\|_1
        \le \frac{\|\vY\|_{F}+\|\vY_{\rm QE}\|_{F}}{n}
        \|\vY-\vY_{\rm QE}\|_{F}
        =o_\P(1).
\]
Thus \(\vK\) and \(\vY_{\rm QE}^\top\vY_{\rm QE}\) have the same ESD
limit.  Moreover, 
\(\vY_{\rm QE}^\top\vY_{\rm QE}\) differs from \(\vG^\top\vG\) by a
finite-rank matrix, and the limiting ESD is the same as \(\vG^\top\vG\).

For the diverging linear outliers, recall that the linear QE term is $ \vT_1=\frac{\theta_{\snr}b_\sigma}{\sqrt N}
        (\vg_1\vv_1^\top+\vg_2\vv_2^\top).$
Since \(N^{-1}\|\vg_k\|^2\to1\) and \(N^{-1}\vg_1^\top\vg_2\to0\) in probability, we have
\[
        \vT_1^\top\vT_1
        =b_\sigma^2\theta_{\snr}^2
        (\vv_1\vv_1^\top+\vv_2\vv_2^\top)+o_\P(\theta_{\snr}^2).
\]
All other terms have operator norm \(O_\prec(1)\) after separating the linear
part, whereas \(\theta_{\snr}^2=(r_0^2\psi/2)n^{1/2}(1+o(1))\).  Weyl's
inequality gives the top two eigenvalue asymptotics, and Davis--Kahan
gives convergence of the two-dimensional projector to \(\vP_{\rm lin}\).  Since
\(\overline\vy\perp\vv_1,\vv_2\), the linear outlier space has no asymptotic
label alignment.

It remains to analyze the order-one spectrum of \(\vK_\perp\).  By
\eqref{eq:large_final_Kperp_Ksharp_close}, it suffices to analyze \(\vK^\sharp\).
Lemma~\ref{lem:large_final_coupled_det} gives the full outlier master equation
\eqref{eq:large_final_full_equation}.  The roots of the label factor are exactly
\(s_y=T^{-1}(1/\beta)\); the roots of the non-label factor are exactly
\(s_\pm=T^{-1}(t_\pm)\) with \(t_\pm\) in \eqref{eq:large_final_tpm}.  By the
proof of Lemma~\ref{lem:additive_scalar_eq}, a real
root produces a separated eigenvalue if and only if \(z'(s)>0\).  The
multiplicity and collision-safe statement follows from
Lemma~\ref{lem:large_final_cluster_count}.  The same lemma gives the label
projector formula and the no-label statement for non-label clusters.  Finally,
\eqref{eq:large_final_Kperp_Ksharp_close} transfers the locations and spectral
projectors from \(\vK^\sharp\) to \(\vK_\perp\) by standard stability of isolated
spectral projectors under an \(o_\P(1)\) operator-norm perturbation.

\section{Proof for Pretrained Weight Regime}
\label{app:pretrained}

In this section we prove Theorem~\ref{thm:trained}.  The
main point is that the spiked weight creates two different effects.  The linear
Hermite component creates one diverging rank-one spike in the sample direction
\(\vs=\vX^\top\mathbf 1_d\).  After this direction is removed, the quadratic
Hermite component leaves an order-one rank-one feature proportional to
\(\vq=\vs^{\odot2}\). The master equation of the outliers is more complicated than the scalar equation in Lemma~\ref{lem:additive_scalar_eq}, since \(\vq\) has a non-negligible component along
\(\vu=\mathbf 1_n/\sqrt n\) and the left vector
\(\va^{\odot2}\) is not orthogonal to \(\mathbf 1_N\).  

Throughout this section, we use the following notation for the sample-direction spike and its projection:
\[
        \vs:=\vX^\top\mathbf 1_d,\qquad
        \vq:=\vs^{\odot2},\qquad
        \widehat\vs:=\frac{\vs}{\|\vs\|},\qquad
        \vPi_s:=\vI_n-\widehat\vs\widehat\vs^\top,
        \qquad
        \vK_s:=\vPi_s\vK\vPi_s .
\] 

\subsection{Moment Estimates and Deterministic Directions}

\begin{lemma}[Moments of \(\vs\) and \(\vq\)]
\label{lem:sw_s_q_moments} Under the assumptions of Theorem~\ref{thm:trained}, we have  
\begin{align}
        \frac1n\|\vs\|^2 &\xrightarrow{\P}\kappa,
        \label{eq:sw_s_norm}\\
        \frac1n\|\vq\|^2 &\xrightarrow{\P}\eta,
        \label{eq:sw_q_norm}\\
        \frac1{\sqrt n}\vu^\top\vq &\xrightarrow{\P}\kappa,
        \qquad
        \frac1n\vy^\top\vq \xrightarrow{\P} -\frac{r^2}{2},
        \label{eq:sw_q_u_y}\\
        \frac1{\sqrt n}\vv_2^\top\vs &\xrightarrow{\P}\frac{r}{\sqrt2},
        \qquad
        \vv_1^\top\widehat\vs=o_\P(1),
        \qquad
        \frac1{\sqrt n}\vy^\top\widehat\vs=o_\P(1),
        \label{eq:sw_s_v2_y}\\
        \vq^\top\widehat\vs&=O_\P(1)
        \label{eq:sw_q_s}
\end{align} 
where $\kappa:=1+\frac{r^2}{2}$ and $\eta:=3+3r^2+\frac{r^4}{2}$. Consequently,
\begin{align}
        \|\vPi_s\vv_1\|^2&\xrightarrow{\P}1,
        &
        \|\vPi_s\vv_2\|^2&\xrightarrow{\P}\frac1\kappa,
        \label{eq:sw_projected_linear_norms}\\
        \frac1n\|\vPi_s\vq\|^2&\xrightarrow{\P}\eta,
        &
        \left\langle \frac{\vPi_s\vq}{\|\vPi_s\vq\|},\vu\right\rangle^2
        &\xrightarrow{\P}\frac{\kappa^2}{\eta},
        \label{eq:sw_projected_q_u}\\
        \left\langle \frac{\vPi_s\vq}{\|\vPi_s\vq\|},\frac{\vy}{\sqrt n}\right\rangle^2
        &\xrightarrow{\P}\chi_y.
        \label{eq:sw_projected_q_y}
\end{align}
Moreover \(\vPi_s\vq/\|\vPi_s\vq\|\) is asymptotically orthogonal to
\(\vv_1\), to \(\vPi_s\vv_2/\|\vPi_s\vv_2\|\), and to the mean-spike direction
\(\vm/\|\vm\|\) from Appendix~\ref{app:proof_finite_snr}, in the sense of the
right-resolvent quadratic forms used below.
\end{lemma}

\begin{proof}
Since \(\vb=\mathbf 1_d=\sqrt d\,\vu_2\), the deterministic shifts of
\(s_j=\mathbf 1_d^\top\vx_j\) on the four balanced XOR blocks are $0, 0, r, -r $ and
the noise part is i.i.d. \(\cN(0,1)\).  Hence
\[
        \frac1n\sum_{j=1}^n\E s_j^2
        =\frac12\cdot 1+\frac12\cdot(1+r^2)=\kappa,
\]
and
\[
        \frac1n\sum_{j=1}^n\E s_j^4
        =\frac12\cdot3+\frac12\cdot(3+6r^2+r^4)
        =3+3r^2+\frac{r^4}{2}=\eta.
\]
The law of large numbers gives \eqref{eq:sw_s_norm}--\eqref{eq:sw_q_norm}.  The
mean of \(s_j^2\) is \(1\) on the first two blocks and \(1+r^2\) on the last
two blocks.  Since \(\vu=\mathbf 1_n/\sqrt n\) and
\(\vy=(\mathbf 1_{n/2},-\mathbf 1_{n/2})\), this gives
\eqref{eq:sw_q_u_y}.  Also
\(\vv_2^\top\vs=r\sqrt{n/2}+o_\P(1)\), while
\(\vv_1^\top\vs=o_\P(1)\) and \(\vy^\top\vs=o_\P(\sqrt n)\), proving
\eqref{eq:sw_s_v2_y}.

Finally, \(\vq^\top\vs=\sum_j s_j^3\).  The third moments in the third and
fourth XOR blocks are \(r^3+3r\) and \(-r^3-3r\), so the average third moment is
zero.  Therefore \(\sum_j s_j^3=O_\P(\sqrt n)\), and after division by
\(\|\vs\|\asymp\sqrt n\) we get \eqref{eq:sw_q_s}.  The projection statements
\eqref{eq:sw_projected_linear_norms}--\eqref{eq:sw_projected_q_y} follow by
subtracting the \(\widehat\vs\)-component and using
\eqref{eq:sw_s_norm}--\eqref{eq:sw_q_s}.  For example,
\[
        \|\vPi_s\vv_2\|^2
        =1-|\vv_2^\top\widehat\vs|^2
        \to 1-\frac{r^2}{2\kappa}=\frac1\kappa,
\]
and the projection changes \(\vq\) by an \(O_\P(1)\)-norm vector.

The last statement follows from the same leave-one-out argument as
Lemma~\ref{lem:right_resolvent}.  The block means of \(\vq\) have no component
along \(\vv_1\) or \(\vv_2\), and the additional projection by \(\vs\) only
changes normalized inner products by \(o_\P(1)\).  For the mean direction,
\(m_j=F(\|\vz_j\|^2)\) is of order \(d^{-1/2}\), while the covariance of
\(s_j^2\) with \(\|\vz_j\|^2-1\) is \(O(d^{-1})\) per sample.  Hence
\(\vq^\top\vm=O_\P(1)\), whereas \(\|\vq\|\asymp\sqrt n\) and
\(\|\vm\|\asymp1\).  The normalized overlap is therefore \(o_\P(1)\), and the
resolvent version is obtained by the same deterministic-equivalent estimate.
\end{proof}

\begin{lemma} 
\label{lem:sw_left_geometry}
Recall $\va$ defined in Theorem~\ref{thm:trained}. Let $\va_0:=\frac{\mathbf 1_N}{\sqrt N}$ and $ \vh:=\frac{\va^{\odot2}}{\|\va^{\odot2}\|}.$
Then
\begin{equation}\label{eq:sw_a2_norm_overlap}
        \|\va^{\odot2}\|^2=\frac3N(1+o_\P(1)),
        \qquad
        \va_0^\top\vh\xrightarrow{\P}\rho:=\frac1{\sqrt3}.
\end{equation}
Moreover \(\vh\) is asymptotically orthogonal to the Gaussian left spike
directions \(\vg_1/\|\vg_1\|\) and \(\vg_2/\|\vg_2\|\) defined in \eqref{eq:QE_formal}, and the left-resolvent
quadratic forms satisfy, uniformly for \(\lambda\) in compact subsets of
\(\mathbb C\setminus(\supp\mu\cup\{0\})\),
\begin{equation}\label{eq:sw_left_resolvent_matrix}
        \vA^\top\vQ_L(\lambda)\vA
        =s(\lambda)
        \begin{pmatrix}
        1&0&0&\rho\\
        0&1&0&0\\
        0&0&1&0\\
        \rho&0&0&1
        \end{pmatrix}
        +o_\P(1),
\end{equation}
where
\[
        \vA:=\left[\va_0,\frac{\vg_1}{\|\vg_1\|},
        \frac{\vg_2}{\|\vg_2\|},\vh\right],
        \qquad
        \vQ_L(\lambda):=(\vG\vG^\top-\lambda\vI_N)^{-1},
        \qquad
        s(\lambda):=\tilde m_\mu(\lambda).
\]
\end{lemma}

\begin{proof}
The first two limits are immediate from the Gaussian moments
\(\E a_i^2=N^{-1}\) and \(\E a_i^4=3N^{-2}\):
\[
        \sum_i a_i^4=\frac3N(1+o_\P(1)),
        \qquad
        \va_0^\top\vh
        =\frac{N^{-1/2}\sum_i a_i^2}{(\sum_i a_i^4)^{1/2}}
        \to\frac1{\sqrt3}.
\]
The vector \(\va\) is independent of \((\vW,\vX)\), hence \(\vh\) is independent
of \(\vG,\vg_1,\vg_2\).  Its centered component
\(\vh-\rho\va_0\) has independent sub-exponential coordinates with variance
\(1-\rho^2+o(1)\), after normalization.  The same Hanson--Wright inequality and
anisotropic-law argument used in Lemma~\ref{lem:left_resolvent} gives
\eqref{eq:sw_left_resolvent_matrix}.  The off-diagonal entry with \(\va_0\)
retains the deterministic overlap \(\rho\), while all other off-diagonal entries
vanish by independence and orthogonality.
\end{proof}

\subsection{QE for Spiked Weight}

Recall the finite-SNR surrogate model from Proposition~\ref{prop:Y_rank3_decomp},
\begin{equation}\label{eq:sw_Ysharp0}
        \vY_0^\sharp
        :=\vG+\frac1{\sqrt N}\mathbf 1_N\vm^\top
        +\frac{\theta_{\snr}b_\sigma}{\sqrt N}
        (\vg_1\vv_1^\top+\vg_2\vv_2^\top),
        \qquad
        \theta_{\snr}=r\sqrt{\frac{n}{2d}}.
\end{equation}

\begin{lemma} 
\label{lem:sw_QE}
Define
\begin{equation}\label{eq:sw_Ysp_def}
        \vY_{\rm sw}^\sharp
        :=\vY_0^\sharp
        +\frac{b_\sigma\theta}{\sqrt N}\va\vs^\top
        +\frac{c_\sigma\theta^2}{2\sqrt N}\va^{\odot2}\vq^\top .
\end{equation}
Then
\begin{equation}\label{eq:sw_QE_feature_error}
        \left\|\vY-\vY_{\rm sw}^\sharp\right\|=O_\prec(n^{-1/4}).
\end{equation}
Consequently, with
\begin{equation}\label{eq:sw_Ys_def}
        \vY_s^\sharp:=(\vY_0^\sharp+\frac{c_\sigma\theta_0^2}{2}\va^{\odot2}\vq^\top)\vPi_s,
        \qquad
        \vK_s^\sharp:=(\vY_s^\sharp)^\top\vY_s^\sharp,
\end{equation}
we have
\begin{equation}\label{eq:sw_Ks_close}
        \left\|\vY\vPi_s-\vY_s^\sharp\right\|=O_\prec(n^{-1/4}),
        \qquad
        \left\|\vK_s-\vK_s^\sharp\right\|=o_\P(1).
\end{equation}
\end{lemma}

\begin{proof}
Write
\[
        \vW_1\vX=\vW\vZ+\Delta_{\rm lin}+\Delta_{\rm sp},
        \qquad
        \Delta_{\rm lin}:=\vW\vM,
        \qquad
        \Delta_{\rm sp}:=\theta\va\vs^\top .
\]
The entries of \(\Delta_{\rm lin}\) are \(O_\prec(n^{-1/2})\), while
\(\|\Delta_{\rm sp}\|_\infty=O_\prec(N^{-1/4})\).  Taylor expansion at
\(\vW\vZ\), using the boundedness of \(\sigma^{(3)}\), gives
\begin{align}
\frac1{\sqrt N}\sigma(\vW_1\vX)
=&\frac1{\sqrt N}\sigma(\vW\vZ)
 +\frac1{\sqrt N}\sigma'(\vW\vZ)\odot\Delta_{\rm lin}
 +\frac1{\sqrt N}\sigma'(\vW\vZ)\odot\Delta_{\rm sp} \\
&+\frac1{2\sqrt N}\sigma''(\vW\vZ)\odot\Delta_{\rm sp}^{\odot2}
 +\vR,\label{eq:sw_Y_Taylor}
\end{align}
where the omitted terms
\((2\sqrt N)^{-1}\sigma''(\vW\vZ)\odot\Delta_{\rm lin}^{\odot2}\),
\(\sqrt N^{-1}\sigma''(\vW\vZ)\odot(\Delta_{\rm lin}\odot\Delta_{\rm sp})\),
and the third-order remainder all have operator norm \(O_\prec(n^{-1/4})\).
Indeed, the mixed term has entries of size
\(O_\prec(n^{-1/2}N^{-1/4})\), and after multiplication by \(1/\sqrt N\) its
operator norm is bounded by
\(\sqrt{Nn}/\sqrt N\cdot O_\prec(n^{-1/2}N^{-1/4})=O_\prec(N^{-1/4})\); the
remainder is analogous.

The first two terms in \eqref{eq:sw_Y_Taylor} are reduced to \(\vY_0^\sharp\) by
Proposition~\ref{prop:Y_rank3_decomp}.  For the spiked linear term in \eqref{eq:sw_Y_Taylor},
\[
        \frac1{\sqrt N}\sigma'(\vW\vZ)\odot\Delta_{\rm sp}
        =\frac{b_\sigma\theta}{\sqrt N}\va\vs^\top+\boldsymbol E_1,
\]
and \(\|\boldsymbol E_1\|=O_\prec(n^{-1/4})\).  This follows by writing
\(\boldsymbol E_1=(\theta/\sqrt N)\operatorname{diag}(\va)
(\sigma'(\vW\vZ)-b_\sigma\mathbf 1_N\mathbf 1_n^\top)
\operatorname{diag}(\vs)\), using
\(\|\operatorname{diag}(\va)\|=O_\prec(N^{-1/2})\),
\(\|\sigma'(\vW\vZ)-b_\sigma\mathbf 1\mathbf 1^\top\|=O_\prec(\sqrt N+\sqrt n)\),
and \(\|\operatorname{diag}(\vs)\|=O_\prec(1)\).  The quadratic spiked term in \eqref{eq:sw_Y_Taylor} satisfies
\[
        \frac1{2\sqrt N}\sigma''(\vW\vZ)\odot\Delta_{\rm sp}^{\odot2}
        =\frac{c_\sigma\theta^2}{2\sqrt N}\va^{\odot2}\vq^\top+\boldsymbol E_2,
        \qquad
        \|\boldsymbol E_2\|=O_\prec(n^{-1/4}),
\]
by the same Hermite-coefficient concentration: after factoring
\(\operatorname{diag}(\va^{\odot2})\) and \(\operatorname{diag}(\vq)\), the
nonconstant part of \(\sigma''(\vW\vZ)\) has operator norm
\(O_\prec(\sqrt N+\sqrt n)\), while
\(\|\operatorname{diag}(\va^{\odot2})\|=O_\prec(N^{-1})\).  Combining these
estimates proves \eqref{eq:sw_QE_feature_error}.

Right multiplication by \(\vPi_s\) kills the diverging linear spike exactly,
\((b_\sigma\theta/\sqrt N)\va\vs^\top\vPi_s=0\), and
\(\theta^2/\sqrt N=\theta_0^2\).  Hence the first estimate in
\eqref{eq:sw_Ks_close} follows.  The matrices \(\vY\vPi_s\) and
\(\vY_s^\sharp\) have operator norm \(O_\P(1)\): \(\vY_0^\sharp\vPi_s\) is
bounded by Proposition~\ref{prop:Y_rank3_decomp}, and the quadratic term has
norm $\frac{|c_\sigma|\theta_0^2}{2}\|\va^{\odot2}\|\,\|\vPi_s\vq\|
        =O_\P(1)$
by Lemmas~\ref{lem:sw_s_q_moments} and~\ref{lem:sw_left_geometry}.  Therefore, we have
\[
        \|\vK_s-\vK_s^\sharp\|
        \le (\|\vY\vPi_s\|+\|\vY_s^\sharp\|)
        \|\vY\vPi_s-\vY_s^\sharp\|=o_\P(1).
\]
\end{proof}

\subsection{The Compressed Covariance Spike}

We analyze the projected covariance spike below similarly to the analysis of the covariance spike present in the
centered null matrix \(\vG\) in Lemma~\ref{lem:K0_bulk_and_tau_outlier}.
Let \(\vSigma\) denote the conditional population covariance of the unscaled
centered row underlying \(\vG\), as in Appendix~\ref{app:proof_finite_snr}.  Thus
\[
        \vSigma=\vSigma_0+\tau\vu\vu^\top+o_\P(1),
        \qquad
        \vSigma_0=(1-b_\sigma^2)\vI_n+b_\sigma^2\vZ^\top\vZ,
        \qquad
        \tau=\frac{c_\sigma^2}{2}\psi .
\]
Define
\[
        \widehat\vu_s:=\frac{\vPi_s\vu}{\|\vPi_s\vu\|},
        \qquad
        \tau_s:=\tau\|\vPi_s\vu\|^2,
        \qquad
        \vK_{0,s}:=(\vG\vPi_s)^\top(\vG\vPi_s).
\]

\begin{lemma}[Covariance spike after \(\vs\)-deflation]
\label{lem:sw_compressed_covariance}
Under the assumptions of Theorem~\ref{thm:trained},
\begin{equation}\label{eq:sw_projected_covariance}
        \vPi_s\vSigma\vPi_s
        =\vPi_s\vSigma_0\vPi_s+\tau_s\widehat\vu_s\widehat\vu_s^\top+o_\P(1),
        \qquad
        \tau_s=\tau+o_\P(1),
        \qquad
        \|\widehat\vu_s-\vu\|=o_\P(1).
\end{equation}
The ESD of \(\vK_{0,s}\) converges weakly in probability to \(\mu\).  Its
possible separated covariance outlier is governed by the same equation as in
Lemma~\ref{lem:K0_bulk_and_tau_outlier}, namely
\begin{equation}\label{eq:sw_covariance_factor}
        1-\tau T(s)=0,
        \qquad z'(s)>0.
\end{equation}
Moreover, uniformly for \(\lambda\) in compact subsets of
\(\mathbb C\setminus(\supp\mu\cup\{0\})\) avoiding the possible pole
\(1-\tau T(s(\lambda))=0\), with
\(\vQ_{R,0}^{(s)}(\lambda):=(\vK_{0,s}-\lambda\vI)^{-1}\) on
\(\operatorname{Range}\vPi_s\),
\begin{align}
        \vb^\top\vQ_{R,0}^{(s)}(\lambda)\vb
        &=m_\mu(\lambda)+o_\P(1),
        \label{eq:sw_R_isotropic_orth}\\
        \widehat\vq_s^\top\vQ_{R,0}^{(s)}(\lambda)\widehat\vq_s
        &=m_\mu(\lambda)R_q(T(s(\lambda)))+o_\P(1),
        \label{eq:sw_R_q_with_cov}\\
        \vb^\top\vQ_{R,0}^{(s)}(\lambda)\widehat\vq_s
        &=o_\P(1),
        \label{eq:sw_R_cross_q}
\end{align}
for every unit vector \(\vb\) among \(\vb_0:=\vm/\|\vm\|\), \(\vv_1\), and
\(\widetilde\vv_2:=\vPi_s\vv_2/\|\vPi_s\vv_2\|\).  Here
\(\widehat\vq_s:=\vPi_s\vq/\|\vPi_s\vq\|\) and
\begin{equation}\label{eq:sw_Rq_cov_lemma}
        R_q(t)=1+\frac{\kappa^2}{\eta}\frac{\tau t}{1-\tau t}
        =\frac{1-\tau\omega_q t}{1-\tau t}.
\end{equation}
\end{lemma}

\begin{proof}
By Lemma~\ref{lem:sw_s_q_moments},
\(\vu^\top\widehat\vs=O_\P(n^{-1/2})\).  Hence
\(\|\vPi_s\vu\|^2=1-|\vu^\top\widehat\vs|^2=1+o_\P(1)\), which proves
\(\tau_s=\tau+o_\P(1)\) and \(\|\widehat\vu_s-\vu\|=o_\P(1)\).  Multiplying
\(\vSigma=\vSigma_0+\tau\vu\vu^\top+o_\P(1)\) on both sides by \(\vPi_s\) gives
\eqref{eq:sw_projected_covariance}.  Since \(\vPi_s\vSigma_0\vPi_s\) differs
from \(\vSigma_0\) by a matrix of rank at most two plus one additional zero
direction, its limiting covariance ESD is still \(\nu\); consequently the limiting
ESD of \(\vK_{0,s}\) is still \(\mu\).  The multiplicative BBP equation for the
rank-one covariance spike in \eqref{eq:sw_projected_covariance} is therefore the
same as in Lemma~\ref{lem:K0_bulk_and_tau_outlier}, with \(\tau_s\) in place of
\(\tau\); since \(\tau_s\to\tau\), this gives \eqref{eq:sw_covariance_factor}.

It remains to record the anisotropic resolvent limits needed below.  The
right-resolvent deterministic equivalent for the base covariance
\(\vPi_s\vSigma_0\vPi_s\) is unchanged by the rank-one compression, because
\(\widehat\vs\) is asymptotically orthogonal to the deterministic finite-dimensional
directions considered here and the removed subspace has fixed rank.  Adding the
rank-one covariance spike
\(\tau_s\widehat\vu_s\widehat\vu_s^\top\) and applying the Sherman--Morrison
formula yields
\begin{equation}\label{eq:sw_SM_u_resolvent}
        \widehat\vu_s^\top\vQ_{R,0}^{(s)}(\lambda)\widehat\vu_s
        =\frac{m_\mu(\lambda)}{1-\tau T(s(\lambda))}+o_\P(1),
\end{equation}
whereas every unit vector asymptotically orthogonal to \(\widehat\vu_s\) has
quadratic form \(m_\mu(\lambda)+o_\P(1)\), and cross forms with
\(\widehat\vu_s\) are negligible.  The vectors \(\vb_0\), \(\vv_1\), and
\(\widetilde\vv_2\) are asymptotically orthogonal to \(\widehat\vu_s\), giving
\eqref{eq:sw_R_isotropic_orth}.  Furthermore, Lemma~\ref{lem:sw_s_q_moments}
gives the decomposition
\[
        \widehat\vq_s
        =\frac{\kappa}{\sqrt\eta}\widehat\vu_s
        +\sqrt{\omega_q}\,\ve_q+o_\P(1),
        \qquad
        \ve_q\perp\widehat\vu_s,
        \qquad \|\ve_q\|=1,
\]
with \(\ve_q\) asymptotically isotropic for the base resolvent and asymptotically
orthogonal to \(\vb_0,\vv_1,\widetilde\vv_2\).  Combining this decomposition
with \eqref{eq:sw_SM_u_resolvent} gives
\[
\begin{aligned}
        \widehat\vq_s^\top\vQ_{R,0}^{(s)}(\lambda)\widehat\vq_s
        &=\frac{\kappa^2}{\eta}\frac{m_\mu(\lambda)}{1-\tau T(s(\lambda))}
          +\omega_q m_\mu(\lambda)+o_\P(1) \\
        &=m_\mu(\lambda)\left(1+\frac{\kappa^2}{\eta}
          \frac{\tau T(s(\lambda))}{1-\tau T(s(\lambda))}\right)+o_\P(1),
\end{aligned}
\]
which is \eqref{eq:sw_R_q_with_cov}--\eqref{eq:sw_Rq_cov_lemma}.  The same
orthogonal decomposition gives the cross estimate \eqref{eq:sw_R_cross_q}.
\end{proof}

\subsection{Proof of Giant Spike and Bulk Limit}

\begin{proof-of-theorem}[\ref{thm:trained}(i)--(ii)]
The rank-one matrix $\vP_s:=\frac{b_\sigma\theta}{\sqrt N}\va\vs^\top$
has squared singular value
\[
        \|\vP_s\|^2
        =\frac{b_\sigma^2\theta^2}{N}(\va^\top\va)\|\vs\|^2
        =b_\sigma^2\theta_0^2\phi\kappa\sqrt N\,(1+o_\P(1)),
\]
using \(\va^\top\va\to1\), \(\|\vs\|^2/n\to\kappa\), and \(n/N\to\phi\).
All terms in \(\vY_{\rm sw}^\sharp-\vP_s\) have operator norm \(O_\P(1)\).
Thus the top eigenvalue of \((\vY_{\rm sw}^\sharp)^\top\vY_{\rm sw}^\sharp\)
is \(\|\vP_s\|^2+o_\P(\sqrt N)\), and the corresponding eigenvector has overlap
\(1-o_\P(1)\) with \(\widehat\vs\).  The feature error in
Lemma~\ref{lem:sw_QE} changes this eigenvalue by \(o_\P(\sqrt N)\), and
Davis--Kahan Theorem gives \eqref{eq:sw_giant_vec}.  Since
\(\vy^\top\widehat\vs=O_\P(1)\), the normalized label overlap tends to zero.
If \(b_\sigma=0\), \(\vP_s\) is absent and no diverging spike is produced by
this mechanism.

For the bulk, \(\vY_{\rm sw}^\sharp\) differs from the finite-SNR random-weight
surrogate \(\vY_0^\sharp\) by two rank-one feature perturbations in \eqref{eq:sw_Ysp_def}.  Hence the
corresponding Gram matrices differ by a finite-rank matrix.  The error
\(\vY-\vY_{\rm sw}^\sharp\) contributes negligibly to the limiting ESD: its product with
the bounded part has normalized trace norm \(o_\P(1)\), and its product with the
single giant rank-one term has rank at most two.  Since the random-weight
surrogate has limiting ESD \(\mu\) by Theorem~\ref{thm:snr_finite_eig}, so does
\(\vK\).  Compressing by \(\vPi_s\) changes at most one eigenvalue, so the same
bulk limit holds for \(\vK_s\).
\end{proof-of-theorem}

\subsection{Proof of Order-one Outliers and Label Alignment}

Let
\[
        \widehat\vq_s:=\frac{\vPi_s\vq}{\|\vPi_s\vq\|},
        \qquad
        \widetilde\vv_2:=\frac{\vPi_s\vv_2}{\|\vPi_s\vv_2\|}.
\]
By Lemmas~\ref{lem:sw_s_q_moments} and~\ref{lem:sw_left_geometry}, the
finite-rank part of \(\vY_s^\sharp\) can be written, up to an \(o_\P(1)\)
operator-norm error, as
\begin{equation}\label{eq:sw_finite_rank_model}
\vY_s^\sharp
=
\vG\vPi_s
+\sqrt\tau\,\va_0\vb_0^\top
+\sqrt{\beta_{\rm lin}}\,\va_1\vv_1^\top
+\sqrt{\frac{\beta_{\rm lin}}{\kappa}}\,\va_2\widetilde\vv_2^\top
+\sqrt{\beta_q}\,\vh\widehat\vq_s^\top
+o_\P(1),
\end{equation}
where
\[
        \vb_0:=\frac{\vm}{\|\vm\|},
        \qquad
        \va_1:=\frac{\vg_1}{\|\vg_1\|},
        \qquad
        \va_2:=\frac{\vg_2}{\|\vg_2\|}.
\]
The factor \(\beta_{\rm lin}/\kappa\) in the third term is the loss of the
\(\vv_2\)-component caused by deflating the giant \(\vs\)-direction.
Set
\[
        \Theta:=\diag\left(\sqrt\tau,\sqrt{\beta_{\rm lin}},
        \sqrt{\beta_{\rm lin}/\kappa},\sqrt{\beta_q}\right),
\qquad
        \vA:=\left[\va_0,\va_1,\va_2,\vh\right],
        \qquad
        \vB:=\left[\vb_0,\vv_1,\widetilde\vv_2,\widehat\vq_s\right].
\]
Let \(\vK_0^{(s)}:=(\vG\vPi_s)^\top(\vG\vPi_s)\), and let
\(\vQ_R^{(s)}(\lambda):=(\vK_0^{(s)}-\lambda\vI)^{-1}\) on the range of
\(\vPi_s\).  By Lemma~\ref{lem:sw_compressed_covariance}, the compressed null
matrix keeps the covariance spike, now in the direction
\(\widehat\vu_s=\vPi_s\vu/\|\vPi_s\vu\|\) and with strength \(\tau+o_\P(1)\).
Thus, with \(t=T(s(\lambda))=\lambda s(\lambda)m_\mu(\lambda)\),
\begin{equation}\label{eq:sw_right_resolvent_matrix}
\vB^\top\vQ_R^{(s)}(\lambda)\vB
=m_\mu(\lambda)
\diag\left(1,1,1,R_q(t)\right)+o_\P(1),
\end{equation}
where \(R_q(t)\) is given by \eqref{eq:sw_Rq_cov_lemma}.
From the proof of Lemma~\ref{lem:sw_compressed_covariance}, we know that
\(\widehat\vq_s\) has squared overlap \(\kappa^2/\eta\) with the projected
covariance-spike direction \(\widehat\vu_s\), while \(\vb_0\), \(\vv_1\), and
\(\widetilde\vv_2\) are asymptotically orthogonal to it.  Hence only the
quadratic spiked-weight direction sees the covariance pole.
Together with \eqref{eq:sw_left_resolvent_matrix}, the finite-rank determinant
for order-one outliers of \(\vK_s^\sharp\) is
\begin{equation}\label{eq:sw_relative_det}
D_{\rm rel}(\lambda)
=
\det\left(\vI_4-
        t\,\Theta
        \begin{pmatrix}
        1&0&0&\rho\\
        0&1&0&0\\
        0&0&1&0\\
        \rho&0&0&1
        \end{pmatrix}
        \Theta
        \diag(1,1,1,R_q(t))
        \right)+o_\P(1).
\end{equation}
Expanding the determinant gives
\begin{equation}\label{eq:sw_Drel_factor}
        D_{\rm rel}(\lambda)
        =\left(1-\beta_{\rm lin}t\right)
        \left(1-\frac{\beta_{\rm lin}}{\kappa}t\right)
        \left[(1-\tau t)(1-\beta_qR_q(t)t)-\rho^2\tau\beta_qR_q(t)t^2\right]
        +o_\P(1).
\end{equation}
The denominator \(1-\tau t\) in \(R_q(t)\) is the pole of the covariance outlier
already present in the base matrix \(\vK_0^{(s)}\), not an additional root.  As
in the collision-safe argument for the uninformative cluster in
Appendix~\ref{app:proof_finite_snr}, the characteristic determinant is obtained
by multiplying by the covariance factor \(1-\tau t\).  Using \(\rho^2=1/3\),
\begin{align}\label{eq:sw_full_factorization}
(1-\tau t)D_{\rm rel}(\lambda)
&=\left(1-\beta_{\rm lin}t\right)
  \left(1-\frac{\beta_{\rm lin}}{\kappa}t\right)
  F_q(t)+o_\P(1),
\end{align}
where \(F_q\) is exactly \eqref{eq:sw_Fq_t}.  This is the desired simplification
of the augmented determinant.

\begin{proof-of-theorem}[\ref{thm:trained}(iii)]
Choose a contour in $\C$ enclosing one candidate location \(\lambda_\star\) and disjoint
from $\supp{\mu}\cup\{0\}$ and from all other candidate locations from \eqref{eq:sw_order_one_factors}.  On this
contour, \eqref{eq:sw_full_factorization} is a uniform approximation to the
finite-dimensional determinant.  Rouch\'e's theorem gives the same number of
zeros as the deterministic product in \eqref{eq:sw_full_factorization}, counted
with algebraic multiplicity.  These zeros are equivalent to eigenvalues of the
linearized matrix, hence to eigenvalues of \(\vK_s^\sharp\), outside the bulk.
Lemma~\ref{lem:sw_QE} transfers the conclusion to \(\vK_s\).  The admissibility
condition \(z'(s)>0\) is the usual BBP threshold from
Lemma~\ref{lem:additive_scalar_eq}; roots on branches with \(z'(s)\le0\) do not
separate from the limiting support $\mu$.   
\end{proof-of-theorem}

\begin{proof-of-theorem}[\ref{thm:trained}(iv)]
We now compute the resolvent in the label direction.  Since
\(\vy^\top\widehat\vs=O_\P(1)\), replacing \(\vy/\sqrt n\) by
\(\vPi_s\vy/\|\vPi_s\vy\|\) changes normalized overlaps by \(o_\P(1)\).  The
linear spike directions \(\vv_1\) and \(\widetilde\vv_2\) are asymptotically
orthogonal to \(\vy/\sqrt n\), and the same is true of the mean direction
\(\vb_0\).  More precisely, uniformly for \(\lambda\) on any fixed contour away
from \(\supp\mu\cup\{0\}\),
\begin{equation}\label{eq:sw_label_cross_vector}
        \frac1{\sqrt n}\vB^\top\vQ_R^{(s)}(\lambda)\vy
        =m_\mu(\lambda)\sqrt{\chi_y}\,\ve_4+o_\P(1).
\end{equation}
The possible covariance pole in \(\vQ_R^{(s)}\) does not change this vector,
because \(\vy\perp\vu\).  Thus the first three right coordinates
\(\vb_0,\vv_1,\widetilde\vv_2\) cannot create a pole in the label resolvent.  In
particular, the roots of the two linear factors have zero label residue, and a
pure mean-spike cluster has zero label residue.  Once \(\beta_q>0\), however,
the mean/covariance cluster is no longer a separate pure \(\vb_0\)-cluster;
it is absorbed into the coupled factor \(F_q\).  Any label mass at such a shifted
location is caused by its mixing with the quadratic direction \(\widehat\vq_s\),
not by \(\vb_0\) itself.

The Woodbury formula for the finite-rank model \eqref{eq:sw_finite_rank_model},
combined with \eqref{eq:sw_left_resolvent_matrix},
\eqref{eq:sw_right_resolvent_matrix}, and \eqref{eq:sw_label_cross_vector}, gives
uniformly away from the bulk
\begin{equation}\label{eq:sw_y_resolvent}
\frac1n\vy^\top(\vK_s^\sharp-\lambda\vI)^{-1}\vy
=
H_y(\lambda)
+
 m_\mu(\lambda)
 \frac{\beta_q\chi_y\,t\,(1-\tau t)
 \left(1-\frac23\tau t\right)}{F_q(t)}
+o_\P(1),
\end{equation}
where \(H_y\) is analytic in any neighborhood not intersecting
\(\supp\mu\cup\{0\}\), and \(t=T(s(\lambda))\).  The absence of the factors
\(1-\beta_{\rm lin}t\) and \(1-\beta_{\rm lin}t/\kappa\) from the denominator in
\eqref{eq:sw_y_resolvent} is the desired zero-alignment statement for the
\(\vv_1\)- and \(\widetilde\vv_2\)-outliers.  The absence of a separate
\(1-\tau t\) denominator is the corresponding zero-alignment statement for a
pure mean/covariance cluster.  The only non-analytic denominator relevant to the
label is \(F_q(t)\).

To see the numerator in \eqref{eq:sw_y_resolvent}, restrict the Woodbury
correction to the coupled coordinates \((\va_0,\vh)\).  The relevant
\(2\times2\) inverse has determinant \(F_q(t)/(1-\tau t)\), and its
\((2,2)\)-entry contributes
\(\beta_q[1-(2/3)\tau t](1-\tau t)/F_q(t)\).  Multiplication by the squared right
overlap \(\chi_y\) gives the singular term.

Let \(\Gamma_\star\) be a positively oriented contour around \(I_\star\).  The
spectral projector satisfies
\[
        \frac1n\vy^\top\widehat\vP_\star^{(s)}\vy
        =-\frac1{2\pi i}\oint_{\Gamma_\star}
        \frac1n\vy^\top(\vK_s-\lambda\vI)^{-1}\vy\,d\lambda .
\]
The approximation \(\|\vK_s-\vK_s^\sharp\|=o_\P(1)\) allows us to replace
\(\vK_s\) by \(\vK_s^\sharp\).  The analytic term \(H_y\) integrates to zero.
Hence, if the contour contains no root of \(F_q(T(s))\), the limit is zero.
If the contour contains a single simple root \(s_\star\) of \(F_q(T(s))\), then
\[
        \frac{d}{d\lambda}F_q(T(s(\lambda)))\Big|_{\lambda=z(s_\star)}
        =F_q'(T(s_\star))\frac{T'(s_\star)}{z'(s_\star)}.
\]
Taking minus the residue of \eqref{eq:sw_y_resolvent} gives
\begin{equation}\label{eq:sw_Gamma}
        \Gamma_{\rm sw}(s)
        :=-\frac{m_\mu(z(s))z'(s)}{T'(s)}
        \frac{\beta_q\chi_y\,T(s)\,(1-\tau T(s))
        \left(1-\frac23\tau T(s)\right)}{F_q'(T(s))} .
\end{equation}  
If a contour encloses several simple \(F_q\)-roots, the
same calculation sums their residues.  This completes the proof of
Theorem~\ref{thm:trained}.
\end{proof-of-theorem}
\section{Proof for Quadratic Sample-size Regime}
\label{sec:quadratic-regime-proof}
In this section we prove Theorem~\ref{thm:quadratic} for the CK matrix in the
quadratic sample-size $n\asymp d^2$ and linear-width $N\asymp n$ regime.
The CK matrix is a sample covariance matrix whose population covariance is a
quadratic polynomial kernel.  This produces two successive BBP transitions:
one in the quadratic kernel of the data governed by
the aspect ratio $\gamma$, and one from the linear-width sample covariance
model governed by $\phi$.
Throughout this section we consider quadratic asymptotic limit of the sample size:
\begin{equation}\label{eq:quadratic_limit}
        p:=\frac{d(d+1)}2,
    \qquad \frac np\to \gamma\in(0,\infty),
    \qquad r \in[0,\infty).
\end{equation}
Set
\begin{equation}\label{eq:right_edge_MP}
     \delta:=\frac{r^2}{2},
    \qquad
    \ell:=\delta^2=\frac{r^4}{4},
    \qquad
    \lambda_+(\gamma):=(1+\sqrt\gamma)^2.    
\end{equation}
   
Let \(h_k\) be the orthonormal Hermite polynomials and define
\[
    a_k:=\E[\sigma(\xi)h_k(\xi)],\qquad \xi\sim\cN(0,1).
\]
Thus
$
    a_1=b_\sigma,
$ 
$
    a_2=\frac{c_\sigma}{\sqrt2},
$
$
    \alpha_2:=a_2^2=\frac{c_\sigma^2}{2},
$ and $
    \alpha_0:=1-a_1^2-a_2^2.
$

\paragraph{Quadratic lift.}
Let \(\mathrm{svec}:\mathbb S^d\to\R^p\) denote the Frobenius-isometric
symmetric vectorization.  Define
\begin{equation}\label{eq:def_q_i_Q}
        \mbi{q}_i:=\mathrm{svec}(\mbi{x}_i\mbi{x}_i^\top),
    \qquad
    \mbi{Q}:=[\mbi{q}_1,\ldots,\mbi{q}_n]\in\R^{p\times n}.
\end{equation}
Then
\begin{equation}\label{eq:def_H}
        \mbi{H}:=(\mbi{X}^\top\mbi{X})^{\odot2}=\mbi{Q}^\top\mbi{Q}.
\end{equation}
For the XOR model, write \(\mbi{x}_i=\mbi{z}_i+\mbi{m}_i\), where
\(\mbi{z}_i\sim\cN(0,\mbi{I}_d/d)\).  Define the second- and first-chaos
lifted matrices \(\mbi{E}_2,\mbi{E}_1\in\R^{p\times n}\) by their $i$-th columns, for $i\in[n]$,
\begin{equation}\label{eq:def_E_1_iE_2_i}
    (\mbi{E}_2)_i
    :=\mathrm{svec}\!\left(\mbi{z}_i\mbi{z}_i^\top-\frac1d\mbi{I}_d\right),
    \qquad
    (\mbi{E}_1)_i
    :=\mathrm{svec}\!\left(\mbi{z}_i\mbi{m}_i^\top+\mbi{m}_i\mbi{z}_i^\top\right).
\end{equation}
Finally set
\begin{equation}\label{eq:def_s_v}
    \mbi{s}:=\frac1{\sqrt2}\mathrm{svec}
    \left(\mbi{u}_1\mbi{u}_1^\top-\mbi{u}_2\mbi{u}_2^\top\right),
    \qquad \|\mbi{s}\|=1,
    \qquad
    \mbi{v}:=\frac{\mbi{y}}{\sqrt n}.
\end{equation}
The unprojected rank-one second-chaos spike is
\begin{equation}\label{eq:def_Y_0}
    \mbi{Y}_0:=\mbi{E}_2+\frac{\delta}{\sqrt p}\mbi{s}\mbi{y}^\top
    =\mbi{E}_2+\theta_n\mbi{s}\mbi{v}^\top,
    \qquad
    \theta_n:=\delta\sqrt{\frac np}\to \theta:=\delta\sqrt\gamma.
\end{equation}

\paragraph{Projected nuisance space.}
For \(s_i=\|\mbi{x}_i\|\), define
\begin{equation}\label{eq:mu-A-nuisance-def_diag}
    \zeta_k^{(i)}:=\E[\sigma(s_i\xi)h_k(\xi)],
    \qquad
    \mbi{D}_k:=\diag(\zeta_k^{(1)},\ldots,\zeta_k^{(n)}),
    \qquad
    \mbi{D}:=\diag(s_1,\ldots,s_n),        
\end{equation}
\begin{equation}\label{eq:mu-A-nuisance-def}
    \mbi{\mu}_\sigma:=(\zeta_0^{(1)},\ldots,\zeta_0^{(n)})^\top,
    \qquad
    \mbi{A}_\sigma:=\diag\left(\frac{\zeta_1^{(i)}}{s_i}\right)_{i=1}^n,
    \qquad
    \mbi{B}_{\sigma}:=\diag\left(\frac{\zeta_3^{(i)}}{s_i}\right)_{i=1}^n.
\end{equation}
We define the BBP-projected CK nuisance space by
\begin{equation}
\label{eq:U-sigma-plus-bbp}
    \mathcal U_\sigma^+
    :=
    \operatorname{span}\{\mbi{1},\mbi{\mu}_\sigma\}
    +\operatorname{Range}(\mbi{A}_\sigma\mbi{X}^\top)
    +\operatorname{Range}(\mbi{B}_{\sigma}\mbi{X}^\top)
    +\operatorname{Range}(\mbi{E}_1^\top).
\end{equation}
Finally, we define the projection by
\begin{equation}
\label{eq:P-sigma-plus-bbp}
    \mbi{P}_\sigma
    :=\operatorname{Proj}\big((\mathcal U_\sigma^+)^\perp\big).
\end{equation}

\paragraph{Limiting spectra and transforms.} Define the limiting population law
\begin{equation}\label{eq:nuq-def}
        \nu_{\rm q}
        :=\alpha_0+\alpha_2 \rho_\gamma^{\mathrm{MP}},
\end{equation}
that is, the push-forward of the MP law $\rho_\gamma^{\mathrm{MP}}$ by
$x\mapsto\alpha_0+\alpha_2x$.  Let
\begin{equation}\label{eq:muq-def}
        \mu_{\rm q}
        :=\rho_\phi^{\mathrm{MP}}\boxtimes\nu_{\rm q}.
\end{equation}
Slightly different from \eqref{eq:z(s)}, for this pair $(\nu_{\rm q},\phi)$ we define the
$z$-transform and $\varphi$-transform as
\begin{align}
        z(s)
        &:=-\frac1s+
        \phi\int\frac{t}{1+ts}\,\nu_{\rm q}(dt),\qquad \varphi(s):=-\frac{s z'(s)}{z(s)}.
        \label{eq:zq-def}
\end{align}
These are the transforms in \eqref{eq:z(s)} with $\nu$
replaced by $\nu_{\rm q}$ in \eqref{eq:nuq-def}.  The upper edge of $\nu_{\rm q}$
is $\tau_+^{\nu}=\alpha_0+\alpha_2\lambda_+(\gamma).$

Recall parameters in \eqref{eq:quadratic_limit} and \eqref{eq:right_edge_MP}.
For later use, define the BBP location and two maps
\begin{align}
        \Lambda_y 
   & :=\alpha_0+\alpha_2\lambda_{\rm out}(\gamma,\ell), \label{eq:lambda-out-def}\\
        \lambda_{\mathrm{out}}(\gamma,\ell)
        &:=1+\gamma+\gamma\ell+\frac1\ell, \qquad
        \mathrm{Align}(\gamma,
    \ell):=\frac{\gamma\ell^2-1}{\gamma\ell(\ell+1)}.
        \label{eq:Align-def}
\end{align}

\subsection{The Projection Preserves the XOR Label}
\label{sec:quadratic-deflation}

\begin{lemma}[The projection preserves the XOR label]
\label{lem:Psigma-preserves-label}
Under the assumptions of Theorem~\ref{thm:quadratic}, let
\(\mbi{P}_\sigma\) be defined in \eqref{eq:P-sigma-plus-bbp}.
Then
\begin{equation}\label{eq:y-projection-preserved}
          \rank{\vI_n-\mbi{P}_\sigma}=o(n),
        \qquad
        \frac1n\vy^\top\mbi{P}_\sigma\vy\xrightarrow{\P}1.
\end{equation}
\end{lemma}
\begin{proof}
We prove the rank bound first.  Clearly, $\dim\operatorname{span}\{\vone,\vmu_\sigma\}\le 2,$
and $\rank{\vA_\sigma\vX^\top}\le d,$ $\rank{\vB_{3,\sigma}\vX^\top}\le d.$
It remains to control \(\rank{\vE_1^\top}\). Recall that $ (\vE_1)_i
        =
        \mathrm{svec}(\vz_i\vm_i^\top+\vm_i\vz_i^\top),$
and in the XOR model $\vm_i
        =
        \eta_i\frac r{\sqrt d}\vu_{\kappa_i},$  $ \eta_i\in\{\pm1\},$ and $ \kappa_i\in\{1,2\}.$
For \(a\in\{1,2\}\), define the \(n\times d\) matrix \(\vZ_a^{\rm row}\) by
\[
        (\vZ_a^{\rm row})_{i\cdot}
        :=
        \mathbf 1_{\{\kappa_i=a\}}\eta_i\frac r{\sqrt d}\vz_i^\top .
\]
Let $\vC_E:=
        \begin{bmatrix}
        \vZ_1^{\rm row} & \vZ_2^{\rm row}
        \end{bmatrix}
        \in\mathbb R^{n\times 2d}.$
We claim that
$
        \operatorname{Range}(\vE_1^\top)\subseteq \operatorname{Range}(\vC_E).
$
For any \(\vw\in\mathbb R^p\), let \(\vW \) be the
symmetric matrix such that \(\vw=\mathrm{svec}(\vW)\).  Then the \(i\)-th entry
of \(\vE_1^\top\vw\) is
\[
\begin{aligned}
        [\vE_1^\top\vw]_i
        &=
        \left\langle
        \mathrm{svec}(\vz_i\vm_i^\top+\vm_i\vz_i^\top),
        \mathrm{svec}(\vW)
        \right\rangle                                                  \\
        &=
        \left\langle
        \vz_i\vm_i^\top+\vm_i\vz_i^\top,
        \vW
        \right\rangle_F                                                \\
        &=
        2\vz_i^\top\vW\vm_i=
        \mathbf 1_{\{\kappa_i=1\}}\eta_i\frac r{\sqrt d}
        \bigl(2\vW\vu_1\bigr)^\top\vz_i
        +
        \mathbf 1_{\{\kappa_i=2\}}\eta_i\frac r{\sqrt d}
        \bigl(2\vW\vu_2\bigr)^\top\vz_i .
\end{aligned}
\]
Thus \(\vE_1^\top\vw\) belongs to the column span of \(\vC_E\).  Therefore
$\rank{\vE_1^\top}\le \rank{\vC_E}\le 2d.$
Then, because \(n\asymp d^2\), we have
\[
        \rank{\vI_n-\mbi{P}_\sigma}
        =
        \dim(\mathcal U_\sigma^+)
        \le 2+d+d+2d
        =
        4d+2
        =
        o(n).
\]

We now prove the label-preservation statement.  It is convenient to work with
a slightly larger nuisance space.  Define the \(n\times q\) matrix
\[
        \vF:=
        \begin{bmatrix}
        \vmu_\sigma &
        \vA_\sigma\vX^\top &
        \vB_{3,\sigma}\vX^\top &
        \vC_E
        \end{bmatrix},
        \qquad
        q\le 1+d+d+2d=4d+1.
\]
Then $\mathcal U_\sigma^+
        \subseteq
        \operatorname{span}\{\vone\}+\operatorname{Range}(\vF).$
Let $\vP:=\vI_n-\frac1n\vone\vone^\top.$
Since \(\vy^\top\vone=0\), we have
\[
        \|\Pi_{\mathcal U_\sigma^+}\vy\|_2^2
        \le
        \left\|
        \Pi_{\operatorname{span}\{\vone\}+\operatorname{Range}(\vF)}
        \vy
        \right\|_2^2 = \|\Pi_{\operatorname{Range}(\vP\vF)}\vy\|_2^2.
\]
Therefore
\begin{equation}\label{eq:Pi_y_bound}
 \|\Pi_{\mathcal U_\sigma^+}\vy\|_2^2
        \le
        (\vF^\top\vy)^\top
        (\vF^\top\vP\vF)^\dagger
        (\vF^\top\vy).
\end{equation} 
Let the \(i\)-th row of \(\vF\) be \(\vf_i^\top\in\mathbb R^q\).  Define   $\bar{\vf}
        :=
        \frac1n\sum_{i=1}^n\E\vf_i,$
and $ \vSigma_F
        :=
        \frac1n
        \sum_{i=1}^n
        \E\bigl[
        (\vf_i-\bar{\vf})(\vf_i-\bar{\vf})^\top
        \bigr].$

We first prove that $ \sum_{i=1}^n y_i\E\vf_i=\vzero .$
Firstly, the first coordinate of \(\vf_i\) is
$
        \zeta_0^{(i)}
        =
        \E[\sigma(s_i\xi)],
$
whose expectation depends only on \(\|\vm_i\|=r/\sqrt d\), the same for all
clusters.  Since \(\sum_i y_i=0\), its label-weighted expectation vanishes.
For the two linear blocks in \(\vf_i\), write $a(s):=\frac{\zeta_1(s)}s,$ and
$b(s):=\frac{\zeta_3(s)}s.$
By rotational symmetry of the Gaussian noise, for any radial scalar function
\(h\),
$
        \E\bigl[h(\|\vz+\vm\|)(\vz+\vm)\bigr]
        =
        \kappa_h\,\vm
$
for some scalar \(\kappa_h\) depending only on \(\|\vm\|\). Therefore
\begin{align}
        \sum_i y_i
        \E\bigl[a(s_i)\vx_i\bigr]
        =
        \kappa_a\sum_i y_i\vm_i
        =
        \vzero,\\
         \sum_i y_i
        \E\bigl[b(s_i)\vx_i\bigr]
        =
        \kappa_b\sum_i y_i\vm_i
        =
        \vzero.
\end{align}
Here we used the XOR symmetry: the two signs on each axis are balanced, so $\sum_i y_i\vm_i=\vzero.$ Finally, the \(\vC_E\) block has zero mean because it is linear in \(\vz_i\), hence $ \E\left[
        \mathbf 1_{\{\kappa_i=a\}}\eta_i\frac r{\sqrt d}\vz_i
        \right]
        =
        \vzero.$
Thus $\sum_i y_i\E\vf_i=\vzero.$
Since \(\vy^\top\vone=0\), we have
$
        \vF^\top\vy
        =
        \sum_{i=1}^n y_i(\vf_i-\bar{\vf}).
$
The previous population orthogonality gives 
\begin{equation}\label{eq:Fy=0}
        \E[\vF^\top\vy]=\vzero.
\end{equation}

We next use sample-covariance concentration bound for
independent non-identically distributed rows in dimension
\(q=o(n)\).  The rows \(\vf_i\) are smooth
finite-dimensional functions of the Gaussian vectors \(\vz_i\), with uniformly
bounded derivatives on the event \(\max_i|s_i-1|=o(1)\), and the remaining
tails are Gaussian. 
We now prove that
\begin{equation}
\label{eq:F-relative-concentration}
        \left\|
        \vSigma_F^{\dagger/2}
        \left(
        \frac1n\vF^\top\vP\vF-\vSigma_F
        \right)
        \vSigma_F^{\dagger/2}
        \right\|
        =
        o_{\P}(1).
\end{equation}
Consider an orthonormal basis \(\vU\in\R^{n\times r_{F}}\)
of \(\operatorname{supp}(\vSigma_F)\) where $r_{F}=\rank{\vSigma_F}$, and define
\[
        \widetilde{\vSigma}_F:=\vU^\top\vSigma_F\vU,
        \qquad
        \vg_i:=
        \widetilde{\vSigma}_F^{-1/2}
        \vU^\top(\vf_i-\bar{\vf}),
\]
then $\frac1n\sum_{i=1}^n\E[\vg_i\vg_i^\top]=\vI.$ Let $\vG\in\R^{n\times r_{F}}$ be the matrix whose $i$-th row is $\vg_i$ . Then $$\left\|
        \vSigma_F^{\dagger/2}
        \left(
        \frac1n\vF^\top\vP\vF-\vSigma_F
        \right)
        \vSigma_F^{\dagger/2}
        \right\|=\left\|\frac1n \vG^\top\vP\vG-\vI_{r_{F}}\right\|, \quad\text{ and }\quad \frac1n\vP\vG^\top\vP\vG=\frac1n\sum_{i=1}^n\vg_i\vg_i^\top-\bar\vg\bar\vg^\top$$
where $\bar\vg=\frac1n\sum_{i=1}^n\vg_i$. Hence to prove \eqref{eq:F-relative-concentration}, we only need to show that
\begin{equation}\label{eq:F-relative-concentration2}
\left\|\frac1n \sum_{i=1}^n\vg_i\vg_i^\top-\vI_{r_{F}}\right\|=o_{\P}(1)
\end{equation}
and $\|\bar\vg\|=o_{\P}(1)$. The second one is derict: since $\sum_{i=1}^n\E[\vg_i]$, $\E[\|\bar\vg\|^2]\le \frac{1}{n^2}\sum_{i=1}^n\Tr\E[\vg_i\vg_i^\top]=O(1/d)$, and then we can apply Markov's inequality to conclude that $\|\bar\vg\|=o_{\P}(1)$. For the first part, we can apply Proposition 2.6 of \cite{nourdin2009universality} for the Gaussian-chaos tail bound of \(F\) and Theorem 6.1 of \cite{vershynin2012close} for sample covariance concentration inequality to conclude \eqref{eq:F-relative-concentration2}. 

Therefore, with probability tending to one, \eqref{eq:F-relative-concentration} imppiles that $\vF^\top\vP\vF
        \succeq
        \frac n2\,\vSigma_F$ on $\operatorname{supp}(\vSigma_F)$
and hence $ (\vF^\top\vP\vF)^\dagger
        \preceq
        \frac2n\vSigma_F^\dagger.$
From \eqref{eq:Pi_y_bound}, it follows that, with probability tending to one,
\begin{align}
        \|\Pi_{\mathcal U_\sigma^+}\vy\|_2^2
        &\le
        (\vF^\top\vy)^\top
        (\vF^\top\vP\vF)^\dagger
        (\vF^\top\vy)                                      
        \le
        \frac2n
        (\vF^\top\vy)^\top
        \vSigma_F^\dagger
        (\vF^\top\vy).\label{eq:Pi_y_bound1}
\end{align}
It remains to bound the whitened score.  Since the rows are independent and
\eqref{eq:Fy=0}, we have
\begin{align}
        &\E\left[
        (\vF^\top\vy)^\top
        \vSigma_F^\dagger
        (\vF^\top\vy)
        \right]                                              
        =
        \Tr\left[
        \vSigma_F^\dagger
        \Cov(\vF^\top\vy)
        \right]                                                  
        =
        \Tr\left[
        \vSigma_F^\dagger
        \sum_{i=1}^n\Cov(\vf_i)
        \right].\label{eq:expectation_Fy_quad}
\end{align}
Moreover, since $\E[(\vf_i-\bar\vf)(\vf_i-\bar\vf)^\top]=\Cov(\vf_i)+(\E[\vf_i]-\bar\vf)(\E[\vf_i]-\bar\vf)^\top$, we have that
\[
        \sum_{i=1}^n\Cov(\vf_i)
        \preceq
        \sum_{i=1}^n
        \E\bigl[
        (\vf_i-\bar{\vf})(\vf_i-\bar{\vf})^\top
        \bigr]
        =
        n\vSigma_F.
\]
Therefore, from \eqref{eq:expectation_Fy_quad}, we have
$
        \E\left[
        (\vF^\top\vy)^\top
        \vSigma_F^\dagger
        (\vF^\top\vy)
        \right]
        \le
        n\,\rank{\vSigma_F}
        \le
        n(4d+1).
$
By Markov's inequality,
$
        (\vF^\top\vy)^\top
        \vSigma_F^\dagger
        (\vF^\top\vy)
        =
        O_{\P}(nd).
$
Consequently, by \eqref{eq:Pi_y_bound1}
$ 
\|\Pi_{\mathcal U_\sigma^+}\vy\|_2^2
        =
        O_{\P}(d).
$
Finally, 
$
\vy^\top\mbi{P}_\sigma\vy
        =
        \|\vy\|_2^2-\|\Pi_{\mathcal U_\sigma^+}\vy\|_2^2
        =
        n-O_{\P}(d),
$
and hence we complete the proof.
\end{proof}

\subsection{QE of the Population Covariance Matrix}
\label{sec:quadratic-QE}

Let $\mbi{w}\sim\cN(\mbi{0},\mbi{I}_d)$ and define the population CK matrix
\[
\mbi{\Phi}(\mbi{X}) \;:=\;\E_{\mbi{w}}\big[\sigma(\mbi{w}^\top\mbi{X})^\top\sigma(\mbi{w}^\top\mbi{X})\mid \mbi{X}\big]
\in\R^{n\times n},
\qquad
\mbi{\Phi}(\mbi{X})_{ij}=\E_{\mbi{w}}[\sigma(\mbi{w}^\top\mbi{x}_i)\sigma(\mbi{w}^\top\mbi{x}_j)\mid \mbi{X}].
\]
The following lemmas derive the QE of the population kernel $\mbi{\Phi}(\mbi{X})$
after the projection $\mbi{P}_\sigma$ defined in \eqref{eq:P-sigma-plus-bbp}.

\begin{lemma}[Expansion of $\mbi{\Phi}(\mbi{X})$]
\label{lem:Mehler-exact}
Recall $\vD_k$ defined in \eqref{eq:mu-A-nuisance-def_diag}. Let $\vG=\vX^\top\vX$. Then
\begin{equation}\label{eq:Mehler-use-Psigma}
         \mbi{\Phi}(\mbi{X}) \;=\;\sum_{k\ge 0}\mbi{D}_k\,\mbi{G}^{\odot k}\,\mbi{D}_k,
\qquad
(\mbi{G}^{\odot k})_{ij} := (\mbi{G}_{ij})^k.
\end{equation}
\end{lemma}
\begin{proof}
The pair $(\mbi{w}^\top\mbi{x}_i,\mbi{w}^\top\mbi{x}_j)$ is
jointly Gaussian with variances $(s_i^2,s_j^2)$ and correlation
$\mbi{R}_{ij}:=\mbi{G}_{ij}/(s_is_j)$.  Writing $\mbi{w}^\top\mbi{x}_i=s_ig_i$ with
$(g_i,g_j)$ standard normal, $\E[g_ig_j]=\mbi{R}_{ij}$, expanding
$\sigma(s_ig_i)=\sum_k\zeta_k^{(i)}h_k(g_i)$ in $L^2$ and using
$\E[h_k(g_i)h_\ell(g_j)]=\delta_{k\ell}\mbi{R}_{ij}^k$ yields \eqref{eq:Mehler-use-Psigma}.  See the proof of Lemma~5.2 in \citet{wang2021deformed} for more details.
\end{proof}

\begin{lemma}
\label{lem:k012-identities}
We have $\mbi{D}_0\mbi{G}^{\odot 0}\mbi{D}_0=\mbi{\mu}_\sigma\mbi{\mu}_\sigma^\top$ and
$\mbi{D}_1\mbi{G}\mbi{D}_1=\mbi{A}_\sigma\,\mbi{G}\,\mbi{A}_\sigma.$
\end{lemma}
\begin{proof}
Immediate from the definitions of $\vmu_\sigma,\vA_\sigma$ in
\eqref{eq:mu-A-nuisance-def}.
\end{proof}

Let $\vT_k:=\mbi{D}_k\,\mbi{G}^{\odot k}\,\mbi{D}_k$. Then from Lemmas~\ref{lem:Mehler-exact} and \ref{lem:k012-identities}, we have 
\begin{equation}\label{lem:Mehler-exact2}
       \vP_\sigma \vPhi(\vX)\vP_\sigma = \sum_{k=2}^\infty \vP_\sigma\vT_k\vP_\sigma.
\end{equation}
In the following two subsections, we control the first three terms in the expansion of \eqref{lem:Mehler-exact2}
\subsubsection{The Quadratic Term}
\begin{lemma}[Radial cancellation by the nuisance projector]
\label{lem:radial-cancellation-Psigma}
Recall $s_i:=\|\mbi x_i\|$ for $i=1,\ldots,n$.  Let
\[
        m(s):=\zeta_0(s)=\E[\sigma(s\xi)],
        \qquad
        q(s):=\zeta_2(s)=\E[\sigma(s\xi)h_2(\xi)].
\]
Assume \(a_2=q(1)\neq0\).  Let \(g\) be any \(C^2\) function in a neighborhood
of \(1\).  For \(g_i:=g(s_i)\), set $ \mbi g:=(g_1,\ldots,g_n)^\top .$
Then $\frac1d\|\mbi P_\sigma \mbi g\|_2^2=o_\P(1).$ In particular,
\begin{equation}
\label{eq:radial-cancellation-v-q}
        \frac1d\|\mbi P_\sigma \mbi v\|_2^2=o_\P(1),
        \qquad
        \frac1d\|\mbi P_\sigma \mbi q\|_2^2=o_\P(1),
\end{equation}
where $\mbi v:=(s_1^2,\ldots,s_n^2)^\top,$ and $\mbi q:=(\zeta_2(s_1),\ldots,\zeta_2(s_n))^\top.$
\end{lemma}
\begin{proof}
By Stein's identity and the definition of \(a_2\), $ m'(1)=\sqrt2\,a_2.$
Since \(a_2\neq0\), we have \(m'(1)\neq0\). For a fixed \(C^2\) function \(g\), define $\beta_g:=\frac{g'(1)}{m'(1)}$ and
$
        r_g(s)
        :=
        g(s)-g(1)-\beta_g\{m(s)-m(1)\}.
$
Then $ r_g(1)=0,$ and $r_g'(1)=g'(1)-\beta_g m'(1)=0.$
Therefore, by Taylor's theorem, there exists a constant \(C_g<\infty\) such that
for all \(s\) sufficiently close to \(1\), $|r_g(s)|\le C_g |s-1|^2.$
Let $\mbi r_g:=(r_g(s_1),\ldots,r_g(s_n))^\top.$ Since
\[
        \mbi g
        =
        \{g(1)-\beta_g m(1)\}\mbi 1
        +
        \beta_g \mbi\mu_\sigma
        +
        \mbi r_g,
\]
and since \(\mbi P_\sigma\mbi 1=\mbi 0\) and
\(\mbi P_\sigma\mbi\mu_\sigma=\mbi 0\), we get $ \mbi P_\sigma\mbi g
        =
        \mbi P_\sigma\mbi r_g.$
Hence by Lemma~\ref{lem:xor_orthonormal}
\[
        \|\mbi P_\sigma\mbi g\|_2^2
        \le
        \|\mbi r_g\|_2^2
        \le
        C_g^2\sum_{i=1}^n |s_i-1|^4 \lesssim\left(\max_i |s_i-1|^2\right)
        \sum_{i=1}^n |s_i-1|^2
        \le
        \tau_n^2 B^2.
\]
where $\tau_n^2=O\!\left(\frac{\log n}{d}\right),$ and $B^2=O(d).$
Thus $\frac1d\|\mbi P_\sigma\mbi g\|_2^2= o_\P(1)$.
Taking \(g(s)=s^2\) and \(g(s)=q(s)=\zeta_2(s)\) implies
\eqref{eq:radial-cancellation-v-q}.
\end{proof}

\begin{lemma}\label{lem:k2-term}
Recall $\vH$ defined in \eqref{eq:def_H}.   We have that  
\begin{equation}\label{eq:k2-to-H}
        \left\|
        \mbi{P}_\sigma\vT_2\mbi{P}_\sigma
        -a_2^2\mbi{P}_\sigma\mbi{H}\mbi{P}_\sigma
        \right\|=o_\P(1).
\end{equation}
\end{lemma}
\begin{proof}
Recall $s_i=\|\mbi x_i\|$ and $\mbi G=\mbi X^\top\mbi X$.  Define the normalized data $ \mbi u_i:=\frac{\mbi x_i}{s_i},$ and $\mbi R_{ij}:=\mbi u_i^\top \mbi u_j .$
Thus $ \mbi X^\top\mbi X
        =
        \mbi D \mbi R \mbi D,$
and $\mbi H
        =
        (\mbi X^\top\mbi X)^{\odot2}
        =
        \mbi D^2 \mbi R^{\odot2}\mbi D^2.$ 
Set $ q_i:=\zeta_2(s_i),$ and $v_i:=s_i^2,$
and define
\[
        \mbi q:=(q_1,\ldots,q_n)^\top,
        \qquad
        \mbi v:=(v_1,\ldots,v_n)^\top,
        \qquad
        \mbi D_q:=\diag(q_1,\ldots,q_n),
        \qquad
        \mbi D_v:=\diag(v_1,\ldots,v_n).
\]
Thus, $\mbi H=\mbi D_v \mbi R^{\odot2}\mbi D_v$
and
$
\vT_2
        =
        \mbi D_q \mbi R^{\odot2}\mbi D_q.$
We now separate the large constant component of \(\mbi R^{\odot2}\).  Define
$
        \mbi\Psi
        :=
        \mbi R^{\odot2}
        -
        \frac1d \mbi 1\mbi 1^\top .
$
This is the degree-two spherical harmonic Gram matrix, since for any \(i,j\in[n]\),
\[
        \Psi_{ij}
        =
        \left\langle
        \mbi u_i\mbi u_i^\top-\frac1d\mbi I_d,\,
        \mbi u_j\mbi u_j^\top-\frac1d\mbi I_d
        \right\rangle_F .
\]
Hence $ \mbi R^{\odot2}
        =
        \frac1d\mbi 1\mbi 1^\top+\mbi\Psi.$
Consequently, if we define $\widetilde{\mbi H}:=\mbi D_v\mbi\Psi\mbi D_v$ and $\vC_\sigma:=\vD_v^{-1} \vD_q$, then we have
\[
        \mbi H
        =
        \frac1d\mbi v\mbi v^\top
        +
        \widetilde{\mbi H},
        \qquad  \vT_2
        =
        \frac1d\mbi q\mbi q^\top
        +
        \vC_\sigma\widetilde{\mbi H} \vC_\sigma.      
\]
Therefore
\begin{align}
        \mbi P_\sigma\vT_2 \vP_\sigma
        -
        a_2^2\mbi P_\sigma\mbi H\mbi P_\sigma                         
        =
        \frac1d
        \mbi P_\sigma
        \left(
        \mbi q\mbi q^\top-a_2^2\mbi v\mbi v^\top
        \right)
        \mbi P_\sigma                                                    
        +
        \mbi P_\sigma
        \left(
        \mbi C_\sigma\widetilde{\mbi H}\mbi C_\sigma
        -
        a_2^2\widetilde{\mbi H}
        \right)
        \mbi P_\sigma .\label{eq:split-rank-one-and-harmonic}
\end{align}
By Lemma~\ref{lem:radial-cancellation-Psigma}, we can control the first term at the right-hand side of \eqref{eq:split-rank-one-and-harmonic} as
\[
\begin{aligned}
        \left\|
        \frac1d
        \mbi P_\sigma
        \left(
        \mbi q\mbi q^\top-a_2^2\mbi v\mbi v^\top
        \right)
        \mbi P_\sigma
        \right\|                                                    
        \le
        \frac1d\|\mbi P_\sigma\mbi q\|_2^2
        +
        \frac{a_2^2}{d}\|\mbi P_\sigma\mbi v\|_2^2
        =
        o_\P(1).
\end{aligned}
\]
Next, we use the centered quadratic-lift bound
\begin{equation}
\label{eq:centered-quadratic-lift-bound}
        \|\widetilde{\mbi H}\|=O_\P(1).
\end{equation}
We defer the proof of \eqref{eq:centered-quadratic-lift-bound} to
Lemma~\ref{lem:centered-quadratic-lift-bound-detailed} below.
Now define $ c(s):=\frac{\zeta_2(s)}{s^2}.$
Since \(c\) is continuous near \(1\), \(c(1)=a_2\), and
\(\max_i|s_i-1|=o_\P(1)\), we have
\[
        \|\mbi C_\sigma-a_2\mbi I_n\|
        =
        \max_i |c(s_i)-a_2|
        =
        o_\P(1).
\]
Also \(\|\mbi C_\sigma\|=O_\P(1)\).  Therefore
\[
\begin{aligned}
        &\left\|
        \mbi P_\sigma
        \left(
        \mbi C_\sigma\widetilde{\mbi H}\mbi C_\sigma
        -
        a_2^2\widetilde{\mbi H}
        \right)
        \mbi P_\sigma
        \right\|                                                     
        =
        o_\P(1).
\end{aligned}
\]
Combining \eqref{eq:split-rank-one-and-harmonic}, we prove \eqref{eq:k2-to-H}.
\end{proof}

\begin{lemma}[Operator norm bound for the centered quadratic lift]
\label{lem:centered-quadratic-lift-bound-detailed}
Assume that $\vX$ is XOR model with $\vx_i=\vz_i+\vm_i,$ where $\vz_i\sim \mathcal N(\vzero,d^{-1}\mbi I_d)$ and $\|\vm_i\|=r/\sqrt{d}$ for $i\in[n]$ where \(r=O(1)\).  Let $s_i:=\|\vx_i\|,$ and $\vu_i:=\frac{\vx_i}{s_i},$
and define
$
        \widetilde{\vq}_i
        :=
        s_i^2\,
        \mathrm{svec}\!\left(
        \vu_i\vu_i^\top-\frac1d\mbi I_d
        \right).
$
Let $ \widetilde{\mbi Q}
        :=
        [\widetilde{\vq}_1,\ldots,\widetilde{\vq}_n]
        \in\mathbb R^{p\times n},$ and $  \widetilde{\mbi H}
        :=
        \widetilde{\mbi Q}^{\top}\widetilde{\mbi Q}.$
If $ \frac np\to \gamma\in(0,\infty),$ 
then $ \|\widetilde{\mbi H}\|=O_{\mathbb P}(1).$
\end{lemma}

\begin{proof} 
Since
\[
        s_i^2\left(
        \vu_i\vu_i^\top-\frac1d\mbi I_d
        \right)
        =
        \vx_i\vx_i^\top-\frac{\|\vx_i\|^2}{d}\mbi I_d,
\]
we have the exact identity
\[
        \widetilde{\vq}_i
        =
        \mathrm{svec}\!\left(
        \vx_i\vx_i^\top-\frac{\|\vx_i\|^2}{d}\mbi I_d
        \right).
\]
Using \(\vx_i=\vz_i+\vm_i\), expand
\[
\begin{aligned}
        \vx_i\vx_i^\top-\frac{\|\vx_i\|^2}{d}\mbi I_d
        &=
        \left(
        \vz_i\vz_i^\top-\frac{\|\vz_i\|^2}{d}\mbi I_d
        \right)                                                     
        +
        \left(
        \vz_i\vm_i^\top+\vm_i\vz_i^\top
        -
        \frac{2\vz_i^\top\vm_i}{d}\mbi I_d
        \right)                                                     
         +
        \left(
        \vm_i\vm_i^\top-\frac{\|\vm_i\|^2}{d}\mbi I_d
        \right).
\end{aligned}
\]
Therefore we have decomposed \(\widetilde{\mbi Q}\) as $ \widetilde{\mbi Q}
        =
        \mbi A+\mbi B+\mbi C,$ 
where the columns of \(\mbi A\), \(\mbi B\), and \(\mbi C\) are, respectively,
\begin{align}
        \va_i
        &=
        \mathrm{svec}\!\left(
        \vz_i\vz_i^\top-\frac{\|\vz_i\|^2}{d}\mbi I_d
        \right),
        \\
        \vb_i
        &=
        \mathrm{svec}\!\left(
        \vz_i\vm_i^\top+\vm_i\vz_i^\top
        -
        \frac{2\vz_i^\top\vm_i}{d}\mbi I_d
        \right),
        \\
        \vc_i
        &=
        \mathrm{svec}\!\left(
        \vm_i\vm_i^\top-\frac{\|\vm_i\|^2}{d}\mbi I_d
        \right).
\end{align}
Thus, to prove the lemma, it suffices to prove that
We prove $\|\mbi A\|=O_{\mathbb P}(1),$ $\|\mbi B\|=O_{\mathbb P}(1),$ and $ \|\mbi C\|=O(1).$

\paragraph{Step 1: the pure quadratic Gaussian part $\vA$.}
Write $\vz_i=\rho_i\vg_i,$ where  $ \rho_i:=\|\vz_i\|$ and $ \vg_i:=\frac{\vz_i}{\|\vz_i\|}.$
Then \(\vg_i\) is uniform on the sphere \(\mathbb S^{d-1}\), independent of
\(\rho_i\), and $\rho_i^2\sim \frac1d\chi_d^2.$
Moreover,
\[
        \va_i
        =
        \rho_i^2
        \mathrm{svec}\!\left(
        \vg_i\vg_i^\top-\frac1d\mbi I_d
        \right).
\]
Let
$\vphi_i
        :=
        \mathrm{svec}\!\left(
        \vg_i\vg_i^\top-\frac1d\mbi I_d
        \right),$
and $\mbi\Phi:=[\vphi_1,\ldots,\vphi_n].$
Then $ \mbi A=\mbi\Phi \mbi D_\rho,$ where $\mbi D_\rho:=\diag(\rho_1^2,\ldots,\rho_n^2).$
By the standard chi-square concentration bound and \(n=O(d^2)\), we have that
$ \max_{1\le i\le n}\rho_i^2=O_{\mathbb P}(1).$
Therefore
\[
        \|\mbi A\|
        \le
        \|\mbi\Phi\|\,\|\mbi D_\rho\|
        =
        O_{\mathbb P}(1)\|\mbi\Phi\|,
\]
so it remains to control \(\|\mbi\Phi\|\). We use the following Lemma~\ref{lem:degree-two-spherical-gram-bound} below to get \(\|\mbi\Phi\|=O_{\mathbb P}(1)\). 

\paragraph{Step 2: the linear cross part \(\vB\).}
For \(\vm\in\mathbb R^d\), define the linear map $ \mathcal{L}_{\vm}:\mathbb R^d\to \mathbb S_0^d$
by
\[
        \mathcal{L}_{\vm}(\vv)
        :=
        \vm \vv^\top + \vv \vm^\top
        -
        \frac{2\vv^\top \vm}{d}\mbi I_d.
\] Here \(\mathbb S_0^d\) denotes the space of \(d\times d\) symmetric matrices with zero trace.
Then $ \vb_i=\mathrm{svec}\bigl(\mathcal{L}_{\vm_i}(\vz_i)\bigr).$ Moreover,
\[
        \|\mathcal{L}_{\vm}(\vv)\|_F
        \le
        2\|\vm\|\|\vv\|
        +
        \frac{2|\vv^\top \vm|}{\sqrt d}
        \le
        4\|\vm\|\|\vv\|.
\]
Hence $ \|\mathcal{L}_{\vm}\|_{\op}\le 4\|\vm\|.$
In the XOR model, $ \vm_i=
        \eta_i\frac r{\sqrt d}\vu_{\kappa_i},$  where $ \eta_i\in\{\pm1\},$ and $\kappa_i\in\{1,2\}.$
Let \(I_a:=\{i:\kappa_i=a\}\).  On \(I_a\),
$
        \vb_i
        =
        \eta_i r
        \mathrm{svec}\bigl(\mathcal{L}_{\vu_a}(\vz_i)\bigr)/{\sqrt d}.
$
Let $\mbi Z_a:=[\vz_i]_{i\in I_a}\in\mathbb R^{d\times |I_a|}.$
Since the entries of \(\mbi Z_a\) are Gaussian with variance \(1/d\),
standard Gaussian matrix norm bounds give
\[
        \|\mbi Z_a\|
        =
        O_{\mathbb P}\left(
        1+\sqrt{\frac{|I_a|}{d}}
        \right)
        =
        O_{\mathbb P}(\sqrt d),
\]
because \(|I_a|\asymp n\asymp d^2\).  Therefore
\[
        \|\mbi B_{I_a}\|
        \le
        \frac r{\sqrt d}
        \|\mathcal L_{\vu_a}\|_{\op}
        \|\mbi Z_a\|
        =
        O_{\mathbb P}(1).
\]
Since there are only two axes \(a=1,2\), we can conclude that
$
        \|\mbi B\|
        =
        O_{\mathbb P}(1).
$

\paragraph{Step 3: the deterministic mean part $\mbi C$.}
For each \(i\in[n]\), we have $ \|\vm_i\|=\frac r{\sqrt d},$
and hence
\[
\begin{aligned}
        \|\vc_i\|_2
        &=
        \left\|
        \vm_i\vm_i^\top-\frac{\|\vm_i\|^2}{d}\mbi I_d
        \right\|_F                                                \\
        &\le
        \|\vm_i\vm_i^\top\|_F
        +
        \frac{\|\vm_i\|^2}{d}\|\mbi I_d\|_F                     
        =
        \|\vm_i\|^2+\frac{\|\vm_i\|^2}{\sqrt d}
        \le
        \frac{Cr^2}{d}.
\end{aligned}
\]
Therefore
$
        \|\mbi C\|
        \le
        \|\mbi C\|_F
        =
        O(1).
$
\end{proof}
\begin{lemma}[Uniform-spherical harmonic Gram bound]
\label{lem:uniform-spherical-harmonic-gram}
Let \(\vg_1,\ldots,\vg_n\iid{\rm Unif}(\mathbb S^{d-1})\), and assume
\(n\asymp d^2\).  For fixed \(m\ge2\), let \(q_m^{(d)}\) be the normalized
Gegenbauer polynomial of degree \(m\) defined in \citet{lu2022equivalence}, and define
\begin{equation}\label{eq:degree-m-spherical-gram-matrix}
        (\vA_m)_{ij}
        :=
        \frac1{\sqrt n}
        q_m^{(d)}(\sqrt d\,\vg_i^\top\vg_j)\mathbf 1_{\{i\ne j\}}.
\end{equation}
Then
$
        \|\vA_m\|_{\op}
        \prec
        d^{\max\{2-m,0\}/2}.
$
Equivalently, if \(\psi_m^{(d)}\) denotes the degree-\(m\) zonal spherical
harmonic kernel normalized so that, for fixed \(m\),
\begin{equation}\label{eq:degree-m-spherical-harmonic-kernel-normalization}
        \psi_m^{(d)}(t)
        =
        c_{m,d}d^{-m/2}q_m^{(d)}(\sqrt d\,t),
        \qquad
        c_{m,d}\asymp 1,
\end{equation}
then
\begin{equation}\label{eq:harmonic-kernel-concentration}
         \left\|
        \bigl(\psi_m^{(d)}(\vg_i^\top\vg_j)\bigr)_{i,j=1}^n
        -
        \psi_m^{(d)}(1)\vI_n
        \right\|
        =
        O_{\P}\!\left(
        \frac{\sqrt n}{d^{m/2}}
        d^{\max\{2-m,0\}/2}
        \right).
\end{equation}  
Thus, for \(n\asymp d^2\),
\begin{align}
        \left\|
        \bigl(\psi_2^{(d)}(\vg_i^\top\vg_j)\bigr)_{i,j}
        -
        \psi_2^{(d)}(1)\vI_n
        \right\|
        &=
        O_{\P}(1),
        \\
        \left\|
        \bigl(\psi_3^{(d)}(\vg_i^\top\vg_j)\bigr)_{i,j}
        -
        \psi_3^{(d)}(1)\vI_n
        \right\|
        &=
        O_{\P}(d^{-1/2})
        =
        o_{\P}(1),
        \\
        \left\|
        \bigl(\psi_4^{(d)}(\vg_i^\top\vg_j)\bigr)_{i,j}
        -
        \psi_4^{(d)}(1)\vI_n
        \right\|
        &=
        O_{\P}(d^{-1})
        =
        o_{\P}(1).
\end{align}
\end{lemma}

\begin{proof}
This is a direct consequence of Proposition~8 of
\citet{lu2022equivalence}.  In their notation, if
\(n=\alpha d^\ell+o(d^\ell)\), then
\[
        \|\vA_m\|_{\op}
        \prec
        d^{\max\{\ell-m,0\}/2}.
\]
Taking \(\ell=2\) gives the first result of this lemma.
For the equivalent harmonic-kernel form, note that 
\[
        \bigl(\psi_m^{(d)}(\vg_i^\top\vg_j)\mathbf 1_{\{i\ne j\}}\bigr)_{i,j}=c_{m,d}d^{-m/2}\sqrt n\, \vA_m.
\]
The diagonal entries are \(\psi_m^{(d)}(1)\). Notice that $c_{m,d}=\Theta(1)$. Therefore, \eqref{eq:harmonic-kernel-concentration} is due to
\[
        \bigl(\psi_m^{(d)}(\vg_i^\top\vg_j)\bigr)_{i,j}
        -
        \psi_m^{(d)}(1)\vI_n
        =
        c_{m,d}d^{-m/2}\sqrt n\,A_m.
\]
\end{proof}

\begin{lemma}[Degree-two spherical Gram bound]
\label{lem:degree-two-spherical-gram-bound}
Let \(\vg_1,\ldots,\vg_n\iid{\rm Unif}(\mathbb S^{d-1})\), and define
\[
        \vphi_i
        :=
        \mathrm{svec}\!\left(
        \vg_i\vg_i^\top-\frac1d\vI_d
        \right),
        \qquad
        \vPhi:=[\vphi_1,\ldots,\vphi_n].
\]
If \(n/d^2=O(1)\), then  $\|\vPhi\|=O_{\P}(1).$
\end{lemma}

\begin{proof}
For \(i,j\in[n]\),
\[
        (\vPhi^\top\vPhi)_{ij}
        =
        \left\langle
        \vg_i\vg_i^\top-\frac1d\vI_d,\,
        \vg_j\vg_j^\top-\frac1d\vI_d
        \right\rangle_F
        =
        (\vg_i^\top\vg_j)^2-\frac1d.
\]
Thus $ \vPhi^\top\vPhi
        =
        \left(1-\frac1d\right)\vI_n
        +
        \vK_2,$
where for \(i\neq j\in[n]\),
$
        (\vK_2)_{ij}
        =
        (\vg_i^\top\vg_j)^2-\frac1d.
$
From \citet{lu2022equivalence}, we know that the normalized Gegenbauer polynomial \(q_2^{(d)}\) is
\begin{equation}\label{eq:q2-normalization}
        q_2^{(d)}(x)
        =
        \frac1{\sqrt2}
        \sqrt{\frac{d+2}{d-1}}\,(x^2-1).
\end{equation}
Therefore, for \(i\ne j\),
\[
        (\vg_i^\top\vg_j)^2-\frac1d
        =
        \frac{\sqrt2}{d}
        \sqrt{\frac{d-1}{d+2}}\,
        q_2^{(d)}(\sqrt d\,\vg_i^\top\vg_j).
\]
Recall the definition of \(\vA_2\) in \eqref{eq:degree-m-spherical-gram-matrix}. Then by Lemma~\ref{lem:uniform-spherical-harmonic-gram}, we have that
\[
        \vK_2
        =
        \frac{\sqrt{2n}}{d}
        \sqrt{\frac{d-1}{d+2}}\,
        \vA_2,
\]
and $\|\vA_2\|_{\op}=O_{\P}(1).$
Since \(\sqrt n/d=O(1)\), it follows that $ \|\vK_2\|=O_{\P}(1).$
Consequently, $\|\vPhi^\top\vPhi\| = O_{\P}(1).$
\end{proof}
 
\subsubsection{The Cubic and Quartic Terms}
\label{sec:k3-k4-terms}

We next analyze the terms $\vT_k=\vD_k\vG^{\odot k}\vD_k,$ for $k=3,4$, from \eqref{eq:Mehler-use-Psigma}. Let
\begin{equation}
\label{eq:zeta-k-i-def}
        \vS:=\diag(s_1,\ldots,s_n),
        \qquad
        \vR:=\vS^{-1}\vG\vS^{-1},
        \qquad
        \vZ_k:=\diag(\zeta_k^{(1)},\ldots,\zeta_k^{(n)}).
\end{equation}
Thus \(\vR_{ij}=\langle \vx_i/s_i,\vx_j/s_j\rangle\) is the normalized
Gram matrix. 
Therefore
$
        \vT_k=\vZ_k\vR^{\odot k}\vZ_k.
$
We shall use the following fixed-degree harmonic Gram estimate.  It is a
standard consequence of the trace-moment method for random inner-product kernel
matrices; see, for instance, \citet{cheng2013spectrum},
\citet{lu2022equivalence}, and the quadratic-regime estimates in
\citet{pandit2024universalitykernelrandommatrices}.  We record the statement in
the form needed here.

\begin{lemma}[Harmonic Gram bounds for normalized XOR directions]
\label{lem:harmonic-gram-xor}
Recall the normalized Gram matrix \(\vR\) defined in \eqref{eq:zeta-k-i-def}.
For fixed \(m\ge2\), let \(\psi_m^{(d)}\) denote the degree-\(m\) zonal
spherical harmonic kernel on \(\mathbb S^{d-1}\), normalized so that
\[
        \psi_m^{(d)}(u^\top v)
        =
        c_{m,d}d^{-m/2}q_m^{(d)}(\sqrt d\,u^\top v),
        \qquad
        c_{m,d}\asymp 1,
\]
where \(q_m^{(d)}\) is the \(L^2\)-normalized Gegenbauer polynomial used in
\citet{lu2022equivalence}.  Define $ \vPsi_m
        :=
        \bigl(\psi_m^{(d)}(\vR_{ij})\bigr)_{i,j=1}^n$ on XOR dataset. 
If \(n\asymp d^2\), then
for each fixed \(m\ge2\),
\[
        \left\|\vPsi_m-\psi_m^{(d)}(1)\vI_n\right\|
        =
        O_{\P}\left(\sqrt{\frac{n}{d^m}}+\frac{n}{d^m}\right).
\]
\end{lemma}

\begin{proof}
The proof has two parts.  First we recall Lemma~\ref{lem:uniform-spherical-harmonic-gram} for standard uniform-sphere distributions.
Then we show that the same estimate remains valid for the normalized XOR
directions.

Let $\vtheta_1,\ldots,\vtheta_n \iid {\rm Unif}(\mathbb S^{d-1})$
and recall $\vA_m$ defined in \eqref{eq:degree-m-spherical-gram-matrix}. Define the corresponding harmonic matrix by
\[
        \vK_m^{\rm sph}
        :=\bigl(\psi_m^{(d)}(\vtheta_i^\top\vtheta_j)\bigr)_{i,j=1}^n
        -
        \psi_m^{(d)}(1)\vI_n=
        \bigl(
        \psi_m^{(d)}(\vtheta_i^\top\vtheta_j)\mathbf 1_{\{i\ne j\}}
        \bigr)_{i,j=1}^n.
\]
Then Lemma~\ref{lem:uniform-spherical-harmonic-gram} implies that $ \|\vK_3^{\rm sph}\|=o_{\P}(1),$ and $\|\vK_4^{\rm sph}\|=o_{\P}(1).$

For the XOR model, conditional on the cluster of \(i\), we have $\vx_i=\frac{\vg_i+\vmu_i}{\sqrt d},$ where $\vg_i\sim \mathcal N(\vzero,\vI_d)$ and $\vmu_i\in\{\pm r\vu_1,\pm r\vu_2\},$ and therefore $ \vu_i=\frac{\vg_i+\vmu_i}{\|\vg_i+\vmu_i\|_2}$ defined in Lemma~\ref{lem:centered-quadratic-lift-bound-detailed}.
Thus \(\vu_i\) is not uniform on the sphere.  However, we can claim that its angular law is
absolutely continuous with respect to the uniform measure, with uniformly
bounded \(L^p\) density for every fixed \(p\). We now verify this claim.  Let \(\sigma_{d-1}\) denote the uniform probability measure on
\(\mathbb S^{d-1}\).  For fixed \(\vmu\) with \(\|\vmu\|\le r\), let
\[
        \mathbb P_{\vmu}
        :=
        {\rm Law}\left(\frac{\vg+\vmu}{\|\vg+\vmu\|_2}\right),
        \qquad
        \vg\sim\mathcal N(\vzero,\vI_d).
\]
Let $L_{\vmu}:=\frac{d\mathbb{P}_{\vmu}}{d\sigma_{d-1}}.$
Using polar coordinates for a standard Gaussian vector, $ \vg=\rho\vtheta,$ where $\vtheta\sim\sigma_{d-1},$ and $\rho\ge0$ is independent of $\vtheta$ with $\rho^2\sim\chi_d^2.$ One can verify that the density \(L_{\vmu}\) is given by
\[
        L_{\vmu}(\vtheta)
        =
        \E_\rho
        \exp\left(
        \rho\,\vmu^\top\vtheta-\frac{\|\vmu\|^2}{2}
        \right).
\] 
We prove that for every fixed \(p<\infty\),
\begin{equation}
\label{eq:Lp-angular-density-bound}
        \sup_{d\ge1}
        \sup_{\|\vmu\|\le r}
        \|L_{\vmu}\|_{L^p(\sigma_{d-1})}
        \le C_{p,r}<\infty .
\end{equation}
To prove this, Jensen's inequality gives
\[
        L_{\vmu}(\vtheta)^p
        \le
        \E_\rho
        \exp\left(
        p\rho\,\vmu^\top\vtheta-\frac{p\|\vmu\|^2}{2}
        \right).
\]
Integrating over \(\theta\) and using the standard spherical sub-Gaussian
bound
\[
        \E_{\vtheta\sim\sigma_{d-1}}
        \exp(t\,\vmu^\top\vtheta)
        \le
        \exp\left(\frac{C t^2\|\vmu\|^2}{d}\right),
\]
we get
\[
        \|L_{\vmu}\|_{L^p(\sigma_{d-1})}^p
        \le
        e^{-p\|\vmu\|^2/2}
        \E_\rho
        \exp\left(
        \frac{C p^2\rho^2\|\vmu\|^2}{d}
        \right).
\]
Since \(\rho^2\sim\chi_d^2\), the last expectation is bounded uniformly in
\(d\) for every \(p\) and \(\|\vmu\|\le r\).  This proves
\eqref{eq:Lp-angular-density-bound}.

Now fix \(m\ge 2\).  Let $ \vK_m^{\rm xor}
        :=
        \bigl(
        \psi_m^{(d)}(\vu_i^\top\vu_j)\mathbf 1_{\{i\ne j\}}
        \bigr)_{i,j=1}^n .$
For an even integer \(2q\),
\[
        \E_{\rm xor}\Tr\bigl((\vK_m^{\rm xor})^{2q}\bigr)
        =
        \sum_{i_1,\ldots,i_{2q}}
        \E_{\rm xor}
        \prod_{\ell=1}^{2q}
        \psi_m^{(d)}(\vu_{i_\ell}^\top\vu_{i_{\ell+1}})
        \mathbf 1_{\{i_\ell\ne i_{\ell+1}\}},
\]
where \(i_{2q+1}=i_1\).  For a fixed index sequence
\((i_1,\ldots,i_{2q})\), let \(V(i)\) be the set of distinct vertices appearing
in the sequence.  Conditional on the cluster assignments, the variables
\((\vu_j)_{j\in V(i)}\) are independent shifted angular variables.  Hence their
joint law has density $\prod_{j\in V(i)}L_{\vmu_j}(\vtheta_j)$
with respect to the corresponding product of uniform spherical measures.

By Hölder's inequality and \eqref{eq:Lp-angular-density-bound}, for every fixed
\(q\) there is a constant \(C_{q,r}\) such that
\[
\begin{aligned}
\left|
\E_{\rm xor}
        \prod_{\ell=1}^{2q}
        \psi_m^{(d)}(\vu_{i_\ell}^\top\vu_{i_{\ell+1}})
        \mathbf 1_{\{i_\ell\ne i_{\ell+1}\}}
\right| 
 \le
C_{q,r}
\left(
\E_{\rm sph}
\left|
        \prod_{\ell=1}^{2q}
        \psi_m^{(d)}(\vtheta_{i_\ell}^\top\vtheta_{i_{\ell+1}})
        \mathbf 1_{\{i_\ell\ne i_{\ell+1}\}}
\right|^{2}
\right)^{1/2}.
\end{aligned}
\]
The right-hand side is controlled by the same moment calculation used for the
uniform-sphere Gegenbauer matrices in \cite[Appendix~E]{lu2022equivalence};
the only change is that the moment order is doubled, which only changes the
constant because \(q\) is fixed.  Therefore, for every fixed \(q\),
\begin{equation}
\label{eq:xor-harmonic-trace-moment}
        \E_{\rm xor}\Tr\bigl((\vK_m^{\rm xor})^{2q}\bigr)
        \le
        C_{q,r}
        n
        \left(
        \sqrt{\frac{n}{d^m}}+\frac{n}{d^m}
        \right)^{2q}
        +
        C_{q,r}.
\end{equation}
Consequently, by Markov's inequality,
\[
        \|\vK_m^{\rm xor}\|
        =
        O_{\P}\left(
        \sqrt{\frac{n}{d^m}}+\frac{n}{d^m}
        \right).
\]
Since $\vPsi_m-\psi_m^{(d)}(1)\vI_n
        =
        \vK_m^{\rm xor},$
the claimed estimates follow.
\end{proof}

We shall also use the following elementary bound on the normalized Gram matrix.

\begin{lemma}[Operator norm of the normalized Gram matrix]
\label{lem:R-op-bound}
Under the XOR model and \(n\asymp d^2\),
\begin{equation}\label{eq:R-op-bound}
        \|\vR\|=O_{\P}\left(\frac nd\right)=O_{\P}(d).
\end{equation}
Moreover, deterministically, $\|\vR\|\ge \frac nd.$
Thus \(\|\vR\|=\Theta_{\P}(d)\).
\end{lemma}

\begin{proof}
Let $\vU:=[\vu_1,\ldots,\vu_n]\in\mathbb R^{d\times n}.$
Then \(\vR=\vU^\top\vU\), so \(\rank{\vR}\le d\) and $\Tr \vR=\sum_{i=1}^n \|\vu_i\|_2^2=n.$
Therefore
\[
        \|\vR\|\ge \frac{\Tr\vR}{\rank{\vR}}\ge \frac nd .
\]

For the upper bound, write $ \vR=\vS^{-1}\vX^\top\vX\vS^{-1}.$
By Lemma~\ref{lem:xor_orthonormal}, there is a constant \(c>0\) such that $\min_{i\in[n]}s_i\ge c$ and $\|\vX\|\le c\sqrt{n/d}$,
with probability tending to one. Hence we conclude \eqref{eq:R-op-bound}.
\end{proof}

\begin{lemma}[The cubic term]
\label{lem:k3-term}
Recall $a_3:=\E[\sigma(\xi)h_3(\xi)]$. Then
\begin{equation}
\label{eq:k3-term-expansion}
        \vP_\sigma\vT_3\vP_\sigma
        =
        a_3^2\vP_\sigma
        +
        o_{\P}(1).
\end{equation} 
\end{lemma}

\begin{proof}
The degree-three zonal spherical harmonic decomposition is
\begin{equation}
\label{eq:cubic-spherical-decomposition}
        t^3=\frac{3}{d+2}t+\psi_3^{(d)}(t),
\end{equation}
where $\psi_3^{(d)}(t):=t^3-\frac{3}{d+2}t$ defined in \eqref{eq:degree-m-spherical-harmonic-kernel-normalization}.
Applying \eqref{eq:cubic-spherical-decomposition} entrywise to \(\vR\), we get
\[
        \vR^{\odot3}
        =
        \frac{3}{d+2}\vR+\vPsi_3,
        \qquad
        (\vPsi_3)_{ij}:=\psi_3^{(d)}(\vR_{ij}).
\]
By Lemma~\ref{lem:harmonic-gram-xor},
$
        \vPsi_3
        =
        \psi_3^{(d)}(1)\vI_n+\vE_3,
$
where 
$
        \|\vE_3\|=o_{\P}(1).
$
Since $\psi_3^{(d)}(1)=1+O(d^{-1}),$
and \(\|\vZ_3-a_3\vI_n\|=o_{\P}(1)\) from Lemma~\ref{lem:xor_orthonormal}, we have
\[
        \vP_\sigma\vZ_3\vPsi_3\vZ_3\vP_\sigma
        =
        \psi_3^{(d)}(1)\vP_\sigma\vZ_3^2\vP_\sigma
        +
        \vP_\sigma\vZ_3\vE_3\vZ_3\vP_\sigma
        =
        a_3^2\vP_\sigma+o_{\P}(1).
\]
Therefore,
\[
        \vP_\sigma\vT_3\vP_\sigma
        =
        \vP_\sigma\vZ_3\vR^{\odot3}\vZ_3\vP_\sigma
        =
        a_3^2\vP_\sigma
        +
        \frac{3}{d+2}
        \vP_\sigma\vZ_3\vR\vZ_3\vP_\sigma
        +
        o_{\P}(1).
\]
Finally, \eqref{eq:k3-term-expansion} is derived from definition of $\vP_\sigma$ in \eqref{eq:P-sigma-plus-bbp} and 
$
        \vZ_3\vR\vZ_3
        =
        \vZ_3\vS^{-1}\vG\vS^{-1}\vZ_3
        =
        \vB_{3,\sigma}\vG\vB_{3,\sigma}.
$
\end{proof}

\begin{lemma}[The quartic term]
\label{lem:k4-term}
Recall $a_4:=\E[\sigma(\xi)h_4(\xi)]$.
Then 
$
       \vP_\sigma\vT_4\vP_\sigma
        =
        a_4^2\vP_\sigma+o_{\P}(1).
$
\end{lemma}

\begin{proof}
The degree-four zonal spherical harmonic decomposition is
\begin{equation}
\label{eq:quartic-spherical-decomposition}
        t^4
        =
        \frac{3}{d(d+2)}
        +
        \frac{6}{d+4}\left(t^2-\frac1d\right)
        +
        \psi_4^{(d)}(t),
\end{equation}
where $\psi_4^{(d)}(t)$ is defined in \eqref{eq:degree-m-spherical-harmonic-kernel-normalization}. 
Applying \eqref{eq:quartic-spherical-decomposition} entrywise to \(\vR\), define
$\vQ_2:=\vR^{\odot2}-\frac1d\vone\vone^\top$ and $(\vPsi_4)_{ij}:=\psi_4^{(d)}(\vR_{ij}).$
Then
\[
        \vR^{\odot4}
        =
        \frac{3}{d(d+2)}\vone\vone^\top
        +
        \frac{6}{d+4}\vQ_2
        +
        \vPsi_4.
\]
Hence
$
        \vP_\sigma\vT_4\vP_\sigma
        =
        \vI_1+\vI_2+\vI_3,
$
where $\vI_1
        :=
        \frac{3}{d(d+2)}
        \vP_\sigma\vZ_4\vone\vone^\top\vZ_4\vP_\sigma,$ $ \vI_2
        :=
        \frac{6}{d+4}
        \vP_\sigma\vZ_4\vQ_2\vZ_4\vP_\sigma,$ and $ \vI_3
        :=
        \vP_\sigma\vZ_4\vPsi_4\vZ_4\vP_\sigma.$
We bound these three terms separately.

First, since \(\vP_\sigma\vone=0\),
$
\vP_\sigma\vZ_4\vone
        =
        \vP_\sigma(\vZ_4-a_4\vI_n)\vone.
$
By Lemma~\ref{lem:xor_orthonormal}, we have
$
        \|\vZ_4-a_4\vI_n\|
        \prec \sqrt{\frac{\log n}{d}} 
$
implying
\[
        \|\vP_\sigma\vZ_4\vone\|^2
        \le
        \|(\vZ_4-a_4\vI_n)\vone\|^2
        \prec
        n\frac{\log n}{d}.
\]
Then we can conclude that $\vI_1=o_{\P}(1)$ since \(n\asymp d^2\).

Second, by Lemma~\ref{lem:harmonic-gram-xor} with $m=2$ and second-oder Gegenbauer polynomial formula in \eqref{eq:q2-normalization}, we know that $\|\vQ_2\|=O_{\P}(1).$
Also \(\|\vZ_4\|=O_{\P}(1)\). Therefore
$
        \|\vI_2\|
        \le
        \frac{C}{d}\|\vQ_2\|
        =
        o_{\P}(1).
$

Third, again by Lemma~\ref{lem:harmonic-gram-xor} with $m=4$, we have
$
        \vPsi_4=\psi_4^{(d)}(1)\vI_n+\vE_4,
$ where
$
        \|\vE_4\|=o_{\P}(1).
$
Moreover,
\[
        \psi_4^{(d)}(1)
        =
        1-\frac{6}{d+4}
        +
        \frac{3}{(d+2)(d+4)}
        =
        1+O(d^{-1}).
\]
Therefore $\vI_3
        =
        \psi_4^{(d)}(1)\vP_\sigma\vZ_4^2\vP_\sigma
        +
        \vP_\sigma\vZ_4\vE_4\vZ_4\vP_\sigma.$
The second term is \(o_{\P}(1)\).  Since
$ \|\vZ_4-a_4\vI_n\|=o_{\P}(1),$
we also have
$
        \vP_\sigma\vZ_4^2\vP_\sigma
        =
        a_4^2\vP_\sigma+o_{\P}(1).
$
Consequently, $ \vI_3
        =
        a_4^2\vP_\sigma+o_{\P}(1).$

Combining the estimates for \(\vI_1,\vI_2,\vI_3\) proves Lemma~\ref{lem:k4-term}.
\end{proof}

\begin{lemma}[Population CK after projection]
\label{lem:population-QE-Psigma}
Under the assumptions of Theorem~\ref{thm:quadratic},
\begin{equation}\label{eq:population-QE-Psigma}
        \left\|
        \mbi{P}_\sigma\mbi{\Phi}(\mbi{X})\mbi{P}_\sigma
        -\left(\alpha_0\mbi{P}_\sigma+
        \alpha_2\mbi{P}_\sigma\mbi{H}\mbi{P}_\sigma\right)
        \right\|\xrightarrow{\P}0.
\end{equation}
\end{lemma}

\begin{proof}
In \eqref{eq:Mehler-use-Psigma}, the $k=0$ term equals $\mbi{\mu}_\sigma\mbi{\mu}_\sigma^\top$
and the $k=1$ term equals $\mbi{A}_\sigma\mbi{G}\mbi{A}_\sigma$ by
Lemma~\ref{lem:k012-identities}; both are killed
by $\mbi{P}_\sigma$.  Following the proof of Lemma~5.2 of \citet{wang2021deformed}, by Lemma~\ref{lem:xor_orthonormal}, we can derive that with high probability
\begin{equation}\label{eq:Phi(X)_concentration}
         \left\|
        \mbi{P}_\sigma\mbi{\Phi}(\mbi{X})\mbi{P}_\sigma
        -\left((\alpha_0-a_3^2-a_4^2)\mbi{P}_\sigma+
        \mbi{P}_\sigma\Big(\sum_{k=2}^4\vT_k\Big)\mbi{P}_\sigma\right)
         \right\|\lesssim \sqrt{\frac{\log n}{d}}.
\end{equation}
Then by applying Lemmas~\ref{lem:k2-term},~\ref{lem:k3-term} and \ref{lem:k4-term} for $\vT_2,\vT_3,\vT_4$, respectively, we can derive \eqref{eq:population-QE-Psigma}.
\end{proof}


\subsection{BBP Transition for the Projected Population CK}
\label{subsec:bbp-projected-population-ck}

\begin{lemma}[Population CK reduction under the projection]
\label{lem:population-CK-BBP-reduction-plus}
Under the assumptions of Theorem~\ref{thm:quadratic}, with $\alpha_2=a_2^2=\frac{c_\sigma^2}{2}$ and 
$
        \alpha_0=\sum_{k\ge3}a_k^2
        =
        1-b_\sigma^2-\frac{c_\sigma^2}{2},
$
we have
\begin{equation}
\label{eq:Sigma-plus-H0-reduction}
        \left\|
        \mbi{P}_\sigma\vPhi(\vX)\mbi{P}_\sigma
        -
        \left(
        \alpha_0\mbi{P}_\sigma
        +
        \alpha_2\mbi{P}_\sigma \vH_0\mbi{P}_\sigma
        \right)
        \right\|
        \xrightarrow{\P}0,
\end{equation}
where $ \vH_0=\vY_0^\top \vY_0$ and $\vY_0$ is defined in \eqref{eq:def_Y_0}.
\end{lemma}
\begin{proof}
Lemma~\ref{lem:population-QE-Psigma} shows that it suffices to control the difference between $\vP_\sigma\vH\vP_\sigma$ and $\vP_\sigma\vH_0\vP_\sigma$. We split the proof into four steps. Notice that $\vH=\vQ^\top \vQ$ where $\vQ$ is defined by \eqref{eq:def_q_i_Q}. For each data point $\vx_i=\vz_i+\vm_i,$ since
$
        \vx_i\vx_i^\top
        =
        \vz_i\vz_i^\top
        +
        \vz_i\vm_i^\top+\vm_i\vz_i^\top
        +
        \vm_i\vm_i^\top,
$
we have that
\begin{equation}
        \vq_i
        =
        \mathrm{svec}\left(\frac1d\vI_d\right)
        +
        (\vE_2)_i
        +
        (\vE_1)_i
        +
        \mathrm{svec}(\vm_i\vm_i^\top)
\end{equation}
where $ (\vE_2)_i$ and $ (\vE_2)_i$ are the $i$-th columns of $\vE_2$ and $\vE_1$, respectively, in \eqref{eq:def_E_1_iE_2_i}. Let $\bar{\mbi{q}}:=\mathrm{svec}(\frac1d\mbi{I}_d+\frac{r^2}{2d}(\mbi{u}_1\mbi{u}_1^\top+\mbi{u}_2\mbi{u}_2^\top))$.
Hence we can decompose \(\vQ\) into 
\begin{equation}
        \vQ
        =
        \bar\vq \mathbf{1}^\top
        +
        \vE_2
        +
        \vE_1
        +
        \frac{\delta_d}{\sqrt p}\vs \vy^\top,
\end{equation}
where $\vs$ is defined in \eqref{eq:def_s_v}, $\delta_d=\delta\sqrt{1+\frac1d}$, and $\delta=\frac{r^2}{2}.$
Thus
$
        \left\|
        \frac{\delta_d-\delta}{\sqrt p}\vs \vy^\top
        \right\|
        =
        |\delta_d-\delta|\sqrt{\frac np}
        =
        o(1).
$

By the definition of $\vP_\sigma$, we have
$
        \vQ\mbi{P}_\sigma
        =
        \vY_0\mbi{P}_\sigma
        +
        o_{\mathbb P}(1).
$
Moreover, we can claim that $ \|\vY_0\mbi{P}_\sigma\|=O_{\mathbb P}(1).$
The traceless second-chaos component in $\vE_2$ has bounded operator norm in the
quadratic regime, while the radial trace component is \(o(1)\). The signal part has norm
$
        \left\|\frac{\delta}{\sqrt p}s y^\top \mbi{P}_\sigma\right\|
        \le
        \delta\sqrt{\frac np}
        =
        O(1).
$ 
Consequently, with Lemma~\ref{lem:centered-quadratic-lift-bound-detailed},
\[
\begin{aligned}
        \|\mbi{P}_\sigma \vH\mbi{P}_\sigma-\mbi{P}_\sigma \vH_0\mbi{P}_\sigma\|
        &=
        \|(\vQ\mbi{P}_\sigma)^\top(\vQ\mbi{P}_\sigma)
        -
        (\vY_0\mbi{P}_\sigma)^\top(\vY_0\mbi{P}_\sigma)\|=
        o_{\mathbb P}(1)
\end{aligned}
\]
which proves the lemma.
\end{proof}

\begin{proposition}[Projected population CK: quadratic BBP transition]
\label{thm:projected-population-ck-bbp}
Assume the XOR model and Assumption~\ref{assump:sigma}.  Suppose
\(n/p\to\gamma\in(0,\infty)\), \(r\in[0,\infty)\) is fixed, and
\(\alpha_2=c_\sigma^2/2>0\).  Let \(\mbi{P}_\sigma\) be defined by
\eqref{eq:U-sigma-plus-bbp}--\eqref{eq:P-sigma-plus-bbp}.
Let \(\Pi_\Phi^+(I)\) denote the spectral projector of
\(\mbi{P}_\sigma\mbi{\Phi}(\mbi{X})\mbi{P}_\sigma\) associated with the
interval \(I\).

\begin{enumerate}
\item[\rm(a)] The ESD of $\mbi{P}_\sigma\mbi{\Phi}(\mbi{X})\mbi{P}_\sigma$ converges weakly to $\nu_{\rm q}$ in probability where $\nu_{\rm q}$ is defined in \eqref{eq:nuq-def}.

\item[\rm(b)] If \(\ell<1/\sqrt\gamma\), then for every fixed \(\eps>0\),
\begin{equation}
\label{eq:subcritical-label-mass}
    \left\|
    \Pi_\Phi^+\left([
        \alpha_0+\alpha_2(\lambda_+(\gamma)+\eps),\infty
    )\right)\frac{\mbi{y}}{\sqrt n}
    \right\|^2\xrightarrow{\P}0
\end{equation}
where $\lambda_+(\gamma)$ is defined by \eqref{eq:right_edge_MP}. Thus there is no label-aligned outlier above the affine MP edge.

\item[\rm(c)] If \(\ell>1/\sqrt\gamma\), then for every fixed sufficiently
small \(\eps>0\),
\begin{equation}
\label{eq:supercritical-label-mass}
    \left\|
    \Pi_\Phi^+\left([
        \Lambda_y -\eps,
        \Lambda_y +\eps
    ]\right)\frac{\mbi{y}}{\sqrt n}
    \right\|^2
    \xrightarrow{\P}
    \mathrm{Align}(\gamma,\ell).
\end{equation}
If, in addition, this interval contains a single eigenvalue with probability
tending to one, then this eigenvalue \(\lambda_{\rm lab}\) and its unit
eigenvector \(\widehat{\mbi{v}}_{\rm lab}\) satisfy $    \lambda_{\rm lab}
    \xrightarrow{\P}
    \Lambda_y ,$ and 
$
    \left|
    \left\langle
        \widehat{\mbi{v}}_{\rm lab},\frac{\mbi{y}}{\sqrt n}
    \right\rangle
    \right|^2
    \xrightarrow{\P}
    \mathrm{Align}(\gamma,\ell)
$ where $\Lambda_{y}$ and $\mathrm{Align}(\gamma,\ell)$ are defined by \eqref{eq:lambda-out-def} and \eqref{eq:Align-def} respectively.
\end{enumerate}
\end{proposition}

\begin{proof}
Lemma~\ref{lem:population-CK-BBP-reduction-plus} shows that it suffices to study the spectrum of $\vP_\sigma\vH_0\vP_\sigma$. Notice that $\vH_0=\vY_0^\top\vY_0$ and 
\[
    \mbi{Y}_0=\mbi{E}_2+\theta_n\mbi{s}\mbi{v}^\top,
    \qquad \mbi{v}=\mbi{y}/\sqrt n,
    \qquad \theta_n\to\theta=\delta\sqrt\gamma.
\]
The columns of \(\mbi{E}_2\) are independent, centered Gaussian quadratic-chaos
vectors with population covariance
\begin{equation}
    \E[(\mbi{E}_2)_i(\mbi{E}_2)_i^\top]
    =\frac{2}{d^2}\mbi{I}_p
    =\frac{1+o(1)}p\mbi{I}_p.
\end{equation}
The anisotropic MP local law \citet{knowles2017anisotropic,fan2026anisotropic} gives
\begin{align}
    \mbi{v}^\top(\lambda\mbi{I}-\mbi{E}_2^\top\mbi{E}_2)^{-1}\mbi{v}
    &\to -m_\gamma(\lambda),\label{eq:amp-v}\\
    \mbi{s}^\top(\lambda\mbi{I}-\mbi{E}_2\mbi{E}_2^\top)^{-1}\mbi{s}
    &\to -\widetilde m_\gamma(\lambda),\label{eq:amp-s}\\
    \mbi{s}^\top\mbi{E}_2(\lambda\mbi{I}-\mbi{E}_2^\top\mbi{E}_2)^{-1}\mbi{v}
    &\to0\label{eq:amp-cross}
\end{align}
uniformly for
\(\lambda\) in compact subsets of \((\lambda_+(\gamma),\infty)\), where \(m_\gamma\) is the Stieltjes transform of the MP law $\rho^\MP_\gamma$ and
\(\widetilde m_\gamma(\lambda)=(1-
\gamma)/\lambda+\gamma m_\gamma(\lambda)\) is the companion transform. 
Let \(x=\sqrt\lambda>\|\mbi{E}_2\|\).  Linearizing the rectangular matrix and
using the matrix determinant lemma, an outlying singular value of
\(\mbi{Y}_0\) must satisfy
\begin{equation}
    1-\theta^2\lambda m_\gamma(\lambda)\widetilde m_\gamma(\lambda)=0.
\end{equation}
This is the rectangular BBP equation of
\citet{benaych2012singular}.  Solving it gives $\lambda= \lambda_{\mathrm{out}}(\gamma,\ell)$
and the solution is outside $\supp{\rho_\gamma^\MP}$ if and only if $\theta^2>\sqrt\gamma$, i.e., $\ell>\frac1{\sqrt\gamma}.$
Moreover, when \(\ell>1/\sqrt\gamma\), if \(\widehat{\mbi{v}}_0\) is the
right singular vector of \(\mbi{Y}_0\) corresponding to the separated singular
value, then the singular-vector formula of \citet{benaych2012singular} yields
\begin{equation}
\label{eq:H0-overlap}
    \left|\left\langle\widehat{\mbi{v}}_0,\mbi{v}\right\rangle\right|^2
    \xrightarrow{\P}
    \frac{\gamma\ell^2-1}{\gamma\ell(\ell+1)}.
\end{equation}
Equivalently, \(\mbi{H}_0=\mbi{Y}_0^\top\mbi{Y}_0\) has a separated eigenvalue
at \(\lambda_{\rm out}(\gamma,\ell)\) with the overlap in
\eqref{eq:H0-overlap}.  In the subcritical regime, there is no separated
label-aligned eigenvalue above \(\lambda_+(\gamma)\).
Also, with Lemma~\ref{lem:Psigma-preserves-label} and \citet{benaych2012singular}, in the supercritical regime, we can get $\|(\mbi{I}-\mbi{P}_\sigma)\widehat{\mbi{v}}_0\|\xrightarrow{\P}0.$

We now use a deterministic compression lemma.  If a symmetric matrix
\(\mbi{A}\) has an isolated eigenpair \((\lambda_*,\mbi{u})\) and an orthogonal
projector \(\mbi{P}\) satisfies \(\|(\mbi{I}-\mbi{P})\mbi{u}\|\to0\), then the
spectral projector of \(\mbi{PAP}\) in any fixed small interval around
\(\lambda_*\) maps \(\mbi{Pu}/\|\mbi{Pu}\|\) to itself up to \(o(1)\).  This
follows immediately from 
\[
    \left\|(\mbi{PAP}-\lambda_*)\frac{\mbi{P}\mbi{u}}{\|\mbi{P}\mbi{u}\|}\right\|
    \le 2\|\mbi{A}\|\|(\mbi{I}-\mbi{P})\mbi{u}\|+o(1)
\]
and the spectral theorem.  Applying this lemma to
\(\mbi{A}=\mbi{H}_0\), \(\mbi{P}=\mbi{P}_\sigma\), and
\(\mbi{u}=\widehat{\mbi{v}}_0\) shows that the BBP spectral mass of
\(\mbi{P}_\sigma\mbi{H}_0\mbi{P}_\sigma\) near
\(\lambda_{\rm out}(\gamma,\ell)\) has the same label overlap as in
\eqref{eq:H0-overlap}.  Since \(\rank{\mbi{I}-\mbi{P}_\sigma}=o(n)\), the
compression also leaves the MP bulk unchanged.

By Lemma~\ref{lem:population-CK-BBP-reduction-plus}, on \(\operatorname{Im}(\mbi{P}_\sigma)\),
\[
    \mbi{P}_\sigma\mbi{\Phi}(\mbi{X})\mbi{P}_\sigma
    =\alpha_0\mbi{P}_\sigma
     +\alpha_2\mbi{P}_\sigma\mbi{H}_0\mbi{P}_\sigma +o_\P(1).
\]
Since \(\alpha_2>0\), this is an affine spectral map
\(\lambda\mapsto\alpha_0+\alpha_2\lambda\), up to an \(o_\P(1)\) perturbation.
Weyl's inequality gives the eigenvalue convergence, and the spectral-projector
formulation follows from the Davis--Kahan theorem.  The MP bulk is mapped to
$\nu_{\rm q}$.  The subcritical and supercritical
claims \eqref{eq:subcritical-label-mass}--\eqref{eq:supercritical-label-mass}
follow from the corresponding claims for \(\mbi{H}_0\).
\end{proof}

\subsection{BBP Transition for the Projected CK}
\begin{proposition}[Second BBP transition for linear-width projected CK]
\label{prop:second-BBP-linear-width-CK-plus}
Let
$
        \vK_\sigma 
        =
        \frac1N
        \vP_\sigma 
        \sigma(\vW\vX)^\top
        \sigma(\vW\vX)
        \vP_\sigma ,
$
and assume \eqref{eq:quadratic_limit} holds.
Assume that $\ell=\frac{r^4}{4}>\gamma^{-1/2}$. Recall $\Lambda_y$ in \eqref{eq:lambda-out-def},
$z$ and $\varphi$ transforms defined in \eqref{eq:zq-def}.
Then the following hold.

\begin{enumerate}
\item[\rm(a)] The ESD of \(\vK_\sigma \) converges weakly in probability to
$
        \mu_{\rm q}
        =
        \rho_\phi^{\rm MP}\boxtimes\nu_{\rm q}.
$

\item[\rm(b)] If
$
        z'\left(-\frac1{\Lambda_y}\right)>0,
$
then \(\vK_\sigma \) has exactly one label-aligned separated eigenvalue,
and for every sufficiently small deterministic interval \(I_y\) around $\lambda_y$ defined in \eqref{eq:lambda-K-y-def},
which is disjoint from \(\supp{\mu_{\rm q}}\), \(\vK_\sigma \) has exactly one eigenvalue in \(I_y\), and this
eigenvalue converges in probability to \(\lambda_y\).
If \(\widehat{\vv}_y\) is the associated unit eigenvector, then the alignment with the label $\langle\widehat{\vv}_y,\vy\rangle/\sqrt{n}$ satisfies \eqref{eq:final-label-overlap}.
\item[\rm(c)] If either $\ell\le \gamma^{-1/2}$
or
$
        z'\left(-\frac1{\Lambda_y}\right)\le0,
$
then the quadratic label channel produces no separated outlier.
\end{enumerate}
\end{proposition}
\begin{proof}
We condition on \(\vX\) throughout the proof.  Let $r_n:=\rank{\vP_\sigma }.$
By Lemma~\ref{lem:Psigma-preserves-label},
$n-r_n=\rank{\vI_n-\vP_\sigma }$ which is $o(n),$ and $\frac{r_n}{N}\to\phi.$
Let
$
        \vU_+\in\mathbb R^{n\times r_n}
$
be an orthonormal basis of \(\operatorname{Range}(\vP_\sigma )\), so that
$
        \vP_\sigma =\vU_+\vU_+^\top .
$
Define the projected feature row
\[
        \vg_a
        :=
        \vU_+^\top \sigma(\vw_a^\top\vX)^\top
        \in\mathbb R^{r_n},
        \qquad a=1,\ldots,N.
\]
Conditional on \(\vX\), the rows \(\vg_1,\ldots,\vg_N\) are independent and
identically distributed.  Since
$
        \E_{\vw}[\sigma(\vw^\top\vX)^\top\mid \vX]
        =
        \vmu_\sigma
$
and \(\vP_\sigma \vmu_\sigma=\vzero\), we have
$
        \E[\vg_a\mid\vX]=\vzero.
$
Their conditional covariance is
$
        \widetilde\vSigma_\sigma^+
        :=
        \E[\vg_a\vg_a^\top\mid\vX]
        =
        \vU_+^\top\vPhi(\vX)\vU_+ .
$
The nonzero eigenvalues of \(\vP_\sigma \vPhi(\vX)\vP_\sigma \)
are exactly the eigenvalues of \(\widetilde\vSigma_\sigma^+\).
Moreover,
\[
        \vK_\sigma 
        =
        \vU_+
        \left(
        \frac1N\sum_{a=1}^N \vg_a\vg_a^\top
        \right)
        \vU_+^\top.
\]
Therefore the nonzero eigenvalues of \(\vK_\sigma \) are exactly those of the
\(r_n\times r_n\) sample covariance matrix
$
        \widetilde \vK_\sigma 
        :=
        \frac1N\sum_{a=1}^N \vg_a\vg_a^\top.
$
We now verify the inputs of the nonlinear spiked covariance theorem from \cite{wang2024nonlinearspikedcovariancematrices}.

First, by Lemma~\ref{lem:Psigma-preserves-label} and Proposition~\ref{thm:projected-population-ck-bbp}, the ESD of \(\widetilde\vSigma_\sigma^+\) converges to $\nu_{\rm q}$ and, in the supercritical regime ($\ell>\gamma^{-1/2}$)
there is a label-aligned population covariance outlier $\widehat\Lambda_y
        =
        \Lambda_y+o_{\P}(1),$
where $\Lambda_y
        =
        \alpha_0+\alpha_2\lambda_{\rm out}(\gamma,\ell).$
If \(\vv_y^{\rm pop}\) denotes the associated unit eigenvector, then the first
BBP phase transition gives
\[
        \left|
        \left\langle
        \vv_y^{\rm pop},\frac{\vy}{\sqrt n}
        \right\rangle
        \right|^2
        \xrightarrow{\P}
        \frac{\gamma\ell^2-1}{\gamma\ell(\ell+1)}.
\]
If \(\ell\le\gamma^{-1/2}\), then there is no separated label-aligned
population spike.

Second, conditional on \(\vX\), we show that the row law of \(\vg_a\) satisfies the
quadratic-form concentration assumptions required for the nonlinear spiked
covariance theorem.  The rows are projected nonlinear Gaussian feature
vectors, and the projection is deterministic after conditioning on \(\vX\).
For globally Lipschitz \(\sigma\), the nonlinear Hanson--Wright inequality gives
for deterministic matrices \(\vA\)
\[
        \vg_a^\top\vA\vg_a
        -
        \Tr(\vA\widetilde\vSigma_\sigma^+)
        \prec
        \|\vA\|_F,
\]
after the standard covariance normalization.  

Thus, on a high-probability event in \(\vX\), we may apply
Theorem~13 \cite{wang2024nonlinearspikedcovariancematrices} to
$
        \widetilde \vK_\sigma 
        =
        \frac1N\sum_{a=1}^N\vg_a\vg_a^\top
$
with population covariance \(\widetilde\vSigma_\sigma^+\), aspect ratio
$\frac{r_n}{N}\to\phi,$ bulk law \(\nu_{\rm q}\), and population spike \(\Lambda_y\). 
We can conclude that a population spike \(\lambda\) produces a separated sample
outlier if and only if $z'\!\left(-\frac1\lambda\right)>0,$ and the sample outlier converges to $z\!\left(-\frac1\lambda\right)$. Additionally, the corresponding sample eigenvector has squared overlap with deterministic
test directions multiplied by the factor $\varphi\!\left(-\frac1\lambda\right).$
Applying this with \(\lambda=\Lambda_y\), if $z'\!\left(-\frac1{\Lambda_y}\right)>0,$ then \(\widetilde \vK_\sigma \), and hence \(\vK_\sigma \), has exactly one
sample eigenvalue converging to $\lambda_y$.

It remains to compute the overlap with the original label vector $\vv:=\frac{\vy}{\sqrt n}.$
Let
$
        \widetilde{\vv}
        :=
        \frac{\vU_+^\top\vv}{\|\vU_+^\top\vv\|_2}.
$
By the enlarged label-preservation lemma,
$
        \|\vU_+^\top\vv\|_2^2
        =
        \vv^\top\vP_\sigma \vv
        \xrightarrow{\P}1.
$
Hence \(\widetilde{\vv}\) is asymptotically the same direction as the projected
label.  Since the population eigenvector belongs to
\(\operatorname{Range}(\vP_\sigma )\), the first BBP phase gives
\[
        \left|
        \left\langle
        \vv_y^{\rm pop},\vv
        \right\rangle
        \right|^2
        \xrightarrow{\P}
        \frac{\gamma\ell^2-1}{\gamma\ell(\ell+1)}.
\]
The eigenvector overlap statement of the nonlinear sample-covariance BBP theorem
therefore yields
\[
        \left|
        \left\langle
        \widehat{\vv}_y,\vv
        \right\rangle
        \right|^2
        \xrightarrow{\P}
        \varphi\!\left(-\frac1{\Lambda_y}\right)
        \frac{\gamma\ell^2-1}{\gamma\ell(\ell+1)}.
\]
Finally, we expand the shrinkage factor.  Since
$
        z'(s)
        =
        \frac1{s^2}
        -
        \phi\int\frac{t^2}{(1+ts)^2}\,\nu_{\rm q}(dt),
$
we have
\[
        z'\!\left(-\frac1{\Lambda_y}\right)
        =
        \Lambda_y^2
        \left(
        1-
        \phi
        \int
        \frac{t^2}{(\Lambda_y-t)^2}\,\nu_{\rm q}(dt)
        \right).
\]
Therefore
\begin{align}
        \varphi\!\left(-\frac1{\Lambda_y}\right)
        &=
        -\frac{(-1/\Lambda_y)
        z'(-1/\Lambda_y)}
        {z(-1/\Lambda_y)}=
        \frac{
        1-\phi\int\frac{t^2}{(\Lambda_y-t)^2}\,\nu_{\rm q}(dt)
        }{
        1+\phi\int\frac{t}{\Lambda_y-t}\,\nu_{\rm q}(dt)
        }.\label{eq:def_varphi_q}
\end{align}
If $z'\!\left(-\frac1{\Lambda_y}\right)\le0$,
then the population spike is not in the outlier-producing index set $I=\{i:z'(-1/\lambda_i)>0\},$ and it sticks to the sample bulk and produces no outlier eigenvalue.  If \(\ell\le\gamma^{-1/2}\), then the first population BBP
transition does not occur, so there is no label-aligned population spike to
feed into the linear-width sample-covariance BBP theorem.  Hence in either
case the quadratic label channel produces no separated linear-width CK outlier.
\end{proof}

\begin{proof-of-theorem}[\ref{thm:quadratic}]
        Item \textbf{(i)} is due to the first part of Proposition~\ref{prop:second-BBP-linear-width-CK-plus}. Item \textbf{(ii)} is proved by Proposition~\ref{thm:projected-population-ck-bbp} since $\E[\vK|\vX]=\vPhi(\vX)$.  Item \textbf{(iii)} is based on the second and the third parts of Proposition~\ref{prop:second-BBP-linear-width-CK-plus}, where $\varphi$ transform is given by \eqref{eq:def_varphi_q}.
\end{proof-of-theorem}

\end{document}